%% file: neurips_2026.tex
\documentclass{article}

\PassOptionsToPackage{authoryear,round}{natbib}
\PassOptionsToPackage{table}{xcolor}

\usepackage[main,final]{neurips_2026}

\usepackage[utf8]{inputenc}
\usepackage[T1]{fontenc}

\usepackage{xcolor}
\usepackage{microtype}
\usepackage{graphicx}
\usepackage{subcaption}
\usepackage[export]{adjustbox}
\usepackage{wrapfig}
\usepackage{comment}
\usepackage{textcomp}
\usepackage{multirow}
\usepackage{booktabs}
\usepackage{nicefrac}
\usepackage{amsmath}
\usepackage{amsfonts}
\usepackage{amssymb}
\usepackage{mathtools}
\usepackage{amsthm}
\usepackage{pifont}
\usepackage{fontawesome5}
\usepackage{algorithm}
\usepackage{algpseudocode}
\algrenewcommand\algorithmiccomment[1]{\hfill$\triangleright$~#1}

\usepackage[textsize=tiny]{todonotes}
\usepackage{tabularx}

\usepackage{hyperref}
\hypersetup{
  colorlinks=true,
  linkcolor=blue,
  citecolor=blue,
  urlcolor=blue,
  filecolor=blue,
}

\usepackage[capitalize,noabbrev]{cleveref}

\theoremstyle{plain}
\newtheorem{theorem}{Theorem}[section]
\newtheorem{proposition}[theorem]{Proposition}
\newtheorem{lemma}[theorem]{Lemma}

\theoremstyle{definition}
\newtheorem{definition}[theorem]{Definition}
\newtheorem{assumption}[theorem]{Assumption}

\theoremstyle{remark}
\newtheorem{remark}[theorem]{Remark}

\usepackage{array}
\definecolor{datasetheader}{HTML}{5E6E82}
\definecolor{datasetrowa}{HTML}{F2F6FA}
\definecolor{datasetrowb}{HTML}{FAFCFE}
\definecolor{headergray}{gray}{0.92}
\definecolor{bestgreen}{RGB}{226,242,226}
\definecolor{effblue}{RGB}{226,236,246}
\definecolor{paretoyes}{RGB}{226,242,226}
\definecolor{paretono}{gray}{0.96}
\definecolor{modelhighlight}{RGB}{248,232,232}

\newcommand{\bestcell}[1]{\cellcolor{bestgreen}\textbf{#1}}
\newcommand{\effcell}[1]{\cellcolor{effblue}\textbf{#1}}
\newcommand{\yespareto}{\cellcolor{paretoyes}\textbf{Yes}}
\newcommand{\nopareto}{\cellcolor{paretono}No}

\definecolor{gainhigh}{RGB}{222,239,221}
\definecolor{gainmid}{RGB}{235,245,234}
\definecolor{gainlow}{RGB}{245,250,244}
\definecolor{sigblue}{RGB}{230,238,250}
\definecolor{wincell}{RGB}{238,246,238}
\definecolor{datasetheader}{HTML}{29465B}
\definecolor{datasetblue}{HTML}{EAF2F8}
\definecolor{computeheader}{HTML}{2F4F4F}
\definecolor{computeblue}{HTML}{EEF6F7}
\definecolor{computewhite}{HTML}{FFFFFF}

\definecolor{rankheader}{HTML}{2E4053}
\definecolor{goldcell}{HTML}{F8E48C}
\definecolor{silvercell}{HTML}{E5E7E9}
\definecolor{bronzecell}{HTML}{E8C39E}
\definecolor{bstabrow}{HTML}{F9F4EC}
\definecolor{realrow}{HTML}{F2F3F4}
\definecolor{modelcol}{HTML}{F7F9F9}

\definecolor{aucBSTabA}{HTML}{EAF2FC}
\definecolor{aucBSTabB}{HTML}{F2F7FD}
\definecolor{aucBSTabName}{HTML}{D8E9F9}
\definecolor{aucBaseA}{HTML}{F5F7FA}
\definecolor{aucBaseB}{HTML}{FFFFFF}
\definecolor{aucTRTR}{HTML}{E4F3EC}
\definecolor{aucAvg}{HTML}{E3EEFA}
\definecolor{aucAvgHead}{HTML}{53718D}

\newcommand{\aucgold}[1]{%
  \cellcolor{goldcell}\(\mathbf{#1}\)}
\newcommand{\aucsilver}[1]{%
  \cellcolor{silvercell}\(#1\)}
\newcommand{\aucbronze}[1]{%
  \cellcolor{bronzecell}\(#1\)}

\definecolor{regimeheader}{HTML}{34495E}
\definecolor{regimeHDLSS}{HTML}{F1EDFA}
\definecolor{regimeHDHSS}{HTML}{EAF3FC}
\definecolor{regimeLDHSS}{HTML}{E8F5EE}
\definecolor{regimeLDLSS}{HTML}{FFF3E8}
\definecolor{regimeMixed}{HTML}{F3F4F6}
\definecolor{theoremborder}{HTML}{D9CDEE}

\definecolor{computeheader}{HTML}{2F4F4F}
\definecolor{computerow}{HTML}{EEF6F7}
\definecolor{computedata}{HTML}{DCECEB}

\newcommand{\goldmedal}[1]{\cellcolor{goldcell}\textbf{#1}}
\newcommand{\silvermedal}[1]{\cellcolor{silvercell}\textbf{#1}}
\newcommand{\bronzemedal}[1]{\cellcolor{bronzecell}\textbf{#1}}

\definecolor{latentheader}{HTML}{34495E}
\definecolor{latentlabel}{HTML}{D1E8E4}
\definecolor{latentvalues}{HTML}{EFF8F5}

\definecolor{theoremborder}{HTML}{D9CDEE}

\definecolor{optunaheader}{HTML}{606B78}
\definecolor{optunarow}{HTML}{F1F3F5}
\definecolor{optunadata}{HTML}{E3E7EB}

\definecolor{lemmaborder}{HTML}{E8D3B8}
\definecolor{gobsheader}{HTML}{DEE5EA}
\definecolor{gobsrow}{HTML}{F3F6F8}
\definecolor{gobsdata}{HTML}{E8EEF2}

\definecolor{ablheadsetting}{HTML}{E6E8FC}
\definecolor{ablheadtstr}{HTML}{CFEFE1}
\definecolor{ablheadtrtr}{HTML}{D9EAFE}
\definecolor{ablheadc2st}{HTML}{FFE4CF}
\definecolor{ablheadeff}{HTML}{EEE0FB}

\definecolor{ablstripe}{HTML}{F8FAFD}
\definecolor{ablbase}{HTML}{EDF4FC}

\definecolor{ablwinmint}{HTML}{B9EACF}
\definecolor{ablwinblue}{HTML}{BFDDFB}
\definecolor{ablwinpeach}{HTML}{FFD3AE}
\definecolor{ablwinpurple}{HTML}{E0CAF5}
\definecolor{theoremborder}{HTML}{D9CDEE}
\definecolor{fidelityheader}{HTML}{DDE5ED}
\definecolor{fidelityKS}{HTML}{E9F3FF}
\definecolor{fidelityW}{HTML}{F0ECFF}
\definecolor{fidelityPearson}{HTML}{E5F6F5}
\definecolor{fidelitySpearman}{HTML}{EAF7E8}
\definecolor{fidelityHOM}{HTML}{FFF1E5}
\definecolor{fidelityMI}{HTML}{FDECF2}
\definecolor{fidelityNaN}{HTML}{F1F3F5}
\definecolor{fidelityPrivacy}{HTML}{FFF8E4}
\definecolor{fidelityC2ST}{HTML}{EEE8FA}

\definecolor{cellBestTSTR}{HTML}{C3EAE1}
\definecolor{cellBestC2ST}{HTML}{DDD0F5}

\definecolor{fidelityBestMI}{HTML}{F7D5E5}
\definecolor{fidelityBestHOM}{HTML}{FFE0C5}
\definecolor{fidelityBestPrivacy}{HTML}{F9E8B8}
\definecolor{gobsBestTime}{HTML}{FFD8B8}
\definecolor{gobsBestMem}{HTML}{CFE5FF}
\definecolor{gobsBestCost}{HTML}{F8D8E6}
\definecolor{gobsBestELR}{HTML}{CFEFD8}

\definecolor{propborder}{HTML}{BEDFCA}

\usepackage{tcolorbox}

\definecolor{algbackground}{HTML}{F3F8FD}
\definecolor{algborder}{HTML}{B8D4E8}
\definecolor{algaccent}{HTML}{6E9FBE}
\definecolor{assumpborder}{HTML}{D5D9DE}

\newcommand{\aucval}[2]{%
    \ensuremath{#1_{\pm #2}}}

\newcommand{\aucgoldstd}[2]{%
    \cellcolor{goldcell}%
    \ensuremath{\mathbf{#1}_{\pm #2}}}

\newcommand{\aucsilverstd}[2]{%
    \cellcolor{silvercell}%
    \ensuremath{#1_{\pm #2}}}

\newcommand{\aucbronzestd}[2]{%
    \cellcolor{bronzecell}%
    \ensuremath{#1_{\pm #2}}}

\newcommand{\fidgold}[1]{%
  \cellcolor{goldcell}\ensuremath{\mathbf{#1}}}

\newcommand{\fidsilver}[1]{%
  \cellcolor{silvercell}\ensuremath{#1}}

\newcommand{\fidbronze}[1]{%
  \cellcolor{bronzecell}\ensuremath{#1}}

\title{BSTabDiff: Block-Subunit Diffusion Priors for High-Dimensional Tabular Data Generation}

\author{
Al Zadid Sultan Bin Habib$^{1}$, Md Younus Ahamed$^{2}$, Prashnna Kumar Gyawali$^{3}$,\\
\textbf{Gianfranco Doretto}$^{4}$, \textbf{Donald A. Adjeroh}$^{5}$\\[3pt]
\begin{minipage}[t]{0.52\textwidth}
\centering\footnotesize
$^{1,2,3,5}$Lane Department of CSEE\\
West Virginia University, Morgantown, WV 26506, USA
\end{minipage}\hfill
\begin{minipage}[t]{0.45\textwidth}
\centering\footnotesize
$^{4}$SCI Institute \& Department of Biomedical Informatics\\
The University of Utah, Salt Lake City, UT 84112, USA
\end{minipage}\\[3pt]
{\scriptsize\texttt{\{\{ah00069,ma00087\}@mix,\{prashna.gyawali,donald.adjeroh\}@mail\}.wvu.edu, doretto@utah.edu}}\\[3pt]
{\small
\href{https://github.com/zadid6pretam/BSTabDiff}{%
\textcolor{black}{\faGithub}\hspace{0.3em}%
\textcolor{blue}{\texttt{https://github.com/zadid6pretam/BSTabDiff}}%
}
}
}

\begin{document}

\maketitle

\begin{abstract}
  High-Dimensional Low-Sample Size (HDLSS) tabular domains (e.g., omics) are characterized by $n \ll m$, where $n$ = number of samples, and $m$ = number of features. Such domains often exhibit strong local correlation groups, sparse cross-group dependencies, heavy-tailed non-Gaussian marginals, heteroscedastic noise, and structured missingness, making direct density learning in $\mathbb{R}^m$ ill-conditioned since $n \ll m$. We propose BSTabDiff, a block-subunit generative framework that partitions the $m$ observed features into $M$ latent blocks ($M \ll m$) and generates each block via a shared low-dimensional subunit variable, concentrating global dependence learning in the compact block-latent space $\mathbb{R}^M$ while decoding to the full feature space with copula-driven dependence, flexible per-feature marginals, and explicit missingness mechanisms. BSTabDiff can induce a data-driven feature ordering and block construction before latent partitioning, allowing the model to better align with local dependency structure when such organization is beneficial. BSTabDiff supports modern deep priors on block latents, including diffusion and normalizing flows, enabling stable synthesis and controllable benchmark generation in the HDLSS regime. Empirically, BSTabDiff produces more realistic and more stable high-dimensional synthetic data when compared with unstructured tabular generators on HDLSS data.
\end{abstract}

\section{Introduction}
Synthetic data has become a practical lever for scaling learning systems when real-world data are scarce, siloed, expensive to curate, or too sensitive to share. In industry settings, synthetic generation is increasingly positioned as a way to bootstrap domain-specific datasets for training and evaluating modern AI pipelines (including agentic systems), helping mitigate data bottlenecks while enabling controlled coverage of rare or safety-critical cases \citep{nvidia_synth_agentic}. At the same time, synthetic data is now also part of the training recipe for tabular foundation models: Prior-data Fitted Networks such as TabPFN variants are trained offline on large collections of synthetic datasets sampled from a prior to approximate Bayesian inference at test time \citep{tabpfn, tabpfnv2, tabpfn2.5}. These trends motivate tabular generators that are not only of high-fidelity, but also scalable and controllable so they can serve as reliable engines for pretraining, simulation, augmentation, and benchmarking across domains.
\newline
However, many high-value scientific tabular domains live in the High-Dimensional Low-Sample Size (HDLSS) regime, where $n$ samples are far fewer than $m$ features ($n\ll m$) \citep{w1,w2, li2011random}. HDLSS data such as omics-like datasets further exhibit strong local correlation groups (modules) \citep{w3}, sparse cross-group dependence, heavy-tailed and non-Gaussian marginals, heteroscedasticity and overdispersion \citep{w4,w5,w6}, and structured missingness mechanisms \citep{w7,w25}. In this regime, directly learning dense dependence in $\mathbb{R}^m$ is often ill-conditioned. Meanwhile, sequence-style tabular generators that treat columns as tokens (e.g., LLM-based synthesis) can become computationally strained as $m$ grows, since standard self-attention scales quadratically in sequence length \citep{attention,great}. These gaps leave a practical need for tabular generators that explicitly exploit HDLSS structure to achieve stable learning and efficient sampling at omics-scale dimensionalities.\newline
\noindent\textbf{Contributions.} We introduce \textbf{BSTabDiff} (\underline{B}lock-\underline{S}ubunit \underline{Tab}ular \underline{Diff}usion), a block-subunit generative framework tailored to HDLSS tabular data. Key novel elements include:

\textbf{1) Block-subunit HDLSS generator:} We propose a generative family that partitions the $m$ observed features into $M$ latent blocks ($M\ll m$), generating each block via a shared low-dimensional subunit variable while preserving feature-wise marginals and structured missingness. \textbf{2) Compact deep priors on block latents:} We concentrate global dependence learning in $\mathbb{R}^M$ by placing modern priors on block latents, including diffusion and normalizing flows, improving stability when $n\ll m$. 
\textbf{3) Structure-aware block construction:} 
We incorporate \underline{G}raph-guided \underline{O}rdering with \underline{B}lock \underline{S}egmentation (GO-BS) to induce a data-driven feature ordering and block assignment before latent partitioning, allowing BSTabDiff to better align latent blocks with local dependency structure when such organization is beneficial. On the theory side, we prove that the GO-BS objective is NP-hard via reduction from minimum bisection, motivating the sparse approximate solver used in our implementation.
\textbf{4) HDLSS-oriented modeling knobs and guarantees:} We provide a block-factorized learning signal and permutation-invariant identifiability (up to block relabeling), accommodating arbitrary observed feature order.
\textbf{5) Empirical stability in high dimension:} We show improved realism and stability over unstructured tabular generators in HDLSS settings, enabling controllable benchmark generation and synthetic pretraining at high feature counts.
\textbf{6) 
Energy-Carbon aware machine learning   perspective:} In addition to synthetic utility, we introduce an energy-carbon-aware evaluation criterion for tabular generative models that jointly assesses TSTR (train on synthetic test on real) performance, energy consumption, and CO$_2$e emissions for training.
\section{Related Work}
\label{related}
A broad set of baselines for tabular data synthesis exists, including GAN/VAE-style generators, diffusion/score-based models, and LLM/foundation-model approaches, and we provide a detailed review in Appendix~\ref{morerelated}.
\section{Methodology}
\label{method}
\textbf{Block-Subunit Generator:} BSTabDiff introduces a tabular generative model that partitions features into blocks and assigns each block a shared latent “subunit” variable, enabling high-dimensional feature generation through low-dimensional block factors while preserving per-feature marginals and missingness patterns. This complements prior tabular synthesis frameworks (e.g., GAN and system-based generators) by explicitly targeting HDLSS structure~\citep{w16,w17}, and is compatible with modern latent priors such as diffusion or flows~\citep{w8,w9,w10,w11,w14,w15}. Fig.~\ref{fig:bstabdiff_architecture} illustrates full generative architecture of BSTabDiff, including latent sampling, emission decoding, and permutation to observed space. Fig.~\ref{fig:subunit_generation} illustrates the subunit generation in BSTabDiff.

\begin{figure}[t]
    \centering
    \includegraphics[width=\textwidth]{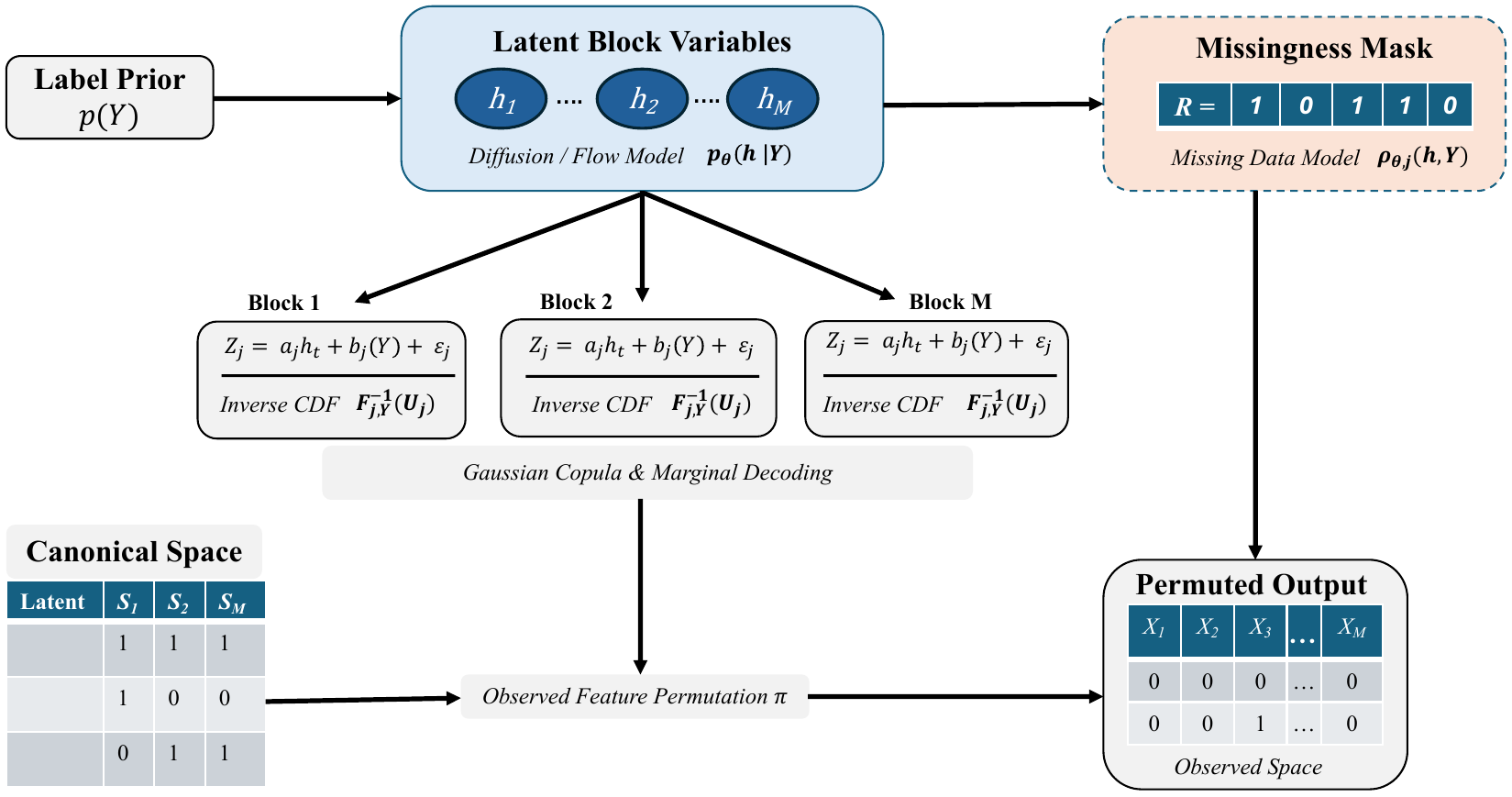}
    \captionsetup{skip=2pt}
    \caption{
    \textbf{Architecture of BSTabDiff.} The model samples a label, then draws low-dimensional block-latents \( h_1,\dots,h_M \) using a learned diffusion/flow prior. Each block governs a subset of features via copula-Gaussian decoding and inverse marginal CDFs to yield realistic marginals. A missingness mask is generated in parallel. The output is then restored to the original feature order, yielding high-dimensional tabular data with structured dependence, marginals, and missingness.
    }
    \label{fig:bstabdiff_architecture}
\end{figure}
\textbf{A Block-Subunit Generative Model for HDLSS Tabular Data}

\textbf{Motivation.}
Real HDLSS tabular datasets (e.g., omics) exhibit behaviors that differ qualitatively from classical asymptotics~\citep{w1,w2} and often show (i) strong local correlation groups~\citep{w3}, (ii) sparse cross-group dependencies, (iii) heavy-tailed / non-Gaussian marginals, (iv) heteroscedastic noise (mean-variance coupling / overdispersion)~\citep{w4,w5,w6}, and often (v) structured missingness~\citep{w7, w25}. We formalize a generative family that captures these properties while remaining analyzable in the $n\ll m$ regime. Our model preserves the core intuition of block structure (namely, shared latent factor per group~\citep{w20}) but extends it beyond simple Gaussian mean-shifts to support realistic marginals, mixed data types, and cross-block coupling. We assume each feature is distinct.\newline
\begin{figure*}[t]
\vspace{-0.5em}
\centering
\includegraphics[width=\textwidth]{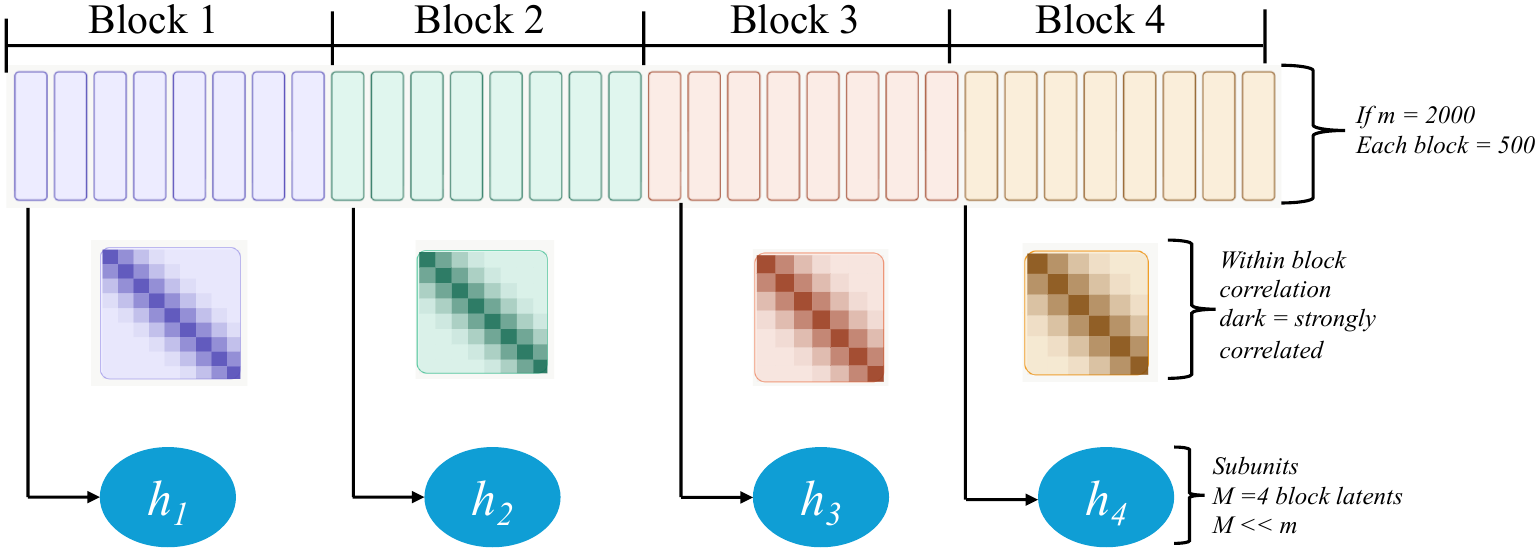}
\vspace{-0.6em}
\caption{
\textbf{Subunit generation in BSTabDiff.} BSTabDiff groups $m$ HDLSS features into $M\ll m$ contiguous blocks, each governed by a shared latent subunit $h_t$. Features are decoded via Gaussian copulas with feature specific noise and learned inverse CDFs.
}
\vspace{-0.4em}
\label{fig:subunit_generation}
\end{figure*}
\textbf{Latent Block Structure with Observed Feature Permutation.}
Let $X \in \mathbb{R}^m$ be a sample with $m$ features and optional label $Y \in \{1,\dots,C\}$. Assume an (unobserved) partition of the canonical feature indices into $M$ disjoint blocks $\{\mathcal{S}_t\}_{t=1}^M$ of sizes $s_t$ (not necessarily equal), with $\sum_{t=1}^M s_t = m$. To decouple the canonical block index space from the observed feature order, we introduce a permutation $\pi\in\mathfrak{S}_m$ relating $\tilde X$ to $X$ via $X=\tilde X_{\pi}$ (and similarly $R=\tilde R_{\pi}$), where $R\in\{0,1\}^m$ is the binary observation mask ($R_j=1$ observed, $R_j=0$ missing/\texttt{NA}). In the GO-BS configuration, the feature permutation $\hat{\pi}$ and contiguous block boundaries $\hat{\mathbf{b}}$ are estimated from the training data before fitting the block-subunit generator. Following Appendix~\ref{sec:gobs}, the observed features and missingness indicators are reordered into canonical space as $X^\dagger=X_{\hat{\pi}^{-1}}$ and $R^\dagger=R_{\hat{\pi}^{-1}}$, where $\hat{\pi}$ maps original feature indices to their canonical ordered positions. The resulting contiguous blocks define feature subsets governed by individual block latents. This data-driven organization builds on the broader motivation for feature ordering and dependency-aware feature sequencing~\citep{rot,tabseq,dynatab2}, while GO-BS additionally optimizes block segmentation for generative modeling. When feature ordering is disabled, $\hat{\pi}=\mathrm{id}$ and original feature order is retained. During generation, synthetic samples produced in canonical space are mapped back to original observed feature order using fixed learned permutation, rather than estimating or randomly sampling a new permutation.
\begin{definition}[Block-subunit HDLSS generative model]
\label{def:block_subunit_hdlss_gen}
We fix blocks $\{\mathcal{S}_t\}_{t=1}^M$ in canonical index space. For each sample, generate:
\begin{enumerate}
\item \textbf{Label (optional):} $Y \sim p(Y)$.
\item \textbf{Block latents with cross-block dependence:} $h=(h_1,\dots,h_M) \in \mathbb{R}^M$ from a label-conditional prior by Eq.~\ref{eq:h_prior} where $p_\theta$ captures cross-block dependence (e.g., sparse graphical prior, normalizing flow, diffusion, or mixture)~\citep{w8,w9,w10,w14,w15}.
\begin{equation}
\label{eq:h_prior}
h \sim p_\theta(h \mid Y)
\end{equation}
\item \textbf{Missingness mask (optional but realistic):} for each feature $j$ in canonical space we get Eq.~\ref{eq:missing_mask} reflecting structured missing-data mechanisms commonly modeled in statistical missingness theory~\citep{w7}.
\begin{equation}
\label{eq:missing_mask}
\tilde R_j \sim \mathrm{Bernoulli}\!\big(\rho_{\theta,j}(h,Y)\big),
\qquad
\tilde X_j=\mathrm{NA}\;\text{if } \tilde R_j=0
\end{equation}
\item \textbf{Block-wise emissions (subunit measurements):} For each block $t$ and each $j\in \mathcal{S}_t$ with $\tilde R_j=1$, we draw an intermediate Gaussian copula variable~\citep{w19} in Eq.~\ref{eq:copula_gauss} and map to $U_j = \Phi(Z_j) \in (0,1)$, and define the observed canonical feature via an inverse marginal CDF by Eq.~\ref{eq:inverse_marginal}.
\begin{equation}
\label{eq:copula_gauss}
Z_j = a_j h_t + b_j(Y) + \xi_j,
\qquad
\xi_j \sim \mathcal{N}\!\big(0,\sigma_j^2(h_t,Y)\big)
\end{equation}
\begin{equation}
\label{eq:inverse_marginal}
U_j = \Phi(Z_j),
\qquad
\tilde X_j = F^{-1}_{j,Y}(U_j)
\end{equation}
Here $F_{j,Y}$ is a (learned) marginal CDF allowing heavy tails, skew, and non-Gaussianity, matching common departures from Gaussian assumptions in some HDLSS data, such as omics datasets~\citep{w4,w5,w6}.
\item \textbf{Permutation to observed space:} We use the learned permutation $\hat{\pi}$ (identity if disabled) and output via Eq.~\ref{eq:gen_perm_output}.
\begin{equation}
\label{eq:gen_perm_output}
X = \tilde X_{\hat{\pi}}, \qquad R = \tilde R_{\hat{\pi}}
\end{equation}
\end{enumerate}
\end{definition}
\textbf{Remarks.}
Eq.~\ref{eq:copula_gauss}-\ref{eq:inverse_marginal} yields a Gaussian copula dependence structure controlled by $h_t$, while $F_{j,Y}$ provides realistic feature-wise marginals~\citep{w19, w18}. Heteroscedasticity 
is captured via $\sigma_j^2(h_t,Y)$, and missingness 
is modeled via Eq.~\ref{eq:missing_mask}~\citep{w4,w5,w6,w7}. This model strictly generalizes the simple additive Gaussian block factor model (shared factor + i.i.d.\ noise)~\citep{w20, w21}.\newline \newline
\textbf{Deep Generative Parameterizations of the Block Prior.}
A key advantage in HDLSS is to learn global dependence at the block-latent level ($M\ll m$), where sample complexity is far more favorable than learning dense dependence directly in $\mathbb{R}^m$~\citep{w1,w2}.

\textbf{Option A: Diffusion on $h$.}
We define a forward noising process $q(h_t \mid h_0)$ and train a conditional denoiser $\epsilon_\theta(\cdot)$ by Eq.~\ref{eq:diffusion_loss}. Then we sample $h_0 \sim p_\theta(h\mid Y)$ via reverse diffusion~\citep{w8,w9} (see also score/SDE views~\citep{w10}) and decode to $\tilde X$ using Eq.~\ref{eq:copula_gauss}-\ref{eq:inverse_marginal}. Diffusion models have also been adapted successfully to mixed-type tabular synthesis~\citep{w11}, and can be run in compact latent spaces~\citep{w12,w13}.
\begin{equation}
\label{eq:diffusion_loss}
\min_\theta \;
\mathbb{E}_{(h_0,Y)}\; \mathbb{E}_{t,\epsilon}
\Big[\big\|\epsilon - \epsilon_\theta(h_t,t,Y)\big\|_2^2\Big]
\end{equation}
\textbf{Option B: Normalizing flow on $h$.}
Let $h = f_\theta(\nu,Y)$ with base $\nu\sim \mathcal{N}(0,I)$. Then $p_\theta(h\mid Y)$ is tractable and can be trained by conditional maximum likelihood on inferred block latents~\citep{w14,w15}.

\textbf{Option C: Mixtures / graphical priors.}
For interpretability, $p_\theta(h\mid Y)$ can be a sparse Gaussian graphical model or a mixture of block states.

\textbf{A Likelihood Factorization and a Block-Level Learning Signal.}
Let $\tilde X_{\mathcal{S}_t}$ denote the subvector in block $t$ (canonical space). Under conditional independence of emissions given $h$ (as in Definition~\ref{def:block_subunit_hdlss_gen}), the joint likelihood factorizes across blocks as defined in Eq.~\ref{eq:block_factorization}.
\begin{tcolorbox}[
    colback=regimeHDHSS,
    colframe=algborder,
    boxrule=0.5pt,
    arc=1mm,
    boxsep=0pt,
    left=3pt,
    right=3pt,
    top=0pt,
    bottom=0pt,
    before skip=2pt,
    after skip=2pt,
    before upper={
        \setlength{\abovedisplayskip}{0pt}
        \setlength{\belowdisplayskip}{0pt}
        \setlength{\abovedisplayshortskip}{0pt}
        \setlength{\belowdisplayshortskip}{0pt}
    }
]

\begin{equation}
\label{eq:block_factorization}
p(\tilde X, \tilde R \mid h, Y)
=
\prod_{t=1}^M
\prod_{j\in \mathcal{S}_t}
p(\tilde R_j \mid h_t,Y)\;
p(\tilde X_j \mid h_t,Y,\tilde R_j=1)
\end{equation}
\end{tcolorbox}
Eq.~\ref{eq:block_factorization} provides a direct learning signal for block structure: each $h_t$ is responsible for explaining a coherent subset of correlated coordinates~\citep{w3}, even if the observed coordinates are arbitrarily ordered (i.e., under an unknown permutation $\pi$). Here, emission factors are understood as observed-data likelihood contributions, taking value one when $\tilde R_j=0$, while the missingness factors remain included for all features.

\textbf{Permutation-Invariant Identifiability (Canonical Up to Block Permutations).}
Because the observed feature order is arbitrary, identifiability should be stated modulo permutation~\citep{w22}. We capture the standard notion: blocks are identifiable up to relabeling when they induce distinct dependence patterns.
\begin{tcolorbox}[
    colback=regimeMixed,
    colframe=assumpborder,
    boxrule=0.5pt,
    arc=1mm,
    boxsep=2pt,
    left=3pt,
    right=3pt,
    top=3pt,
    bottom=3pt,
    before skip=5pt,
    after skip=5pt
]
\begin{assumption}[Distinct block dependence]
\label{ass:distinct_blocks}
For any two blocks $t\neq t'$, the pairwise dependence structure among coordinates in $\mathcal{S}_t$ differs from that in $\mathcal{S}_{t'}$ (e.g., different correlation spectra or different copula parameters). In addition, each feature is assumed to be unique, so coordinates are not exchangeable copies, even within a block. Cross-block dependencies are sparse in $p_\theta(h\mid Y)$.
\end{assumption}
\end{tcolorbox}
\begin{tcolorbox}[
    colback=regimeLDHSS,
    colframe=propborder,
    boxrule=0.5pt,
    arc=1mm,
    boxsep=2pt,
    left=3pt,
    right=3pt,
    top=3pt,
    bottom=3pt,
    before skip=5pt,
    after skip=5pt
]
\begin{proposition}[Identifiability up to block permutation]
\label{prop:identifiability}
Under Assumption~\ref{ass:distinct_blocks} and with $a_j\neq 0$ for features that participate in block dependence, the partition $\{\mathcal{S}_t\}$ is identifiable from the population distribution of $X$ up to permutation of block labels, even if the observed coordinates are permuted by an unknown $\pi$~\citep{w22}.
\end{proposition}
\end{tcolorbox}
\vspace{-0.8em}
\begin{proof}[Proof sketch]
Because within-block dependence is induced through a shared scalar $h_t$ (plus block-specific emissions),
each block yields a characteristic dependence signature (e.g., rank-1 plus noise structure in copula Gaussian space). Distinctness implies there exists a clustering of coordinates that maximizes within-group dependence and minimizes cross-group dependence (e.g., as commonly done to identify correlated modules in high-dimensional biology ~\citep{w3}). Unknown coordinate permutation changes feature order but not these dependence relations, hence the recovered partition matches the true one up to block relabeling.
\end{proof}
\vspace{-0.8em}
\textbf{HDLSS Stability: Why Block Latents Are Learnable When $n\ll m$.}
The model is designed so the effective degrees of freedom scale with $M$ rather than $m$, where $M$ is the number of latent blocks (equivalently, the dimension of the block-latent vector $h\in\mathbb{R}^M$), aligning with HDLSS analyses that emphasize low-dimensional structure underlying high-dimensional observations~\citep{w1,w2}.
\begin{tcolorbox}[
    colback=regimeLDHSS,
    colframe=propborder,
    boxrule=0.5pt,
    arc=1mm,
    boxsep=2pt,
    left=3pt,
    right=3pt,
    top=3pt,
    bottom=3pt,
    before skip=5pt,
    after skip=5pt
]
\begin{proposition}[Block-latent sample complexity advantage (informal)]
\label{prop:sample_complexity}
Assume $p_\theta(h\mid Y)$ has $O(M)$ to $O(M\log M)$ effective parameters (e.g., sparse couplings), and emissions $p(\tilde X_j \mid h_t,Y)$ share parameters within blocks. Then consistent estimation of the generative mechanism can be achieved with sample sizes scaling primarily with $M$ (up to log factors and emission complexity), rather than with $m$~\citep{w1,w2}. In contrast, unstructured generators that model dense dependence directly in $\mathbb{R}^m$ face parameter growth at least linear (often quadratic) in $m$, which is ill-conditioned in HDLSS regimes~\citep{w1,w2}.
\end{proposition}
\end{tcolorbox}
\textbf{Interpretation.}
Even when $m$ is in the thousands and $n$ is in the tens, block-level dependence can be learned because it compresses the global structure into $M\ll m$ latent variables and block-shared emissions~\citep{w1,w2}. \newline
\textbf{SNR Scaling Within a Block (Connection to HDLSS Separability).}
To connect to the classical HDLSS 
characterization, we show how aggregation within a block boosts signal-to-noise ratio.
\begin{tcolorbox}[
    colback=regimeLDLSS,
    colframe=lemmaborder,
    boxrule=0.5pt,
    arc=1mm,
    boxsep=2pt,
    left=3pt,
    right=3pt,
    top=3pt,
    bottom=3pt,
    before skip=5pt,
    after skip=5pt
]
\begin{lemma}[Within-block SNR scaling under mean-shift emissions]
\label{lem:snr_scaling}
Consider a single block $\mathcal{S}$ of size $s$
and a simplified continuous emission
$\tilde X_j \mid Y=y \sim \mathcal{N}(\mu_y, \sigma^2)$
i.i.d.\ for $j\in\mathcal{S}$.
Let
$\bar X_{\mathcal{S}} = \frac{1}{s}\sum_{j\in\mathcal{S}}\tilde X_j$.
Then we get Eq.~\ref{eq:snr_scaling}.
\begin{equation}
\label{eq:snr_scaling}
\mathrm{Var}(\bar X_{\mathcal{S}} \mid Y) = \sigma^2/s,
\qquad
\mathrm{SNR}(\bar X_{\mathcal{S}})
=
\frac{(\mu_1-\mu_0)^2}{\sigma^2/s}
=
s\cdot \frac{(\mu_1-\mu_0)^2}{\sigma^2}
\end{equation}
where $\mu_0$ and $\mu_1$ denote the
class-conditional means for $Y=0$ and $Y=1$,
respectively, i.e.,
$\mu_y \equiv \mathbb{E}[\tilde X_j\mid Y=y]$.
\end{lemma}
\end{tcolorbox}
\textbf{Takeaway.}
Blocks provide a natural mechanism for SNR amplification, which is critical when $n$ is very small, as in HDLSS regimes~\citep{w1,w2}. This intuition can be extended to non-Gaussian marginals and copula dependence via Eq.~\ref{eq:inverse_marginal}.

\textbf{Synthetic HDLSS Dataset Generation as a Controlled Benchmark Suite.}
Definition~\ref{def:block_subunit_hdlss_gen} yields a family of synthetic HDLSS benchmarks with knobs controlling:
(i) number of blocks $M$ and size distribution $\{s_t\}$,
(ii) cross-block sparsity/strength in $p_\theta(h\mid Y)$,
(iii) marginal tail-heaviness via $F_{j,Y}$,
(iv) heteroscedasticity via $\sigma_j^2(h_t,Y)$,
(v) missingness mechanisms via $\rho_{\theta,j}(h,Y)$,
and (vi) label-dependence via $b_j(Y)$ or label-conditional $p_\theta(h\mid Y)$. These knobs enable systematic evaluation of tabular generative modeling and synthetic pretraining in the HDLSS regime~\citep{w1,w2,w16,w17}.

\textbf{Generation Procedure.}
After training, synthetic generation follows a direct forward sampling pipeline that operates in the low-dimensional block-latent space and then decodes into the original $m$-dimensional feature space. The procedure produces $(X,R,Y)$ and is summarized in Algorithm~\ref{alg:block_subunit_gen} in Appendix~\ref{pseudocode}. Crucially, although the model uses only $M\ll m$ latent degrees of freedom through $h\in\mathbb{R}^M$, the output $X\in\mathbb{R}^m$ is not compressed: each feature $j$ is generated explicitly via its block subunit and its learned marginal transform.\newline
\textbf{(1) Optional label sampling.}
If class-conditional generation is enabled, we either (i) sample $Y$ from a specified class prior (uniform by default) or (ii) fix $Y=c$ to generate class-specific samples. For unconditional generation, $Y$ is omitted and $h\sim p_\theta(h)$.\newline
\textbf{(2) Sample block latents from 
learned prior.} We first sample the block latent vector in Eq.~\ref{eq:gen_h_sample} where $p_\theta$ is implemented as a deep prior in $\mathbb{R}^M$. In our main instantiation, $p_\theta(h\mid Y)$ is a diffusion prior: sampling is performed by running the reverse diffusion chain from Gaussian noise to get $h$.
\begin{equation}
\label{eq:gen_h_sample}
h=(h_1,\dots,h_M)\sim p_\theta(h\mid Y)
\end{equation}
This concentrates the model's global dependence learning in latent space, avoiding ill-conditioned
high-dimensional density modeling in $\mathbb{R}^m$ in HDLSS settings~\citep{w1,w2}. (For related tabular diffusion variants and latent-space diffusion instantiations, see~\citep{w11,w13}.)\newline
\textbf{(3) Sample missingness and emit features block-wise.}
Given $h$ and (optionally) $Y$, we generate a binary observed-mask $\tilde R\in\{0,1\}^m$ using the missingness model by Eq.~\ref{eq:missing_mask} which can be unconditional or class-conditional and may vary across features, consistent with structured missing-data modeling~\citep{w7}. If $\tilde R_j=0$ we set $\tilde X_j=\texttt{NA}$; otherwise we emit $\tilde X_j$ using the block subunit corresponding to the block membership of feature $j$. For each continuous feature $j\in\mathcal{S}_t$ with $\tilde R_j=1$, we first sample a Gaussian copula variable by Eq.~\ref{eq:copula_gauss} which maps it to a uniform variable $U_j=\Phi(Z_j)$, and then apply the learned inverse marginal CDF by Eq.~\ref{eq:inverse_marginal}.This yields (i) dependence driven by shared block subunit $h_t$ in copula space and (ii) realistic, potentially heavy-tailed feature-wise marginals through $F_{j,Y}$, consistent with common departures from Gaussianity in omics-like measurements~\citep{w4,w5,w6}. \newline
\textbf{(4) Permutation to observed feature order.}
GO-BS estimates a dependency-aware permutation $\hat{\pi}$ and contiguous block boundaries from training data before fitting BSTabDiff. During generation, fixed learned permutation maps synthetic features \& missingness masks from canonical space back to original feature order (Eq.~\ref{eq:gen_perm_output}). When GO-BS is disabled, $\hat{\pi}=\mathrm{id}$. Details are provided in App.~\ref{sec:gobs}.\newline
\textbf{Summary.}
Overall, generation follows the chain $Y \rightarrow h \rightarrow (\tilde R,\tilde X) \rightarrow (R,X)$, where the learned permutation $\hat{\pi}$ maps the canonical outputs back to the original feature order. The key HDLSS advantage is that global structure is learned and sampled in $\mathbb{R}^M$, while the full $m$-dimensional output is produced via block-conditioned emissions with learned marginals and missingness, yielding high-dimensional synthetic data~\citep{w1,w2}.\newline
\noindent\textbf{Implementation note.}
In practice, once $(\theta,\phi)$ are fitted, generation requires only sampling $h$ from the prior and a single forward decode pass over features; no optimization is performed at generation time. This makes sampling efficient even when $m$ is large (e.g., omics-scale) because all expensive dependence learning is confined to $M\ll m$.

\textbf{Training procedure.}
We train the model by combining (i) a modern prior on block latents $p_\theta(h\mid Y)$ (diffusion or flow)
and (ii) a block-factorized emission model $p_\phi(\tilde X,\tilde R\mid h,Y)$ (Eq.~\ref{eq:block_factorization}).
In HDLSS, learning a high-capacity generator directly in $\mathbb{R}^m$ is often ill-conditioned; instead,
we fit the global dependence in $\mathbb{R}^M$ and decode to $\mathbb{R}^m$ via block emissions.\newline
\textbf{Marginal estimation and Gaussianization.}
For each continuous feature $j$, we estimate empirical marginal CDFs $\hat F_{j,y}$ from observed training values, optionally conditioned on class $Y=y$, and retain interpolated inverse CDFs $\hat F^{-1}_{j,y}$ for generation via Eq.~\ref{eq:inverse_marginal}. For latent inference, observed continuous values are independently rank-Gaussianized using pooled training samples across classes: $z_{ij}=\Phi^{-1}((r_{ij}-0.5)/N_j)$, where $r_{ij}$ is the observed rank and $N_j$ is the number of observed values for feature $j$. Thus, forward Gaussianization is class-independent, whereas inverse marginal decoding can be class-conditional. All transformations are fitted within each training fold to avoid data leakage.\newline
\textbf{Latent inference.}
After GO-BS ordering and Gaussianization, block latents $h\in\mathbb{R}^M$ are inferred deterministically by averaging observed Gaussianized continuous features within each block and standardizing the resulting block factors. This lightweight factor-score construction~\citep{w23,w24} avoids training a separate inference encoder. Missing entries are excluded from block averages and emission fitting. The inferred latent vectors are subsequently used for block-conditioned emission estimation and joint prior learning. 
See Appendix~\ref{app:additional_methodological_notes} for further clarification on amortized inference.\newline
\textbf{Objective.}
We fit continuous emission parameters $\phi$ by least-squares regression on observed Gaussianized features, categorical emissions by cross-entropy, and observation probabilities from empirical missingness rates, optionally conditioned on class labels. These steps instantiate block-factorized model in Eq.~\ref{eq:block_factorization} without directly maximizing its full joint likelihood. Compact prior $p_\theta(h\mid Y)$ is then trained on inferred block latents using conditional maximum likelihood for flows or denoising score matching for diffusion (Eq.~\ref{eq:diffusion_loss}). Algs.~\ref{alg:block_subunit_gen} and~\ref{alg:block_subunit_train} in App.~\ref{pseudocode} detail generation and training.\newline
\textbf{Permutation estimation via GO-BS.}
BSTabDiff estimates an observed feature permutation $\hat{\pi}$ using GO-BS (Appendix~\ref{sec:gobs}), which orders dependent features contiguously to align the input space with the latent block structure. Given a weighted feature-affinity graph $G=(V,E,W)$, GO-BS estimates $\hat{\pi}$ and $M-1$ boundaries $\hat{\mathbf{b}}$ by approximately minimizing Eq.~\ref{eq:gobs_main_short}.
\begin{tcolorbox}[
    colback=regimeHDHSS,
    colframe=algborder,
    boxrule=0.5pt,
    arc=1mm,
    boxsep=0pt,
    left=3pt,
    right=3pt,
    top=1pt,
    bottom=1pt,
    before skip=3pt,
    after skip=3pt,
    before upper={
        \setlength{\abovedisplayskip}{0pt}
        \setlength{\belowdisplayskip}{0pt}
        \setlength{\abovedisplayshortskip}{0pt}
        \setlength{\belowdisplayshortskip}{0pt}
    }
]
\begin{equation}
\label{eq:gobs_main_short}
\resizebox{0.88\linewidth}{!}{$\displaystyle
\mathcal{J}_{\mathrm{GO\mbox{-}BS}}(\pi,\mathbf{b})
=
-\sum_{t=1}^{M}
\sum_{\substack{i,j \in B_t(\pi,\mathbf{b})\\ i<j}} W_{ij}
+
\lambda
\sum_{1 \le t < t' \le M}
\sum_{\substack{i \in B_t(\pi,\mathbf{b})\\ j \in B_{t'}(\pi,\mathbf{b})}} W_{ij}
+
\gamma\,\Omega(\mathbf{b})
$}
\end{equation}
\end{tcolorbox}
where $B_t(\pi,\mathbf{b})$ is the $t$-th contiguous block.
This objective rewards within-block affinity, penalizes cross-block leakage, and regularizes boundaries.
Afterward, BSTabDiff trains in the GO-BS canonical order and maps generated samples back to the original feature order. Details, ablations, and pseudocode for BSTabDiff and GO-BS are provided in App.~\ref{sec:gobs}, \ref{app:gobs_ablation}, \ref{app:gobsmeta}, \ref{app:same}, \ref{app:gobsextreme}, and~\ref{pseudocode}. We provide a $(\rho,\kappa)$-based GO-BS/GO-BS-FC selection rule with empirical validation in App.~\ref{app:gobs_variant_selection}.

\noindent\textbf{Micro-Macro dataset regime refinement.}
We extend the 5 dataset regimes based on $(n,m,\rho=m/n)$ by introducing $\kappa=nm$ as a complementary measure of data-matrix scale, partitioning each regime into Micro and Macro subsets using a regime-specific median threshold. Fig.~\ref{fig:micro_macro_regime_map} illustrates the 10 possible refined categories, with formal definitions provided in Appendix~\ref{app:micro_macro_regimes}. Additional methodological clarifications concerning baseline selection, block construction, latent inference, marginal estimation, mixed-type handling, and practical configuration are provided in Appendix~\ref{app:additional_methodological_notes}.
\section{Experiments}
\label{exp}
\vspace{-1em}
In this section, we evaluate BSTabDiff against existing methods for HDLSS tabular generation.

\textbf{Datasets.}
We evaluate HDLSS generation (Appendix~\ref{hdlss}) on eight publicly available real-world datasets (Table~\ref{tab:hdlss_datasets}) from the repository~\citep{hdlsss}, previously used by~\citet{protogate} and~\citet{gotabpfn}; see Appendix~\ref{extreme} for extreme dimensional experiments. These datasets span diverse domains, with approximately 2K-20K+ numerical features, few samples, and no missing values in their processed versions. Several have also been studied in prior HDLSS research, including ProtoGate~\citep{protogate}, LLSPIN~\citep{spin}, and GOTabPFN~\citep{gotabpfn}. We further evaluate nine additional datasets spanning diverse dimensionality regimes, including mixed-type features and heterogeneous missingness, as exemplified by TCGA-BRCA and Springleaf (Appendix~\ref{scope}). Overall, our experiments cover 18 distinct real-world datasets.

\textbf{Baselines.}
Given the large and growing set of tabular generative models, we benchmark BSTabDiff against representative, widely used methods from each major family. We further restrict comparisons to baselines that provide publicly available implementations, and we exclude LLM-based generators because their computational overhead becomes prohibitive in high-dimensional settings. Specifically, we include SMOTE~\citep{smote} as a classical baseline; CTGAN~\citep{w16}, CTAB-GAN~\citep{ctabgan}, and CTAB-GAN+~\citep{ctabgan+} as GAN-based methods; TVAE~\citep{w16} as a VAE-based approach; and TabDDPM~\citep{w11}, TabDiff~\citep{tabdiff}, BinaryDiffusion~\citep{bfiff}, ForestDiffusion~\citep{fdiff} as diffusion-style generators, and CtrTab~\citep{ctr} for high-dimensionality.

\textbf{Evaluation and Implementation.}
Our primary evaluation measure is Machine Learning Efficiency (MLE), following the standard protocol set by \citet{w11, stasy, codi, w13, tabdiff}. MLE measures how well classifiers trained on synthetic data perform on a real test set (TSTR), with the upper bound given by training and testing on real data (TRTR). High-quality synthetic data should yield models that approach or sometimes exceed TRTR performance. We report ML efficiency using two protocols: (i) the common approach of averaging efficiency relative to a classical baseline (logistic regression)~\citep{w16,ctabgan, w11}, and (ii) an evaluation relative to strong modern tabular models CatBoost~\citep{catboost}, TANDEM~\citep{tandem}, and TabPFN-2.5~\citep{tabpfn2.5} (for Colon), which represent competitive state-of-the-art baselines. We evaluate all classifiers using 5$\times$5 cross-validation (5 repeats $\times$ 5 folds = 25 runs), following the HDLSS evaluation protocol established by~\citet{protogate} and~\citet{gotabpfn}. We also performed Optuna~\citep{optuna}-based hyperparameter tuning with 2000 trials (see Appendix~\ref{hyp}), following the tuning protocol of~\citet{w11} and~\citet{gotabpfn}. Experiments used compute resources: Vulcan (8$\times$49\,GB NVIDIA RTX A6000 GPUs, 2$\times$ Intel Xeon Gold 5320 CPUs, 503\,GB RAM). Dataset and baseline source links are provided in Appendix~\ref{app:dataset}.
\begin{table}[t]
\centering
\caption{
BSTabDiff compute, energy, and CO$_2$e for one run on Vulcan (no permutation), measured with CodeCarbon~\citep{g8} and psutil~\citep{psutil}. See App.~\ref{app:green_ai} and~\ref{run} for more.
}
\label{tab:bstabdiff_compute_optuna_noperm}
\scriptsize
\setlength{\tabcolsep}{3.0pt}
\renewcommand{\arraystretch}{1.15}
\begin{tabular}{|l|c|c|c|c|c||l|c|c|c|c|c|}
\hline
\rowcolor{computeheader}
\textcolor{white}{\textbf{Data}} &
\textcolor{white}{\textbf{Time}} &
\textcolor{white}{\textbf{GPU}} &
\textcolor{white}{\textbf{CPU}} &
\textcolor{white}{\textbf{Energy}} &
\textcolor{white}{\textbf{CO$_2$e}} &
\textcolor{white}{\textbf{Data}} &
\textcolor{white}{\textbf{Time}} &
\textcolor{white}{\textbf{GPU}} &
\textcolor{white}{\textbf{CPU}} &
\textcolor{white}{\textbf{Energy}} &
\textcolor{white}{\textbf{CO$_2$e}} \\
\rowcolor{computeheader}
\textcolor{white}{\textbf{}} &
\textcolor{white}{\textbf{(s)}} &
\textcolor{white}{\textbf{(GiB)}} &
\textcolor{white}{\textbf{(GiB)}} &
\textcolor{white}{\textbf{(kWh)}} &
\textcolor{white}{\textbf{(kg)}} &
\textcolor{white}{\textbf{}} &
\textcolor{white}{\textbf{(s)}} &
\textcolor{white}{\textbf{(GiB)}} &
\textcolor{white}{\textbf{(GiB)}} &
\textcolor{white}{\textbf{(kWh)}} &
\textcolor{white}{\textbf{(kg)}} \\
\hline
\rowcolor{computeblue}
COL & 6.24 & 0.037 & 1.514 & 4.79e-4 & 4.29e-4 &
SMK & 8.18 & 0.062 & 1.647 & 4.68e-4 & 4.20e-4 \\
\rowcolor{computewhite}
GLI & 7.00 & 0.042 & 1.466 & 3.69e-4 & 3.31e-4 &
TOX & 4.36 & 0.031 & 1.529 & 1.92e-4 & 1.72e-4 \\
\rowcolor{computeblue}
LNG & 6.88 & 0.030 & 1.548 & 3.17e-4 & 2.84e-4 &
AML & 6.57 & 0.036 & 1.532 & 1.31e-3 & 1.17e-3 \\
\rowcolor{computewhite}
PRS & 5.28 & 0.209 & 2.519 & 2.52e-4 & 2.26e-4 &
ARC & 9.34 & 0.030 & 1.697 & 4.30e-4 & 3.85e-4 \\
\hline
\end{tabular}
\vspace{-5pt}
\end{table}
\begin{table*}[t]
\caption{
MLE-based TSTR synthesis comparison using Logistic Regression on 8 HDLSS datasets.
Accuracy is reported as mean$_{\pm \mathrm{std}}$ over evaluation folds.
See App.~\ref{bsgob} and~\ref{extreme} for more.
}
\label{tab:bstabdiff_mle_lr_hdlss}
\centering
\scriptsize
\setlength{\tabcolsep}{4.2pt}
\renewcommand{\arraystretch}{1.15}
\begin{tabular}{|l|c|c|c|c|c|c|c|c|}
\hline
\rowcolor{rankheader}
\textcolor{white}{\textbf{Model}} &
\textcolor{white}{\textbf{COL}} &
\textcolor{white}{\textbf{LNG}} &
\textcolor{white}{\textbf{GLI}} &
\textcolor{white}{\textbf{SMK}} &
\textcolor{white}{\textbf{AML}} &
\textcolor{white}{\textbf{PRS}} &
\textcolor{white}{\textbf{ARC}} &
\textcolor{white}{\textbf{TOX}} \\
\hline

\cellcolor{modelcol}CTAB-GAN
& 62.32$_{\pm 5.56}$ & 69.36$_{\pm 3.56}$ & 60.30$_{\pm 6.68}$ & 55.32$_{\pm 6.68}$ & 69.48$_{\pm 6.24}$ & 71.64$_{\pm 5.47}$ & 64.32$_{\pm 6.86}$ & 62.14$_{\pm 6.45}$ \\

\cellcolor{modelcol}CTAB-GAN+
& 63.38$_{\pm 7.32}$ & 70.72$_{\pm 4.86}$ & 58.74$_{\pm 5.56}$ & 56.46$_{\pm 5.56}$ & 71.65$_{\pm 5.46}$ & 70.86$_{\pm 6.36}$ & 66.46$_{\pm 4.56}$ & 63.15$_{\pm 5.46}$ \\

\cellcolor{modelcol}ForestDiff
& 74.65$_{\pm 5.34}$ & 76.17$_{\pm 5.79}$ & 72.34$_{\pm 4.56}$ & 62.14$_{\pm 4.48}$ & 80.81$_{\pm 4.47}$ & 76.12$_{\pm 5.86}$ & 70.15$_{\pm 5.22}$ & 76.12$_{\pm 6.68}$ \\

\cellcolor{modelcol}BinaryDiff
& 72.16$_{\pm 4.65}$ & 73.76$_{\pm 4.62}$ & 70.71$_{\pm 5.68}$ & 61.13$_{\pm 6.13}$ & 78.67$_{\pm 4.49}$ & 73.71$_{\pm 6.73}$ & 69.12$_{\pm 5.74}$ & 71.69$_{\pm 4.45}$ \\

\cellcolor{modelcol}TabDiff
& \bronzemedal{79.72$_{\pm 8.72}$}
& \bronzemedal{88.64$_{\pm 2.32}$}
& 76.48$_{\pm 5.24}$
& 68.32$_{\pm 4.36}$
& 92.68$_{\pm 2.42}$
& 86.42$_{\pm 6.36}$
& 76.72$_{\pm 4.56}$
& \bronzemedal{86.32$_{\pm 4.84}$} \\

\cellcolor{modelcol}SMOTE
& 78.71$_{\pm 9.21}$
& 87.03$_{\pm 2.79}$
& 76.47$_{\pm 4.63}$
& \bronzemedal{68.44$_{\pm 3.32}$}
& \bronzemedal{93.22$_{\pm 2.91}$}
& \bronzemedal{87.16$_{\pm 5.56}$}
& \bronzemedal{79.08$_{\pm 3.12}$}
& 84.32$_{\pm 4.98}$ \\

\cellcolor{modelcol}TVAE
& 79.03$_{\pm 6.35}$
& 82.27$_{\pm 3.15}$
& \bronzemedal{79.23$_{\pm 4.81}$}
& 66.76$_{\pm 4.19}$
& 92.67$_{\pm 3.43}$
& 83.45$_{\pm 7.81}$
& 78.12$_{\pm 6.76}$
& 81.54$_{\pm 5.23}$ \\

\cellcolor{modelcol}CTGAN
& 61.29$_{\pm 7.56}$ & 68.47$_{\pm 3.89}$ & 30.59$_{\pm 7.26}$ & 54.55$_{\pm 5.98}$ & 68.89$_{\pm 6.21}$ & 59.80$_{\pm 8.76}$ & 60.50$_{\pm 5.34}$ & 56.99$_{\pm 5.07}$ \\

\cellcolor{modelcol}TabDDPM
& 69.35$_{\pm 8.91}$ & 68.47$_{\pm 2.89}$ & 65.88$_{\pm 3.97}$ & 54.01$_{\pm 4.76}$ & 73.61$_{\pm 6.58}$ & 69.17$_{\pm 8.27}$ & 65.12$_{\pm 5.76}$ & 59.77$_{\pm 5.19}$ \\

\cellcolor{modelcol}CtrTab
& 43.38$_{\pm 13.77}$ & 58.60$_{\pm 4.99}$ & 62.35$_{\pm 10.78}$ & 55.05$_{\pm 8.91}$ & 41.62$_{\pm 12.36}$ & 35.30$_{\pm 8.37}$ & 47.00$_{\pm 7.68}$ & 22.83$_{\pm 5.69}$ \\

\rowcolor{bstabrow}
\cellcolor{modelcol}\textbf{BSTabDiff$^\P$}
& \silvermedal{83.21$_{\pm 9.81}$}
& \silvermedal{94.48$_{\pm 4.51}$}
& \goldmedal{84.94$_{\pm 7.64}$}
& \silvermedal{69.39$_{\pm 6.07}$}
& \goldmedal{95.56$_{\pm 8.31}$}
& \silvermedal{88.40$_{\pm 9.71}$}
& \silvermedal{85.30$_{\pm 6.05}$}
& \goldmedal{89.36$_{\pm 7.14}$} \\

\rowcolor{bstabrow}
\cellcolor{modelcol}\textbf{BSTabDiff$^\S$}
& \goldmedal{83.26$_{\pm 12.40}$}
& \goldmedal{95.96$_{\pm 2.85}$}
& \silvermedal{82.35$_{\pm 8.96}$}
& \goldmedal{72.08$_{\pm 5.58}$}
& \silvermedal{95.30$_{\pm 7.45}$}
& \goldmedal{90.76$_{\pm 9.25}$}
& \goldmedal{86.50$_{\pm 6.24}$}
& \silvermedal{87.70$_{\pm 6.20}$} \\
\hline\hline

\rowcolor{realrow}
\cellcolor{modelcol}\textbf{Real (TRTR)}
& 86.13$_{\pm 8.56}$ & 96.54$_{\pm 2.43}$ & 89.88$_{\pm 6.96}$ & 72.34$_{\pm 7.81}$ & 98.04$_{\pm 3.15}$ & 91.76$_{\pm 5.34}$ & 86.70$_{\pm 6.19}$ & 94.38$_{\pm 4.22}$ \\
\hline
\end{tabular}

\vspace{2pt}
\begin{minipage}{0.98\textwidth}
\scriptsize
\textit{Notes.}
Gold, silver, and bronze cell colors denote the best, second-best, and third-best \emph{synthetic} baselines in each dataset column, respectively.
Real (TRTR) is shown only as an upper-bound reference and is not included in ranking.
$^\S$: default configuration without tuning or permutation.
$^\P$: Optuna-tuned configuration without permutation.
TSTR: train on synthetic, test on real; TRTR: train on real, test on real.
\end{minipage}
\vspace{-10pt}
\end{table*}
\begin{table}[t]
\caption{
COL downstream performance: BSTabDiff synthetic data vs. real data TRTR.
}
\label{tab:colon_downstream_classifiers}
\centering
\small
\setlength{\tabcolsep}{5pt}
\renewcommand{\arraystretch}{1.15}
\begin{tabular}{|l|c|c||c|c||c|c|}
\hline

\rowcolor{rankheader}
\textcolor{white}{\textbf{Data}} &
\textcolor{white}{\textbf{$+$Acc}} &
\textcolor{white}{\textbf{$+$AUC}} &
\textcolor{white}{\textbf{$*$Acc}} &
\textcolor{white}{\textbf{$*$AUC}} &
\textcolor{white}{\textbf{$\diamond$Acc}} &
\textcolor{white}{\textbf{$\diamond$AUC}} \\
\hline

\rowcolor{bstabrow}
\textbf{BSTabDiff$^\P$}
& 81.92$_{\pm 10.56}$
& \cellcolor{goldcell}\textbf{88.97}$_{\pm 10.33}$
& 86.44$_{\pm 9.80}$
& 88.80$_{\pm 8.11}$
& \cellcolor{goldcell}\textbf{85.51}$_{\pm 10.16}$
& 88.82$_{\pm 8.97}$ \\

\rowcolor{bstabrow}
\textbf{BSTabDiff$^\S$}
& \cellcolor{goldcell}\textbf{81.97}$_{\pm 12.07}$
& 87.85$_{\pm 9.33}$
& \cellcolor{goldcell}\textbf{86.74}$_{\pm 10.10}$
& \cellcolor{goldcell}\textbf{89.25}$_{\pm 7.82}$
& 85.46$_{\pm 10.81}$
& \cellcolor{goldcell}\textbf{89.22}$_{\pm 9.76}$ \\

\hline\hline

\rowcolor{realrow}
\textbf{Real (TRTR)}
& 79.44$_{\pm 11.10}$
& 87.02$_{\pm 9.25}$
& 86.18$_{\pm 8.68}$
& 90.53$_{\pm 8.64}$
& 82.28$_{\pm 7.86}$
& 87.20$_{\pm 9.68}$ \\
\hline
\end{tabular}

\vspace{2pt}
\begin{minipage}{0.98\columnwidth}
\scriptsize
\textit{Notes.}
Gold indicates the better mean result among the two
BSTabDiff synthetic configurations for each metric.
Real (TRTR) is shown separately as a reference.
$+$: TANDEM; $*$: TabPFN\nobreakdash-2.5;
$\diamond$: CatBoost.
$^\S$: default, no tuning/permutation.
$^\P$: Optuna-tuned, no permutation.
\end{minipage}
\vspace{-4pt}
\end{table}
\begin{table*}[t]
\centering
\scriptsize
\setlength{\tabcolsep}{2.5pt}
\renewcommand{\arraystretch}{1.20}

\caption{
\textbf{Extended TSTR AUC (\%).} Results on nine additional datasets. Gold, silver, and bronze cells indicate the first, second, and third-highest distinct synthetic data results, respectively; ties share the same color. Bold values denote the best synthetic data result for each dataset. Standard deviations are shown as subscripts. TRTR is the real data reference and is excluded from ranking. The average is computed from the nine dataset-level means. See Tables~\ref{tab:extended_highdim_results} \&~\ref{tab:cifar10_fidelity_variants} for TSTR acc. \& fidelity.
}
\label{tab:extended_highdim_auc}

\resizebox{\textwidth}{!}{%
\begin{tabular}{lcccccccccc}
\toprule

\rowcolor{rankheader}
\textcolor{white}{\textbf{Model}}
&
\textcolor{white}{
\shortstack{
\textbf{TCGA}\\
\textbf{BRCA}\\
\(^{*,\dagger,\ddagger}\)}}
&
\textcolor{white}{
\shortstack{
\textbf{CIFAR-10}\\
\(^{+}\)}}
&
\textcolor{white}{
\shortstack{
\textbf{MiniBooNE}\\
\(^{\#}\)}}
&
\textcolor{white}{
\shortstack{
\textbf{Springleaf}\\
\(^{+,\dagger,\ddagger}\)}}
&
\textcolor{white}{
\shortstack{
\textbf{ORLraws10P}\\
\(^{*}\)}}
&
\textcolor{white}{
\shortstack{
\textbf{Adult}\\
\(^{\#}\)}}
&
\textcolor{white}{
\shortstack{
\textbf{BASEHOCK}\\
\(^{\diamond}\)}}
&
\textcolor{white}{
\shortstack{
\textbf{RELATHE}\\
\(^{\diamond}\)}}
&
\textcolor{white}{
\shortstack{
\textbf{PCMAC}\\
\(^{\diamond}\)}}
&
\cellcolor{aucAvgHead}
\textcolor{white}{\textbf{Avg.}}
\\
\midrule


\rowcolor{aucBSTabA}
\cellcolor{aucBSTabName}
\textbf{BSTabDiff-no-\(p\)}
& \aucgoldstd{87.33}{2.88}
& \aucgoldstd{98.40}{0.09}
& \aucbronzestd{92.18}{1.44}
& \aucsilverstd{63.71}{1.37}
& \aucgoldstd{99.31}{0.72}
& \aucval{85.25}{1.08}
& \aucgoldstd{95.50}{3.12}
& \aucsilverstd{94.20}{1.51}
& \aucsilverstd{86.97}{4.15}
& \aucgold{89.21}
\\

\rowcolor{aucBSTabB}
\cellcolor{aucBSTabName}
\textbf{BSTabDiff-BS}
& \aucsilverstd{87.25}{3.17}
& \aucsilverstd{98.39}{0.08}
& \aucval{91.74}{4.10}
& \aucgoldstd{63.79}{1.54}
& \aucsilverstd{99.30}{0.70}
& \aucgoldstd{85.59}{0.84}
& \aucbronzestd{94.61}{3.57}
& \aucgoldstd{94.25}{1.47}
& \aucgoldstd{87.03}{4.33}
& \aucsilver{89.11}
\\

\rowcolor{aucBSTabA}
\cellcolor{aucBSTabName}
\textbf{BSTabDiff-FC}
& \aucbronzestd{87.19}{3.34}
& \aucgoldstd{98.40}{0.10}
& \aucsilverstd{92.30}{1.39}
& \aucval{63.10}{1.54}
& \aucbronzestd{99.18}{0.83}
& \aucbronzestd{85.54}{1.12}
& \aucsilverstd{94.99}{3.35}
& \aucbronzestd{93.90}{1.37}
& \aucbronzestd{85.80}{5.36}
& \aucbronze{88.93}
\\

\rowcolor{aucBSTabB}
\cellcolor{aucBSTabName}
\textbf{BSTabDiff-d}
& \aucval{86.79}{3.73}
& \aucbronzestd{98.37}{0.10}
& \aucgoldstd{92.70}{0.85}
& \aucbronzestd{63.70}{1.31}
& \aucval{99.12}{0.91}
& \aucval{85.51}{0.91}
& \aucval{90.64}{1.89}
& \aucval{78.06}{2.90}
& \aucval{75.99}{3.55}
& \cellcolor{aucAvg}85.65
\\

\midrule


\rowcolor{aucBaseA}
TabDiff
& \aucval{81.65}{5.54}
& \aucval{95.82}{0.83}
& \aucval{86.74}{2.02}
& \aucval{61.92}{3.33}
& \aucval{82.53}{7.98}
& \aucval{83.49}{1.64}
& \aucval{85.72}{5.15}
& \aucval{80.17}{6.79}
& \aucval{69.60}{7.96}
& \cellcolor{aucAvg}80.85
\\

\rowcolor{aucBaseB}
CtrTab
& \aucval{83.51}{4.20}
& \aucval{97.15}{0.31}
& \aucval{86.33}{1.90}
& \aucval{63.19}{1.51}
& \aucval{98.71}{0.54}
& \aucval{83.06}{1.87}
& \aucval{91.64}{2.65}
& \aucval{81.20}{2.75}
& \aucval{73.82}{3.52}
& \cellcolor{aucAvg}84.29
\\

\rowcolor{aucBaseA}
TabDDPM
& \aucval{32.26}{5.74}
& \aucval{51.39}{3.95}
& \aucval{82.53}{4.74}
& \aucval{50.45}{4.75}
& \aucval{49.44}{9.76}
& \aucval{78.75}{4.46}
& \aucval{51.34}{4.28}
& \aucval{52.01}{2.96}
& \aucval{50.81}{3.56}
& \cellcolor{aucAvg}55.44
\\

\rowcolor{aucBaseB}
BinaryDiff
& \aucval{52.49}{7.23}
& \aucval{51.16}{1.33}
& \aucval{58.91}{11.87}
& \aucval{51.39}{2.90}
& \aucval{82.89}{8.56}
& \aucval{80.44}{2.40}
& \aucval{58.40}{6.03}
& \aucval{55.63}{2.52}
& \aucval{57.54}{2.84}
& \cellcolor{aucAvg}60.98
\\

\rowcolor{aucBaseA}
SMOTE
& \aucval{86.39}{4.35}
& \aucval{97.74}{0.35}
& \aucval{90.56}{2.17}
& \aucval{62.83}{1.71}
& \aucval{98.91}{0.25}
& \aucsilverstd{85.58}{0.92}
& \aucval{88.40}{3.19}
& \aucval{80.12}{4.25}
& \aucval{75.87}{2.90}
& \cellcolor{aucAvg}85.16
\\

\midrule


\rowcolor{aucTRTR}
\textbf{TRTR}
& \aucval{92.51}{2.52}
& \aucval{99.09}{0.09}
& \aucval{94.18}{0.19}
& \aucval{71.53}{0.74}
& \aucval{100.00}{0.00}
& \aucval{88.29}{0.53}
& \aucval{99.51}{0.17}
& \aucval{95.07}{1.34}
& \aucval{94.56}{1.03}
& \cellcolor{aucAvg}92.75
\\

\bottomrule
\end{tabular}%
}

\vspace{0.5em}

\raggedright
\footnotesize
\textbf{Dataset notation:}
\(^{*}\)HDLSS;
\(^{+}\)HDHSS;
\(^{\#}\)LDHSS;
\(^{\diamond}\)MixedRegime;
\(^{\dagger}\)mixed continuous and categorical features;
\(^{\ddagger}\)feature-wise missing values.
The three MixedRegime datasets are high-dimensional bag of words text representations.

\vspace{-0.5em}

\end{table*}
\begin{figure*}[t]
    \centering

    \begin{subfigure}[t]{0.32\textwidth}
        \centering
        \includegraphics[width=\linewidth]{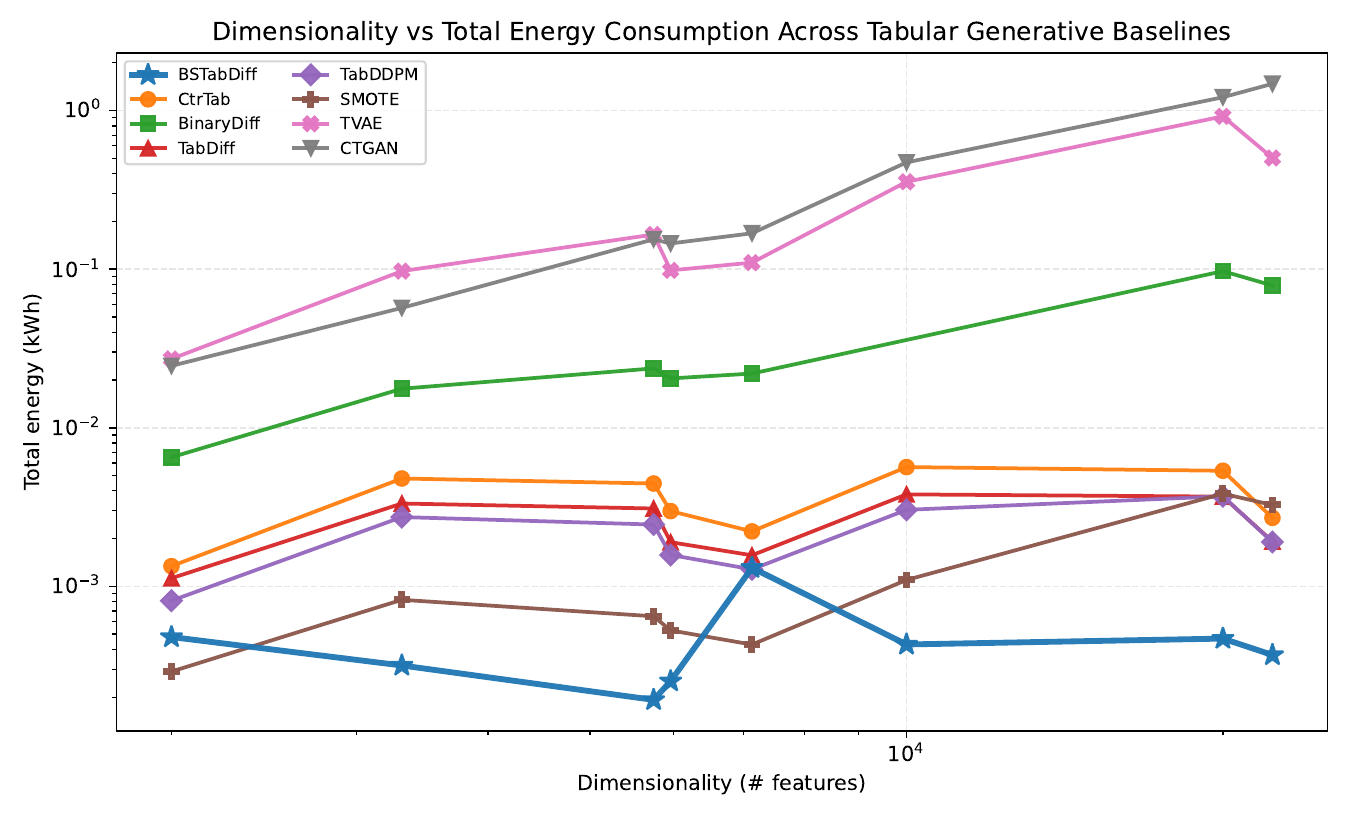}
        \caption{Energy vs. dimensionality.}
        \label{fig:energy_dim}
    \end{subfigure}
    \hfill
    \begin{subfigure}[t]{0.32\textwidth}
        \centering
        \includegraphics[width=\linewidth]{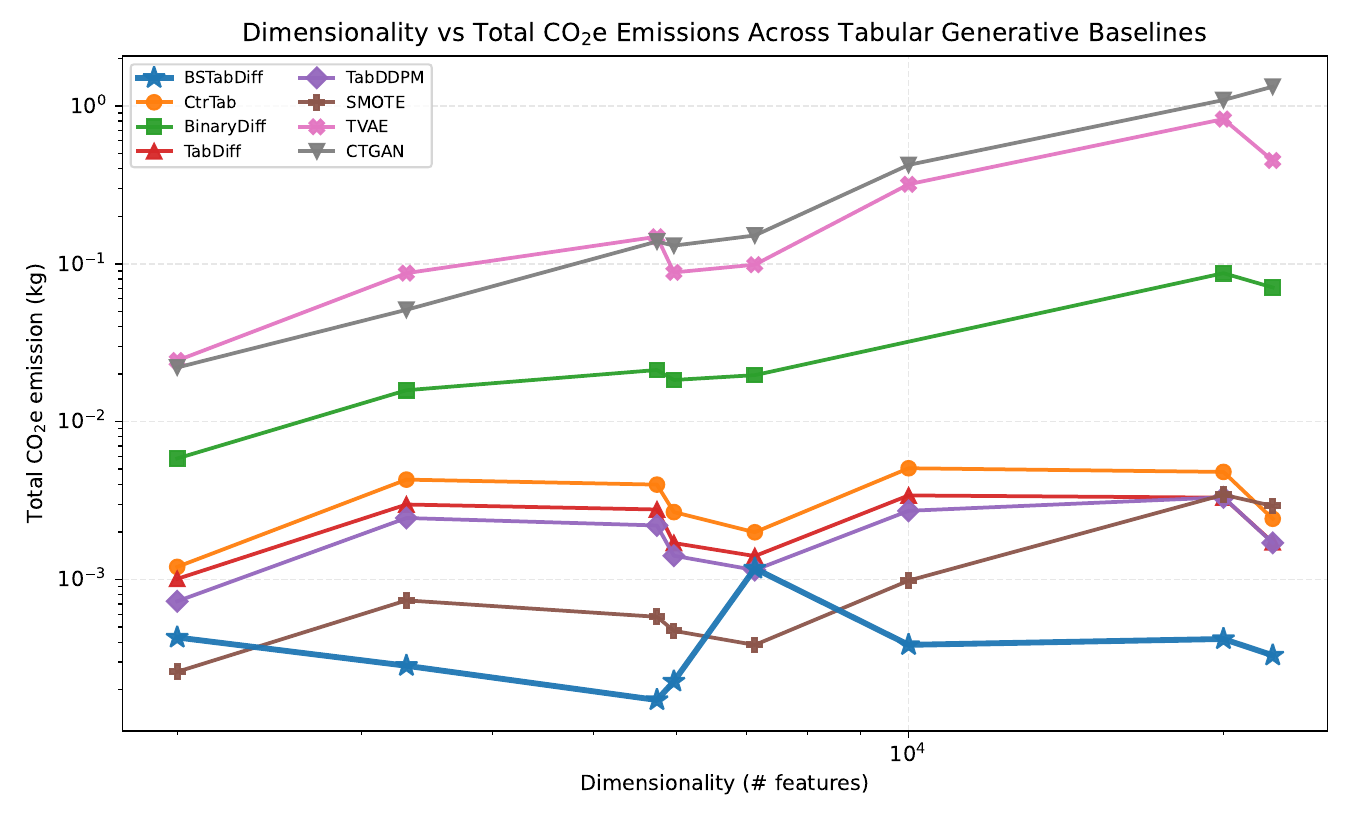}
        \caption{CO$_2$e vs. dimensionality.}
        \label{fig:co2_dim}
    \end{subfigure}
    \hfill
    \begin{subfigure}[t]{0.32\textwidth}
        \centering
        \includegraphics[width=\linewidth]{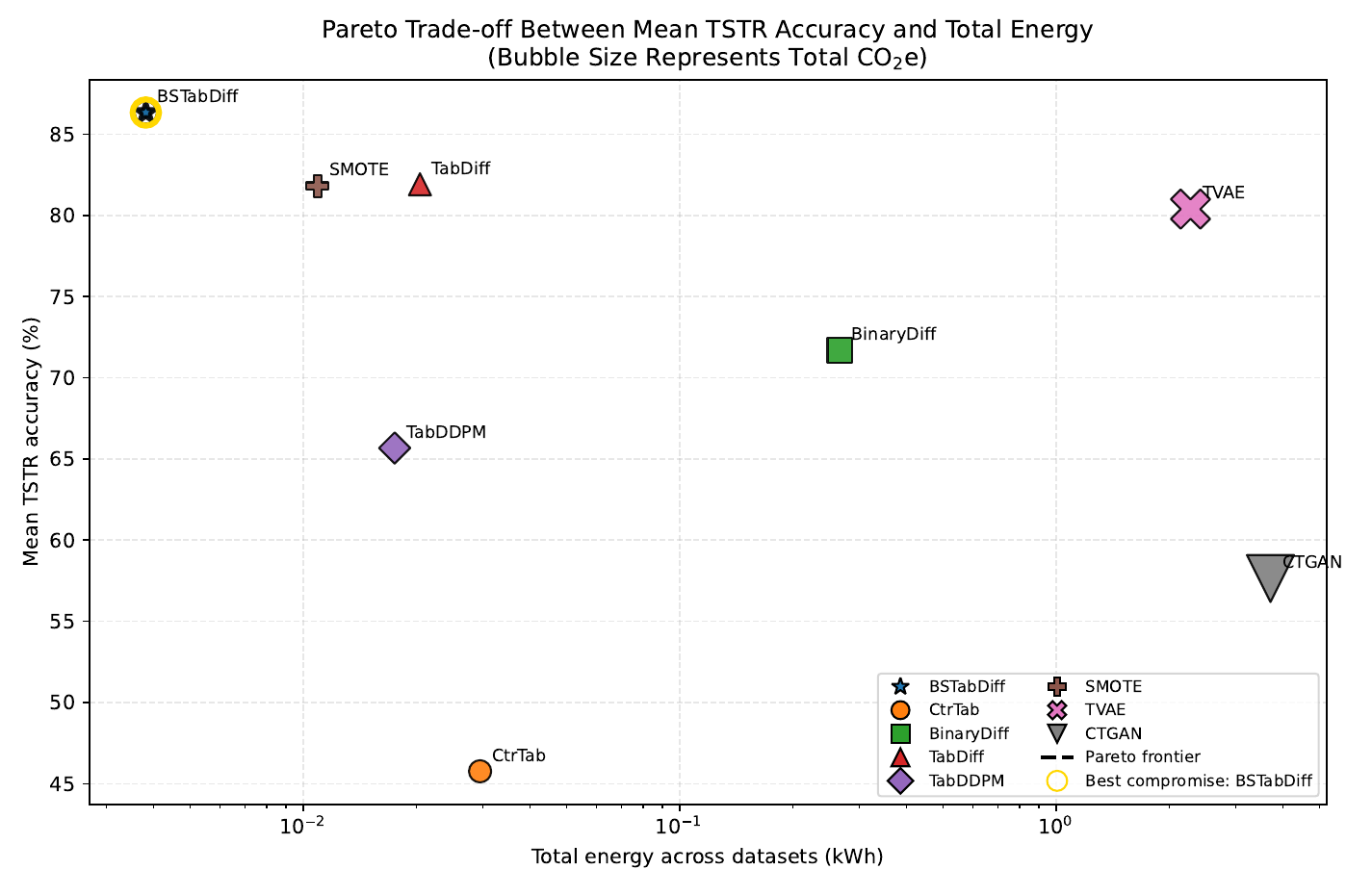}
        \caption{TSTR-energy Pareto trade-off.}
        \label{fig:pareto_energy}
    \end{subfigure}

    \caption{
Energy-carbon-aware evaluation on HDLSS benchmarks.
BSTabDiff achieves the best TSTR-energy-CO$_2$e trade-off and is the unique Pareto-efficient method. See Appendix~\ref{app:green_ai} for details.
}
    \label{fig:energy_carbon_pareto}
\vspace{-8pt}
\end{figure*}

\textbf{Computational analysis.}
BSTabDiff trains a low-dimensional latent prior in $\mathbb{R}^M$ (with $M\ll m$) and decodes to $m$ features via block-wise emissions. For a dataset with $n$ samples, $m$ features, $M$ blocks, diffusion horizon $T$, and prior training epochs $E$, the dominant costs are (i) block-latent inference and emission fitting, which are linear in the observed entries and scale as $\mathcal{O}(nm)$ (e.g., Gaussianization/rank steps plus per-feature regression-like fits), and (ii) prior learning in latent space, which scales as $\mathcal{O}(E \cdot \min\{b,n\} \cdot T \cdot M \cdot H)$ for diffusion (or $\mathcal{O}(E \cdot \min\{b,n\} \cdot L \cdot M \cdot H)$ for a flow with $L$ coupling layers), where $b$ is batch size and $H$ is the prior network width. Crucially, unlike generators that model dense dependence directly in $\mathbb{R}^m$, the expensive global dependence learning term depends on $M$ rather than $m$, while the $\mathcal{O}(nm)$ emission-side work is a single pass over features and samples. This scaling is consistent with the empirical efficiency in Table~\ref{tab:bstabdiff_compute_optuna_noperm}: despite large $m$ (up to $22{,}283$ features in GLI-85 and $19{,}993$ in SMK), training remains fast and memory-light, typically requiring under 10 seconds, below 0.1\,GiB peak GPU memory, and about 1.5-1.7\,GiB CPU RSS (Resident Set Size). This reflects the use of an \(M\)-dimensional latent prior and simple block-conditioned decoding. See statistical significance analysis in Appendix~\ref{stats}.

\textbf{Performance Comparison.} 
Table~\ref{tab:bstabdiff_mle_lr_hdlss} shows that BSTabDiff is the strongest synthetic data method under MLE-style Logistic Regression evaluation across 8 HDLSS datasets. Across its two configurations, BSTabDiff achieves the best synthetic accuracy on every dataset: the default setting performs best on COL, LNG, SMK, PRS, and ARC, while the tuned setting performs best on GLI, AML, and TOX. It also closely approaches the real data upper bound (TRTR) on several datasets, including LNG (95.96\% vs.\ 96.54\%), SMK (72.08\% vs.\ 72.34\%), ARC (86.50\% vs.\ 86.70\%), AML (95.56\% vs.\ 98.04\%), and PRS (90.76\% vs.\ 91.76\%). Table~\ref{tab:colon_downstream_classifiers} further shows that, on Colon, BSTabDiff synthetic data remains competitive with real data training across TANDEM~\citep{tandem}, TabPFN-2.5~\citep{tabpfn2.5}, and CatBoost~\citep{catboost}, matching or slightly exceeding TRTR in several accuracy metrics while maintaining comparable AUC. Together, these results suggest that BSTabDiff preserves decision-relevant structure and produces high-utility synthetic samples that transfer across datasets and downstream evaluators in the HDLSS regime.

\textbf{Generalization beyond the Original HDLSS Scope.}
We further evaluate BSTabDiff on nine additional datasets spanning HDLSS, HDHSS, LDHSS, and MixedRegime settings, including mixed-type omics and marketing data, image embeddings, and high-dimensional text representations. As shown in Table~\ref{tab:extended_highdim_auc}, at least one BSTabDiff variant achieves the highest synthetic data TSTR AUC on each dataset. BSTabDiff no-$p$ obtains the best average AUC of $89.21\%$, outperforming the strongest non-BSTabDiff baseline (SMOTE, $85.16\%$) and approaching the real data TRTR reference ($92.75\%$). These findings suggest that the block-subunit design generalizes beyond original HDLSS setting. Additional TSTR accuracy results and detailed analyses are deferred to Appendix~\ref{scope}.

\textbf{Fidelity Diagnostics.}
BSTabDiff’s fidelity diagnostics in Tables~\ref{tab:fidelity-bstabdiff} and~\ref{tab:fidelity-bstabdiff-noperm} (Appendix~\ref{fid}) show that synthetic samples generally preserve real marginals, dependence structure, and privacy separation without trivial memorization. Across both settings, per-feature KS / W1 distances remain relatively small on COL and AML, while LNG is consistently the most difficult case and GLI shows the largest scale-sensitive W1 and higher-order moment discrepancies, reflecting its extreme dimensionality and feature magnitude. Pearson/Spearman $|\Delta\mathrm{corr}|$ are typically moderate where defined, indicating substantial preservation of the correlation structure, and $|\Delta \mathrm{MI}(\text{feature},y)|$ remains reasonably controlled on COL, AML, and GLI but is weaker on LNG. NaN rates are negligible throughout, nearest neighbor privacy distances suggest limited sample copying, and in the default setting C2ST~\citep{c2st} is near chance for COL/AML/GLI but substantially higher on LNG, again identifying it as the hardest dataset. See Appendix~\ref{t-SNE} for t-SNE plots.

\textbf{Energy-Carbon-Aware Utility Evaluation.}
Fig.~\ref{fig:energy_carbon_pareto} evaluates tabular generative models from an energy-carbon-aware perspective.
While utility metrics such as TSTR accuracy are central, practical deployment also depends on computational cost, environmental footprint, and energy-reporting accountability~\citep{energy}. Fig.~\ref{fig:energy_dim} and~\ref{fig:co2_dim} show that several deep tabular generators incur rapidly increasing energy and carbon dioxide equivalent (CO$_2$e) costs as dimensionality grows. In contrast, BSTabDiff remains consistently energy efficient across HDLSS datasets. Fig.~\ref{fig:pareto_energy} summarizes the trade-off between mean TSTR accuracy and total energy consumption, with bubble size indicating total CO$_2$e emissions. BSTabDiff attains the highest mean TSTR accuracy while also achieving the lowest total energy and CO$_2$e, yielding the best utility-sustainability trade-off among all evaluated methods. To reduce measurement noise, each run was executed as a single isolated process on a fixed GPU/node and monitored with \texttt{nvidia-smi}/\texttt{psutil}~\citep{psutil}; CodeCarbon~\citep{g8} estimates local CPU/GPU/RAM energy only, so other nodes do not affect the estimate, while residual co-located process noise is treated as measurement uncertainty. See Appendix~\ref{app:green_ai} for CodeCarbon~\citep{g8} protocol details and dataset-wise energy/carbon results.

\textbf{Ablation studies.}
Ablation results on Colon (TabPFN-2.5) show that BSTabDiff is broadly robust to reasonable hyperparameter changes, with TSTR AUC remaining in a tight range across compact variants (Fig.~\ref{fig:bstabdiff_ablation_4rows}, Table~\ref{tab:bstabdiff_compact_ablations} in Appendix~\ref{abla}). Disabling class-conditional marginals causes the clearest degradation in downstream utility, while switching to a flow prior and increasing the synthetic sample budget can slightly improve TSTR, with $n_{\text{syn}}{=}500$ achieving the best mean TSTR AUC. AUC efficiency (TSTR/TRTR) stays close to 1 across settings, indicating that synthetic-utility gains reflect real data performance rather than overfitting, and the scaling curves suggest only mild sensitivity to prior training epochs and the number of blocks $M$ within the tested range. We provide a practical $(\rho,\kappa)$-based $M$ initialization rule and a completed $5\times5$ CV sensitivity analysis of the block-latent dimension $M$ in App.~\ref{app:m_selection_sensitivity}. A matched dense-factor ablation demonstrates the parameter efficiency and metric-dependent benefits of block-sparse generation (App.~\ref{app:block_subunit_necessity}). A complete ``one factor at a time” component sensitivity analysis and practical tuning guide are provided in App.~\ref{app:component_sensitivity}.

\textbf{Limitations and Impact Statement.}
BSTabDiff primarily targets HDLSS data ($n\ll m$), although our additional experiments demonstrate promising generalization across HDHSS, LDHSS, and mixed dimensionality regimes (Appendix~\ref{scope}). Its effectiveness depends partly on meaningful feature dependencies; when block structure is weak, GO-BS may introduce computational overhead with limited benefit. BSTabDiff supports synthetic data augmentation and benchmarking in data-scarce domains, with explicit energy and CO$_2$e reporting. Nevertheless, synthetic data should complement, not replace, real data validation and appropriate safeguards in high-stakes applications.
\vspace{-1em}
\section{Conclusion}
\label{sec:conclusion}
\vspace{-1em}
We introduced BSTabDiff, a block-subunit generative framework for HDLSS tabular synthesis that learns global dependence in a compact block-latent space ($M\ll m$) and decodes to high-dimensional observations through copula-driven, feature-wise structured emissions. By scaling with latent block dimension rather than ambient feature dimension, BSTabDiff enables stable generation when $n\ll m$. We further connect BSTabDiff with GO-BS, a graph-guided ordering and block segmentation procedure that aligns observed features into dependency-aware contiguous blocks. Empirically, BSTabDiff outperforms strong GAN/VAE/diffusion baselines, often approaches real data performance under classical and modern downstream classifiers, and preserves key marginal and dependence structure without trivial memorization. From a Green AI perspective, our energy-carbon-aware analysis shows that BSTabDiff achieves the best TSTR-energy-CO$_2$e trade-off among evaluated top methods. Overall, BSTabDiff provides an effective foundation for controllable HDLSS benchmark generation, data augmentation, and synthetic pretraining.

\section*{Acknowledgements}
\vspace{-1em}
This work was supported in part by the US National Science Foundation under Awards \#1920920, \#2125872, \#2223793, and \#2318708. We sincerely thank the anonymous NeurIPS 2026 reviewers for their valuable feedback and constructive suggestions. We are especially grateful to the anonymous Area Chair for their thoughtful insights and guidance, which significantly contributed to strengthening this work. We also gratefully acknowledge Zaigham Abbas Randhawa, Pouya Iranmanesh, and Ghazaleh Mirzaee for their assistance with GPU clusters.

\bibliographystyle{plainnat} 
\bibliography{ref}
\newpage
\appendix
\section{Appendix}
This technical appendix supports our main paper \textit{BSTabDiff: Block-Subunit Diffusion Priors for High-Dimensional Tabular Data Generation} (Accepted to the Fortieth Annual Conference on Neural Information Processing Systems - NeurIPS 2026). Specifically, it includes:
\begin{itemize}
    \item Detailed Related Work in Appendix~\ref{morerelated}
    \item Permutation Estimation via Graph-guided Ordering with Block Segmentation (GO-BS) in Appendix~\ref{sec:gobs}
    \item BSTabDiff with GO-BS in Appendix~\ref{bsgob}
    \item Ablation on Sample Clustering vs. Feature Clustering in GO-BS in Appendix~\ref{app:gobs_ablation}
    \item GO-BS Ordering Ablation against Classical Metaheuristics in Appendix~\ref{app:gobsmeta}
    \item Ordering-Solver Ablation on a Fixed Global Affinity Graph in Appendix~\ref{app:same}
    \item Ablation of GO-BS in Extreme Dimensionality in Appendix~\ref{app:gobsextreme}
    \item BSTabDiff in Extreme Dimensionality in Appendix~\ref{extreme}
    \item Energy Use, Carbon Emissions, and Green AI Considerations in Appendix~\ref{app:green_ai}
    \item Pseudocode in Appendix~\ref{pseudocode}
    \item Additional Fidelity Diagnostics in Appendix~\ref{fid}
    \item Additional Ablation Studies in Appendix~\ref{abla}
    \item HDLSS Literature and Datasets Used in Appendix~\ref{hdlss}
    \item Representation Quality via t-SNE in Appendix~\ref{t-SNE}
    \item Statistical Significance Analysis in Appendix~\ref{stats}
    \item BSTabDiff Hyperparameters in Appendix~\ref{hyp}
    \item Additional Runtime and Resource Usage in Appendix~\ref{run}
    \item Generalization across Diverse Data Regimes in Appendix~\ref{scope}
    \item Block-Latent Dimension Selection and Sensitivity Analysis in Appendix~\ref{app:m_selection_sensitivity}
    \item Necessity of the Block-Subunit Structure in Appendix~\ref{app:block_subunit_necessity}
    \item Component Sensitivity Analysis and Practical Configuration Guide in Appendix~\ref{app:component_sensitivity}
    \item Additional Methodological Notes and Clarifications in Appendix~\ref{app:additional_methodological_notes}
    \item Dataset and Baseline Sources in Appendix~\ref{app:dataset}
\end{itemize}
\vspace{-1em}
\renewcommand{\thesection}{A1}
\renewcommand{\thesubsection}{A1.\arabic{subsection}}
\setcounter{figure}{0}\renewcommand{\thefigure}{A1.\arabic{figure}}
\setcounter{table}{0}\renewcommand{\thetable}{A1.\arabic{table}}
\section{Detailed Related Work}
\label{morerelated}
\vspace{-1em}
Prior work related to tabular generation includes both generative synthesis methods and feature augmentation approaches (e.g., AquaAugmentor~\citep{aqua} or AugTab~\citep{augtab}). Here, we focus on generative models, since BSTabDiff is a generative framework for HDLSS tabular synthesis.

\paragraph{GAN/VAE-based tabular generators.}
Tabular synthesis has evolved from classical oversampling strategies such as SMOTE~\citep{smote} to deep generative models designed for mixed-type tables with complex dependencies. Early representative methods include CTGAN~\citep{w16} (and its VAE counterpart, TVAE~\citep{w16}), which introduced conditional training strategies to better handle skewed categorical variables; CTAB-GAN~\citep{ctabgan}, which improved mixed-type modeling and missing-value handling via redesigned encodings and conditional vectors; and CTAB-GAN+~\citep{ctabgan+}, which incorporated differential privacy mechanisms (e.g., DP-SGD) to strengthen privacy-utility trade-offs. Privacy-centric and domain-specific variants such as PATE-GAN~\citep{pate} and medGAN~\citep{med} further highlight the importance of controlled disclosure and domain constraints in synthetic data generation.

\paragraph{Diffusion and score-based tabular generators.}
More recently, diffusion/score-based methods have emerged as strong general-purpose approaches for tabular synthesis. TabDDPM~\citep{w11} established diffusion as a competitive baseline for tabular generation, motivating subsequent refinements for mixed-type data and improved dependency modeling, including STaSy~\citep{stasy}, CoDi~\citep{codi}, and TabDiff~\citep{tabdiff}. Latent-space diffusion has also been explored to better manage heterogeneous feature types and reduce modeling difficulty, as in TabSyn~\citep{w13}. Complementary hybrid directions combine diffusion/flow ideas with tree-based learners, such as ForestDiffusion~\citep{fdiff}, while discretization-first strategies such as Binary Diffusion~\citep{bfiff} provide alternative pipelines for handling tabular features. HARPOON~\citep{harpoon} is an inference-time, manifold-guided conditional tabular diffusion method designed mainly for tasks such as imputation and inequality constraints, but unlike our approach it does not explicitly model contiguous latent block structure, and thus does not address the sample-efficiency and structural compression benefits that are central to HDLSS generation.

\paragraph{LLM/foundation-model and structure-aware generators.}
Beyond diffusion, structure-aware models explicitly represent column relationships, exemplified by GOGGLE~\citep{goggle}, which learns relational structure among features for improved synthesis. In parallel, LLM-based tabular synthesis treats rows as sequences and leverages pretrained language models: GReaT~\citep{great} adapts autoregressive LMs with flexible conditioning, HARMONIC~\citep{harmonic} emphasizes joint utility and privacy evaluation for LLM-based synthesizers, and EPIC~\citep{epic} proposed  effective prompting for imbalanced-class tabular synthesis. Finally, CTSyn~\citep{ctsyn} targets cross-table generalization using schema-conditioned latent diffusion over a shared representation space.

\paragraph{HDLSS-oriented and selection-based generative directions.}
Recent work has also explored complementary directions for wide tabular generation. TabPFN-Wide~\citep{tabpfnw} develops a customized prior for synthetic HDLSS data generation and uses continued pre-training to adapt tabular foundation models to extreme feature counts. CtrTab~\citep{ctr} proposes a control-conditioned diffusion model for HDLSS tabular synthesis, but most of its evaluation is on datasets with only moderate feature dimensionality and relatively large sample sizes, with Arcene being the main truly HDLSS case, so its empirical setting only partially matches the breadth implied by its title. SYNTHONY~\citep{synthony} instead focuses on intent-conditioned selection, proposing a stress-aware framework for choosing among GAN, diffusion, and LLM-based tabular synthesizers under fidelity, privacy, and utility objectives. 

\textbf{Positioning.} Unlike prior tabular generators that largely treat all features in a flat manner, operate directly in the ambient feature space, BSTabDiff is designed specifically for the HDLSS regime, where $n \ll m$ and feature dependencies are often structured into local correlation groups. Our key distinction is to introduce a block-subunit generative view that compresses global dependence learning into a low-dimensional block-latent space while retaining flexible feature-wise decoding through copula-based emissions, non-Gaussian marginals, and explicit missingness modeling. In this sense, BSTabDiff is not simply another diffusion-based tabular generator; rather, it contributes a structure-aware HDLSS generative framework that can use diffusion or flow priors within a block-factorized design, improving stability, and controllability for high-dimensional tabular synthesis. Table~\ref{tab:gen_benchmark_scope} and Fig.~\ref{fig:bstabdiff_scope} position BSTabDiff relative to prior non-LLM tabular generators. BSTabDiff is evaluated on 18 real datasets against 10 distinct baselines, covering sample sizes from 62 to 130,064 and feature dimensions from 14 to 42,728. Although some methods evaluate on more datasets, their reported benchmarks predominantly cover lower-dimensional settings, whereas BSTabDiff additionally demonstrates applicability to extreme HDLSS data. Specifically, Fig.~\ref{fig:bstabdiff_scope}(b) highlights the initial eight HDLSS benchmarks and Cell Cycle, spanning 62-1,067 samples and 2,000-42,728 features, while Fig.~\ref{fig:bstabdiff_scope}(a) reflects the complete 18 dataset evaluation, including nine additional datasets across diverse dimensionality regimes. We exclude LLM-based or LLM-assisted proposed models from this comparison because tokenizing thousands of features can incur substantial computational costs, particularly under quadratic self-attention, making conventional LLM-based generation less suitable for direct extreme HDLSS synthesis. For CtrTab, we consider its full reported evaluation of nine datasets, comprising seven main benchmarks and two additional high-dimensional case studies, Stadyn and Arcene. Dataset details and original sources are provided in Appendix~\ref{app:dataset}; see also Appendix~\ref{hdlss} for related HDLSS literature. 
Table~\ref{tab:hdlss_datasets} summarizes eight HDLSS datasets spanning 62-203 samples, 2,000-22,283 features, and 2-5 classes, highlighting their extreme feature-to-sample ratios and varying class distributions. See other datasets in Table~\ref{tab:all_datasets_summary}.


\begin{figure*}[t]
  \centering
  \includegraphics[width=0.95\textwidth]{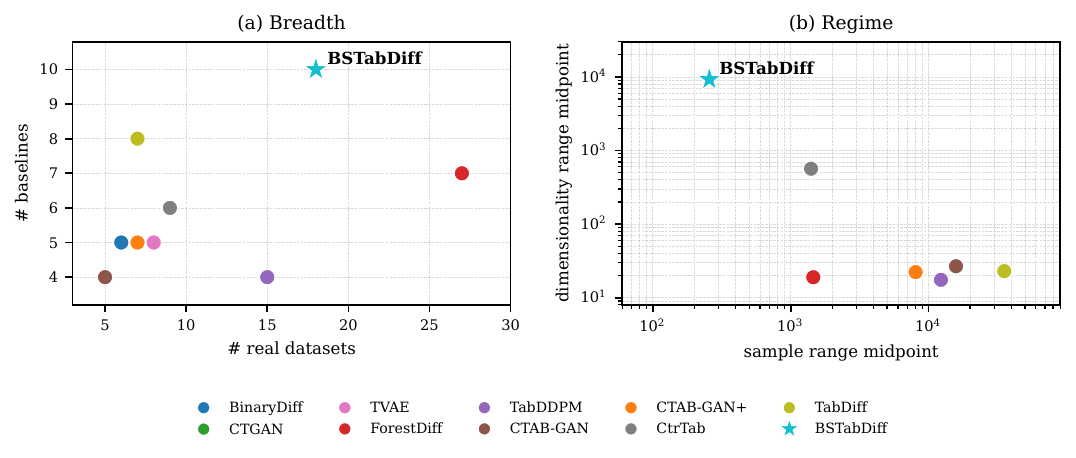}
  \caption{\textbf{Benchmark scope comparison of non-LLM tabular generative models.} (a) Evaluation breadth measured by the number of real datasets and distinct baseline methods; BSTabDiff covers 18 datasets and 10 baselines. (b) Benchmark regimes represented by the geometric midpoints of reported sample size and feature dimensionality ranges on logarithmic axes. For BSTabDiff, panel (b) specifically considers the eight original HDLSS benchmarks and the extreme dimensional Cell Cycle dataset ($n=62$-$1067$, $m=2000$-$42728$), highlighting its extreme HDLSS evaluation. The complete 18 dataset scope is summarized in Table~\ref{tab:gen_benchmark_scope}.}
  \label{fig:bstabdiff_scope}
\end{figure*}


\begin{table*}[t]
\centering
\caption{
Benchmark scope of tabular synthetic data generation methods using real datasets only. LLM-based or LLM-assisted proposed models are excluded as rows, but LLM baselines are counted when used by non-LLM methods. ``NR'' indicates that the exact range was not reported. For CtrTab, we include its seven main benchmarks plus Stadyn and Arcene, following the full tested range in the original paper. For BSTabDiff, the evaluation includes eight initial HDLSS benchmarks, the extreme dimensional Cell Cycle dataset, and nine additional datasets spanning HDLSS, HDHSS, LDHSS, and mixed dimensionality regimes. The number of baselines denotes distinct comparison methods across the evaluation, not necessarily the number evaluated on every dataset. Dataset details are provided in Appendix~\ref{app:dataset}, and the extreme dimensional evaluation is discussed in Appendix~\ref{extreme}.
}
\label{tab:gen_benchmark_scope}

\scriptsize
\setlength{\tabcolsep}{5pt}
\renewcommand{\arraystretch}{1.20}

\begin{tabular}{
>{\columncolor{modelcol}}l
>{\columncolor{fidelityKS}}c
>{\columncolor{fidelityW}}c
>{\columncolor{fidelityPearson}}c
>{\columncolor{fidelityHOM}}c
}
\toprule

\rowcolor{fidelityheader}
\textbf{Model} &
\textbf{\# Real Datasets} &
\textbf{\# Baselines Compared} &
\textbf{Sample Range} &
\textbf{Dimensionality Range} \\

\midrule

BinaryDiff
& 6
& 5
& NR
& NR \\

CTGAN
& 8
& 5
& NR
& NR \\

TVAE
& 8
& 5
& NR
& NR \\

ForestDiff
& 27
& 7
& [103 - 20640]
& [4 - 90] \\

SMOTE
& 9
& N/A
& [937 - 443872]
& NR \\

TabDDPM
& 15
& 4
& [768 - 197080]
& [6 - 51] \\

CTAB-GAN
& 5
& 4
& [5000 - 50000]
& [13 - 55] \\

CTAB-GAN+
& 7
& 5
& [1300 - 50000]
& [9 - 55] \\

CtrTab
& 9
& 6
& [200 - 9873]
& [32 - 10001] \\

TabDiff
& 7
& 8
& [12330 - 101766]
& [11 - 48] \\

\midrule

\rowcolor{bstabrow}
\textbf{BSTabDiff}
& \textbf{18}
& \textbf{10}
& \textbf{[62 - 130064]}
& \textbf{[14 - 42728]} \\

\bottomrule
\end{tabular}
\end{table*}
\begin{table}[t]
\caption{Summary of the HDLSS datasets used in our experiments ($n$=\#samples, $m$=\#features, $C$=\#classes; distribution shown as class counts). See other datasets in Table~\ref{tab:all_datasets_summary}.}
\label{tab:hdlss_datasets}
\centering
\scriptsize
\setlength{\tabcolsep}{3.5pt}
\renewcommand{\arraystretch}{1.15}
\begin{tabular}{|l|l|r|r|r|l||l|l|r|r|r|l|}
\hline
\rowcolor{datasetheader}
\textcolor{white}{\textbf{Abbr}} &
\textcolor{white}{\textbf{Name}} &
\textcolor{white}{\textbf{$n$}} &
\textcolor{white}{\textbf{$m$}} &
\textcolor{white}{\textbf{$C$}} &
\textcolor{white}{\textbf{Dist.}} &
\textcolor{white}{\textbf{Abbr}} &
\textcolor{white}{\textbf{Name}} &
\textcolor{white}{\textbf{$n$}} &
\textcolor{white}{\textbf{$m$}} &
\textcolor{white}{\textbf{$C$}} &
\textcolor{white}{\textbf{Dist.}} \\
\hline

\rowcolor{datasetrowa}
COL & Colon & 62 & 2000 & 2 & [40, 22]
& SMK & SMK & 187 & 19993 & 2 & [90, 97] \\

\rowcolor{datasetrowb}
GLI & GLI-85 & 85 & 22283 & 2 & [26, 59]
& TOX & TOX171 & 171 & 5748 & 4 & [45, 45, 39, 42] \\

\rowcolor{datasetrowa}
LNG & Lung & 203 & 3312 & 5 & [139, 17, 21, 20, 6]
& AML & ALLAML & 72 & 7129 & 2 & [47, 25] \\

\rowcolor{datasetrowb}
PRS & Prostate & 102 & 5966 & 2 & [50, 52]
& ARC & Arcene & 200 & 10000 & 2 & [112, 88] \\
\hline
\end{tabular}
\end{table}

\vspace{-1em}
\renewcommand{\thesection}{A2}
\renewcommand{\thesubsection}{A2.\arabic{subsection}}
\setcounter{figure}{0}\renewcommand{\thefigure}{A2.\arabic{figure}}
\setcounter{table}{0}\renewcommand{\thetable}{A2.\arabic{table}}
\setcounter{algorithm}{0}\renewcommand{\thealgorithm}{A2.\arabic{algorithm}}
\section{Permutation Estimation via Graph-guided Ordering with Block Segmentation (GO-BS)}
\label{sec:gobs}
\vspace{-1em}
BSTabDiff estimates a dependency-aware feature permutation and contiguous block partition from the training data before fitting the block-subunit generator. GO-BS constructs a sparse weighted feature-affinity graph and seeks an ordering under which strongly dependent features become adjacent, followed by block segmentation that favors high within-block affinity and low cross-block leakage. Unlike a generic linear arrangement criterion, the GO-BS objective jointly considers feature ordering and block consistency. The estimated permutation and boundaries define the canonical feature organization for latent inference, emission fitting, and synthetic generation. This construction is related to classical seriation and graph partitioning methods~\citep{atkins,shi}. Earlier work on feature graph architectures explored correlation-guided feature grouping and permutation for hierarchical predictive learning~\citep{feature}; in contrast, GO-BS explicitly constructs a sparse feature-affinity graph and jointly estimates a dependency-aware ordering and contiguous block segmentation for generative modeling.

\paragraph{Feature graph.}
Let $X \in \mathbb{R}^{n \times m}$ be the training data matrix, where columns correspond to features and typically $n \ll m$. We construct a weighted undirected feature-affinity graph by Eq.~\ref{eq:g}.
\begin{equation}
\label{eq:g}
    G = (V,E,W), \qquad V = \{1,\dots,m\}
\end{equation}
\begin{figure*}[t]
    \centering
    \includegraphics[width=\textwidth]{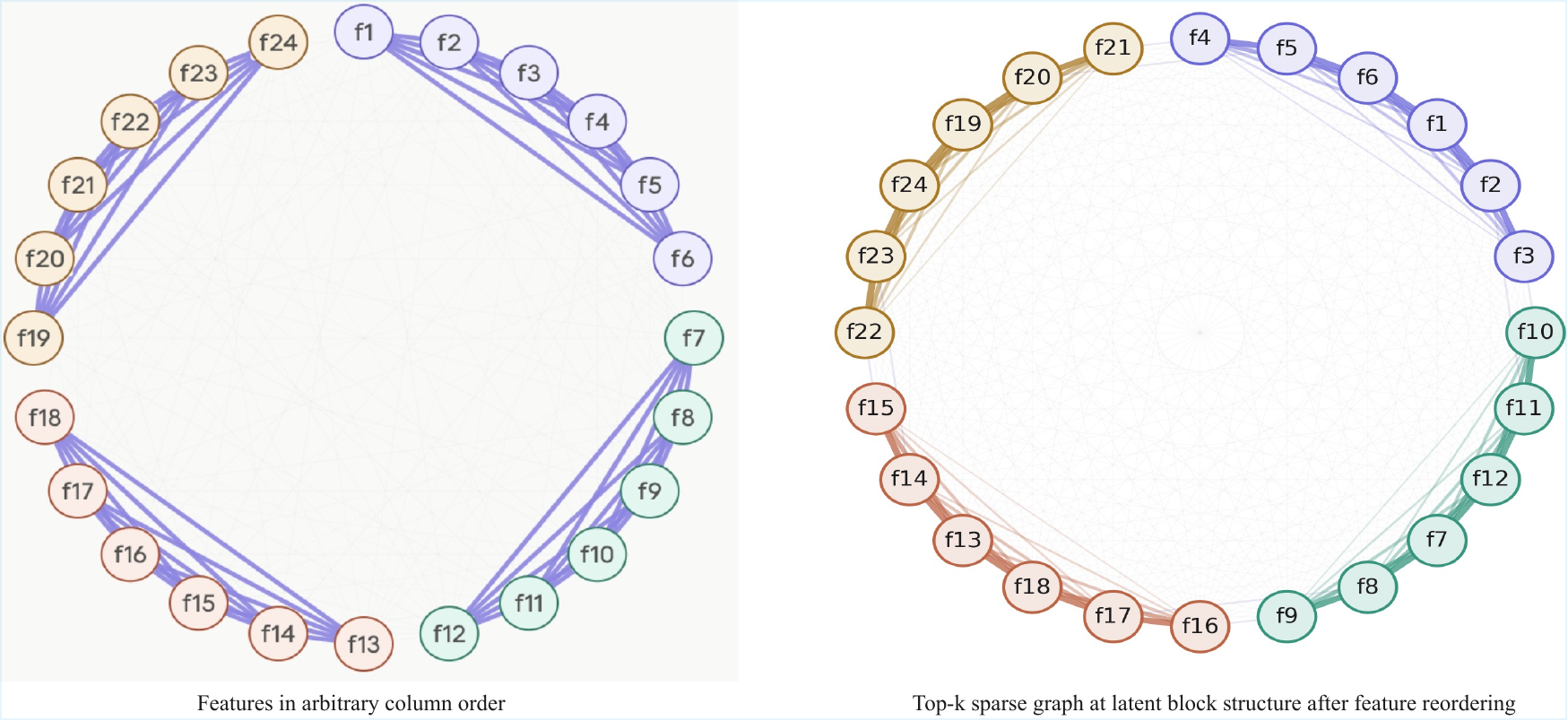}
    \caption{
\textbf{Feature reordering in the feature-affinity graph.}
Nodes \(f_1,\dots,f_{24}\) denote feature columns. 
In the arbitrary input order (left), related features need not be adjacent. We build a sparse feature-affinity graph by retaining the top-\(k\) strongest neighbors per feature; strong edges indicate within-group affinity and weak cross-group links are suppressed. 
After reordering (right), related features become contiguous, revealing dense local connectivity and reduced long range mixing.
}
    \label{fig:feature_reordering_graph}
\end{figure*}
where each vertex corresponds to one feature and $W_{ij} \ge 0$ measures affinity between features $i$ and $j$ (e.g., absolute correlation, cosine similarity, or a label-aware dependence score when labels are available). To scale to very high dimensions, we retain only the top-$k$ strongest neighbors per feature and symmetrize the resulting graph. Thus the practical graph is sparse, with $|E| = O(mk)$, where $k$ is a user-chosen sparsification parameter, typically treated as a small constant independent of $m$. To scale to very high dimensions, we sparsify the affinity graph by retaining only the top-$k$ neighbors per feature, a standard strategy in $k$-nearest neighbor graph construction and approximation~\citep{dong,dashti}. As shown in Fig.~\ref{fig:feature_reordering_graph}, reordering the feature-affinity graph places highly connected features adjacently, reducing long range mixing and revealing the underlying block structure.

\paragraph{Ordering and contiguous blocks.}
The ordering variable $\pi$ plays the role of a seriation map, converting latent relational structure into an approximately contiguous one-dimensional arrangement~\citep{atkins}. A feature ordering is a bijection (Eq.~\ref{eq:b}).
\begin{equation}
\label{eq:b}
    \pi : V \to \{1,\dots,m\}
\end{equation}
or equivalently a permutation of the feature indices. Let the block count be $M$. Given an ordering $\pi$ and a boundary vector is defined by Eq.~\ref{eq:o}.
\begin{equation}
\label{eq:o}
    \mathbf{b} = (b_1,\dots,b_{M-1}), \qquad 1 \le b_1 < \cdots < b_{M-1} < m
\end{equation}
with the conventions $b_0 := 0$ and $b_M := m$, we define the induced contiguous ordered blocks
\begin{equation}
\label{eq:gobs_blocks}
B_t(\pi,\mathbf{b})
=
\{\, i \in V : b_{t-1} < \pi(i) \le b_t \,\}
\qquad t = 1,\dots,M.
\end{equation}
Hence GO-BS jointly estimates a permutation and a contiguous segmentation of that permutation. Our objective is inspired by classical graph partitioning principles, rewarding strong within-part connectivity while penalizing cross-part leakage, but is specialized here to contiguous blocks induced by a feature ordering~\citep{ker, shi}.

\begin{tcolorbox}[
    colback=regimeHDHSS,
    colframe=algborder,
    boxrule=0.5pt,
    arc=1mm,
    boxsep=2pt,
    left=3pt,
    right=3pt,
    top=3pt,
    bottom=3pt,
    before skip=5pt,
    after skip=5pt
]
\begin{definition}[GO-BS block-consistency objective]
\label{def:gobs_obj}

Given a weighted feature graph $G=(V,E,W)$, a target block count $M$, and nonnegative coefficients $\lambda,\gamma\geq 0$, define

\begin{equation}
\label{eq:gobs_main_obj}
\resizebox{0.85\linewidth}{!}{$\displaystyle
\mathcal{J}_{\mathrm{GO\mbox{-}BS}}(\pi,\mathbf{b})
=
-\sum_{t=1}^{M}
\sum_{\substack{i,j\in B_t(\pi,\mathbf{b})\\i<j}} W_{ij}
+
\lambda\sum_{1\leq t<t'\leq M}
\sum_{\substack{i\in B_t(\pi,\mathbf{b})\\j\in B_{t'}(\pi,\mathbf{b})}}W_{ij}
+
\gamma\,\Omega(\mathbf{b})
$}
\end{equation}

where $\Omega(\mathbf{b})$ is a block-size regularizer. The GO-BS problem is

\begin{equation}
\label{eq:gobs_opt}
(\pi^\ast,\mathbf{b}^\ast)
\in
\arg\min_{\pi,\mathbf{b}}
\mathcal{J}_{\mathrm{GO\mbox{-}BS}}(\pi,\mathbf{b})
\end{equation}

\end{definition}
\end{tcolorbox}

The first term rewards within-block affinity, the second penalizes cross-block leakage, and the third regularizes the block sizes. A convenient choice is
\begin{equation}
\label{eq:omega_balance}
\Omega(\mathbf{b})
=
\sum_{t=1}^{M}
\left( |B_t(\pi,\mathbf{b})| - \tau_t \right)^2
\end{equation}
where $\tau_t$ is a target size (e.g., $\tau_t = m/M$ for nearly balanced blocks). Other regularizers may be used to prevent degenerate blocks or enforce application-specific size priors.

\begin{remark}[Equivalent cut form]
\label{rem:constant_weight}
Let
\begin{equation}
\label{eq:c}
    W_{\mathrm{tot}} = \sum_{(i,j)\in E} W_{ij}
\end{equation}
Since every edge either lies within one block or crosses between two distinct blocks, we have Eq.~\ref{eq:block}.
\begin{equation}
\label{eq:block}
    W_{\mathrm{tot}}
=
\sum_{t=1}^{M}
\sum_{\substack{i,j \in B_t(\pi,\mathbf{b})\\ i<j}} W_{ij}
+
\sum_{1 \le t < t' \le M}
\sum_{\substack{i \in B_t(\pi,\mathbf{b})\\ j \in B_{t'}(\pi,\mathbf{b})}} W_{ij}
\end{equation}
Therefore Eq.~\ref{eq:gobs_main_obj} can be rewritten, up to an additive constant, as
\begin{equation}
\label{eq:gobs_cut_form}
\mathcal{J}_{\mathrm{GO\mbox{-}BS}}(\pi,\mathbf{b})
=
-\,W_{\mathrm{tot}}
+
(1+\lambda)
\sum_{1 \le t < t' \le M}
\sum_{\substack{i \in B_t(\pi,\mathbf{b})\\ j \in B_{t'}(\pi,\mathbf{b})}} W_{ij}
+
\gamma\,\Omega(\mathbf{b})
\end{equation}
Thus GO-BS is equivalently a contiguous balanced cut minimization problem.
\end{remark}
Fig.~\ref{fig:gobs_block_cut} schematically illustrates the segmentation step after feature ordering, using a simplified contiguous feature layout for visual clarity.
\begin{figure*}[t]
    \centering
    \includegraphics[width=\textwidth]{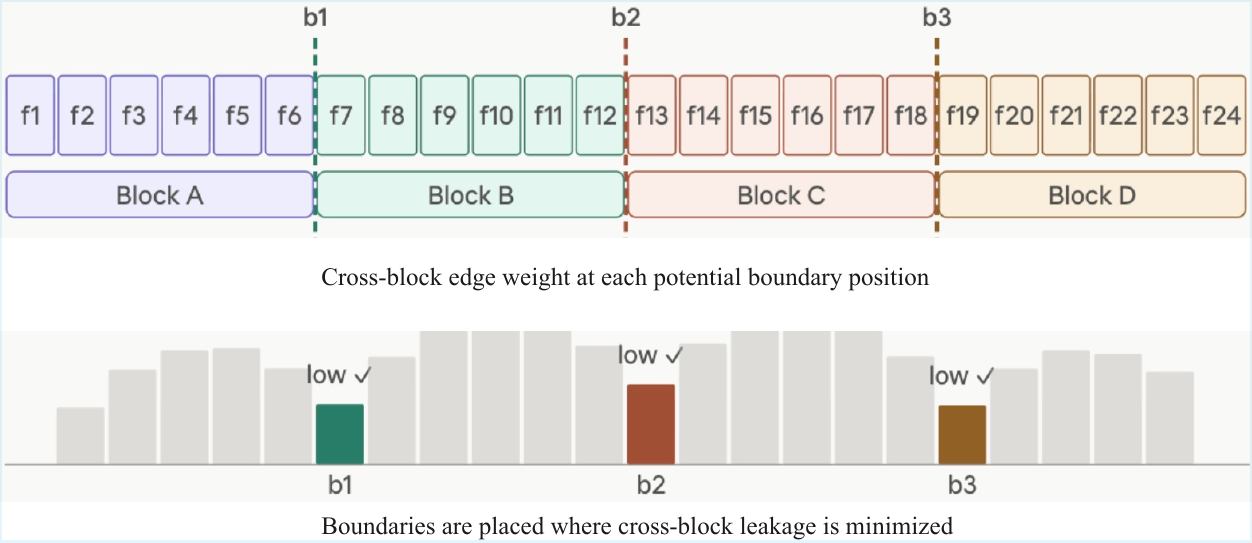}
    \caption{
\textbf{Contiguous block segmentation in GO-BS.} Given an ordered feature sequence, GO-BS selects \(M-1\) low-leakage boundaries to form \(M\) contiguous blocks, minimizing cross-block edge weight while preserving within-block affinity.
}
    \label{fig:gobs_block_cut}
\end{figure*}
\paragraph{Integration with BSTabDiff.}
Once GO-BS returns $(\hat{\pi},\hat{\mathbf{b}})$, we reorder the observed data into canonical ordered space by
\begin{equation}
\label{eq:obs_to_canonical}
X^\dagger = X_{\hat{\pi}^{-1}}
\end{equation}
The contiguous canonical blocks are then
\begin{equation}
\label{eq:canonical_blocks}
\mathcal{S}_t
=
\{\,\hat b_{t-1}+1,\dots,\hat b_t\,\},
\qquad t=1,\dots,M
\end{equation}
where $\hat b_0=0$ and $\hat b_M=m$. BSTabDiff is trained in this canonical ordered space using the block partition $\{\mathcal{S}_t\}_{t=1}^{M}$. At generation time, a synthetic canonical sample $\tilde{X}$ is mapped back to observed space by
\begin{equation}
\label{eq:canonical_to_obs}
X = \tilde{X}_{\hat{\pi}}
\end{equation}
This turns the permutation mechanism into an explicit learned preprocessing stage that aligns observed coordinates with the block-subunit generator.

\paragraph{Decision version and complexity.}
We now show that GO-BS is computationally intractable in the worst case. We reduce from the minimum bisection problem, a classical NP-hard graph partitioning problem~\citep{gar}. Minimum bisection and related balanced partitioning problems remain central hard problems in graph partitioning theory~\citep{gar,diaz}.

\begin{tcolorbox}[
    colback=regimeHDHSS,
    colframe=algborder,
    boxrule=0.5pt,
    arc=1mm,
    boxsep=2pt,
    left=3pt,
    right=3pt,
    top=3pt,
    bottom=3pt,
    before skip=5pt,
    after skip=5pt
]
\begin{definition}[Decision version of GO-BS]
\label{def:gobs_decision}
Given a weighted graph $G=(V,E,W)$, an integer $M \ge 1$, a regularizer coefficient $\gamma \ge 0$, a leakage coefficient $\lambda \ge 0$, and a threshold $K \in \mathbb{R}$, the decision version of GO-BS asks whether there exists a permutation $\pi$ and a valid boundary vector $\mathbf{b}$ such that Eq.~\ref{eq:de} holds:
\begin{equation}
\label{eq:de}
\mathcal{J}_{\mathrm{GO\mbox{-}BS}}(\pi,\mathbf{b}) \le K
\end{equation}
\end{definition}
\end{tcolorbox}

\begin{tcolorbox}[
    colback=fidelityW,
    colframe=theoremborder,
    boxrule=0.5pt,
    arc=1mm,
    boxsep=2pt,
    left=3pt,
    right=3pt,
    top=3pt,
    bottom=3pt,
    before skip=5pt,
    after skip=5pt
]
\begin{theorem}[NP-hardness of GO-BS]
\label{thm:gobs_nphard}
The GO-BS optimization problem in Eq.~\ref{eq:gobs_opt} is NP-hard. In fact, its decision version is NP-hard even in the restricted case of exactly two contiguous blocks of equal size.
\end{theorem}
\end{tcolorbox}

\begin{proof}
We reduce from the minimum bisection problem, which is NP-hard. An instance of minimum bisection consists of an undirected weighted graph $G=(V,E,W)$ with $|V|=m$ even, and a threshold $K_{\mathrm{bis}}$. The question is whether there exists a partition by Eq.~\ref{part}
\begin{equation}
    \label{part}
    V = S \cup \bar{S},
\qquad
S \cap \bar{S} = \emptyset,
\qquad
|S| = |\bar{S}| = m/2
\end{equation}
such that the cut weight
\begin{equation}
\label{eq:min_bis_cut}
\mathrm{Cut}(S,\bar{S})
=
\sum_{i \in S,\; j \in \bar{S}} W_{ij}
\end{equation}
is at most $K_{\mathrm{bis}}$. From this instance, construct a GO-BS instance on the same graph as follows Eq.~\ref{cut}.
\begin{equation}
    \label{cut}
    M = 2,
\qquad
\lambda = 1,
\qquad
\gamma = 0
\end{equation}
and enforce the two block sizes to be exactly $m/2$ each by restricting the admissible boundaries to the unique value in Eq.~\ref{unique}.
\begin{equation}
\label{unique}
    b_1 = m/2
\end{equation}
Equivalently, one may keep $\gamma > 0$ and choose a hard balance regularizer that is zero only when $|B_1| = |B_2| = m/2$ and $+\infty$ otherwise; both formulations are equivalent for the reduction. Under this restriction, any permutation $\pi$ induces exactly two contiguous blocks in Eq.~\ref{two}
\begin{equation}
    \label{two}
    B_1(\pi) = \{\, i \in V : \pi(i) \le m/2 \,\},
\qquad
B_2(\pi) = \{\, i \in V : \pi(i) > m/2 \,\}
\end{equation}
with $|B_1| = |B_2| = m/2$. Since $\gamma = 0$, the GO-BS objective reduces to
\begin{equation}
\label{eq:reduction_obj_1}
\mathcal{J}_{\mathrm{GO\mbox{-}BS}}(\pi)
=
-\sum_{\substack{i,j \in B_1(\pi)\\ i<j}} W_{ij}
-\sum_{\substack{i,j \in B_2(\pi)\\ i<j}} W_{ij}
+\sum_{\substack{i \in B_1(\pi)\\ j \in B_2(\pi)}} W_{ij}
\end{equation}
Using the decomposition of total weight in Eq.~\ref{weight},
\begin{equation}
    \label{weight}
    W_{\mathrm{tot}}
=
\sum_{\substack{i,j \in B_1(\pi)\\ i<j}} W_{ij}
+
\sum_{\substack{i,j \in B_2(\pi)\\ i<j}} W_{ij}
+
\sum_{\substack{i \in B_1(\pi)\\ j \in B_2(\pi)}} W_{ij}
\end{equation}
we obtain
\begin{equation}
\label{eq:reduction_obj_2}
\mathcal{J}_{\mathrm{GO\mbox{-}BS}}(\pi)
=
-\,W_{\mathrm{tot}}
+
2\sum_{\substack{i \in B_1(\pi)\\ j \in B_2(\pi)}} W_{ij}
\end{equation}
Therefore minimizing $\mathcal{J}_{\mathrm{GO\mbox{-}BS}}(\pi)$ is equivalent to minimizing the cut weight between the two induced equal-size blocks.

Now observe two directions.

(If) Suppose $(S,\bar{S})$ is a feasible minimum-bisection solution with $|S| = |\bar{S}| = m/2$ and $\mathrm{Cut}(S,\bar{S}) \le K_{\mathrm{bis}}$. Construct any permutation $\pi$ that lists all vertices of $S$ in the first $m/2$ positions and all vertices of $\bar{S}$ in the last $m/2$ positions. Then $B_1(\pi)=S$ and $B_2(\pi)=\bar{S}$, so by Eq.~\ref{eq:reduction_obj_2}, we get Eq.~\ref{obs}
\begin{equation}
    \label{obs}
    \mathcal{J}_{\mathrm{GO\mbox{-}BS}}(\pi)
=
-\,W_{\mathrm{tot}} + 2\,\mathrm{Cut}(S,\bar{S})
\le
-\,W_{\mathrm{tot}} + 2K_{\mathrm{bis}}
\end{equation}
(Only if) Conversely, suppose there exists a GO-BS permutation $\pi$ such that in Eq.~\ref{conv}
\begin{equation}
    \label{conv}
    \mathcal{J}_{\mathrm{GO\mbox{-}BS}}(\pi) \le -\,W_{\mathrm{tot}} + 2K_{\mathrm{bis}}
\end{equation}
By Eq.~\ref{eq:reduction_obj_2}, the induced equal-size blocks satisfy Eq.~\ref{sat}.
\begin{equation}
    \label{sat}
    \sum_{\substack{i \in B_1(\pi)\\ j \in B_2(\pi)}} W_{ij} \le K_{\mathrm{bis}}
\end{equation}
hence $(B_1(\pi),B_2(\pi))$ is a feasible minimum bisection. Thus the minimum bisection instance is a YES-instance if and only if the constructed GO-BS instance is a YES-instance. The reduction is polynomial time. Since minimum bisection is NP-hard, the restricted GO-BS decision problem is NP-hard. Therefore the full GO-BS optimization problem is NP-hard as well.
\end{proof}

\begin{remark}
The hardness is driven by the combinatorial difficulty of deciding which features should belong to the same contiguous ordered block. The reduction does not rely on any generic linear arrangement term; it applies directly to the block-consistent objective itself.
\end{remark}

\paragraph{A recovery-style sufficient condition.}
The NP-hardness result motivates approximate solvers. Still, the intended behavior of GO-BS can be formalized under a simple block-dominance condition. Proposition~\ref{prop:exact_recovery} formalizes the intended inductive bias behind GO-BS in the spirit of classical seriation and graph-partitioning methods~\citep{atkins, shi}.

\begin{tcolorbox}[
    colback=regimeMixed,
    colframe=assumpborder,
    boxrule=0.5pt,
    arc=1mm,
    boxsep=2pt,
    left=3pt,
    right=3pt,
    top=3pt,
    bottom=3pt,
    before skip=5pt,
    after skip=5pt
]
\begin{assumption}[Ordered block-dominance]
\label{ass:block_dominance}

There exists a ground-truth permutation $\pi^\dagger$ and a boundary vector $\mathbf{b}^\dagger$ inducing contiguous blocks in Eq.~\ref{dag}.

\begin{equation}
\label{dag}
B_t^\dagger :=
B_t(\pi^\dagger,\mathbf{b}^\dagger),
\qquad t=1,\dots,M
\end{equation}

such that for some constants $a_{\mathrm{in}},a_{\mathrm{out}}\geq 0$ with $a_{\mathrm{in}}>a_{\mathrm{out}}$, Eqs.~\ref{con} and~\ref{three} hold:

\begin{equation}
\label{con}
W_{ij}\geq a_{\mathrm{in}}
\quad
\text{whenever }
i,j\in B_t^\dagger
\text{ for some }t
\end{equation}

and

\begin{equation}
\label{three}
W_{ij}\leq a_{\mathrm{out}}
\quad
\text{whenever }
i\in B_t^\dagger,\;
j\in B_{t'}^\dagger,\;
t\neq t'
\end{equation}

\end{assumption}
\end{tcolorbox}
\begin{tcolorbox}[
    colback=regimeLDHSS,
    colframe=propborder,
    boxrule=0.5pt,
    arc=1mm,
    boxsep=2pt,
    left=3pt,
    right=3pt,
    top=3pt,
    bottom=3pt,
    before skip=5pt,
    after skip=5pt
]
\begin{proposition}[Exact recovery under strict block dominance]
\label{prop:exact_recovery}
Assume $\gamma = 0$ and suppose Assumption~\ref{ass:block_dominance} holds. If all admissible segmentations use the same block sizes as $\mathbf{b}^\dagger$, then every minimizer of Eq.~\ref{eq:gobs_opt} induces the same partition of $V$ as $\{B_t^\dagger\}_{t=1}^M$, up to permutation of block labels and arbitrary within-block reorderings.
\end{proposition}
\end{tcolorbox}

\begin{proof}
Because all admissible segmentations have the same block sizes, the total number of within-block vertex pairs is fixed across feasible solutions, and so is the total number of cross-block pairs. For any feasible partition $\mathcal{P} = \{B_t\}_{t=1}^M$, denote by Eq.~\ref{seg}.
\begin{equation}
    \label{seg}
    \mathrm{In}(\mathcal{P})
=
\sum_{t=1}^{M}\sum_{\substack{i,j \in B_t\\ i<j}} W_{ij},
\qquad
\mathrm{Out}(\mathcal{P})
=
\sum_{1 \le t < t' \le M}\sum_{\substack{i \in B_t\\ j \in B_{t'}}} W_{ij}
\end{equation}
Since $\gamma = 0$, minimizing Eq.~\ref{eq:gobs_main_obj} is equivalent to maximizing $\mathrm{In}(\mathcal{P})$, or equivalently minimizing $\mathrm{Out}(\mathcal{P})$.

Under Assumption~\ref{ass:block_dominance}, every pair placed correctly within its ground-truth block contributes at least $a_{\mathrm{in}}$, while every pair split across distinct ground-truth blocks contributes at most $a_{\mathrm{out}}$. Consider any feasible partition $\mathcal{P}$ that does not coincide with the ground-truth partition up to block relabeling. Then at least one pair of vertices that belong to the same ground-truth block is separated across different blocks of $\mathcal{P}$, or equivalently at least one pair from different ground-truth blocks is incorrectly grouped together. Because $a_{\mathrm{in}} > a_{\mathrm{out}}$, each such disagreement strictly decreases the within-block total relative to the ground-truth partition.

Therefore the ground-truth partition $\{B_t^\dagger\}_{t=1}^M$ uniquely maximizes $\mathrm{In}(\mathcal{P})$ among all feasible partitions, up to relabeling of the blocks. Since the objective depends only on block membership and not on the internal order of features inside each block, any within-block permutation yields the same optimum value. Hence every minimizer induces the same block partition up to block relabeling and within-block reorderings.
\end{proof}

\begin{remark}
Proposition~\ref{prop:exact_recovery} is intentionally simple: it formalizes the intended inductive bias of GO-BS, namely that if the affinity graph is approximately block-diagonal after some latent permutation, then the optimal segmentation recovers those blocks. This is directly consistent with BSTabDiff’s identifiability discussion, where blocks are recoverable up to permutation when they induce distinct dependence patterns.
\end{remark}

\paragraph{Sparse complexity.}
The exact optimization problem is NP-hard, so BSTabDiff uses a scalable sparse approximation. The key reason this is practical in HDLSS is that the sample dimension $n$ is small even when the feature dimension $m$ is large.

\begin{tcolorbox}[
    colback=regimeLDHSS,
    colframe=propborder,
    boxrule=0.5pt,
    arc=1mm,
    boxsep=2pt,
    left=3pt,
    right=3pt,
    top=3pt,
    bottom=3pt,
    before skip=5pt,
    after skip=5pt
]
\begin{proposition}[Sparse construction and refinement complexity]
\label{prop:sparse_complexity}

Assume each feature retains at most $k$ neighbors in the affinity graph, so $|E| = O(mk)$. Then:

\begin{enumerate}
\setlength{\itemsep}{2pt}
\setlength{\parsep}{0pt}
\setlength{\topsep}{3pt}

\item Constructing the sparse top-$k$ graph from feature embeddings of dimension $d$ costs $O(mdk)$ to $O(md\log m)$, depending on the neighbor-search routine;

\item Evaluating the block leakage term for fixed $(\pi,\mathbf{b})$ costs $O(mk)$;

\item One local boundary refinement sweep over all $M-1$ boundaries with a constant-size search window costs $O(mk)$;

\item Hence $P$ local refinement passes cost $O(Pmk)$ after graph construction.

\end{enumerate}
\end{proposition}
\end{tcolorbox}

\begin{proof}[Proof sketch]
Since each vertex has at most $k$ retained neighbors after sparsification, the graph contains $O(mk)$ stored edges. Any objective term in Eq.~\ref{eq:gobs_main_obj} can therefore be evaluated by scanning only these edges once and checking whether the endpoints lie in the same block or in different blocks under the current segmentation; this yields $O(mk)$ time. For local boundary refinement, if each boundary is moved only within a constant-size neighborhood, then each candidate move changes the status of only edges incident to vertices near that boundary, and the aggregated work over a full sweep remains linear in the number of stored edges, i.e., $O(mk)$. Repeating for $P$ passes gives $O(Pmk)$ after graph construction.
\end{proof}

\paragraph{Practical GO-BS solver.}
Exact optimization of Eq.~\ref{eq:gobs_opt} is infeasible for large $m$, so we adopt a sparse approximate solver. First, a sparse top-$k$ feature-affinity graph is built on GPU when available. Second, a coarse dependency-aware ordering is obtained from the feature-affinity graph. Third, the ordered sequence is segmented into $M$ contiguous blocks by minimizing the leakage-based boundary objective. Finally, local refinement updates both the ordering and the boundaries within short windows. Our practical solver combines sparse top-$k$ graph construction, a graph-guided ordering initialization, and local partition-refinement heuristics in the spirit of classical graph partitioning methods~\citep{dong,atkins,ker,fidu}.

\paragraph{GO-BS-FC variant.}
We also consider a feature-clustered variant, denoted GO-BS-FC, which preserves the same BSTabDiff interface and the same block-consistency goal as GO-BS, but changes the way the ordering problem is initialized. In standard GO-BS, samples are first clustered into regimes and each regime induces a feature-affinity graph over the original feature set, after which local feature orders are aggregated into a global permutation and segmented into contiguous blocks. In contrast, GO-BS-FC operates from the transposed view \(X^\top \in \mathbb{R}^{m \times n}\): features themselves are treated as the clustering units, so the method first partitions the \(m\) feature vectors into \(M\) feature clusters and then forms blocks directly from these feature groups. Mathematically, standard GO-BS estimates both a permutation \(\pi\) and boundaries \(\mathbf{b}\), inducing contiguous blocks \(B_t(\pi,\mathbf{b})\), whereas GO-BS-FC estimates a feature-cluster assignment \(c:V\to\{1,\dots,M\}\) and uses \(B_t^{\mathrm{FC}}=\{i\in V:c(i)=t\}\), ordering features within each cluster to obtain a canonical sequence. Thus both variants seek high within-block affinity and low cross-block leakage, but GO-BS enforces contiguity through a learned global ordering, while GO-BS-FC obtains the block partition primarily through feature clustering. This makes GO-BS-FC especially attractive in extreme dimensional settings, because it avoids constructing and refining a dense global ordering over all \(m\) features while still producing the block partition required by BSTabDiff.

\paragraph{Use inside BSTabDiff.}
After GO-BS returns $(\hat{\pi},\hat{\mathbf{b}})$, BSTabDiff applies Eqs.~\ref{eq:obs_to_canonical}-\ref{eq:canonical_to_obs}: the training data are reordered into canonical space, the induced contiguous segments become the latent blocks of the block-subunit generator, and synthetic samples are finally mapped back to the original observed coordinate indexing. This converts BSTabDiff’s previously permutation mechanism into a learned, block-aware preprocessing stage tailored to the model’s block-factorized generative structure.

\subsection{Assumption Diagnostics and Failure Modes}
\label{app:block_assumption_diagnostics}

\paragraph{Meaning of sparse cross-block dependence.}
Assumption~\ref{ass:distinct_blocks} introduces distinct within-block dependence structures and sparse cross-block interactions as structural conditions motivating the identifiability and sample complexity arguments of BSTabDiff. These conditions are not strict requirements for fitting the generative model. To clarify the meaning of sparsity, let $\mathcal{G}_h=(\{1,\ldots,M\},E_h)$ denote the conditional dependence graph of the $M$ block latents, where an edge $(t,t')\in E_h$ indicates that $h_t$ and $h_{t'}$ remain conditionally dependent given the other block latents and the optional label $Y$ as Eq.~\ref{eq:latent_dependency_graph}.
\begin{equation}
\label{eq:latent_dependency_graph}
(t,t')\in E_h
\iff
h_t\not\!\perp h_{t'}
\mid h_{-\{t,t'\}},Y
\end{equation}
Operationally, after discarding interactions below a small dependence threshold $\tau$, we consider cross-block dependence effectively sparse when the resulting graph $E_h^{(\tau)}$ contains $O(M)$ or $O(M\log M)$ meaningful edges rather than $O(M^2)$. This distinction does not imply mutual independence among blocks: each block may interact strongly with a limited subset of other blocks, and the joint prior $p_\theta(h\mid Y)$ models these dependencies. The thresholded graph provides a practical approximation to the underlying conditional dependence structure rather than an exact conditional independence guarantee.

\paragraph{Practical verification of structural assumptions.}
The presence of recoverable block structure is primarily a property of the training data distribution rather than of a particular downstream classifier. We identify three complementary diagnostics for assessing the plausibility of the assumptions. \textbf{(1) Block coherence:} evaluate whether the GO-BS partition exhibits high within-block feature affinity and low normalized cross-block leakage, indicating that strongly dependent features are concentrated within the recovered blocks. \textbf{(2) Latent conditional dependence density:} estimate a thresholded conditional dependence graph from the inferred block latents $h$, using partial correlations or conditional mutual information, and examine whether its edge density is consistent with sparse interactions. Partial correlation provides a tractable diagnostic of linear conditional dependence, while conditional mutual information can capture more general relationships when estimation is reliable. \textbf{(3) Bootstrap stability:} repeatedly resample the training data, recover the feature partitions, and assess their consistency using partition-similarity measures that are invariant to block-label permutations. High block coherence, limited cross-block leakage, sparse estimated latent interactions, and stable partitions provide practical evidence supporting the structural assumptions. However, these diagnostics are not formal tests of identifiability, and their reliability may be limited in extreme HDLSS settings.

\paragraph{Consequences of assumption violations.}
BSTabDiff does not fail automatically when cross-block dependencies are dense or the recovered feature blocks are weakly separated. The diffusion prior continues to learn a joint distribution $p_\theta(h\mid Y)$ and can represent dependencies among all block latents, including interactions that violate the effective sparsity assumption. Nevertheless, the conditions motivating Proposition~\ref{prop:identifiability} and the informal sample-complexity argument in Proposition~\ref{prop:sample_complexity} may no longer hold. In particular, block recovery can become less stable, and dependence fidelity may degrade when the available sample size is insufficient to estimate dense interactions among the block latents. Under such conditions, practitioners may reduce the cross-block leakage penalty, adjust the block budget $M$, disable GO-BS when ordering provides no validation benefit, or use $M=m$ to avoid block-level compression and approach a per-feature latent representation. These alternatives trade structural simplicity and computational efficiency against greater modeling flexibility.

\paragraph{Relationship to theoretical recovery conditions.}
The distinct dependence and sparsity assumptions motivate the identifiability and effective complexity arguments of BSTabDiff, but should not be interpreted as universally necessary conditions for successful generation. Appendix~\ref{app:block_subunit_necessity} provides a matched comparison with a dense latent-factor model, illustrating that block-subunit factorization can improve selected downstream utility and fidelity metrics while reducing decoder loading complexity. Separately, Assumption~\ref{ass:block_dominance} and Proposition~\ref{prop:exact_recovery} provide stronger sufficient conditions for idealized GO-BS partition recovery under strict within-block dominance. The diagnostics described above offer practical guidance for assessing whether comparable structural organization is plausible in real datasets, without claiming that such measurements alone establish the theoretical recovery conditions.

\vspace{-1em}
\renewcommand{\thesection}{A3}
\renewcommand{\thesubsection}{A3.\arabic{subsection}}
\setcounter{figure}{0}\renewcommand{\thefigure}{A3.\arabic{figure}}
\setcounter{table}{0}\renewcommand{\thetable}{A3.\arabic{table}}
\setcounter{algorithm}{0}\renewcommand{\thealgorithm}{A3.\arabic{algorithm}}
\section{BSTabDiff with GO-BS}
\label{bsgob}
\vspace{-1em}
\paragraph{Downstream utility with permutation.}
We evaluate BSTabDiff with feature ordering using GO-BS and GO-BS-FC, with results summarized in Table~\ref{tab:optional_perm_tstr}. Compared with the no-permutation setting in Table~\ref{tab:bstabdiff_mle_lr_hdlss}, both ordering variants remain consistently competitive across all HDLSS datasets under TSTR evaluation. Across datasets, GO-BS-FC shows small but consistent gains in AUC. For example, on COL the AUC improves from 88.78\% to 89.35\% (+0.57\%), on GLI from 89.79\% to 89.96\% (+0.17\%), on SMK from 70.53\% to 71.28\% (+0.75\%), on TOX from 89.76\% to 90.01\% (+0.25\%), and on ARC from 73.37\% to 77.14\% (+3.77\%), indicating improved discriminative structure in the generated data. The most pronounced improvement is observed on AML, where accuracy increases from 85.33\% to 88.38\% (+3.05\%) and AUC from 89.88\% to 94.69\% (+4.81\%). GO-BS remains competitive and occasionally slightly better in accuracy (e.g., PRS: 84.57\% vs 83.94\%, TOX: 89.35\% vs 89.11\%), but these differences are minor (typically $<1\%$). On high-performing datasets such as LNG, both methods achieve near-identical results (e.g., AUC $\approx$ 99.55\%-99.57\%), suggesting limited headroom for improvement when performance is already saturated. Importantly, TRTR values remain unchanged and serve as the real data upper-bound reference. Overall, these results indicate that feature ordering does not drastically alter performance but provides consistent incremental gains, with GO-BS-FC offering a slightly more favorable tradeoff in downstream utility.

\begin{table*}[t]
\centering
\scriptsize
\setlength{\tabcolsep}{4.2pt}
\renewcommand{\arraystretch}{1.18}
\caption{
Downstream utility of BSTabDiff with feature ordering
under strict \(5\times5\) TSTR evaluation using
Logistic Regression.
Results are mean$_{\pm\mathrm{std}}$ percentages;
TRTR is the real data reference on the same folds.
GO-BS denotes sample-clustered Graph-guided Ordering
with Block Segmentation, and GO-BS-FC denotes its
feature-clustered variant.
Gold indicates the higher TSTR mean within each
dataset; ties and TRTR references are unhighlighted.
}
\label{tab:optional_perm_tstr}

\begin{tabular}{
>{\columncolor{fidelityheader}}l
>{\columncolor{modelcol}}l
>{\columncolor{fidelityPearson}}c
>{\columncolor{fidelitySpearman}}c
>{\columncolor{fidelityKS}}c
>{\columncolor{fidelityW}}c
}
\toprule

\textbf{Data} &
\textbf{Ordering} &
\textbf{TSTR Acc. (\%) $\uparrow$} &
\textbf{TSTR AUC (\%) $\uparrow$} &
\textbf{TRTR Acc. (\%) $\uparrow$} &
\textbf{TRTR AUC (\%) $\uparrow$} \\

\midrule

\textbf{COL} & GO-BS
& $82.87_{\pm9.94}$
& $88.78_{\pm9.61}$
& $86.13_{\pm8.56}$
& $87.45_{\pm10.56}$ \\

\textbf{COL} & GO-BS-FC
& \cellcolor{goldcell}$\mathbf{83.21}_{\pm10.09}$
& \cellcolor{goldcell}$\mathbf{89.35}_{\pm9.17}$
& $86.13_{\pm8.56}$
& $87.45_{\pm10.56}$ \\

\midrule

\textbf{GLI} & GO-BS
& $84.94_{\pm7.64}$
& $89.79_{\pm8.76}$
& $89.88_{\pm6.96}$
& $93.42_{\pm7.64}$ \\

\textbf{GLI} & GO-BS-FC
& $84.94_{\pm7.64}$
& \cellcolor{goldcell}$\mathbf{89.96}_{\pm8.08}$
& $89.88_{\pm6.96}$
& $93.42_{\pm7.64}$ \\

\midrule

\textbf{LNG} & GO-BS
& $94.20_{\pm4.26}$
& \cellcolor{goldcell}$\mathbf{99.57}_{\pm0.40}$
& $96.54_{\pm2.43}$
& $99.56_{\pm0.44}$ \\

\textbf{LNG} & GO-BS-FC
& \cellcolor{goldcell}$\mathbf{94.58}_{\pm4.28}$
& $99.55_{\pm0.42}$
& $96.54_{\pm2.43}$
& $99.56_{\pm0.44}$ \\

\midrule

\textbf{PRS} & GO-BS
& \cellcolor{goldcell}$\mathbf{84.57}_{\pm9.78}$
& $81.57_{\pm9.50}$
& $91.76_{\pm5.34}$
& $94.91_{\pm4.40}$ \\

\textbf{PRS} & GO-BS-FC
& $83.94_{\pm8.46}$
& \cellcolor{goldcell}$\mathbf{81.70}_{\pm8.41}$
& $91.76_{\pm5.34}$
& $94.91_{\pm4.40}$ \\

\midrule

\textbf{SMK} & GO-BS
& $69.07_{\pm6.10}$
& $70.53_{\pm7.85}$
& $72.34_{\pm7.81}$
& $79.13_{\pm6.09}$ \\

\textbf{SMK} & GO-BS-FC
& \cellcolor{goldcell}$\mathbf{69.45}_{\pm5.65}$
& \cellcolor{goldcell}$\mathbf{71.28}_{\pm7.22}$
& $72.34_{\pm7.81}$
& $79.13_{\pm6.09}$ \\

\midrule

\textbf{TOX} & GO-BS
& \cellcolor{goldcell}$\mathbf{89.35}_{\pm6.61}$
& $89.76_{\pm3.39}$
& $94.38_{\pm4.22}$
& $99.50_{\pm0.96}$ \\

\textbf{TOX} & GO-BS-FC
& $89.11_{\pm7.32}$
& \cellcolor{goldcell}$\mathbf{90.01}_{\pm3.30}$
& $94.38_{\pm4.22}$
& $99.50_{\pm0.96}$ \\

\midrule

\textbf{AML} & GO-BS
& $85.33_{\pm7.82}$
& $89.88_{\pm9.77}$
& $98.04_{\pm3.15}$
& $99.56_{\pm0.87}$ \\

\textbf{AML} & GO-BS-FC
& \cellcolor{goldcell}$\mathbf{88.38}_{\pm7.68}$
& \cellcolor{goldcell}$\mathbf{94.69}_{\pm6.89}$
& $98.04_{\pm3.15}$
& $99.56_{\pm0.87}$ \\

\midrule

\textbf{ARC} & GO-BS
& \cellcolor{goldcell}$\mathbf{85.30}_{\pm5.97}$
& $73.37_{\pm7.46}$
& $86.70_{\pm6.19}$
& $91.91_{\pm4.68}$ \\

\textbf{ARC} & GO-BS-FC
& $84.70_{\pm6.96}$
& \cellcolor{goldcell}$\mathbf{77.14}_{\pm7.03}$
& $86.70_{\pm6.19}$
& $91.91_{\pm4.68}$ \\

\bottomrule
\end{tabular}
\end{table*}
\paragraph{Runtime and resource usage.}
Table~\ref{tab:optional_perm_compute} shows that ordering remains practical across all datasets, with GPU memory generally below 1\,GiB except for the largest feature spaces, GLI and SMK. GO-BS-FC is faster than GO-BS on most datasets, with especially large reductions on SMK (151.93s to 60.61s) and ARC (62.24s to 24.42s), while maintaining comparable GPU usage. GO-BS is only slightly faster on AML, and otherwise the feature-clustered variant is typically more efficient. Overall, GO-BS-FC provides the stronger runtime-structure tradeoff.

\begin{table*}[t]
\centering
\scriptsize
\setlength{\tabcolsep}{3.0pt}
\renewcommand{\arraystretch}{1.15}
\caption{
Runtime, resource usage, and energy footprint of
BSTabDiff with feature ordering.
GPU memory denotes peak CUDA usage, CPU memory
denotes end-of-run RSS, energy is in kWh, and
CO$_2$e is in kg.
GO-BS is the sample-clustered variant;
GO-BS-FC is the feature-clustered variant.
Gold indicates the lower value between the two
ordering variants for each dataset and metric.
}
\label{tab:optional_perm_compute}

\resizebox{\textwidth}{!}{%
\begin{tabular}{@{}
>{\columncolor{fidelityheader}}l
>{\columncolor{modelcol}}l
>{\columncolor{fidelityHOM}}c
>{\columncolor{fidelityKS}}c
>{\columncolor{fidelityW}}c
>{\columncolor{fidelityMI}}c
>{\columncolor{fidelityPrivacy}}c
@{\hspace{8pt}}
>{\columncolor{fidelityheader}}l
>{\columncolor{modelcol}}l
>{\columncolor{fidelityHOM}}c
>{\columncolor{fidelityKS}}c
>{\columncolor{fidelityW}}c
>{\columncolor{fidelityMI}}c
>{\columncolor{fidelityPrivacy}}c
@{}}
\toprule

\textbf{Data} &
\textbf{Ordering} &
\textbf{Time} &
\textbf{GPU} &
\textbf{CPU} &
\textbf{Energy} &
\textbf{CO$_2$e} &
\textbf{Data} &
\textbf{Ordering} &
\textbf{Time} &
\textbf{GPU} &
\textbf{CPU} &
\textbf{Energy} &
\textbf{CO$_2$e} \\

& &
\textbf{(s)} &
\textbf{(GiB)} &
\textbf{(GiB)} &
\textbf{(kWh)} &
\textbf{(kg)} &
& &
\textbf{(s)} &
\textbf{(GiB)} &
\textbf{(GiB)} &
\textbf{(kWh)} &
\textbf{(kg)} \\

\midrule


\textbf{COL} & GO-BS
& 11.30
& \cellcolor{goldcell}\textbf{0.069}
& \cellcolor{goldcell}\textbf{1.518}
& 5.77e-4
& 5.17e-4
& \textbf{SMK} & GO-BS
& 151.93
& \cellcolor{goldcell}\textbf{6.012}
& 2.264
& 7.19e-3
& 6.44e-3 \\

\textbf{COL} & GO-BS-FC
& \cellcolor{goldcell}\textbf{10.21}
& 0.087
& 1.825
& \cellcolor{goldcell}\textbf{5.40e-4}
& \cellcolor{goldcell}\textbf{4.84e-4}
& \textbf{SMK} & GO-BS-FC
& \cellcolor{goldcell}\textbf{60.61}
& 6.023
& \cellcolor{goldcell}\textbf{1.806}
& \cellcolor{goldcell}\textbf{3.73e-3}
& \cellcolor{goldcell}\textbf{3.34e-3} \\

\midrule


\textbf{GLI} & GO-BS
& 54.35
& 7.450
& 2.312
& \cellcolor{goldcell}\textbf{3.11e-3}
& \cellcolor{goldcell}\textbf{2.79e-3}
& \textbf{TOX} & GO-BS
& 31.98
& \cellcolor{goldcell}\textbf{0.523}
& 2.107
& 1.73e-3
& 1.55e-3 \\

\textbf{GLI} & GO-BS-FC
& \cellcolor{goldcell}\textbf{39.67}
& \cellcolor{goldcell}\textbf{7.421}
& \cellcolor{goldcell}\textbf{1.552}
& 6.50e-3
& 5.83e-3
& \textbf{TOX} & GO-BS-FC
& \cellcolor{goldcell}\textbf{27.50}
& 0.526
& \cellcolor{goldcell}\textbf{1.802}
& \cellcolor{goldcell}\textbf{8.13e-4}
& \cellcolor{goldcell}\textbf{7.28e-4} \\

\midrule


\textbf{LNG} & GO-BS
& 14.78
& \cellcolor{goldcell}\textbf{0.177}
& \cellcolor{goldcell}\textbf{1.505}
& 6.90e-4
& 6.18e-4
& \textbf{AML} & GO-BS
& \cellcolor{goldcell}\textbf{29.78}
& \cellcolor{goldcell}\textbf{0.770}
& \cellcolor{goldcell}\textbf{1.481}
& \cellcolor{goldcell}\textbf{1.77e-3}
& \cellcolor{goldcell}\textbf{1.58e-3} \\

\textbf{LNG} & GO-BS-FC
& \cellcolor{goldcell}\textbf{10.88}
& 0.195
& 1.809
& \cellcolor{goldcell}\textbf{5.71e-4}
& \cellcolor{goldcell}\textbf{5.11e-4}
& \textbf{AML} & GO-BS-FC
& 31.67
& 0.784
& 1.784
& 2.38e-3
& 2.13e-3 \\

\midrule


\textbf{PRS} & GO-BS
& 38.92
& \cellcolor{goldcell}\textbf{0.557}
& 2.032
& \cellcolor{goldcell}\textbf{2.18e-3}
& \cellcolor{goldcell}\textbf{1.95e-3}
& \textbf{ARC} & GO-BS
& 62.24
& \cellcolor{goldcell}\textbf{1.534}
& 2.172
& 2.71e-3
& 2.43e-3 \\

\textbf{PRS} & GO-BS-FC
& \cellcolor{goldcell}\textbf{36.43}
& 0.563
& \cellcolor{goldcell}\textbf{1.795}
& 5.29e-2
& 4.74e-2
& \textbf{ARC} & GO-BS-FC
& \cellcolor{goldcell}\textbf{24.42}
& 1.541
& \cellcolor{goldcell}\textbf{1.831}
& \cellcolor{goldcell}\textbf{2.03e-3}
& \cellcolor{goldcell}\textbf{1.82e-3} \\

\bottomrule
\end{tabular}%
}
\end{table*}
\paragraph{Fidelity diagnostics.}
Table~\ref{tab:optional_perm_fidelity} compares fidelity diagnostics for GO-BS and GO-BS-FC. Both variants produce numerically stable samples with zero NaN rates across all datasets. GO-BS more often achieves lower marginal and moment discrepancies, especially on TOX, AML, ARC, PRS, and LNG, while GO-BS-FC improves correlation preservation on several datasets, including COL, GLI, LNG, SMK, and TOX. GO-BS-FC also gives lower \(|\Delta\mathrm{MI}|\) on LNG, PRS, SMK, and ARC, suggesting better label-feature alignment in some cases. Overall, ordering improves dataset quality in a metric-dependent way: GO-BS tends to favor marginal/moment fidelity, whereas GO-BS-FC often improves correlation or label-related structure while remaining competitive in downstream utility.

\begin{table*}[t]
\centering
\scriptsize
\setlength{\tabcolsep}{3.2pt}
\renewcommand{\arraystretch}{1.15}
\caption{
Fidelity diagnostics for BSTabDiff with feature ordering.
Lower is better for KS, W1, correlation discrepancies,
HOM, and \(|\Delta \mathrm{MI}|\);
higher privacy distance suggests less direct copying.
Gold highlights the better result between GO-BS and
GO-BS-FC within each dataset; ties are unhighlighted.
KS = Kolmogorov-Smirnov, W1 = Wasserstein-1,
HOM = higher-order moment discrepancy,
MI = mutual information, and Priv. =
nearest neighbor distance from synthetic to real samples.
}
\label{tab:optional_perm_fidelity}

\begin{tabular}{
>{\columncolor{fidelityheader}}l
>{\columncolor{modelcol}}l
>{\columncolor{fidelityKS}}c
>{\columncolor{fidelityW}}c
>{\columncolor{fidelityPearson}}c
>{\columncolor{fidelitySpearman}}c
>{\columncolor{fidelityHOM}}c
>{\columncolor{fidelityMI}}c
>{\columncolor{fidelityNaN}}c
>{\columncolor{fidelityPrivacy}}c
}
\toprule

\textbf{Data} &
\textbf{Ordering} &
\textbf{KS mean} &
\textbf{W1 mean} &
\textbf{Pearson} &
\textbf{Spearman} &
\textbf{HOM mean} &
\textbf{\(|\Delta\mathrm{MI}|\)} &
\textbf{NaN} &
\textbf{Priv. mean} \\

\midrule

\textbf{COL} & GO-BS
& \cellcolor{goldcell}\textbf{0.1126}
& \cellcolor{goldcell}\textbf{0.1760}
& 0.2188
& 0.2177
& \cellcolor{goldcell}\textbf{1.42e-01}
& \cellcolor{goldcell}\textbf{0.2213}
& 0
& \cellcolor{goldcell}\textbf{5.01e1} \\

\textbf{COL} & GO-BS-FC
& 0.1168
& 0.1797
& \cellcolor{goldcell}\textbf{0.1690}
& \cellcolor{goldcell}\textbf{0.1684}
& 1.45e-01
& 0.2236
& 0
& 4.60e1 \\

\midrule

\textbf{GLI} & GO-BS
& \cellcolor{goldcell}\textbf{0.0903}
& \cellcolor{goldcell}\textbf{1.99e2}
& 0.2896
& 0.3136
& \cellcolor{goldcell}\textbf{1.49e2}
& \cellcolor{goldcell}\textbf{0.2116}
& 0
& \cellcolor{goldcell}\textbf{4.81e5} \\

\textbf{GLI} & GO-BS-FC
& 0.0907
& 2.00e2
& \cellcolor{goldcell}\textbf{0.2894}
& \cellcolor{goldcell}\textbf{0.3132}
& 1.50e2
& 0.2119
& 0
& 4.79e5 \\

\midrule

\textbf{LNG} & GO-BS
& \cellcolor{goldcell}\textbf{0.1987}
& \cellcolor{goldcell}\textbf{0.1108}
& 0.1701
& 0.1774
& \cellcolor{goldcell}\textbf{7.66e-02}
& 0.5723
& 0
& 1.73e1 \\

\textbf{LNG} & GO-BS-FC
& 0.2005
& 0.1124
& \cellcolor{goldcell}\textbf{0.1656}
& \cellcolor{goldcell}\textbf{0.1746}
& 7.80e-02
& \cellcolor{goldcell}\textbf{0.5676}
& 0
& \cellcolor{goldcell}\textbf{1.76e1} \\

\midrule

\textbf{PRS} & GO-BS
& \cellcolor{goldcell}\textbf{0.1688}
& \cellcolor{goldcell}\textbf{0.0290}
& \cellcolor{goldcell}\textbf{0.2817}
& \cellcolor{goldcell}\textbf{0.2324}
& \cellcolor{goldcell}\textbf{2.01e-02}
& 0.1505
& 0
& 2.16e1 \\

\textbf{PRS} & GO-BS-FC
& 0.1756
& 0.0326
& 0.2924
& 0.2508
& 2.43e-02
& \cellcolor{goldcell}\textbf{0.1385}
& 0
& \cellcolor{goldcell}\textbf{2.21e1} \\

\midrule

\textbf{SMK} & GO-BS
& 0.1579
& 0.1699
& 0.3254
& 0.3088
& 8.76e-02
& 0.1728
& 0
& \cellcolor{goldcell}\textbf{8.74e1} \\

\textbf{SMK} & GO-BS-FC
& \cellcolor{goldcell}\textbf{0.1532}
& \cellcolor{goldcell}\textbf{0.1531}
& \cellcolor{goldcell}\textbf{0.3117}
& \cellcolor{goldcell}\textbf{0.2972}
& \cellcolor{goldcell}\textbf{4.84e-02}
& \cellcolor{goldcell}\textbf{0.1653}
& 0
& 8.30e1 \\

\midrule

\textbf{TOX} & GO-BS
& \cellcolor{goldcell}\textbf{0.0927}
& \cellcolor{goldcell}\textbf{3.63e1}
& 0.1501
& 0.1493
& \cellcolor{goldcell}\textbf{1.41e1}
& \cellcolor{goldcell}\textbf{0.3031}
& 0
& \cellcolor{goldcell}\textbf{3.71e4} \\

\textbf{TOX} & GO-BS-FC
& 0.1008
& 5.52e1
& \cellcolor{goldcell}\textbf{0.1497}
& \cellcolor{goldcell}\textbf{0.1492}
& 1.85e1
& 0.3086
& 0
& 3.36e4 \\

\midrule

\textbf{AML} & GO-BS
& \cellcolor{goldcell}\textbf{0.0914}
& \cellcolor{goldcell}\textbf{0.1402}
& \cellcolor{goldcell}\textbf{0.1555}
& \cellcolor{goldcell}\textbf{0.1592}
& \cellcolor{goldcell}\textbf{1.02e-01}
& \cellcolor{goldcell}\textbf{0.2145}
& 0
& \cellcolor{goldcell}\textbf{9.35e1} \\

\textbf{AML} & GO-BS-FC
& 0.0955
& 0.1454
& 0.1563
& 0.1623
& 1.04e-01
& 0.2207
& 0
& 9.19e1 \\

\midrule

\textbf{ARC} & GO-BS
& \cellcolor{goldcell}\textbf{0.1931}
& \cellcolor{goldcell}\textbf{6.5563}
& \cellcolor{goldcell}\textbf{0.1659}
& \cellcolor{goldcell}\textbf{0.1578}
& \cellcolor{goldcell}\textbf{5.1824}
& 0.0858
& 0
& 9.39e3 \\

\textbf{ARC} & GO-BS-FC
& 0.2040
& 10.5546
& 0.1763
& 0.1642
& 8.2353
& \cellcolor{goldcell}\textbf{0.0798}
& 0
& \cellcolor{goldcell}\textbf{9.60e3} \\

\bottomrule
\end{tabular}
\end{table*}
\paragraph{Effect of ordering on dataset quality.}
Comparing Table~\ref{tab:optional_perm_fidelity} with the no-permutation baseline in Table~\ref{tab:fidelity-bstabdiff-noperm}, feature ordering improves several important aspects of synthetic data quality. GO-BS strengthens marginal fidelity on multiple datasets, reducing KS/W1 or moment discrepancies on TOX, AML, GLI, and LNG, while GO-BS-FC improves structural preservation by lowering Pearson/Spearman correlation gaps on COL, SMK, LNG, and GLI. Ordering also improves label-related structure in several cases: GO-BS-FC lowers \(|\Delta \mathrm{MI}|\) on LNG, PRS, SMK, and ARC, indicating better preservation of feature-label dependence. These gains occur while keeping NaN rates at zero and maintaining privacy distances comparable to the no-permutation setting, suggesting that ordering improves dataset quality without increasing numerical instability or sample-copying risk. Figure~\ref{fig:optional_ordering_effects} summarizes the effect of feature ordering from complementary perspectives. GO-BS-FC improves TSTR AUC over GO-BS on most datasets, with especially visible gains on AML and ARC. GO-BS improves selected fidelity diagnostics relative to the no-permutation baseline, while GO-BS-FC reduces Pearson and Spearman correlation discrepancies on COL, indicating improved structural preservation. The metric-count summary further shows that ordering improves multiple fidelity criteria across datasets, supporting that feature ordering provides useful structural bias rather than only changing downstream scores.
\begin{figure*}[t]
    \centering
    \includegraphics[width=0.48\textwidth]{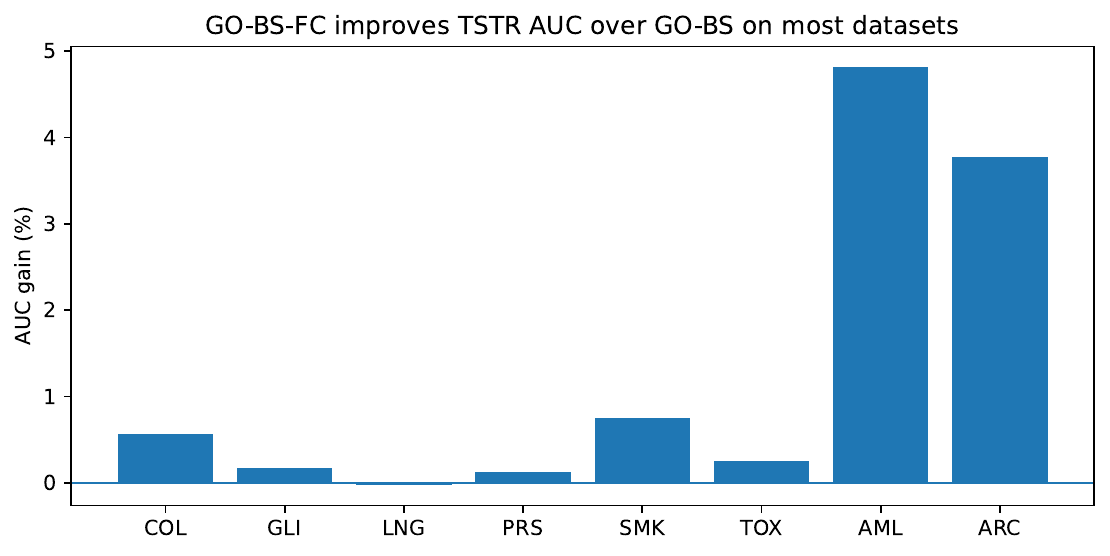}
    \includegraphics[width=0.48\textwidth]{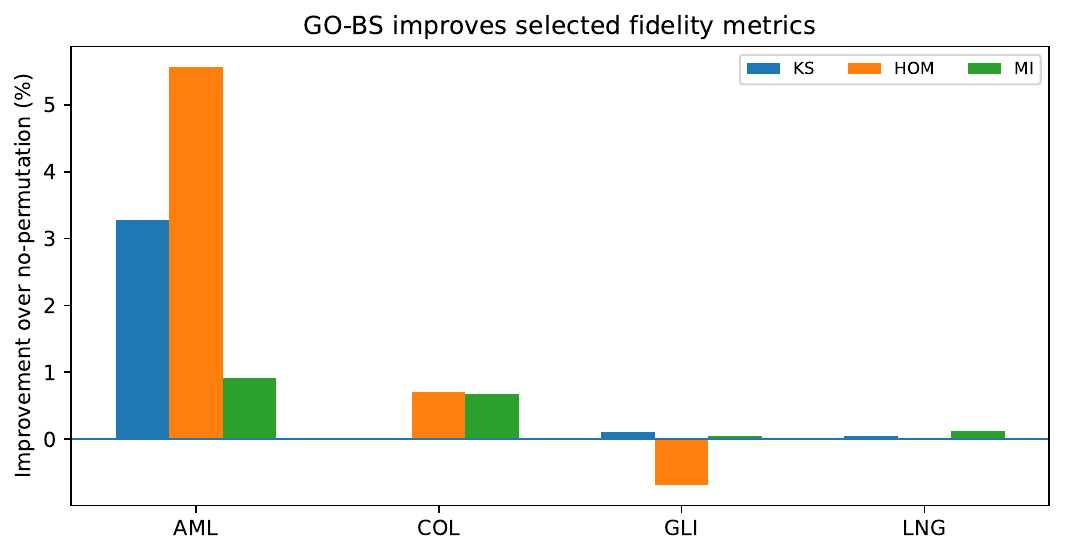}
    
    \vspace{0.5em}
    
    \includegraphics[width=0.48\textwidth]{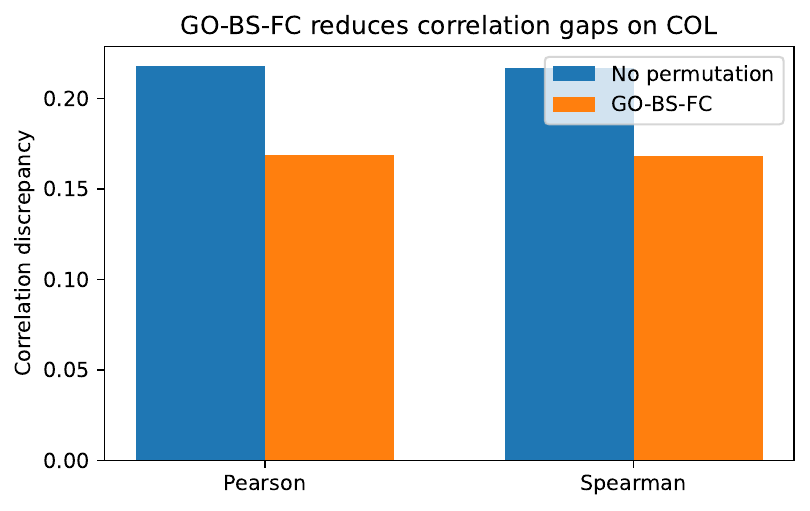}
    \includegraphics[width=0.48\textwidth]{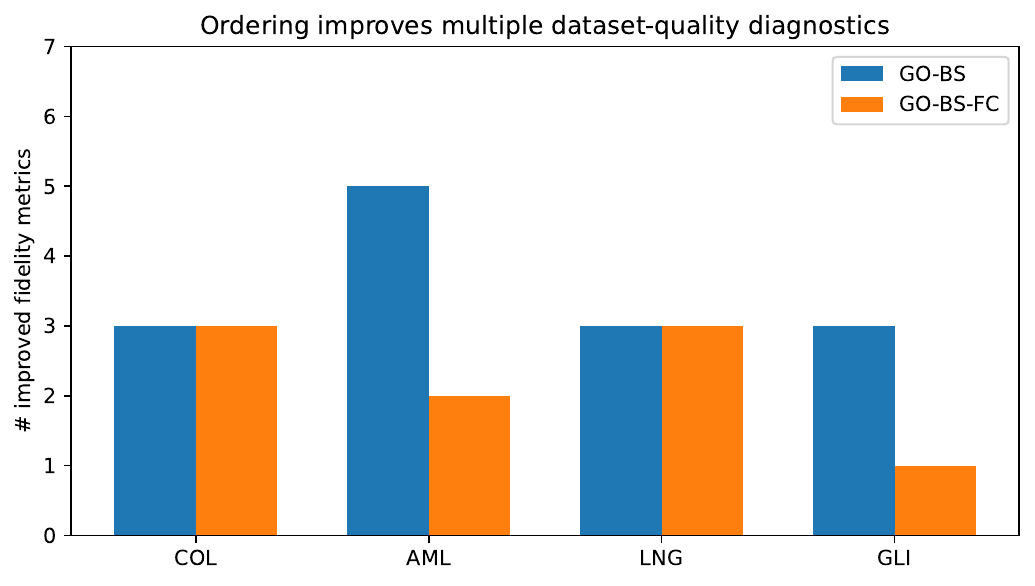}
    
    \caption{
\textbf{Effect of feature ordering.}
GO-BS-FC generally improves TSTR AUC over GO-BS, while GO-BS improves selected fidelity diagnostics over the no-permutation baseline. 
Ordering also reduces COL correlation discrepancies and improves multiple fidelity metrics across datasets, indicating better structural preservation.
}
    \label{fig:optional_ordering_effects}
\end{figure*}
\vspace{-1em}
\renewcommand{\thesection}{A4}
\renewcommand{\thesubsection}{A4.\arabic{subsection}}
\setcounter{figure}{0}\renewcommand{\thefigure}{A4.\arabic{figure}}
\setcounter{table}{0}\renewcommand{\thetable}{A4.\arabic{table}}
\setcounter{algorithm}{0}\renewcommand{\thealgorithm}{A4.\arabic{algorithm}}
\section{Ablation on Sample Clustering vs.\ Feature Clustering in GO-BS}
\label{app:gobs_ablation}
\vspace{-1em}
We study a design choice in GO-BS: whether the graph-guided ordering pipeline should begin by clustering samples or by clustering features. Our main method uses sample clustering, from which we construct cluster-specific feature-affinity graphs, derive local feature orderings, and then aggregate them into a global feature order before contiguous block segmentation. Here, we compare this design against a feature-clustering alternative and against a simpler no-clustering baseline.
\paragraph{Motivation.}
The goal of GO-BS is to estimate a data-driven permutation under which strongly dependent features become contiguous and can be segmented into coherent blocks for BSTabDiff. This is closely related to seriation, i.e., finding a suitable linear order for a set of objects~\citep{getting}, such that more similar items are placed closer together in the recovered order~\citep{spectral}. A natural question is whether these blocks should be inferred from:
\begin{enumerate}
    \item \textbf{sample-local feature structure}, obtained by first partitioning samples into local regimes and then constructing a feature graph within each regime; or
    \item \textbf{global feature structure}, obtained by directly clustering features in the transposed data matrix.
\end{enumerate}
GO-BS adopts the first strategy. The intuition is that in HDLSS data, especially in biological and biomedical settings, feature-dependence patterns may vary across subpopulations, phenotypes, or local sample regimes. This is consistent with prior work on heterogeneous biomedical tabular data, which emphasizes ``the subset of most relevant features for each sample''~\citep{spin}, as well as work on sample-specific regulatory networks and subtype-specific co-expression modules~\citep{kui,sun}. A single global feature view may therefore blur regime-specific structure. Related work has explicitly contrasted sample-specific networks with aggregate ones~\citep{kui}, and has shown that different cancer subtypes can exhibit distinct co-expression modules rather than one universal global module~\citep{sun}. By contrast, sample clustering allows GO-BS to recover multiple local feature graphs, each capturing dependence patterns specific to a subset of samples. This design is motivated by prior evidence that network structure can be sample-specific~\citep{kui} or subtype-specific~\citep{sun}, rather than being adequately summarized by a single global dependency graph.
\paragraph{Compared Variants.} We compare the following variants.

\textbf{1. GO-BS (sample-clustered; main method).}
Given training data $X\in\mathbb{R}^{n\times m}$, we first partition samples into $C$ clusters. For each sample cluster $c$, we construct a feature-affinity graph
\begin{equation}
G_c=(V,E,W^{(c)}), \qquad V=\{1,\dots,m\}
\label{eq:gobs_ablation_graph}
\end{equation}
we derive a local feature ordering $\pi_c$, and integrate the local orderings into a global order $\Pi^\ast$. The globally ordered feature sequence is then segmented into $M$ contiguous blocks using the GO-BS objective in Eq.~\ref{eq:gobs_main_obj}.

\textbf{2. GO-BS-FC (feature-clustered ablation).}
As an alternative, we directly cluster feature vectors by transposing the data matrix and treating each feature as an object in $\mathbb{R}^{n}$:
\begin{equation}
X^\top \in \mathbb{R}^{m\times n}
\label{eq:gobs_ablation_xt}
\end{equation}
We partition the $m$ features into $C_f$ clusters in this transposed space, order the feature clusters, and then order features within each cluster using the same graph metric as in the main method. The resulting concatenated order is then segmented into $M$ contiguous blocks.

\textbf{3. No clustering.}
As a simpler baseline, we also consider building a single global feature-affinity graph from the full dataset without any sample or feature clustering, followed by one global ordering and contiguous segmentation.

\paragraph{Why Sample Clustering is the Main Design.} We choose sample clustering as the main GO-BS design for four reasons.

\textbf{(1) Local dependence regimes.}
Sample clustering allows GO-BS to capture regime-specific feature structure. If two features are strongly associated only within a subset of samples, a global feature-clustering approach may weaken that relationship by averaging over the entire dataset. In contrast, a cluster-specific feature graph can preserve such local structure.

\textbf{(2) Local-to-global ordering.}
The GO-BS pipeline is designed as a local-to-global ordering mechanism: local feature orderings are derived within sample clusters and then integrated into a single global order. This aggregation step is naturally related to rank aggregation, where one seeks a ``globally consistent solution'' from potentially inconsistent partial order information~\citep{ailon}. In our view, this is better aligned with graph-guided ordering than direct feature clustering, because seriation methods are explicitly concerned with recovering a linear order from similarity information~\citep{getting,spectral}, whereas feature clustering primarily yields discrete groups.

\textbf{(3) Better compatibility with contiguous segmentation.}
The final step of GO-BS is to partition the global ordered sequence into contiguous blocks for BSTabDiff. Sample-clustered local orderings naturally produce a global rank aggregation problem, which then feeds cleanly into contiguous segmentation. Direct feature clustering tends instead to produce a discrete grouping first, making the subsequent ordering step less central.

\textbf{(4) Consistency with the intended role of GO-BS.}
GO-BS is not intended to replace BSTabDiff with a feature-clustering model. Rather, it serves as a permutation-estimation module that recovers a canonical feature order before fitting the original block-subunit generator. The sample-clustered design better supports this interpretation.

\paragraph{Experimental Setup.} For this ablation, we keep all BSTabDiff components fixed and vary only the front-end ordering strategy. In particular, the following are held constant across variants:
\begin{itemize}
    \item the number of latent blocks $M$,
    \item the prior type,
    \item the emission model,
    \item the latent inference procedure, and
    \item the training and sampling protocol.
\end{itemize}
Thus, any performance difference is attributable to the ordering mechanism rather than to changes in the generative model itself. We evaluate the three variants on the Colon dataset using the same downstream protocol as in the main paper. Specifically, we report:
\begin{itemize}
    \item \textbf{synthetic utility} via strict fold-wise 5$\times$5 CV TSTR/TRTR evaluation,
    \item \textbf{distributional fidelity} via per-feature KS/Wasserstein statistics and correlation-structure recovery,
    \item \textbf{label-information preservation} via mutual-information discrepancy, and
    \item \textbf{runtime and memory usage}.
\end{itemize}
All variants use the same BSTabDiff hyperparameters; only the ordering front-end is changed.

\paragraph{Ablation Results.} Table~\ref{tab:gobs_ablation_utility} reports synthetic utility across four downstream classifiers. The three variants are broadly competitive on Colon, with only modest differences overall. Nevertheless, several trends are visible.

- First, the proposed sample-clustered GO-BS attains the best average TSTR accuracy across the four classifiers ($0.8556$), slightly ahead of feature clustering ($0.8538$) and no clustering ($0.8533$). It also gives the strongest TSTR accuracy for CatBoost and Logistic Regression, while matching the best TabPFN-2.5 result. This indicates that the sample-clustered ordering is a robust default when the goal is downstream synthetic utility across diverse classifier families.

- Second, GO-BS-FC is highly competitive and in some cases yields the best TSTR AUC, especially with TabPFN-2.5. However, its gains are not uniform across classifiers, and its average TSTR accuracy remains marginally below the sample-clustered variant.

- Third, the no-clustering baseline is surprisingly strong in several cases. This suggests that the Colon dataset already contains a fairly coherent global dependence structure, so even a single global affinity graph can recover a useful ordering. As a result, the benefit of sample clustering on this dataset is real but necessarily modest.

- TRTR scores are nearly identical across all three variants, as expected, since the real data classifier pipeline is unchanged and only the synthetic generator front-end differs. Hence, the main signal in this ablation comes from TSTR and fidelity rather than TRTR.

\begin{table*}[t]
\centering
\small
\setlength{\tabcolsep}{5pt}
\renewcommand{\arraystretch}{1.18}
\caption{
Synthetic utility of GO-BS ordering variants on Colon
under strict fold-wise \(5\times5\) CV.
We report TSTR/TRTR mean\(\pm\)std across four
downstream classifiers.
Sample-clustered GO-BS (\(+\)) gives the best
average TSTR accuracy, while all variants remain
competitive; \(*\) denotes feature-clustered GO-BS-FC.
Gold highlights the best TSTR values for each
classifier and the best aggregate TSTR results.
}
\label{tab:gobs_ablation_utility}

\begin{tabular}{llcccc}
\toprule

\cellcolor{fidelityheader}\textbf{Variant} &
\cellcolor{fidelityheader}\textbf{Classifier} &
\cellcolor{fidelityPearson}\textbf{TSTR Acc. $\uparrow$} &
\cellcolor{fidelitySpearman}\textbf{TSTR AUC $\uparrow$} &
\cellcolor{fidelityKS}\textbf{TRTR Acc. $\uparrow$} &
\cellcolor{fidelityW}\textbf{TRTR AUC $\uparrow$} \\

\midrule


\multirow{4}{*}{\textbf{GO-BS$+$}}
& \cellcolor{modelcol}TabPFN-2.5
& \cellcolor{fidelityPearson}0.8715$_{\pm 0.0991}$
& \cellcolor{fidelitySpearman}0.8988$_{\pm 0.1047}$
& \cellcolor{fidelityKS}0.8464$_{\pm 0.0951}$
& \cellcolor{fidelityW}0.8825$_{\pm 0.1075}$ \\

& \cellcolor{modelcol}TANDEM
& \cellcolor{fidelityPearson}0.8431$_{\pm 0.1118}$
& \cellcolor{fidelitySpearman}0.8938$_{\pm 0.0978}$
& \cellcolor{fidelityKS}0.8333$_{\pm 0.0857}$
& \cellcolor{fidelityW}0.8888$_{\pm 0.0913}$ \\

& \cellcolor{modelcol}CatBoost
& \cellcolor{goldcell}\textbf{0.8551}$_{\pm 0.1098}$
& \cellcolor{fidelitySpearman}0.8770$_{\pm 0.1073}$
& \cellcolor{fidelityKS}0.8246$_{\pm 0.1017}$
& \cellcolor{fidelityW}\textbf{0.8920}$_{\pm 0.1094}$ \\

& \cellcolor{modelcol}Log. Reg.
& \cellcolor{goldcell}\textbf{0.8526}$_{\pm 0.1074}$
& \cellcolor{fidelitySpearman}0.8928$_{\pm 0.0999}$
& \cellcolor{fidelityKS}\textbf{0.8431}$_{\pm 0.1036}$
& \cellcolor{fidelityW}0.8703$_{\pm 0.1092}$ \\

\midrule


\multirow{4}{*}{\textbf{GO-BS-FC$*$}}
& \cellcolor{modelcol}TabPFN-2.5
& \cellcolor{fidelityPearson}0.8715$_{\pm 0.0991}$
& \cellcolor{goldcell}\textbf{0.9032}$_{\pm 0.1008}$
& \cellcolor{fidelityKS}0.8464$_{\pm 0.0951}$
& \cellcolor{fidelityW}0.8825$_{\pm 0.1075}$ \\

& \cellcolor{modelcol}TANDEM
& \cellcolor{goldcell}\textbf{0.8459}$_{\pm 0.1281}$
& \cellcolor{fidelitySpearman}0.8892$_{\pm 0.1084}$
& \cellcolor{fidelityKS}0.8333$_{\pm 0.0857}$
& \cellcolor{fidelityW}0.8888$_{\pm 0.0913}$ \\

& \cellcolor{modelcol}CatBoost
& \cellcolor{fidelityPearson}0.8487$_{\pm 0.1089}$
& \cellcolor{goldcell}\textbf{0.8895}$_{\pm 0.0988}$
& \cellcolor{fidelityKS}0.8246$_{\pm 0.1017}$
& \cellcolor{fidelityW}\textbf{0.8920}$_{\pm 0.1094}$ \\

& \cellcolor{modelcol}Log. Reg.
& \cellcolor{fidelityPearson}0.8492$_{\pm 0.1166}$
& \cellcolor{goldcell}\textbf{0.8978}$_{\pm 0.1006}$
& \cellcolor{fidelityKS}\textbf{0.8431}$_{\pm 0.1036}$
& \cellcolor{fidelityW}0.8703$_{\pm 0.1092}$ \\

\midrule


\multirow{4}{*}{\textbf{No clustering}}
& \cellcolor{modelcol}TabPFN-2.5
& \cellcolor{fidelityPearson}0.8715$_{\pm 0.0991}$
& \cellcolor{fidelitySpearman}0.9017$_{\pm 0.1077}$
& \cellcolor{fidelityKS}0.8464$_{\pm 0.0951}$
& \cellcolor{fidelityW}0.8825$_{\pm 0.1075}$ \\

& \cellcolor{modelcol}TANDEM
& \cellcolor{fidelityPearson}0.8431$_{\pm 0.1190}$
& \cellcolor{fidelitySpearman}0.8915$_{\pm 0.1079}$
& \cellcolor{fidelityKS}0.8333$_{\pm 0.0857}$
& \cellcolor{fidelityW}0.8888$_{\pm 0.0913}$ \\

& \cellcolor{modelcol}CatBoost
& \cellcolor{fidelityPearson}0.8526$_{\pm 0.0994}$
& \cellcolor{fidelitySpearman}0.8805$_{\pm 0.1079}$
& \cellcolor{fidelityKS}0.8246$_{\pm 0.1017}$
& \cellcolor{fidelityW}\textbf{0.8920}$_{\pm 0.1094}$ \\

& \cellcolor{modelcol}Log. Reg.
& \cellcolor{fidelityPearson}0.8459$_{\pm 0.1109}$
& \cellcolor{fidelitySpearman}0.8938$_{\pm 0.0967}$
& \cellcolor{fidelityKS}\textbf{0.8431}$_{\pm 0.1036}$
& \cellcolor{fidelityW}0.8703$_{\pm 0.1092}$ \\

\midrule


\rowcolor{fidelityheader}
\multicolumn{2}{l}{\textbf{Avg.: GO-BS}}
& \cellcolor{goldcell}\textbf{0.8556}
& 0.8906
& 0.8369
& 0.8834 \\

\rowcolor{fidelityheader}
\multicolumn{2}{l}{\textbf{Avg.: GO-BS-FC}}
& 0.8538
& \cellcolor{goldcell}\textbf{0.8949}
& 0.8369
& 0.8834 \\

\rowcolor{fidelityheader}
\multicolumn{2}{l}{\textbf{Avg.: No clustering}}
& 0.8533
& 0.8919
& 0.8369
& 0.8834 \\

\bottomrule
\end{tabular}
\end{table*}
Table~\ref{tab:gobs_ablation_fidelity_runtime} reports fidelity and resource usage. The picture here is more nuanced. The sample-clustered variant yields the best mean KS discrepancy and remains competitive on Wasserstein distance and mutual-information preservation. The feature-clustered variant achieves the best Pearson/Spearman correlation recovery and is also the fastest variant, reducing training time by more than $2\times$ relative to the sample-clustered pipeline. Meanwhile, no clustering gives the best KS max, W1 max, and marginally best mutual-information discrepancy. Taken together, these results suggest that no single variant dominates every criterion on Colon. Instead, the proposed sample-clustered GO-BS offers the best overall balance for our intended use: it gives the strongest average downstream synthetic utility while remaining competitive on fidelity, at the cost of additional front-end computation. Feature clustering is an efficient alternative when runtime is a stronger concern, whereas no clustering serves as a strong but less expressive baseline.
\begin{table*}[t]
\centering
\small
\setlength{\tabcolsep}{5pt}
\renewcommand{\arraystretch}{1.18}
\caption{
Fidelity and resource comparison for GO-BS ordering variants on Colon.
Lower is better for fidelity discrepancies and runtime;
higher NN distance suggests less direct copying.
Feature clustering is fastest and best for correlation recovery,
while sample clustering gives the strongest synthetic utility
in Table~\ref{tab:gobs_ablation_utility}.
Here, \(+\) denotes sample-clustered GO-BS and \(*\)
denotes feature-clustered GO-BS-FC.
Here, KS = Kolmogorov-Smirnov,
W1 = Wasserstein-1, MI = mutual information,
NN = nearest neighbor, and RSS = CPU resident set size.
}
\label{tab:gobs_ablation_fidelity_runtime}
\vspace{0.25em}

\begin{tabular}{lcccccc}
\toprule

\cellcolor{fidelityheader}\textbf{Variant} &
\cellcolor{fidelityKS}\textbf{KS mean $\downarrow$} &
\cellcolor{fidelityKS}\textbf{KS max $\downarrow$} &
\cellcolor{fidelityW}\textbf{W1 mean $\downarrow$} &
\cellcolor{fidelityW}\textbf{W1 max $\downarrow$} &
\cellcolor{fidelityPearson}\textbf{Pearson $\downarrow$} &
\cellcolor{fidelitySpearman}\textbf{Spearman $\downarrow$} \\

\midrule

\cellcolor{modelcol}\textbf{GO-BS$+$}
& \cellcolor{goldcell}\textbf{0.0931}
& \cellcolor{fidelityKS}0.2960
& \cellcolor{goldcell}\textbf{0.1375}
& \cellcolor{fidelityW}0.4842
& \cellcolor{fidelityPearson}0.2178
& \cellcolor{fidelitySpearman}0.2174 \\

\cellcolor{modelcol}\textbf{GO-BS-FC$*$}
& \cellcolor{fidelityKS}0.0978
& \cellcolor{fidelityKS}0.2960
& \cellcolor{fidelityW}0.1495
& \cellcolor{fidelityW}0.5216
& \cellcolor{goldcell}\textbf{0.1807}
& \cellcolor{goldcell}\textbf{0.1800} \\

\cellcolor{modelcol}\textbf{No clustering}
& \cellcolor{fidelityKS}0.0943
& \cellcolor{goldcell}\textbf{0.2848}
& \cellcolor{fidelityW}0.1397
& \cellcolor{goldcell}\textbf{0.4484}
& \cellcolor{fidelityPearson}0.2137
& \cellcolor{fidelitySpearman}0.2136 \\

\bottomrule
\end{tabular}

\vspace{0.5em}

\begin{tabular}{lccccc}
\toprule

\cellcolor{fidelityheader}\textbf{Variant} &
\cellcolor{fidelityMI}\textbf{MI diff $\downarrow$} &
\cellcolor{fidelityPrivacy}\textbf{NN dist $\uparrow$} &
\cellcolor{fidelityHOM}\textbf{Time (s) $\downarrow$} &
\cellcolor{fidelityNaN}\textbf{CPU RSS (GiB) $\downarrow$} &
\cellcolor{fidelityKS}\textbf{GPU (GiB) $\downarrow$} \\

\midrule

\cellcolor{modelcol}\textbf{GO-BS$+$}
& \cellcolor{fidelityMI}0.2252
& \cellcolor{goldcell}\textbf{49.9973}
& \cellcolor{fidelityHOM}27.30
& \cellcolor{goldcell}\textbf{1.6231}
& \cellcolor{fidelityKS}0.0357 \\

\cellcolor{modelcol}\textbf{GO-BS-FC$*$}
& \cellcolor{fidelityMI}0.2282
& \cellcolor{fidelityPrivacy}45.7295
& \cellcolor{goldcell}\textbf{11.97}
& \cellcolor{fidelityNaN}1.6279
& \cellcolor{fidelityKS}0.0357 \\

\cellcolor{modelcol}\textbf{Global ordering}
& \cellcolor{goldcell}\textbf{0.2251}
& \cellcolor{fidelityPrivacy}49.5847
& \cellcolor{fidelityHOM}15.03
& \cellcolor{fidelityNaN}1.6329
& \cellcolor{fidelityKS}0.0357 \\

\bottomrule
\end{tabular}
\end{table*}
\paragraph{Discussion.} The ablation supports three conclusions.

\textbf{(1) Sample clustering is a strong default for GO-BS.}
Although the margins are small on Colon, the proposed sample-clustered variant gives the best average TSTR accuracy across four downstream classifiers. This is the most relevant criterion for our use case, since GO-BS is intended as a front-end permutation estimator for improving the practical utility of BSTabDiff-generated samples.

\textbf{(2) Feature clustering is efficient but slightly less robust as a default.}
GO-BS-FC is substantially faster and recovers global correlation structure slightly better, but its downstream gains are less consistent across classifiers. This is compatible with the interpretation that feature clustering captures a strong global organization, yet is less expressive than sample-conditioned local graphs when local regimes matter.

\textbf{(3) Colon appears relatively homogeneous.}
The no-clustering baseline remains surprisingly competitive. This indicates that Colon may already exhibit a coherent global dependence structure, limiting the headroom available to more sophisticated local-to-global ordering strategies. We therefore view this ablation as conservative: on datasets with stronger sample heterogeneity, the relative advantage of sample clustering is expected to be larger.

Overall, the empirical evidence is consistent with our design choice. Even on a dataset where the simpler alternatives are already strong, the proposed sample-clustered GO-BS remains the best-performing overall configuration for downstream synthetic utility, while preserving a competitive fidelity profile.

\paragraph{Practical Note on Implementation.} We use GPU acceleration for sample clustering, while feature-graph construction and local refinement are performed on CPU. This design reflects the HDLSS regime: sample clustering benefits from parallel distance computation, whereas dense feature-graph materialization scales quadratically in the number of features and local refinement involves irregular indexing and sequential updates, making a CPU implementation more memory-stable and practical. In our implementation, all three variants used negligible GPU memory for the ordering front-end, and the main computational difference came from the additional clustering and aggregation steps in the sample-clustered pipeline.
\paragraph{Summary.}
Overall, this ablation isolates an important design choice in GO-BS. The main conclusion is that GO-BS is most naturally formulated as
\[
\small
\begin{aligned}
\text{sample clustering}
&\to \text{cluster-specific feature graphs}
\to \text{local feature orderings} \\
&\to \text{global order aggregation}
\to \text{contiguous block segmentation}
\end{aligned}
\]
whereas direct feature clustering serves as a simpler and faster, but somewhat less expressive, alternative. On Colon, all three variants are competitive, but the sample-clustered design provides the strongest overall downstream synthetic utility and therefore remains the main design used in BSTabDiff.
\vspace{-1em}
\renewcommand{\thesection}{A5}
\renewcommand{\thesubsection}{A5.\arabic{subsection}}
\setcounter{figure}{0}\renewcommand{\thefigure}{A5.\arabic{figure}}
\setcounter{table}{0}\renewcommand{\thetable}{A5.\arabic{table}}
\setcounter{algorithm}{0}\renewcommand{\thealgorithm}{A5.\arabic{algorithm}}
\section{GO-BS Ordering Ablation against Classical Metaheuristics}
\label{app:gobsmeta}
\vspace{-1em}
We compare the GO-BS family against classical metaheuristic feature ordering modules inside BSTabDiff on Colon. The compared methods are: (i) sample-clustered GO-BS, which derives ordering from sample-cluster-specific feature affinities; (ii) feature-clustered GO-BS-FC; and global explicit orderings from (iii) greedy affinity traversal, (iv) Christofides~\citep{t6}, (v) Simulated Annealing TSP~\citep{t3}, (vi) Threshold Accepting TSP~\citep{t7}, (vii) Genetic Algorithm~\citep{t4}, and (viii) Ant Colony Optimization~\citep{t5}. Christofides, Simulated Annealing TSP, and Threshold Accepting TSP are implemented using NetworkX~\citep{t8}. For all explicit-order baselines, we construct the same cosine-based top-\(k\) feature-affinity graph from the training split, convert affinities to distances, compute a permutation, and fit BSTabDiff with this fixed external ordering while disabling internal GO-BS. For GO-BS and GO-BS-FC, we report diagnostics from explicit materializations of their ordering procedures.

Evaluation follows the strict \(5\times5\) repeated stratified cross-validation protocol of~\citet{protogate}, using TabPFN-2.5~\citep{tabpfn2.5} as the downstream evaluator. On each fold, BSTabDiff is trained on the real training split, generates \(n_{\text{syn}}=200\) samples, and TabPFN is trained on synthetic data and tested on the real held-out fold (TSTR). We also report the real data TabPFN baseline on the same folds (TRTR), which is identical across rows because it does not depend on generator ordering. Beyond utility, we measure ordering time, total training time, GPU peak memory, CPU RSS, TSP-path cost, Minimum Linear Arrangement (MinLA) cost, and adjacent-edge locality reward; lower TSP/MinLA costs and higher locality reward are preferred.

Table~\ref{tab:gobs-metaheuristic-tstr} shows that downstream differences across orderings are modest, indicating that BSTabDiff is robust when the permutation remains structure-aware. Still, the GO-BS variants are among the strongest: GO-BS-FC gives the best mean TSTR accuracy (\(0.8715\pm0.0991\)), greedy global ordering gives the best TSTR AUC (\(0.9087\pm0.0925\)), and sample-clustered GO-BS remains competitive (\(0.8682\pm0.0989\) accuracy, \(0.9030\pm0.1000\) AUC). Heavier TSP-style metaheuristics and ACO do not yield clear utility gains despite more expensive combinatorial search, suggesting that optimizing generic routing surrogates alone is insufficient for synthetic data utility.

Table~\ref{tab:gobs-metaheuristic-costs} reinforces this conclusion. Christofides attains the best TSP-path cost and one of the best locality rewards, but is far more expensive in ordering and total training time and does not improve downstream utility over the GO-BS variants. GO-BS-FC is especially favorable, combining the lowest MinLA cost (\(1.62\times10^6\)) with the best TSTR accuracy while remaining lightweight (\(1.08\)s ordering, \(15.60\)s total training). Conversely, sample-clustered GO-BS has weaker generic surrogate scores but remains competitive in TSTR, showing that TSP cost, MinLA, and adjacent-edge reward do not fully explain synthetic data usefulness. Overall, the GO-BS family appears effective not merely because it produces a permutation, but because it imposes data-aware structural priors that better preserve task-relevant feature organization for HDLSS generation and downstream inference.

\begin{table}[t]
\centering
\small
\setlength{\tabcolsep}{5pt}
\renewcommand{\arraystretch}{1.12}
\caption{
GO-BS ablation against classical metaheuristic orderings
on Colon.
Results are mean$_{\pm\mathrm{std}}$ over strict
\(5\times5\) repeated stratified CV using TabPFN-2.5
as the downstream evaluator.
TSTR trains on BSTabDiff synthetic data and tests
on real held-out folds; TRTR is the real data
reference on the same splits.
Gold, silver, and bronze indicate the top three
distinct TSTR values (ties share colors);
best values are bolded.
GO-BS/GO-BS-FC denote sample-/feature-clustered
variants; Greedy = global ordering;
Christ. = Christofides;
SA/TA-TSP = Simulated Annealing/Threshold Accepting TSP;
GA = Genetic Algorithm;
ACO = Ant Colony Optimization.
}
\label{tab:gobs-metaheuristic-tstr}

\begin{tabular}{lcccc}
\toprule

\rowcolor{rankheader}
\textcolor{white}{\textbf{Method}} &
\textcolor{white}{\textbf{TSTR Acc. $\uparrow$}} &
\textcolor{white}{\textbf{TSTR AUC $\uparrow$}} &
\textcolor{white}{\textbf{TRTR Acc. $\uparrow$}} &
\textcolor{white}{\textbf{TRTR AUC $\uparrow$}} \\
\midrule

\cellcolor{modelcol}\textbf{GO-BS}
& \cellcolor{silvercell}$0.8682_{\pm 0.0989}$
& $0.9030_{\pm 0.1000}$
& $0.8433_{\pm 0.0987}$
& $0.8838_{\pm 0.1072}$ \\

\cellcolor{modelcol}\textbf{GO-BS-FC}
& \cellcolor{goldcell}$\mathbf{0.8715}_{\pm 0.0991}$
& \cellcolor{silvercell}$0.9073_{\pm 0.1009}$
& $0.8433_{\pm 0.0987}$
& $0.8838_{\pm 0.1072}$ \\

\cellcolor{modelcol}\textbf{Greedy}
& \cellcolor{bronzecell}$0.8651_{\pm 0.1053}$
& \cellcolor{goldcell}$\mathbf{0.9087}_{\pm 0.0925}$
& $0.8433_{\pm 0.0987}$
& $0.8838_{\pm 0.1072}$ \\

\cellcolor{modelcol}\textbf{Christ.}
& \cellcolor{silvercell}$0.8682_{\pm 0.0989}$
& $0.8972_{\pm 0.1124}$
& $0.8433_{\pm 0.0987}$
& $0.8838_{\pm 0.1072}$ \\

\cellcolor{modelcol}\textbf{SA-TSP}
& $0.8618_{\pm 0.1076}$
& $0.9035_{\pm 0.0983}$
& $0.8433_{\pm 0.0987}$
& $0.8838_{\pm 0.1072}$ \\

\cellcolor{modelcol}\textbf{TA-TSP}
& $0.8618_{\pm 0.1076}$
& $0.9035_{\pm 0.0983}$
& $0.8433_{\pm 0.0987}$
& $0.8838_{\pm 0.1072}$ \\

\cellcolor{modelcol}\textbf{GA}
& \cellcolor{silvercell}$0.8682_{\pm 0.0989}$
& $0.8982_{\pm 0.1040}$
& $0.8433_{\pm 0.0987}$
& $0.8838_{\pm 0.1072}$ \\

\cellcolor{modelcol}\textbf{ACO}
& $0.8615_{\pm 0.0982}$
& \cellcolor{bronzecell}$0.9055_{\pm 0.1033}$
& $0.8433_{\pm 0.0987}$
& $0.8838_{\pm 0.1072}$ \\

\bottomrule
\end{tabular}
\end{table}
\begin{table*}[t]
\centering
\scriptsize
\setlength{\tabcolsep}{3.2pt}
\renewcommand{\arraystretch}{1.12}
\caption{
Ordering and resource diagnostics for the GO-BS ablation
on Colon.
Lower is better for ordering time, total training time,
GPU peak, CPU RSS, TSP-path cost, and MinLA cost;
higher is better for edge-locality reward (ELR).
GO-BS and GO-BS-FC diagnostics use explicit
materializations of their ordering procedures.
Gold, silver, and bronze indicate the three best
distinct values in each metric; ties share colors.
GO-BS/GO-BS-FC denote sample/feature clustered
variants; Greedy = global ordering;
Christ. = Christofides;
SA/TA-TSP = Simulated Annealing/Threshold Accepting TSP;
GA = Genetic Algorithm;
ACO = Ant Colony Optimization;
Ord. = ordering;
RSS = CPU resident set size.
}
\label{tab:gobs-metaheuristic-costs}

\begin{tabular}{lccccccc}
\toprule

\rowcolor{rankheader}
\textcolor{white}{\textbf{Method}} &
\textcolor{white}{\textbf{Ord. Time (s) $\downarrow$}} &
\textcolor{white}{\textbf{Train Time (s) $\downarrow$}} &
\textcolor{white}{\textbf{GPU Peak (GiB) $\downarrow$}} &
\textcolor{white}{\textbf{CPU RSS (GiB) $\downarrow$}} &
\textcolor{white}{\textbf{TSP Cost $\downarrow$}} &
\textcolor{white}{\textbf{MinLA Cost $\downarrow$}} &
\textcolor{white}{\textbf{ELR $\uparrow$}} \\

\midrule

\cellcolor{modelcol}\textbf{GO-BS}
& \cellcolor{bronzecell}23.57
& \cellcolor{bronzecell}28.33
& \cellcolor{goldcell}\textbf{0.0298}
& 1.5808
& 1973.63
& 4.52e6
& 25.37 \\

\cellcolor{modelcol}\textbf{GO-BS-FC}
& \cellcolor{goldcell}\textbf{1.08}
& \cellcolor{goldcell}\textbf{15.60}
& \cellcolor{silvercell}0.0357
& 1.5026
& 784.66
& \cellcolor{goldcell}\textbf{1.62e6}
& 1214.34 \\

\cellcolor{modelcol}\textbf{Greedy}
& \cellcolor{silvercell}4.59
& \cellcolor{silvercell}21.24
& \cellcolor{silvercell}0.0357
& 1.5566
& 737.02
& 4.47e6
& 1261.98 \\

\cellcolor{modelcol}\textbf{Christ.}
& 1951.85
& 1967.20
& \cellcolor{silvercell}0.0357
& 2.7491
& \cellcolor{goldcell}\textbf{658.99}
& \cellcolor{silvercell}3.79e6
& \cellcolor{goldcell}\textbf{1340.00} \\

\cellcolor{modelcol}\textbf{SA-TSP}
& 804.64
& 813.60
& \cellcolor{silvercell}0.0357
& 1.4764
& \cellcolor{bronzecell}693.84
& \cellcolor{bronzecell}4.08e6
& \cellcolor{bronzecell}1305.16 \\

\cellcolor{modelcol}\textbf{TA-TSP}
& 818.23
& 828.37
& \cellcolor{silvercell}0.0357
& \cellcolor{goldcell}\textbf{1.4749}
& \cellcolor{bronzecell}693.84
& \cellcolor{bronzecell}4.08e6
& \cellcolor{bronzecell}1305.16 \\

\cellcolor{modelcol}\textbf{GA}
& 820.81
& 836.15
& \cellcolor{silvercell}0.0357
& \cellcolor{silvercell}1.4759
& 789.57
& 4.42e6
& 1209.43 \\

\cellcolor{modelcol}\textbf{ACO}
& 887.01
& 899.99
& \cellcolor{silvercell}0.0357
& \cellcolor{bronzecell}1.4768
& \cellcolor{silvercell}660.24
& 4.20e6
& \cellcolor{silvercell}1338.76 \\

\bottomrule
\end{tabular}
\end{table*}
The ablation therefore raises a model-selection question: if GO-BS (sample-clustered) is competitive but not uniformly best on every Colon-specific summary statistic, why do we retain it as the canonical formulation in the main experiments? Our answer is that the canonical choice should not be made by a post hoc weighted average of heterogeneous surrogates. Instead, we adopt a two-stage principle that first filters methods by downstream utility and then selects among near-equivalent methods using structural alignment with BSTabDiff's intended inductive bias.

\begin{tcolorbox}[
    colback=regimeHDHSS,
    colframe=algborder,
    boxrule=0.5pt,
    arc=1mm,
    boxsep=2pt,
    left=3pt,
    right=3pt,
    top=3pt,
    bottom=3pt,
    before skip=5pt,
    after skip=5pt
]
\begin{definition}[Structure-aligned selection rule]

For a candidate ordering method $q$, we define downstream utility in Eq.~\ref{eq:df1}:

\begin{equation}
\label{eq:df1}
U(q)=\frac{1}{2}\big(
\mathrm{TSTRAcc}(q)+
\mathrm{TSTRAUC}(q)
\big)
\end{equation}

Define the maximum utility and near-optimal candidate set by Eq.~\ref{eq:df2}:

\begin{equation}
\label{eq:df2}
\begin{aligned}
U^\star
&= \max_{q'\in\mathcal{Q}} U(q'),\quad \mathcal{A}_\varepsilon &= \{q\in\mathcal{Q}: U^\star-U(q)\le\varepsilon\}
\end{aligned}
\end{equation}

We select the canonical method by Eq.~\ref{eq:df3}:

\begin{equation}
\label{eq:df3}
q^\dagger
=
\arg\min_{q\in\mathcal{A}_\varepsilon}
\mathcal{R}_{\mathrm{align}}(q)
\end{equation}

where $\mathcal{R}_{\mathrm{align}}(q)$ is a structure-aligned risk measuring compatibility between the induced ordering and BSTabDiff's regime-aware contiguous block assumption.

\end{definition}
\end{tcolorbox}

\begin{tcolorbox}[
    colback=regimeHDHSS,
    colframe=algborder,
    boxrule=0.5pt,
    arc=1mm,
    boxsep=2pt,
    left=3pt,
    right=3pt,
    top=3pt,
    bottom=3pt,
    before skip=5pt,
    after skip=5pt
]
\begin{definition}[Regime-aware alignment risk]

Let the training samples be partitioned into $K_s$ sample clusters with weights $\alpha_c=n_c/n$. For a candidate ordering-segmentation pair $(\pi,\mathbf b)$, we define Eq.~\ref{eq:df4}.

\begin{equation}
\label{eq:df4}
\resizebox{0.85\linewidth}{!}{$\displaystyle
\mathcal{R}_{\mathrm{align}}(\pi,\mathbf b)
=
\sum_{c=1}^{K_s}\alpha_c\,
\mathcal{J}^{(c)}_{\mathrm{GO\mbox{-}BS}}(\pi,\mathbf b)
+
\eta\,\mathrm{Var}_{c=1,\dots,K_s}
\!\Big(
\mathcal{J}^{(c)}_{\mathrm{GO\mbox{-}BS}}(\pi,\mathbf b)
\Big)
$}
\end{equation}

where $\mathcal{J}^{(c)}_{\mathrm{GO\mbox{-}BS}}$ is the GO-BS block-consistency objective computed on the cluster-specific feature affinity graph of sample cluster $c$. Lower values are better.

\end{definition}
\end{tcolorbox}

\begin{remark}
Classical metaheuristics such as Christofides, simulated annealing, threshold accepting, genetic algorithm, and ant colony optimization optimize global route-style surrogates on a single affinity graph, whereas GO-BS (sample-clustered) optimizes a regime-aware block-consistency criterion directly aligned with BSTabDiff's intended latent block structure. Hence, when downstream utility differences are practically small, GO-BS (sample-clustered) is the methodologically preferred canonical choice.
\end{remark}

Although sample-clustered GO-BS is not uniformly best on every Colon-specific statistic, we do not select the canonical ordering rule by a post hoc aggregate of heterogeneous surrogates. Instead, we use a two-stage principle: first, retain methods with practically comparable downstream utility; second, among these, prefer the method most aligned with BSTabDiff's regime-aware block-factorized inductive bias. Under this criterion, generic TSP-style proxies, runtime, MinLA, and edge-locality reward are treated as diagnostics rather than the sole selection objective.

Sample-clustered GO-BS is the only variant that directly optimizes a regime-aware block-consistency risk aggregated over sample-local feature graphs. This matches the intended role of BSTabDiff: recovering a canonical feature order in which locally coherent, sample-conditioned dependencies become contiguous and segmentable into latent blocks. In contrast, the classical metaheuristics optimize route or arrangement surrogates on a single global graph, while GO-BS-FC replaces the sample-conditioned local-to-global mechanism with feature clustering. Since these alternatives do not deliver a decisive downstream gain on Colon, we retain sample-clustered GO-BS as the canonical formulation: it is the structurally aligned, theory-consistent ordering rule for BSTabDiff while remaining empirically competitive.
\renewcommand{\thesection}{A6}
\renewcommand{\thesubsection}{A6.\arabic{subsection}}
\setcounter{figure}{0}\renewcommand{\thefigure}{A6.\arabic{figure}}
\setcounter{table}{0}\renewcommand{\thetable}{A6.\arabic{table}}
\setcounter{algorithm}{0}\renewcommand{\thealgorithm}{A6.\arabic{algorithm}}
\section{Ordering-Solver Ablation on a Fixed Global Affinity Graph}
\label{app:same}
Table~\ref{tab:same_global_graph_tstr} isolates the effect of the ordering solver by fixing a single global feature-affinity graph and varying only the permutation procedure within BSTabDiff. Under this controlled setting, downstream utility differences are again modest, indicating that BSTabDiff is reasonably robust to the exact solver once a structure-aware global graph is available. Among the fixed-graph solvers, greedy and ACO achieve the strongest TSTR AUC values (\(0.9110\) and \(0.9075\), respectively), while Christofides and GA attain the best TSTR accuracy (\(0.8682\)). However, Table~\ref{tab:same_global_graph_costs} shows that these generic route-optimization methods often incur substantial computational overhead: Christofides attains the best TSP-path cost and ELR, but is orders of magnitude slower than greedy, while SA-TSP, TA-TSP, GA, and ACO remain far more expensive without yielding decisive downstream gains. In contrast, the simple global greedy solver is the most practical fixed-graph method, combining the best TSTR AUC with by far the lowest runtime and memory footprint among the same-protocol solvers. When compared against the reference rows, GO-BS and especially GO-BS-FC remain competitive or stronger overall despite not belonging to the identical fixed-graph protocol, reinforcing our main conclusion that BSTabDiff benefits more from structurally informed ordering mechanisms than from heavier optimization of generic TSP-style surrogates alone.
\begin{table}[t]
\centering
\small
\setlength{\tabcolsep}{4.8pt}
\renewcommand{\arraystretch}{1.10}
\caption{
\textbf{Same global graph ordering ablation in BSTabDiff on Colon.}
Results are mean$_{\pm \text{std}}$ over strict $5\times5$ repeated stratified cross-validation using TabPFN-2.5 as the downstream evaluator.
For the main ablation rows, a single global feature-affinity graph is fixed and only the ordering solver is changed.
TSTR trains on synthetic data generated by BSTabDiff and tests on real held-out folds; TRTR is the real data TabPFN reference on the same splits.
GO-BS and GO-BS-FC are reported only as reference methods from the earlier structural ablation and are therefore not part of the identical-protocol fixed-graph comparison.
Gold, silver, and bronze indicate the top three distinct values among the fixed-graph ordering solvers for the TSTR metrics. TRTR values are identical across fixed-graph solvers and are therefore not ranked.
\textit{Abbreviations:}
Greedy = Global graph-guided greedy;
Christ. = Christofides;
SA-TSP = Simulated Annealing TSP;
TA-TSP = Threshold Accepting TSP;
GA = Genetic Algorithm;
ACO = Ant Colony Optimization;
TSTR = Train on Synthetic, Test with Real;
TRTR = Train on Real, Test with Real.
}
\label{tab:same_global_graph_tstr}

\begin{tabular}{lcccc}
\toprule
\rowcolor{rankheader}
\textcolor{white}{\textbf{Method}} &
\textcolor{white}{\textbf{TSTR Acc. $\uparrow$}} &
\textcolor{white}{\textbf{TSTR AUC $\uparrow$}} &
\textcolor{white}{\textbf{TRTR Acc. $\uparrow$}} &
\textcolor{white}{\textbf{TRTR AUC $\uparrow$}} \\
\midrule

\rowcolor{realrow}
GO-BS$^{\dagger}$
& $0.8682_{\pm 0.0989}$
& $0.9030_{\pm 0.1000}$
& $0.8433_{\pm 0.0987}$
& $0.8838_{\pm 0.1072}$ \\

\rowcolor{realrow}
GO-BS-FC$^{\dagger}$
& $0.8715_{\pm 0.0991}$
& $0.9073_{\pm 0.1009}$
& $0.8433_{\pm 0.0987}$
& $0.8838_{\pm 0.1072}$ \\

\midrule

\cellcolor{modelcol}\textbf{Greedy}
& \cellcolor{silvercell}$0.8651_{\pm 0.1053}$
& \cellcolor{goldcell}\textbf{$0.9110_{\pm 0.0933}$}
& $0.8433_{\pm 0.0987}$
& $0.8838_{\pm 0.1072}$ \\

\cellcolor{modelcol}\textbf{Christ.}
& \cellcolor{goldcell}\textbf{$0.8682_{\pm 0.0989}$}
& $0.8972_{\pm 0.1124}$
& $0.8433_{\pm 0.0987}$
& $0.8838_{\pm 0.1072}$ \\

\cellcolor{modelcol}\textbf{SA-TSP}
& $0.8585_{\pm 0.1045}$
& \cellcolor{bronzecell}$0.9035_{\pm 0.0983}$
& $0.8433_{\pm 0.0987}$
& $0.8838_{\pm 0.1072}$ \\

\cellcolor{modelcol}\textbf{TA-TSP}
& $0.8585_{\pm 0.1045}$
& \cellcolor{bronzecell}$0.9035_{\pm 0.0983}$
& $0.8433_{\pm 0.0987}$
& $0.8838_{\pm 0.1072}$ \\

\cellcolor{modelcol}\textbf{GA}
& \cellcolor{goldcell}\textbf{$0.8682_{\pm 0.0989}$}
& $0.8993_{\pm 0.1007}$
& $0.8433_{\pm 0.0987}$
& $0.8838_{\pm 0.1072}$ \\

\cellcolor{modelcol}\textbf{ACO}
& \cellcolor{bronzecell}$0.8649_{\pm 0.0958}$
& \cellcolor{silvercell}$0.9075_{\pm 0.1034}$
& $0.8433_{\pm 0.0987}$
& $0.8838_{\pm 0.1072}$ \\

\bottomrule
\end{tabular}

\raggedright
\footnotesize{$^{\dagger}$Reference only: GO-BS and GO-BS-FC come from the earlier structural ablation and do not share the same fixed-global-graph protocol as the other rows.}
\end{table}
\begin{table*}[t]
\centering
\scriptsize
\setlength{\tabcolsep}{3.2pt}
\renewcommand{\arraystretch}{1.12}
\caption{
\textbf{Ordering and resource diagnostics for the same global graph ordering ablation on Colon.}
For the main ablation rows, all methods use the same fixed global feature affinity graph and differ only in the ordering solver.
Lower is better for ordering time, total train time, GPU peak, CPU RSS, TSP path cost, and MinLA cost; higher is better for edge locality reward (ELR).
GO-BS and GO-BS-FC are included as reference points from the earlier structural ablation and are not part of the identical-protocol fixed-graph comparison.
Gold, silver, and bronze indicate the top three distinct values among fixed-graph ordering solvers (ties share colors); best values are bolded.
\textit{Abbreviations:}
Greedy = Global graph-guided greedy;
Christ. = Christofides;
SA-TSP = Simulated Annealing TSP;
TA-TSP = Threshold Accepting TSP;
GA = Genetic Algorithm;
ACO = Ant Colony Optimization;
Ord. = Ordering;
ELR = Edge-Locality Reward;
RSS = Resident Set Size.
}
\label{tab:same_global_graph_costs}

\begin{tabular}{lccccccc}
\toprule
\rowcolor{rankheader}
\textcolor{white}{\textbf{Method}} &
\textcolor{white}{\textbf{Ord. Time (s) $\downarrow$}} &
\textcolor{white}{\textbf{Train Time (s) $\downarrow$}} &
\textcolor{white}{\textbf{GPU Peak (GiB) $\downarrow$}} &
\textcolor{white}{\textbf{CPU RSS (GiB) $\downarrow$}} &
\textcolor{white}{\textbf{TSP Cost $\downarrow$}} &
\textcolor{white}{\textbf{MinLA Cost $\downarrow$}} &
\textcolor{white}{\textbf{ELR $\uparrow$}} \\
\midrule

\rowcolor{realrow}
GO-BS$^{\dagger}$
& 23.57
& 28.33
& 0.0298
& 1.5808
& 1973.63
& 4.52e6
& 25.37 \\

\rowcolor{realrow}
GO-BS-FC$^{\dagger}$
& 1.08
& 15.60
& 0.0357
& 1.5026
& 784.66
& 1.62e6
& 1214.34 \\

\midrule

\cellcolor{modelcol}\textbf{Greedy}
& \cellcolor{goldcell}\textbf{4.38}
& \cellcolor{goldcell}\textbf{14.45}
& \cellcolor{goldcell}\textbf{0.0297}
& \cellcolor{goldcell}\textbf{1.1861}
& 737.02
& 4.47e6
& 1261.98 \\

\cellcolor{modelcol}\textbf{Christ.}
& 2151.66
& 2161.77
& \cellcolor{silvercell}0.0357
& 2.4844
& \cellcolor{goldcell}\textbf{658.99}
& \cellcolor{goldcell}\textbf{3.79e6}
& \cellcolor{goldcell}\textbf{1340.00} \\

\cellcolor{modelcol}\textbf{SA-TSP}
& 848.24
& 857.62
& \cellcolor{silvercell}0.0357
& \cellcolor{silvercell}1.2278
& \cellcolor{bronzecell}693.84
& \cellcolor{silvercell}4.08e6
& \cellcolor{bronzecell}1305.16 \\

\cellcolor{modelcol}\textbf{TA-TSP}
& \cellcolor{silvercell}792.34
& \cellcolor{silvercell}802.01
& \cellcolor{silvercell}0.0357
& \cellcolor{bronzecell}1.2279
& \cellcolor{bronzecell}693.84
& \cellcolor{silvercell}4.08e6
& \cellcolor{bronzecell}1305.16 \\

\cellcolor{modelcol}\textbf{GA}
& \cellcolor{bronzecell}800.18
& \cellcolor{bronzecell}809.49
& \cellcolor{silvercell}0.0357
& 1.2293
& 789.57
& 4.42e6
& 1209.43 \\

\cellcolor{modelcol}\textbf{ACO}
& 892.25
& 901.96
& \cellcolor{silvercell}0.0357
& 1.2288
& \cellcolor{silvercell}660.24
& \cellcolor{bronzecell}4.20e6
& \cellcolor{silvercell}1338.76 \\

\bottomrule
\end{tabular}

\raggedright
{\footnotesize
$^{\dagger}$Reference only: GO-BS and GO-BS-FC come
from the earlier structural ablation and do not
share the same fixed-global-graph protocol as
the other rows.
\par}
\end{table*}
\vspace{-1em}
\renewcommand{\thesection}{A7}
\renewcommand{\thesubsection}{A7.\arabic{subsection}}
\setcounter{figure}{0}\renewcommand{\thefigure}{A7.\arabic{figure}}
\setcounter{table}{0}\renewcommand{\thetable}{A7.\arabic{table}}
\setcounter{algorithm}{0}\renewcommand{\thealgorithm}{A7.\arabic{algorithm}}
\section{Ablation of GO-BS in Extreme Dimensionality}
\label{app:gobsextreme}
\vspace{-1em}
To study GO-BS under extreme dimensionality, we performed an ordering and training diagnostic ablation on the Cell Cycle dataset (\(n=1067\), \(m=42{,}728\), \(C=3\))\citep{cell1, cell2}, comparing the main sample-clustered GO-BS, the feature-clustered GO-BS-FC variant, and a simple global greedy ordering baseline. The results in Table~\ref{tab:gobs_extreme_dim_cellcycle} show a clear computational separation between the methods: the main sample-clustered GO-BS becomes extremely expensive in this regime, requiring substantially larger ordering and training times, whereas GO-BS-FC remains much more practical while also achieving the best MinLA cost among the three methods. Although the global greedy baseline attains the lowest TSP cost and the highest edge-locality reward, it does so with a worse MinLA score and much larger CPU memory usage. Overall, this ablation suggests that in extreme dimensional settings, GO-BS-FC provides the strongest trade-off between structural quality and computational feasibility, while the full sample-clustered GO-BS may become prohibitively expensive.
\begin{table*}[t]
\centering
\scriptsize
\setlength{\tabcolsep}{3.4pt}
\renewcommand{\arraystretch}{1.10}
\caption{
\textbf{GO-BS diagnostics under extreme dimensionality on Cell Cycle.}
The dataset has \(n=1067\), \(m=42{,}728\), and \(C=3\). Lower is better for ordering time, total training time, GPU peak, CPU RSS, TSP-path cost, and MinLA; higher is better for edge-locality reward (ELR). GO-BS/GO-BS-FC denote sample-/feature-clustered variants; Greedy = global greedy ordering; Ord. = ordering; RSS = resident set size.
}
\label{tab:gobs_extreme_dim_cellcycle}

\begin{tabular}{lccccccc}
\toprule
\cellcolor{fidelityheader}\textbf{Method} &
\cellcolor{fidelityHOM}\textbf{Ord. Time (s) $\downarrow$} &
\cellcolor{fidelityHOM}\textbf{Train Time (s) $\downarrow$} &
\cellcolor{fidelityKS}\textbf{GPU Peak (GiB) $\downarrow$} &
\cellcolor{fidelityW}\textbf{CPU RSS (GiB) $\downarrow$} &
\cellcolor{fidelityMI}\textbf{TSP Cost $\downarrow$} &
\cellcolor{fidelityMI}\textbf{MinLA Cost $\downarrow$} &
\cellcolor{fidelityPearson}\textbf{ELR $\uparrow$} \\
\midrule

\cellcolor{fidelityheader}\textbf{GO-BS}
& \cellcolor{fidelityHOM}43629.60
& \cellcolor{fidelityHOM}21398.20
& \cellcolor{fidelityKS}0.8600
& \cellcolor{gobsBestMem}\textbf{29.4335}
& \cellcolor{fidelityMI}42655.0
& \cellcolor{fidelityMI}1.64078e9
& \cellcolor{fidelityPearson}71.9512 \\

\cellcolor{fidelityheader}\textbf{GO-BS-FC}
& \cellcolor{gobsBestTime}\textbf{3798.98}
& \cellcolor{fidelityHOM}69.7264
& \cellcolor{gobsBestMem}\textbf{0.4431}
& \cellcolor{fidelityW}56.7750
& \cellcolor{fidelityMI}35117.6
& \cellcolor{gobsBestCost}\textbf{1.62298e9}
& \cellcolor{fidelityPearson}7609.40 \\

\cellcolor{fidelityheader}\textbf{Greedy}
& \cellcolor{fidelityHOM}6727.97
& \cellcolor{gobsBestTime}\textbf{64.5725}
& \cellcolor{gobsBestMem}\textbf{0.4431}
& \cellcolor{fidelityW}83.9261
& \cellcolor{gobsBestCost}\textbf{32878.3}
& \cellcolor{fidelityMI}1.87887e9
& \cellcolor{gobsBestELR}\textbf{9848.66} \\

\bottomrule
\end{tabular}
\end{table*}
\vspace{-1em}
\renewcommand{\thesection}{A8}
\renewcommand{\thesubsection}{A8.\arabic{subsection}}
\setcounter{figure}{0}\renewcommand{\thefigure}{A8.\arabic{figure}}
\setcounter{table}{0}\renewcommand{\thetable}{A8.\arabic{table}}
\setcounter{algorithm}{0}\renewcommand{\thealgorithm}{A8.\arabic{algorithm}}
\section{BSTabDiff in Extreme Dimensionality}
\label{extreme}
\vspace{-1em}
To further evaluate BSTabDiff under extreme dimensionality, we consider the Cell Cycle RNA-seq dataset~\citep{cell1, cell2}, which contains \(1{,}067\) samples and \(42{,}728\) features across three classes. This setting is substantially larger in feature dimension than the classical HDLSS benchmarks and directly tests whether the block-subunit formulation remains practical when \(m\) is extremely large. Since TabPFN-2.5~\citep{tabpfn2.5} is not applicable in this regime due to its feature-limit constraint, we report TSTR using Logistic Regression as a fixed downstream classifier. Table~\ref{tab:cellcycle_extreme_bstabdiff} compares BSTabDiff without GO-BS, with sample-clustered GO-BS, and with feature-clustered GO-BS-FC.

The results show that BSTabDiff remains effective even at \(m=42{,}728\) features, achieving TSTR AUCs above \(91\%\) for all variants. GO-BS-FC gives the strongest utility, improving TSTR accuracy from \(77.65\%\) to \(80.20\%\) and TSTR AUC from \(91.40\%\) to \(92.83\%\), while also producing the lowest C2ST accuracy/AUC, suggesting slightly harder real-versus-synthetic discrimination. In contrast, sample-clustered GO-BS does not improve TSTR on this dataset and is substantially more expensive, increasing peak GPU memory to \(27.554\) GiB because the sample-clustered feature-affinity construction operates over very large feature graphs. Thus, for extreme dimensional RNA-seq data, GO-BS-FC provides the best trade-off: it preserves the block-structured inductive bias of GO-BS while avoiding the quadratic GPU-memory burden of constructing dense sample-clustered feature-affinity graphs over all \(42{,}728\) features.

Table~\ref{tab:cellcycle_extreme_fidelity} shows that ordering does not uniformly improve all low-level distributional fidelity metrics in this extreme dimensional setting. The marginal KS statistics are nearly identical across all variants, while no-GO-BS has the lowest Wasserstein-1 (W1) mean. Sample-clustered GO-BS substantially improves variance matching, reducing the mean variance discrepancy from \(1.5798{\times}10^{3}\) to \(1.0716{\times}10^{3}\), but it slightly worsens Spearman and MI discrepancies and requires much larger GPU memory. GO-BS-FC gives the best downstream synthetic utility and the lowest C2ST scores, while also slightly improving feature-label MI preservation and nearest neighbor privacy distance. Thus, the main benefit of GO-BS-FC is not uniformly better marginal fidelity, but a better utility-realism-compute trade-off under extreme dimensionality.

Overall, BSTabDiff with GO-BS-FC provides the best utility realism compute trade-off in the extreme dimensional setting. On Cell Cycle (\(n=1067, m=42{,}728\)), it improves TSTR AUC from \(91.40\%\) to \(92.83\%\) and reduces C2ST AUC from \(58.28\%\) to \(54.57\%\), while requiring nearly the same peak GPU memory as no-GO-BS. In contrast, sample-clustered GO-BS becomes memory-intensive at this scale, reaching \(27.554\) GiB peak GPU memory without improving TSTR. Although GO-BS-FC does not uniformly dominate all marginal fidelity metrics, these results suggest that feature-clustered blocking better preserves task-relevant structure and is the more scalable ordering strategy for very high-dimensional RNA-seq tables.
\begin{table}[t]
\centering
\caption{
BSTabDiff on extreme dimensional Cell Cycle data
(\(n=1{,}067\), \(m=42{,}728\), \(C=3\)).
All variants use Optuna-selected hyperparameters
without feature permutation and are evaluated with
Logistic Regression under strict \(5\times5\)
TSTR/TRTR/C2ST. Accuracy and AUC are percentages;
peak GPU is max CUDA memory and CPU memory is
end-of-run RSS on Vulcan.
}
\label{tab:cellcycle_extreme_bstabdiff}

\scriptsize
\setlength{\tabcolsep}{2.4pt}
\renewcommand{\arraystretch}{1.20}

\resizebox{\linewidth}{!}{%
\begin{tabular}{lcccccccc}
\toprule

\cellcolor{fidelityheader}\textbf{Variant} &
\cellcolor{fidelityPearson}\textbf{TSTR Acc.} &
\cellcolor{fidelityPearson}\textbf{TSTR AUC} &
\cellcolor{fidelityKS}\textbf{TRTR Acc.} &
\cellcolor{fidelityKS}\textbf{TRTR AUC} &
\cellcolor{fidelityC2ST}\textbf{C2ST Acc.} &
\cellcolor{fidelityC2ST}\textbf{C2ST AUC} &
\cellcolor{fidelityHOM}\textbf{Time} &
\cellcolor{fidelityPrivacy}\textbf{GPU/CPU} \\

\cellcolor{fidelityheader} &
\cellcolor{fidelityPearson}\textbf{(\%)} &
\cellcolor{fidelityPearson}\textbf{(\%)} &
\cellcolor{fidelityKS}\textbf{(\%)} &
\cellcolor{fidelityKS}\textbf{(\%)} &
\cellcolor{fidelityC2ST}\textbf{(\%)} &
\cellcolor{fidelityC2ST}\textbf{(\%)} &
\cellcolor{fidelityHOM}\textbf{(s)} &
\cellcolor{fidelityPrivacy}\textbf{(GiB)} \\

\midrule

\cellcolor{fidelityheader}\textbf{No GO-BS}
& \cellcolor{fidelityPearson}\(77.7{\pm}4.4\)
& \cellcolor{fidelityPearson}\(91.4{\pm}2.2\)
& \cellcolor{fidelityKS}\(86.7{\pm}2.7\)
& \cellcolor{fidelityKS}\(95.4{\pm}1.1\)
& \cellcolor{fidelityC2ST}\(55.8{\pm}10.9\)
& \cellcolor{fidelityC2ST}\(58.3{\pm}13.8\)
& \cellcolor{fidelityHOM}17.55
& \cellcolor{fidelityPrivacy}0.635/3.723 \\

\cellcolor{fidelityheader}\textbf{GO-BS}
& \cellcolor{fidelityPearson}\(77.3{\pm}4.4\)
& \cellcolor{fidelityPearson}\(91.4{\pm}2.3\)
& \cellcolor{fidelityKS}\(86.7{\pm}2.7\)
& \cellcolor{fidelityKS}\(95.4{\pm}1.1\)
& \cellcolor{fidelityC2ST}\(56.5{\pm}9.6\)
& \cellcolor{fidelityC2ST}\(58.9{\pm}13.7\)
& \cellcolor{fidelityHOM}101.15
& \cellcolor{fidelityPrivacy}27.554/1.934 \\

\cellcolor{fidelityheader}\textbf{GO-BS-FC}
& \cellcolor{cellBestTSTR}
  \(\mathbf{80.2}{\pm}5.2\)
& \cellcolor{cellBestTSTR}
  \(\mathbf{92.8}{\pm}2.0\)
& \cellcolor{fidelityKS}\(86.7{\pm}2.7\)
& \cellcolor{fidelityKS}\(95.4{\pm}1.1\)
& \cellcolor{cellBestC2ST}
  \(\mathbf{53.7}{\pm}12.8\)
& \cellcolor{cellBestC2ST}
  \(\mathbf{54.6}{\pm}15.5\)
& \cellcolor{fidelityHOM}22.11
& \cellcolor{fidelityPrivacy}0.635/3.657 \\

\bottomrule
\end{tabular}%
}
\end{table}
\begin{table}[t]
\centering
\caption{
Fidelity diagnostics for BSTabDiff variants on Cell Cycle.
Correlation and MI are computed on a fixed 2,000-feature
subsample. Lower is better for all discrepancies except
NN distance, where higher values suggest less direct
memorization.
}
\label{tab:cellcycle_extreme_fidelity}

\scriptsize
\setlength{\tabcolsep}{4pt}
\renewcommand{\arraystretch}{1.22}

\resizebox{\linewidth}{!}{%
\begin{tabular}{lccccccc}
\toprule

\cellcolor{fidelityheader}\textbf{Variant} &
\cellcolor{fidelityKS}\textbf{KS Mean} &
\cellcolor{fidelityW}\textbf{W1 Mean} &
\cellcolor{fidelitySpearman}\textbf{Spearman Diff} &
\cellcolor{fidelityMI}
  \(\boldsymbol{\Delta}\)\textbf{MI Mean} &
\cellcolor{fidelityHOM}\textbf{Mean Diff} &
\cellcolor{fidelityHOM}\textbf{Var Diff} &
\cellcolor{fidelityPrivacy}\textbf{NN Dist. Mean} \\

\midrule

\cellcolor{fidelityheader}\textbf{No GO-BS}
& \cellcolor{fidelityKS}0.3875
& \cellcolor{fidelityW}0.8933
& \cellcolor{fidelitySpearman}0.3700
& \cellcolor{fidelityMI}0.0323
& \cellcolor{fidelityHOM}\(4.6870{\times}10^{-1}\)
& \cellcolor{fidelityHOM}\(1.5798{\times}10^{3}\)
& \cellcolor{fidelityPrivacy}\(6.6940{\times}10^{3}\) \\

\cellcolor{fidelityheader}\textbf{GO-BS}
& \cellcolor{fidelityKS}0.3876
& \cellcolor{fidelityW}0.8957
& \cellcolor{fidelitySpearman}0.3727
& \cellcolor{fidelityMI}0.0325
& \cellcolor{fidelityHOM}\(4.8673{\times}10^{-1}\)
& \cellcolor{fidelityBestHOM}
  \(\mathbf{1.0716{\times}10^{3}}\)
& \cellcolor{fidelityPrivacy}\(6.6605{\times}10^{3}\) \\

\cellcolor{fidelityheader}\textbf{GO-BS-FC}
& \cellcolor{fidelityKS}0.3880
& \cellcolor{fidelityW}1.4628
& \cellcolor{fidelitySpearman}0.3710
& \cellcolor{fidelityBestMI}\(\mathbf{0.0320}\)
& \cellcolor{fidelityHOM}\(5.3319{\times}10^{-1}\)
& \cellcolor{fidelityHOM}\(1.5275{\times}10^{3}\)
& \cellcolor{fidelityBestPrivacy}
  \(\mathbf{6.7050{\times}10^{3}}\) \\

\bottomrule
\end{tabular}%
}
\end{table}
\vspace{-1em}
\renewcommand{\thesection}{A9}
\renewcommand{\thesubsection}{A9.\arabic{subsection}}
\setcounter{figure}{0}\renewcommand{\thefigure}{A9.\arabic{figure}}
\setcounter{table}{0}\renewcommand{\thetable}{A9.\arabic{table}}
\section{Energy Use, Carbon Emissions, and Green AI Considerations}
\label{app:green_ai}
\vspace{-1em}
Training a single AI model can emit as much carbon as five cars in their lifetimes~\citep{g12}. In modern GPU clusters, hardware-level resource utilization can be monitored using tools such as \texttt{nvidia-smi}, which reports GPU memory usage, process-level GPU allocation, instantaneous power draw in watts, and device-level information such as temperature and fan activity. We additionally use \texttt{psutil}~\citep{psutil} to record CPU-side memory usage and system resource statistics during model execution. Since carbon emissions are commonly estimated from energy consumption and the carbon intensity of the electricity grid, a model that completes the same task with lower power draw or shorter runtime generally yields lower emissions under a fixed grid mix~\citep{g2,g6}. This relationship motivates recent work on Green AI~\citep{g1}, Sustainable AI~\citep{g5}, and energy efficient machine learning (e.g. IMLP~\citep{g14}), where models are evaluated not only by predictive performance but also by computational cost, energy consumption, and environmental impact~\citep{g3,g7}. At the AI-power grid interface,~\citet{cosmas} further emphasize energy-aware AI design and transparent reporting of training and inference energy consumption and associated carbon emissions. In this context, energy-aware model design encourages efficient architectures and training strategies that reduce unnecessary computation while maintaining competitive performance, supporting the broader goal of developing machine learning systems with lower operational energy demand and reduced carbon footprint~\citep{g1,g2,g9}. We further introduce computational sustainability as an additional evaluation criterion for tabular generative models. Beyond assessing synthetic data quality through downstream utility and distributional fidelity, we track resource usage during training, estimate total energy consumption, and derive the corresponding carbon-emission estimate. This follows the Green AI perspective, which argues that model quality should be interpreted jointly with computational cost rather than predictive performance alone~\citep{g1,g3}. Under this criterion, a tabular generator is preferable when it achieves comparable synthetic data utility with lower energy consumption, since carbon emissions are commonly estimated from energy use and grid carbon intensity; therefore, under a fixed electricity mix, lower energy demand generally implies lower emissions~\citep{g2,g4,g6}.

\paragraph{Practical measurement setup.}
During training, we monitor GPU telemetry with \texttt{nvidia-smi}, including memory usage, process-level allocation, temperature, fan speed, and instantaneous power draw. We use \texttt{psutil}~\citep{psutil} to record CPU-side memory and system resource statistics. We also use Python-based tools such as CarbonTracker~\citep{g4} and CodeCarbon~\citep{g8} to estimate energy consumption and carbon emissions during model training~\citep{g6}.

\paragraph{Power and energy formulation.}
Let \(P_{\mathrm{GPU}}(t)\) denote the instantaneous GPU power draw (in watts) at time \(t\), and let \(P_{\mathrm{CPU}}(t)\) denote the corresponding CPU power draw if available.
If measurements are collected at discrete times \(t_1,\dots,t_T\), then the total training energy can be approximated by Eq.~\ref{eq:energy_discrete}.
\begin{equation}
E_{\mathrm{train}}
\;\approx\;
\sum_{i=1}^{T-1}
\bigl(P_{\mathrm{GPU}}(t_i) + P_{\mathrm{CPU}}(t_i)\bigr)\,
\Delta t_i
\label{eq:energy_discrete}
\end{equation}
where \(\Delta t_i = t_{i+1}-t_i\) is measured in hours, so that \(E_{\mathrm{train}}\) is expressed in watt-hours (Wh).
Equivalently, in integral form, we get Eq.~\ref{eq:energy_cont}.
\begin{equation}
E_{\mathrm{train}}
=
\int_0^{T_{\mathrm{train}}}
\bigl(P_{\mathrm{GPU}}(t) + P_{\mathrm{CPU}}(t)\bigr)\,dt
\label{eq:energy_cont}
\end{equation}
If only average GPU power is available, one may use the simpler approximation using Eq.~\ref{eq:energy_avg}.
\begin{equation}
E_{\mathrm{GPU}}
\approx
\bar{P}_{\mathrm{GPU}} \cdot T_{\mathrm{train}}
\label{eq:energy_avg}
\end{equation}
where \(\bar{P}_{\mathrm{GPU}}\) is the average GPU power (W) and \(T_{\mathrm{train}}\) is training time in hours.

\paragraph{Including data-center overhead.}
A more complete estimate may account for infrastructure overhead through Power Usage Effectiveness (PUE), yielding Eq.~\ref{eq:pue_energy}.
\begin{equation}
E_{\mathrm{total}}
=
\mathrm{PUE}\times E_{\mathrm{train}}
\label{eq:pue_energy}
\end{equation}
where \(\mathrm{PUE}\ge 1\) captures additional cooling and facility-level energy costs~\citep{g2,g13}.

\paragraph{Carbon-emission estimate.}
Let \(I_{\mathrm{grid}}\) denote the carbon intensity of the electricity supply in units of \(\mathrm{kg\ CO_2e/kWh}\).
Then the estimated operational carbon footprint of training is calculated using Eq.~\ref{eq:co2e}.
\begin{equation}
\mathrm{CO}_{2}\mathrm{e}
=
E_{\mathrm{total}}^{(\mathrm{kWh})}\times I_{\mathrm{grid}}
\label{eq:co2e}
\end{equation}
where \(E_{\mathrm{total}}^{(\mathrm{kWh})}\) is the total energy converted to kWh. Thus, if two models are trained on the same hardware under the same electricity grid mix, the model with lower total energy consumption will yield lower estimated \(\mathrm{CO}_{2}\mathrm{e}\) emissions~\citep{g2,g4,g6}.

\paragraph{Interpretation in terms of Green AI.}
A model should not be considered ``green'' solely because it exhibits low instantaneous memory usage or low peak power draw. Rather, from the Green AI perspective, a method is more sustainable when it achieves a favorable trade-off between predictive utility and computational cost while reducing total energy consumption, estimated carbon emissions, or both~\citep{g1,g5,g7}. In practice, useful indicators include: (i) lower training energy in Wh or kWh, (ii) lower estimated \(\mathrm{CO}_{2}\mathrm{e}\), (iii) comparable or improved predictive performance per unit energy, (iv) efficient hardware utilization, and (v) transparent reporting of runtime, hardware configuration, energy-measurement methodology, and emissions assumptions. Since there is no single universal numeric threshold for Green AI, the concept is typically operationalized through comparative efficiency, emissions-aware reporting, and environmentally informed model design~\citep{g1,g2,g5,g7,g9}. Recent work has further extended sustainability analysis to specific AI application domains, such as recommender systems, and to broader discussions of the environmental costs of rapid AI innovation~\citep{g10,g11}.
\paragraph{Context.}
Prior work has shown that the carbon cost of large-scale AI training can be substantial. For example,~\citet{g3} estimated that certain large NLP model-development pipelines can emit carbon at a scale comparable to the lifetime emissions of multiple automobiles. While this does not imply that every machine learning experiment is environmentally costly to the same extent, it motivates reporting energy consumption and estimated emissions as part of model evaluation.
\paragraph{Energy-Carbon Pareto Selection.}
To jointly evaluate synthetic utility and resource efficiency, we define an energy-aware Pareto criterion over the set of tabular generative models. For each generator $g \in \mathcal{G}$, let $A_g$ denote the mean TSTR accuracy across datasets, $E_g$ denote the total measured energy consumption in kWh, and $C_g$ denote the total CO$_2$e emission in kg. Since higher TSTR accuracy is preferred while lower energy and carbon cost are preferred, generator $g_i$ is said to dominate generator $g_j$ if it satisfies Eq.~\ref{eq:c1}
\begin{equation}
\label{eq:c1}
A_{g_i} \geq A_{g_j}, \qquad
E_{g_i} \leq E_{g_j}
\end{equation}
with at least one strict inequality. Equivalently, when carbon is also included as an objective, dominance requires Eq.~\ref{eq:c2}.
\begin{equation}
\label{eq:c2}
A_{g_i} \geq A_{g_j}, \qquad
E_{g_i} \leq E_{g_j}, \qquad
C_{g_i} \leq C_{g_j}
\end{equation}
again with at least one strict improvement. The energy-carbon Pareto frontier is therefore the set of non-dominated generators by Eq.~\ref{eq:c3}:
\begin{equation}
\label{eq:c3}
\mathcal{P}_{\mathrm{EC}}
=
\left\{
g \in \mathcal{G}
:
\nexists g' \in \mathcal{G}
\ \text{s.t.}\
A_{g'} \geq A_g,\ 
E_{g'} \leq E_g,\ 
C_{g'} \leq C_g,
\ \text{and at least one inequality is strict}
\right\}
\end{equation}
In our evaluation, BSTabDiff is the only non-dominated method: it achieves the highest mean TSTR accuracy while simultaneously requiring the lowest total energy consumption and producing the lowest total CO$_2$e emissions. Thus, BSTabDiff is selected as the preferred Pareto-optimal operating point.
\paragraph{Interpretation of Energy and Carbon Results.}
The dimensionality-energy (Fig.~\ref{fig:energy_dim}) and dimensionality-CO$_2$e (Fig.~\ref{fig:co2_dim}) plots reveal clear differences in the resource profiles of tabular generative models.
CTGAN~\citep{w16} and TVAE~\citep{w16} exhibit the largest energy and carbon footprints, especially as feature dimensionality increases. BinaryDiff~\citep{bfiff} also shows substantially higher resource consumption than the more lightweight baselines. Classical and lower-cost methods such as SMOTE~\citep{smote}, TabDDPM~\citep{w11}, TabDiff~\citep{tabdiff}, and CtrTab~\citep{ctr} are generally more efficient, but they do not achieve the same synthetic utility as BSTabDiff. BSTabDiff occupies the most favorable region of the trade-off space: it achieves the highest mean TSTR accuracy while maintaining the lowest measured energy and CO$_2$e emissions across the benchmark suite. This indicates that the proposed block-subunit design is not only beneficial for HDLSS tabular synthesis quality, but also computationally efficient and environmentally safer.

\paragraph{Pareto Trade-off Analysis for Energy and $CO_2$e.}
The Pareto plot (Fig.~\ref{fig:pareto_energy}) compares each method using total energy consumption on the horizontal axis and mean TSTR accuracy on the vertical axis, with bubble size representing total CO$_2$e emissions. A desirable method should move toward the upper-left region of the plot: higher synthetic utility and lower energy cost. A method is Pareto-efficient if no other method achieves both higher accuracy and lower energy. Under this criterion, BSTabDiff is the sole Pareto-efficient method in our results.
Although methods such as SMOTE~\citep{smote} and TabDiff~\citep{tabdiff} obtain competitive TSTR accuracy, they consume more energy than BSTabDiff while also achieving lower mean TSTR accuracy. Similarly, TVAE~\citep{w16} and CTGAN~\citep{w16} incur substantially larger energy and carbon costs without improving utility.
Therefore, the Pareto analysis provides a compact summary of the utility-sustainability trade-off and identifies BSTabDiff as the best overall operating point. Table~\ref{tab:energy_carbon_pareto_summary} shows that BSTabDiff achieves the highest mean TSTR accuracy while using the lowest total energy and $CO_2$e, making it the only Pareto-efficient method among the evaluated baselines. We selected the strongest representative baselines for the Pareto-frontier analysis, but observed that GAN-based methods were computationally expensive and that ForestDiffusion~\citep{fdiff} was by far the heaviest overall; even under a light configuration, its dataset-wise energy/$CO_2$e/time costs were COL: 0.039 kWh, 0.035 kg, 646.81 s; LNG: 0.205 kWh, 0.184 kg, 3235.18 s; GLI: 4.625 kWh, 4.145 kg, 71755.59 s; SMK: 4.856 kWh, 4.351 kg, 74935.06 s; AML: 0.457 kWh, 0.410 kg, 7405.28 s; PRS: 0.414 kWh, 0.371 kg, 6741.68 s; ARC: 1.172 kWh, 1.051 kg, 18638.22 s; and TOX: 0.752 kWh, 0.674 kg, 11709.77 s, so we excluded ForestDiffusion~\citep{fdiff}, CTAB-GAN~\citep{ctabgan}, and CTAB-GAN+~\citep{ctabgan+} from the Pareto-frontier analysis to avoid visually compressing the remaining baselines.
\begin{table*}[t]
\centering
\caption{
Energy-carbon Pareto summary across HDLSS datasets. Mean TSTR accuracy is reported in percentage.
Energy is measured in kWh and carbon emission in kg CO$_2$e. Measurements are obtained using CodeCarbon~\citep{g8}.
Light green marks the best predictive performance, light blue marks the lowest energy and carbon footprint, and bold Pareto indicators identify non-dominated models.
}
\label{tab:energy_carbon_pareto_summary}
\scriptsize
\setlength{\tabcolsep}{6pt}
\renewcommand{\arraystretch}{1.12}
\resizebox{\textwidth}{!}{%
\begin{tabular}{lrrrrc}
\toprule
\rowcolor{headergray}
\textbf{Model} 
& \textbf{Mean TSTR Acc. (\%)} 
& \textbf{Total Energy (kWh)} 
& \textbf{Total CO$_2$e (kg)} 
& \textbf{\#Datasets} 
& \textbf{Pareto-efficient} \\
\midrule

\cellcolor{modelhighlight}\textbf{BSTabDiff}
& \bestcell{86.33}
& \effcell{0.003816}
& \effcell{0.003420}
& 8
& \yespareto \\

CtrTab
& 45.77
& 0.029485
& 0.026424
& 8
& \nopareto \\

BinaryDiff
& 71.69
& 0.266292
& 0.238648
& 8
& \nopareto \\

TabDiff
& 81.91
& 0.020414
& 0.018295
& 8
& \nopareto \\

TabDDPM
& 65.67
& 0.017491
& 0.015675
& 8
& \nopareto \\

SMOTE
& 81.80
& 0.010911
& 0.009778
& 8
& \nopareto \\

TVAE
& 80.38
& 2.274939
& 2.038773
& 8
& \nopareto \\

CTGAN
& 57.64
& 3.709245
& 3.324182
& 8
& \nopareto \\

\bottomrule
\end{tabular}%
}
\end{table*}
\paragraph{Energy and carbon footprint analysis for BSTabDiff.}
Fig.~\ref{fig:bstabdiff_energy_co2_configs} summarizes the energy and carbon footprint of BSTabDiff across four configurations. 
The Optuna-tuned no-ordering setting is highly efficient, with most datasets requiring only \(10^{-4}\)-\(10^{-3}\) kWh and sub-\(10^{-3}\) kg CO\(_2\)e, even for high-dimensional HDLSS data. 
The fixed/default no-GO-BS setting increases energy moderately, mainly due to the longer 10k-epoch training schedule, but remains in a low absolute carbon range.  Adding GO-BS introduces extra structure-learning cost, especially on high-dimensional datasets such as GLI and SMK.  GO-BS-FC is usually competitive but shows a clear PRS-specific outlier, suggesting sensitivity of feature-clustered ordering to dataset geometry and clustering cost. 
Overall, the plots show a clear trade-off: ordering-based structure can improve utility and dataset quality (Fig.~\ref{fig:optional_ordering_effects}) but adds computational overhead, with GO-BS providing more controlled cost and GO-BS-FC offering stronger structural bias at the risk of occasional energy outliers.
\begin{figure*}[t]
    \centering

    \begin{subfigure}[t]{0.98\textwidth}
        \centering
        \includegraphics[width=\linewidth]{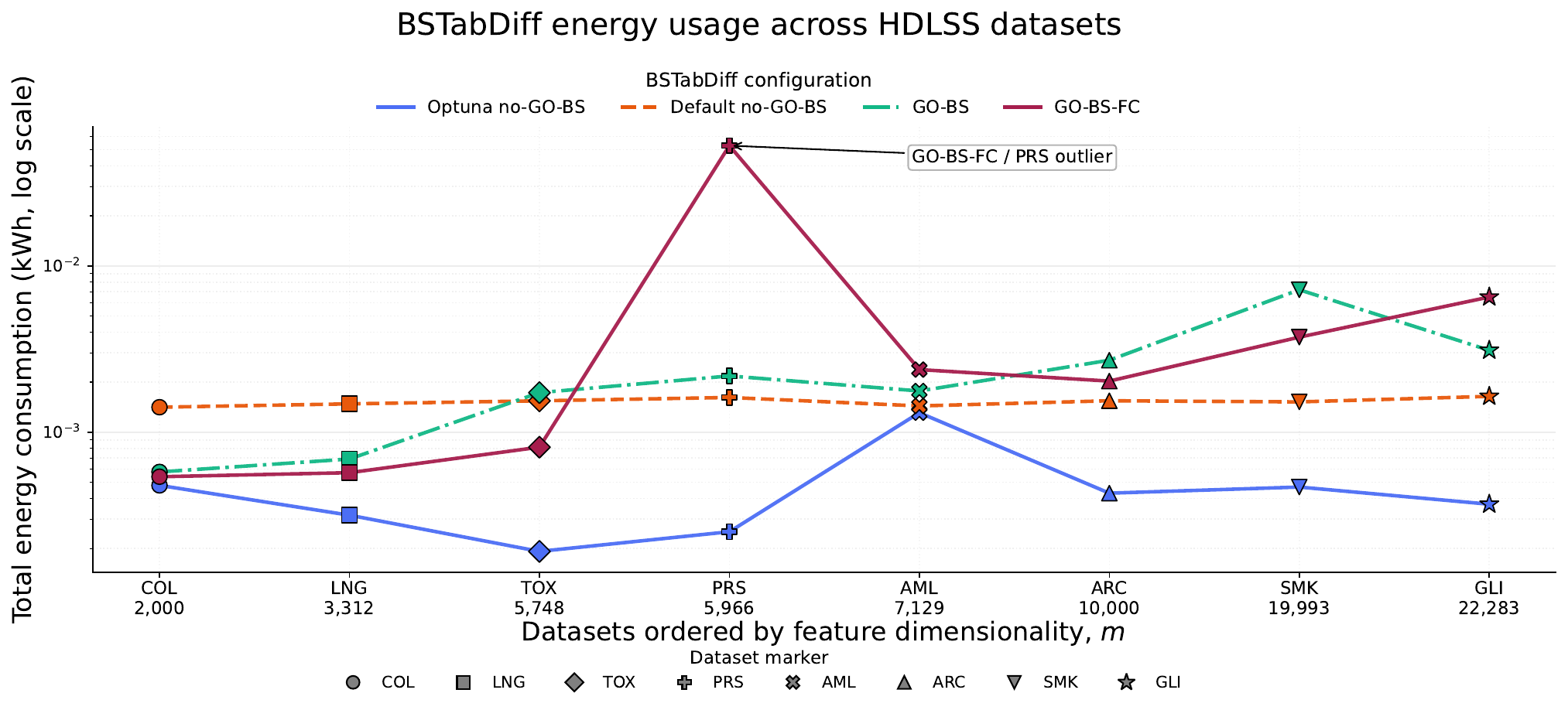}
        \caption{Energy consumption.}
        \label{fig:bstabdiff_energy_clean}
    \end{subfigure}

    \vspace{0.25em}

    \begin{subfigure}[t]{0.98\textwidth}
        \centering
        \includegraphics[width=\linewidth]{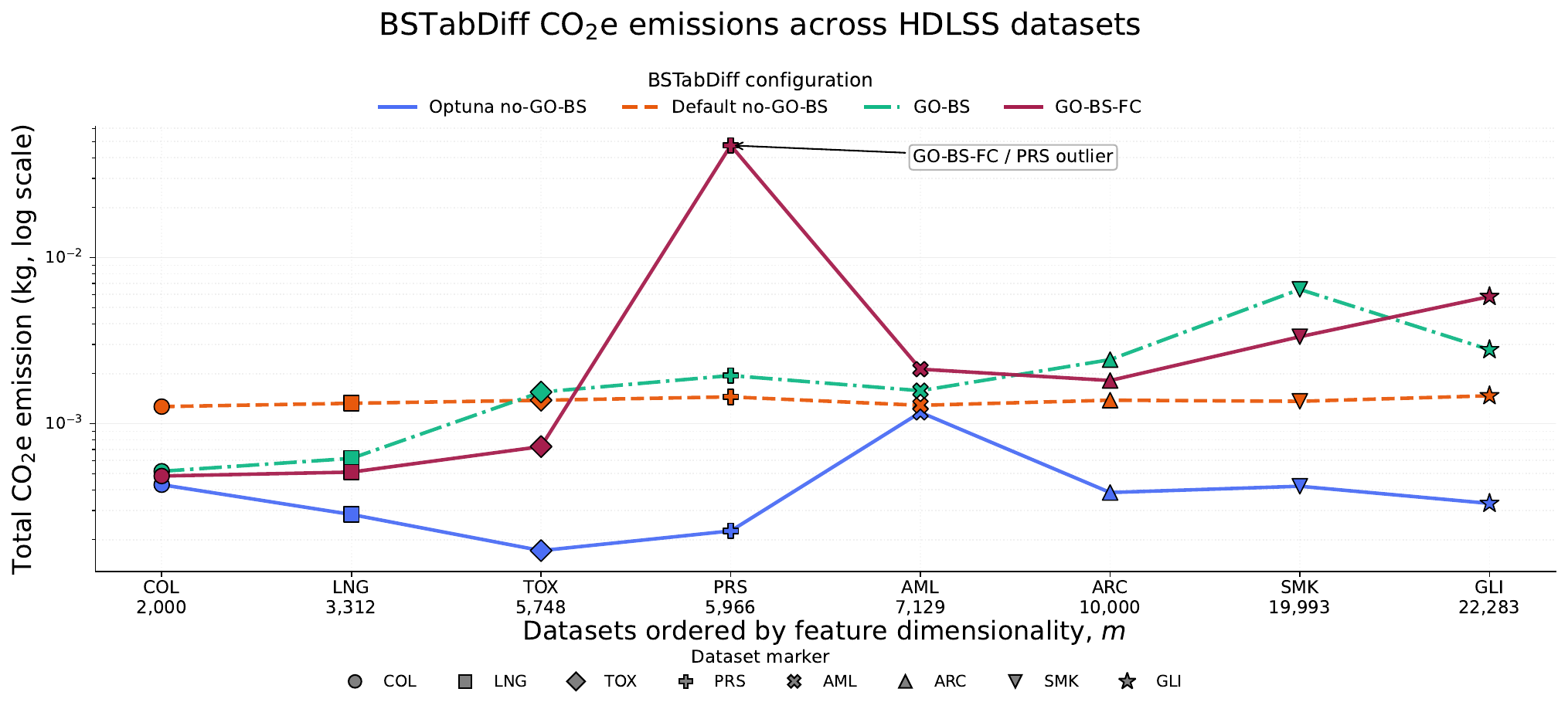}
        \caption{CO$_2$e emissions.}
        \label{fig:bstabdiff_co2_clean}
    \end{subfigure}

    \caption{
    Energy and carbon footprint of BSTabDiff variants across HDLSS datasets ordered by feature dimensionality.
    Curves denote configurations and marker shapes denote datasets. Both axes use log-scale resource measurements.
    }
    \label{fig:bstabdiff_energy_co2_configs}
    \vspace{-0.8em}
\end{figure*}
\paragraph{Pareto Trade-off Analysis for Resource Usage.}
Total energy usage is approximately linearly related to $CO_2$e emissions, while higher memory consumption and longer training time generally increase energy usage and, consequently, $CO_2$e emissions. We further analyze the resource-utility trade-off of top tabular generative models by comparing dataset dimensionality with CPU RSS and GPU memory usage, and by constructing Pareto frontiers over memory/time versus TSTR accuracy. For each generator $g$, we compute the mean TSTR accuracy $\bar{A}_g$, mean training time $\bar{T}_g$, and mean total memory $\overline{\mathrm{Mem}}_g$ across the $D=8$ HDLSS datasets, where we get Eq.~\ref{eq:mem}.
\begin{equation}
    \label{eq:mem}
    \overline{\mathrm{Mem}}_g =
    \frac{1}{D}\sum_{d=1}^{D}
    \left(
    \mathrm{Mem}^{\mathrm{CPU}}_{g,d}
    +
    \mathrm{Mem}^{\mathrm{GPU}}_{g,d}
    \right)
\end{equation}
A generator is Pareto-efficient if no other generator achieves both lower resource usage and higher TSTR accuracy, i.e., generator $g_i$ is dominated by generator $g_j$ when $R_{g_j} \leq R_{g_i}$ and $\bar{A}_{g_j} \geq \bar{A}_{g_i}$, with at least one strict inequality, where \(R_g\) denotes either \(\overline{\mathrm{Mem}}_g\) or \(\bar{T}_g\). As shown in Fig.~\ref{fig:resource_accuracy_tradeoffs} and Table~\ref{tab:resource_pareto_summary}, BSTabDiff is the only generator that remains Pareto-efficient under both resource criteria. It achieves the highest mean TSTR accuracy of $86.33\%$ while requiring only $6.73$ seconds of mean training time and $1.741$ GiB of mean total memory. In contrast, several baselines obtain competitive TSTR accuracy but require substantially larger GPU memory or training time; for example, TabDiff reaches $81.91\%$ mean TSTR accuracy but uses $41.535$ GiB mean total memory, while TVAE and CTGAN require several thousand seconds on average. These results indicate that BSTabDiff provides a stronger accuracy-efficiency balance, especially in HDLSS settings where scaling with feature dimensionality and memory footprint is critical.
\begin{figure*}[!t]
    \centering

    \begin{subfigure}[t]{0.48\textwidth}
        \centering
        \includegraphics[width=\linewidth]{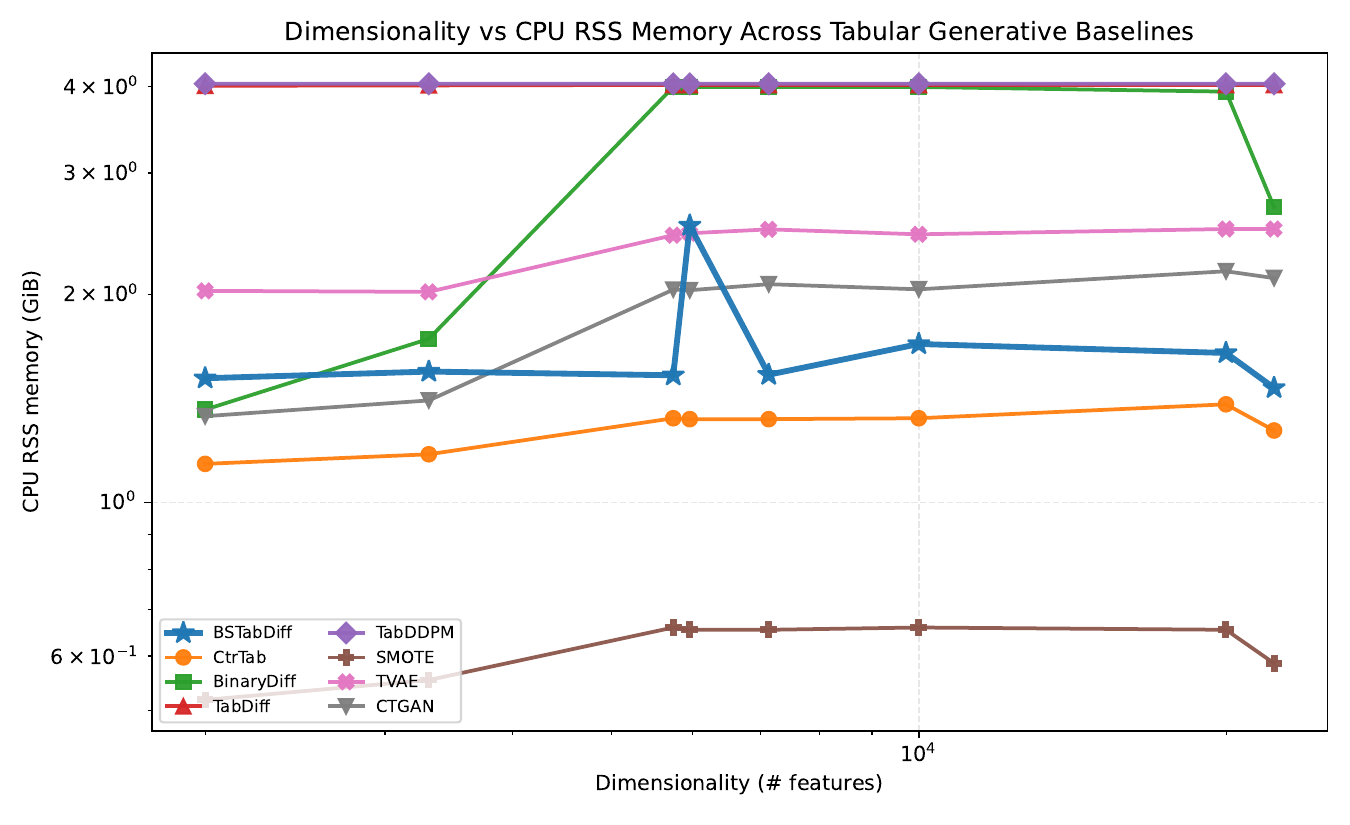}
        \caption{Dimensionality vs. CPU RSS memory.}
        \label{fig:dim_cpu_rss}
    \end{subfigure}
    \hfill
    \begin{subfigure}[t]{0.48\textwidth}
        \centering
        \includegraphics[width=\linewidth]{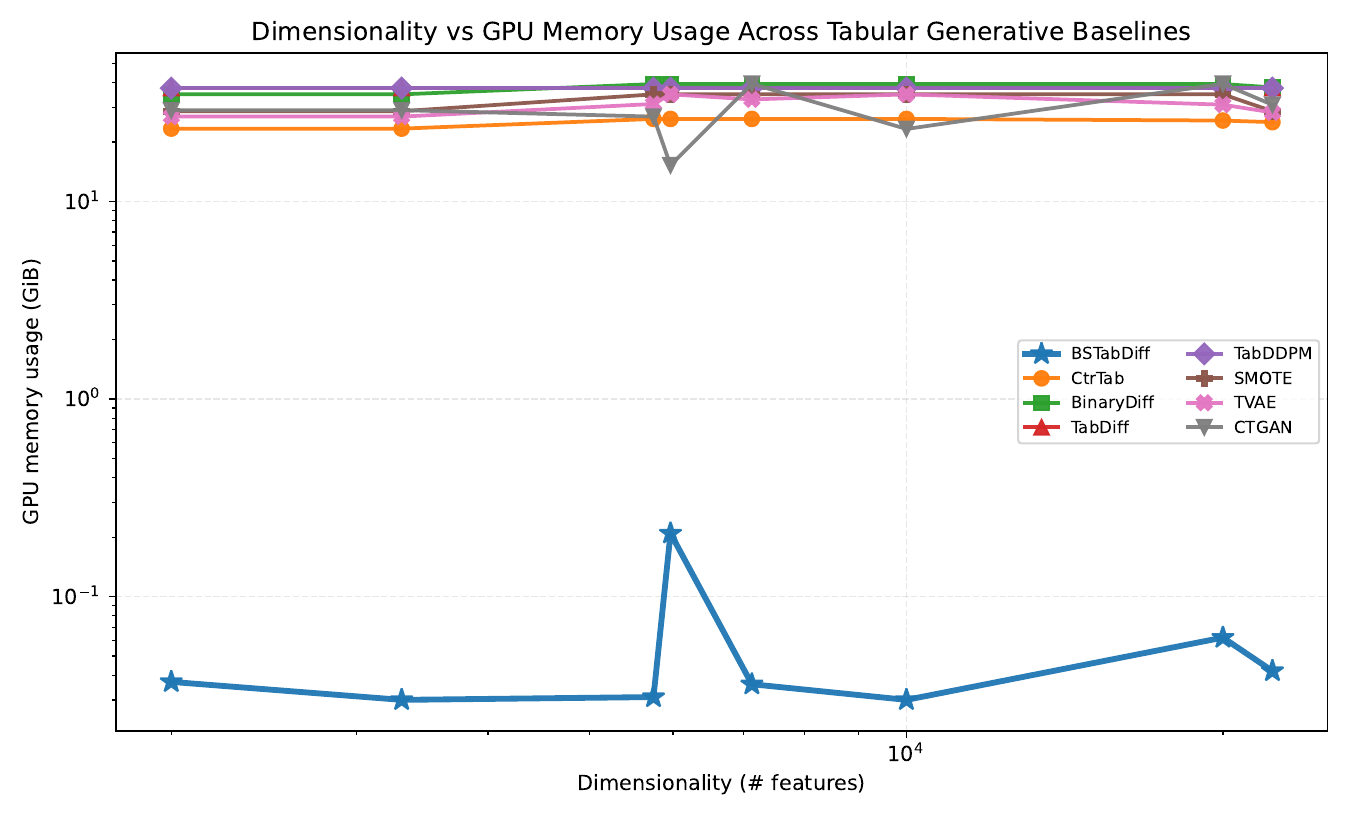}
        \caption{Dimensionality vs. GPU memory usage.}
        \label{fig:dim_gpu_mem}
    \end{subfigure}

    \vspace{4pt}

    \begin{subfigure}[t]{0.48\textwidth}
        \centering
        \includegraphics[width=\linewidth]{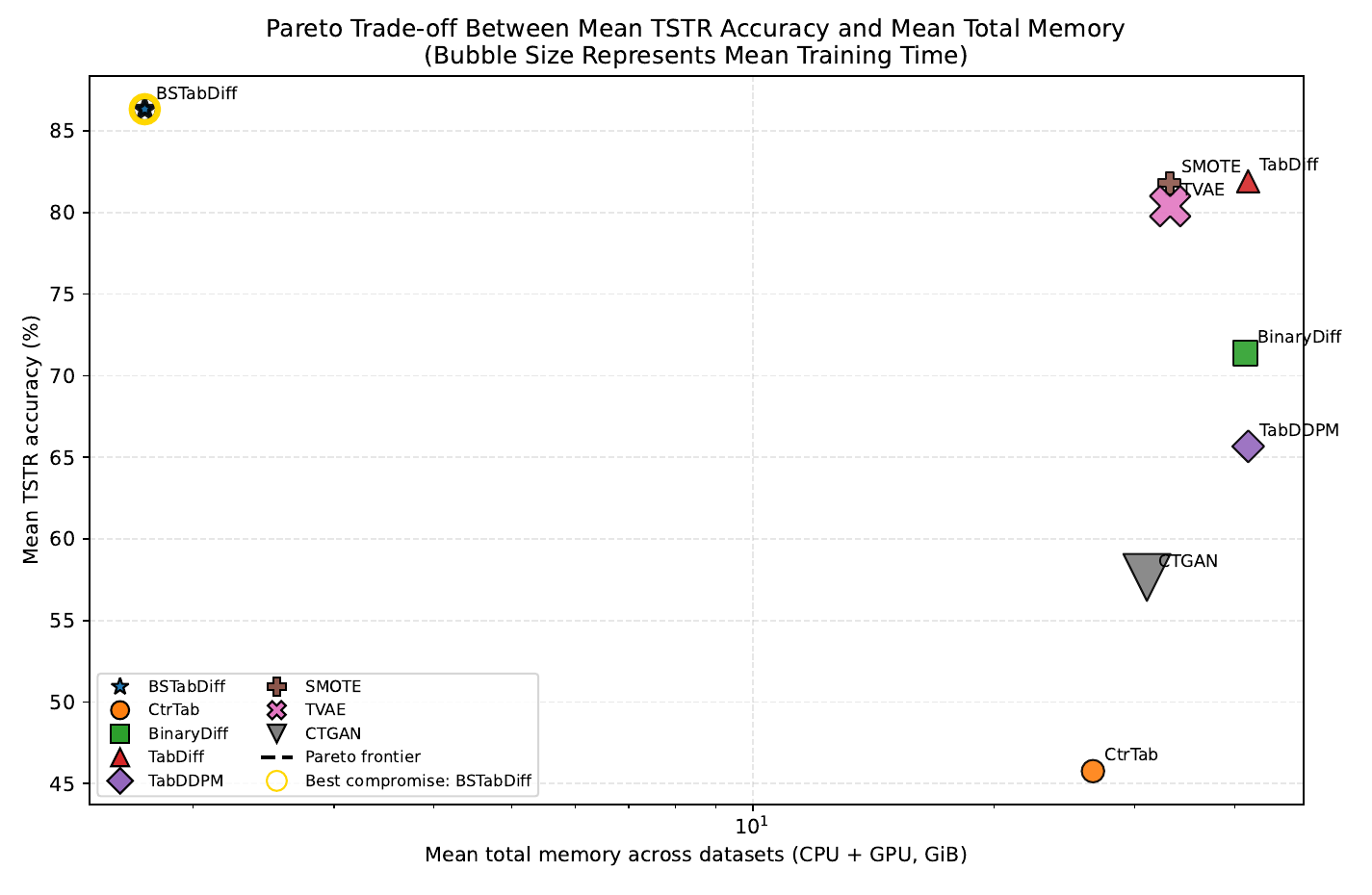}
        \caption{Pareto trade-off between mean total memory and mean TSTR accuracy.}
        \label{fig:pareto_memory_tstr}
    \end{subfigure}
    \hfill
    \begin{subfigure}[t]{0.48\textwidth}
        \centering
        \includegraphics[width=\linewidth]{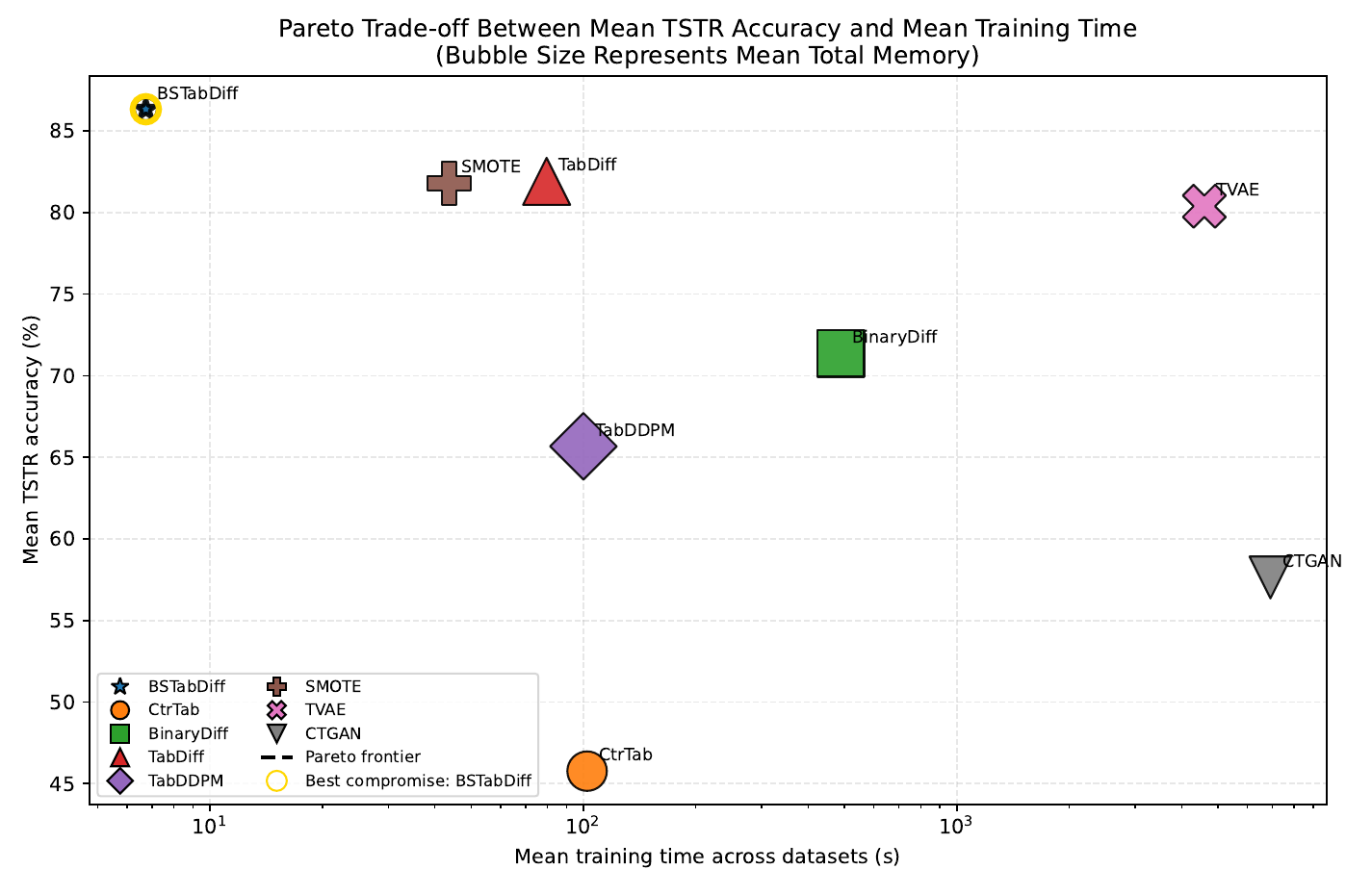}
        \caption{Pareto trade-off between mean training time and mean TSTR accuracy.}
        \label{fig:pareto_time_tstr}
    \end{subfigure}

    \caption{Resource-TSTR accuracy trade-offs across tabular generative baselines. 
    Top row: dataset dimensionality versus CPU RSS memory and GPU memory usage. 
    Bottom row: Pareto frontiers comparing mean TSTR accuracy against mean total memory and mean training time.}
    \label{fig:resource_accuracy_tradeoffs}
\end{figure*}
\begin{table*}[t]
\centering
\caption{
Resource Pareto summary across HDLSS datasets. Mean TSTR accuracy is reported in percentage.
Training time is measured in seconds, and memory usage is measured in GiB.
Total memory is computed as CPU RSS memory plus GPU memory usage.
Light green marks the best predictive performance, light blue marks the most efficient resource values, and bold Pareto indicators identify non-dominated models.
}
\label{tab:resource_pareto_summary}
\scriptsize
\setlength{\tabcolsep}{5pt}
\renewcommand{\arraystretch}{1.12}
\resizebox{\textwidth}{!}{%
\begin{tabular}{lrrrrrrcc}
\toprule
\rowcolor{headergray}
\textbf{Model}
& \textbf{Mean TSTR Acc. (\%)}
& \textbf{Mean Time (s)}
& \textbf{Mean CPU RSS (GiB)}
& \textbf{Mean GPU Mem. (GiB)}
& \textbf{Mean Total Mem. (GiB)}
& \textbf{\#Datasets}
& \textbf{Mem.-Pareto}
& \textbf{Time-Pareto} \\
\midrule

\cellcolor{modelhighlight}\textbf{BSTabDiff}
& \bestcell{86.33}
& \effcell{6.73}
& 1.682
& \effcell{0.060}
& \effcell{1.741}
& 8
& \yespareto
& \yespareto \\

CtrTab
& 45.77
& 102.29
& 1.283
& 25.304
& 26.588
& 8
& \nopareto
& \nopareto \\

BinaryDiff
& 71.37
& 488.55
& 3.212
& 38.017
& 41.229
& 8
& \nopareto
& \nopareto \\

TabDiff
& 81.91
& 79.67
& 4.023
& 37.512
& 41.535
& 8
& \nopareto
& \nopareto \\

TabDDPM
& 65.67
& 100.00
& 4.040
& 37.512
& 41.552
& 8
& \nopareto
& \nopareto \\

SMOTE
& 81.80
& 43.70
& \effcell{0.618}
& 32.537
& 33.155
& 8
& \nopareto
& \nopareto \\

TVAE
& 80.38
& 4599.47
& 2.356
& 30.840
& 33.196
& 8
& \nopareto
& \nopareto \\

CTGAN
& 57.64
& 6914.84
& 1.899
& 29.150
& 31.049
& 8
& \nopareto
& \nopareto \\

\bottomrule
\end{tabular}%
}
\end{table*}
\vspace{-1em}
\renewcommand{\thesection}{A10}
\renewcommand{\thesubsection}{A10.\arabic{subsection}}
\setcounter{figure}{0}\renewcommand{\thefigure}{A10.\arabic{figure}}
\setcounter{table}{0}\renewcommand{\thetable}{A10.\arabic{table}}
\setcounter{algorithm}{0}\renewcommand{\thealgorithm}{A10.\arabic{algorithm}}
\section{Pseudocode}
\label{pseudocode}
\vspace{-1em}

Algorithms~\ref{alg:block_subunit_gen} and~\ref{alg:block_subunit_train} describe the generation and training procedures of BSTabDiff, respectively. Algorithm~\ref{alg:block_subunit_train} first determines the canonical feature order and block partition using GO-BS or GO-BS-FC, and estimates empirical inverse marginal CDFs from observed training values, optionally conditioned on class labels. It then infers block latents $h\in\mathbb{R}^M$ through pooled rank-Gaussianization and standardized within-block averages, fits continuous emissions by least squares and categorical emissions by cross-entropy on observed entries, and estimates missingness probabilities from training masks. Finally, it learns the compact prior $p_\theta(h\mid Y)$ using conditional maximum likelihood for flows or denoising score matching for diffusion. Algorithm~\ref{alg:block_subunit_gen} synthesizes samples by optionally drawing $Y$, sampling $h\sim p_\theta(h\mid Y)$, generating a canonical missingness mask $\tilde R$ from empirical observation probabilities, and decoding each feature through its block subunit $h_t$ using copula-Gaussian emissions and inverse marginals or categorical logits. The generated features and masks are restored to the original feature order using the fixed learned permutation $\hat{\pi}$. Algorithm~\ref{alg:gobs} details GO-BS, which constructs a sparse feature-affinity graph, estimates a global feature ordering, initializes contiguous block boundaries by minimizing cross-block leakage, and refines the ordering and boundaries. GO-BS-FC (Algorithm~\ref{alg:gobsfc}) provides a feature-clustered alternative that clusters features in $X^\top$, constructs local feature orderings, orders the clusters using representative vectors, and concatenates the resulting sequences. The learned permutation and boundaries $(\hat{\pi},\hat{\mathbf b})$ define the canonical block structure used by BSTabDiff, with GO-BS-FC offering a computationally efficient alternative for high-dimensional datasets.

\begin{algorithm}[t]
\begin{tcolorbox}[
    colback=algbackground,
    colframe=algborder,
    boxrule=0.5pt,
    arc=1mm,
    boxsep=3pt,
    left=3pt,
    right=3pt,
    top=4pt,
    bottom=4pt
]

\caption{Block-Subunit HDLSS Tabular Generation}
\label{alg:block_subunit_gen}

\begin{algorithmic}[1]

\Require Canonical block partition $\{\mathcal{S}_t\}_{t=1}^M$, optional class label or label sampler $p(Y)$, learned prior $p_\theta(h\mid Y)$, empirical observation probabilities $\{\hat{\rho}_j(Y)\}$, inverse continuous marginals $\{\hat F^{-1}_{j,Y}\}$, categorical emission parameters $\{W_j,c_j(Y)\}$ (if applicable), and fixed permutation $\hat{\pi}$ (identity if ordering is disabled)

\Ensure Synthetic sample $(X,R,Y)$ in the original feature order

\State Select label $Y$
\Comment{sample or specify; omit if unconditional}

\State Sample block latents
$h=(h_1,\dots,h_M)\sim p_\theta(h\mid Y)$
\Comment{diffusion or flow prior}

\State Initialize canonical vectors
$\tilde X\leftarrow(\mathrm{NA},\ldots,\mathrm{NA})$
and $\tilde R\leftarrow\mathbf{0}_m$

\For{$t=1$ to $M$}
  \For{each feature $j\in\mathcal{S}_t$}

    \State Sample observation indicator
    $\tilde R_j\sim
    \mathrm{Bernoulli}\!\big(\hat{\rho}_j(Y)\big)$

    \If{$\tilde R_j=0$}

      \State $\tilde X_j\leftarrow\mathrm{NA}$

    \Else

      \If{$j$ is continuous}

        \State Sample copula-Gaussian latent
        $Z_j\leftarrow a_jh_t+b_j(Y)+\xi_j$,
        $\xi_j\sim\mathcal{N}(0,\hat{\sigma}_j^2)$

        \State Map to uniform
        $U_j\leftarrow\Phi(Z_j)$

        \State Apply inverse marginal
        $\tilde X_j\leftarrow
        \hat F^{-1}_{j,Y}(U_j)$

      \Else \Comment{categorical / discrete}

        \State Compute logits
        $\ell_j\leftarrow W_jh_t+c_j(Y)$

        \State Sample
        $\tilde X_j\sim
        \mathrm{Categorical}\!\big(
        \mathrm{softmax}(\ell_j)\big)$

      \EndIf

    \EndIf

  \EndFor
\EndFor

\State Restore original feature order:
$X\leftarrow\tilde X_{\hat{\pi}}$,
$R\leftarrow\tilde R_{\hat{\pi}}$

\State \Return $(X,R,Y)$

\end{algorithmic}
\end{tcolorbox}
\end{algorithm}

\begin{algorithm}[t]
\caption{Training BSTabDiff HDLSS Generator (diffusion or flow prior)}
\label{alg:block_subunit_train}

\begin{tcolorbox}[
    colback=algbackground,
    colframe=algborder,
    boxrule=0.5pt,
    arc=1mm,
    boxsep=2pt,
    left=2pt,
    right=2pt,
    top=3pt,
    bottom=3pt
]

\begin{algorithmic}[1]

\Require Training dataset $\mathcal{D}=\{(X^{(i)},R^{(i)},Y^{(i)})\}_{i=1}^{n}$ (labels optional), feature types, block budget $M$, ordering configuration, and prior family $p_\theta(h\mid Y)$ (diffusion or flow)

\Ensure Fitted prior $\theta$, emission parameters $\phi$, empirical marginals, observation probabilities, canonical blocks $\{\mathcal{S}_t\}_{t=1}^{M}$, and permutation $\hat{\pi}$

\If{GO-BS or GO-BS-FC is enabled}

  \State Estimate feature ordering and boundaries
  $(\hat{\pi},\hat{\mathbf b})$
  using the selected GO-BS variant

  \State Construct contiguous canonical blocks
  $\{\mathcal{S}_t\}_{t=1}^{M}$

\Else

  \State Set $\hat{\pi}\leftarrow\mathrm{id}$
  and use supplied or equally divided blocks

\EndIf

\State Reorder features, masks, and feature specifications into canonical space:
$X^\dagger\leftarrow X_{\hat{\pi}^{-1}}$,
$R^\dagger\leftarrow R_{\hat{\pi}^{-1}}$

\State Fit empirical continuous marginals
$\{\hat F_{j,Y}\}$
using observed training values
\Comment{class-conditional if enabled}

\State Retain interpolated inverse marginals
$\{\hat F^{-1}_{j,Y}\}$ for generation

\State Rank-Gaussianize observed continuous features:
$Z_{ij}\leftarrow
\Phi^{-1}\!\left(
\frac{r_{ij}-1/2}{N_j}
\right)$
\Comment{$r_{ij}$ is the observed rank; $N_j$ is the observed count}

\State Use pooled ranks across classes;
set unobserved Gaussianized entries to zero

\For{each training sample $i$ and block $t$}

  \State Infer the scalar block factor
  $\hat h_{it}\leftarrow
  \frac{
  \sum_{j\in\mathcal{S}_t^{\mathrm{cont}}}
  R^\dagger_{ij}Z_{ij}
  }{
  \max\!\left(
  1,\sum_{j\in\mathcal{S}_t^{\mathrm{cont}}}R^\dagger_{ij}
  \right)
  }$

\EndFor

\State Standardize each column of the inferred latent matrix
$\hat H\in\mathbb{R}^{n\times M}$

\State Estimate empirical observation probabilities
$\hat{\rho}_j(Y)$
from training masks
\Comment{class-conditional if available}

\State Fit continuous emission coefficients
$a_j,b_j(Y)$ and residual variances
$\hat{\sigma}_j^2$ by least squares on observed Gaussianized entries

\State Fit categorical emission parameters
$W_j,c_j(Y)$ by cross-entropy on observed categories

\For{each prior-training minibatch $(h,Y)$ from $(\hat H,Y)$}

  \If{flow prior on $h$}

    \State Update $\theta$ by maximizing
    $\log p_\theta(h\mid Y)$
    \Comment{conditional MLE}

  \Else \Comment{diffusion prior on $h$}

    \State Sample timestep $\tau$
    and noise $\epsilon$

    \State Form noisy latent
    $h_\tau$ via $q(h_\tau\mid h)$

    \State Update $\theta$ to minimize
    $\|\epsilon-\epsilon_\theta(h_\tau,\tau,Y)\|_2^2$
    \Comment{score matching}

  \EndIf

\EndFor

\State \Return
$(\theta,\phi,\hat{\pi},
\{\mathcal{S}_t\},
\{\hat F_{j,Y},\hat F^{-1}_{j,Y}\},
\{\hat{\rho}_j(Y)\})$

\end{algorithmic}
\end{tcolorbox}
\end{algorithm}


\begin{algorithm}[t]
\caption{GO-BS: Graph-guided Ordering with Block Segmentation}
\label{alg:gobs}

\begin{tcolorbox}[
    colback=algbackground,
    colframe=algborder,
    boxrule=0.5pt,
    arc=1mm,
    boxsep=2pt,
    left=2pt,
    right=2pt,
    top=3pt,
    bottom=3pt
]

\begin{algorithmic}[1]

\Require Training data $X\in\mathbb{R}^{n\times m}$,
number of blocks $M$, affinity metric $\phi$,
top-$k$ sparsification parameter $k$,
local window size $w$, refinement passes $P$

\Ensure Estimated permutation $\hat{\pi}$
and boundary vector $\hat{\mathbf{b}}$

\State Normalize or standardize feature columns of $X$

\State Build sparse feature-affinity graph
$G=(V,E,W)$ using metric $\phi$

\State Retain top-$k$ neighbors per feature
and symmetrize the graph

\State Compute a coarse ordering
$\pi^{(0)} \gets
\big(\mathrm{SpectralOrder}(G)\big)^{-1}$
\Comment{feature-to-position map}

\State Compute initial boundaries:

\State $\mathbf{b}^{(0)} \gets
\mathrm{SegmentByLeakage}\!\big(
\pi^{(0)},G,M\big)$

\State Initialize accepted solution:
$\hat{\pi}\gets\pi^{(0)}$,
$\hat{\mathbf{b}}\gets\mathbf{b}^{(0)}$

\For{$p=1$ to $P$}

    \State $\pi' \gets
    \mathrm{LocalReorder}\!\big(
    \hat{\pi},\hat{\mathbf{b}},G,w\big)$

    \State $\mathbf{b}' \gets
    \mathrm{BoundaryRefine}\!\big(
    \pi',\hat{\mathbf{b}},G,M,w\big)$

    \If{$\mathcal{J}_{\mathrm{GO\mbox{-}BS}}\!\big(
    \pi',\mathbf{b}'\big)
    \ge
    \mathcal{J}_{\mathrm{GO\mbox{-}BS}}\!\big(
    \hat{\pi},\hat{\mathbf{b}}\big)$}

        \State \textbf{break}

    \EndIf

    \State $\hat{\pi}\gets\pi'$

    \State $\hat{\mathbf{b}}\gets\mathbf{b}'$

\EndFor

\State \Return $(\hat{\pi},\hat{\mathbf{b}})$

\end{algorithmic}

\end{tcolorbox}
\end{algorithm}

\begin{algorithm}[t]
\caption{GO-BS-FC: Feature-Clustered Graph-guided Ordering with Block Segmentation}
\label{alg:gobsfc}

\begin{tcolorbox}[
    colback=algbackground,
    colframe=algborder,
    boxrule=0.5pt,
    arc=1mm,
    boxsep=2pt,
    left=2pt,
    right=2pt,
    top=3pt,
    bottom=3pt
]

\begin{algorithmic}[1]

\Require Training data $X\in\mathbb{R}^{n\times m}$,
number of feature clusters $K_f$,
number of blocks $M$, affinity metric $\phi$,
top-$k$ sparsification parameter $k$,
refinement passes $P$

\Ensure Estimated permutation $\hat{\pi}$
and boundary vector $\hat{\mathbf{b}}$

\State Transpose the data to obtain
feature-as-sample representation
$X^\top \in \mathbb{R}^{m\times n}$

\State Normalize or standardize
feature vectors in $X^\top$

\State Cluster features into $K_f$ groups:

\State $\{\mathcal{F}_1,\ldots,\mathcal{F}_{K_f}\}
\gets \mathrm{FeatureCluster}(X^\top,K_f)$

\For{$c=1$ to $K_f$}

    \State Build sparse within-cluster affinity
    graph $G_c=(V_c,E_c,W_c)$ over features
    in $\mathcal{F}_c$ using metric $\phi$

    \State Retain top-$k$ neighbors per feature
    and symmetrize $G_c$

    \State Compute local feature order
    $\mathbf{o}_c \gets
    \mathrm{GraphOrder}(G_c,P)$

    \State Compute cluster representative
    $r_c \gets
    \mathrm{Aggregate}(X_{:,\mathbf{o}_c})$

\EndFor

\State Build cluster-level affinity graph
$G_{\mathrm{clust}}$ over representatives
$\{r_c\}_{c=1}^{K_f}$

\State Compute cluster ordering
$\boldsymbol{\zeta} \gets
\mathrm{GraphOrder}(G_{\mathrm{clust}},P)$

\State Concatenate local feature orders
according to $\boldsymbol{\zeta}$:

\State $\mathbf{o} \gets
\mathbf{o}_{\zeta_1} \Vert
\mathbf{o}_{\zeta_2} \Vert \cdots \Vert
\mathbf{o}_{\zeta_{K_f}}$

\State Define $\hat{\pi}(\mathbf{o}_r)\gets r$
for $r=1,\ldots,m$
\Comment{feature-to-position map}

\State Build global sparse feature-affinity
graph $G_{\mathrm{feat}}=(V,E,W)$
using metric $\phi$

\State Compute block boundaries:

\State $\hat{\mathbf{b}} \gets
\mathrm{SegmentByLeakage}\!\big(
\hat{\pi},G_{\mathrm{feat}},M\big)$

\State \Return $(\hat{\pi},\hat{\mathbf{b}})$

\end{algorithmic}

\end{tcolorbox}
\end{algorithm}
\vspace{-1em}
\renewcommand{\thesection}{A11}
\renewcommand{\thesubsection}{A11.\arabic{subsection}}
\setcounter{figure}{0}\renewcommand{\thefigure}{A11.\arabic{figure}}
\setcounter{table}{0}\renewcommand{\thetable}{A11.\arabic{table}}
\section{Additional Fidelity Diagnostics}
\label{fid}
\vspace{-1em}
BSTabDiff’s fidelity diagnostics in Table~\ref{tab:fidelity-bstabdiff} show that, on most HDLSS datasets, the generator produces synthetic distributions that are reasonably close to the real data while still avoiding trivial memorization. We first probe marginal fidelity via per-feature KS and W1 distances, which are modest on Colon and ALLAML and somewhat larger on Lung and (especially in scale) GLI-85, indicating that BSTabDiff captures many univariate marginals but struggles more on heavily skewed, high-variance features. Pairwise structure is assessed through Pearson and Spearman $|\Delta\mathrm{corr}|$, which remain in a moderate regime (mean $\approx 0.17$–0.22) across datasets, suggesting that the generator preserves a substantial fraction of the real correlation graph rather than collapsing to independent noise. Higher-order moments (HOM: mean, variance, skewness, kurtosis) further reveal that Colon, ALLAML, and Lung have relatively small average discrepancies, whereas GLI-85 exhibits inflated variance- and mean-scale differences, consistent with its extreme HDLSS regime and raw feature scaling. Label-conditional structure is evaluated via $|\Delta \mathrm{MI}(\text{feature}, y)|$, where Colon and ALLAML show moderate MI gaps, and Lung has a larger mean MI discrepancy, implying that, for Lung, synthetic samples only coarsely approximate the most discriminative features. We also monitor NaN rates, which are essentially zero in the real data and remain at $\mathcal{O}(10^{-4})$ in the synthetic tables, and a nearest neighbor privacy diagnostic (Priv.), where large NN distances between synthetic and real points indicate that the generator does not simply copy training samples. Finally, the C2ST accuracy/AUC quantifies how easily a classifier can distinguish real from synthetic data: for Colon, ALLAML, and GLI-85, C2ST performance is at or even below chance, indicating that simple discriminators do not find a stable separating signal across folds, whereas Lung exhibits a much stronger C2ST signal, highlighting it as the most challenging dataset where BSTabDiff leaves a clearer “synthetic” footprint despite still providing useful downstream TSTR performance.
\begin{table}[t]
\centering
\scriptsize
\caption{
Fidelity diagnostics for BSTabDiff synthetic data on four HDLSS datasets under the default configuration. We report marginal distances (KS, W1), correlation discrepancies, higher-order moment (HOM) gaps, label-feature MI gaps, NaN rates, nearest neighbor privacy distance, and C2ST mean\(\pm\)std over \(5\times5\) CV. Here, KS = Kolmogorov-Smirnov, W1 = Wasserstein-1, MI = mutual information, and Priv. = synthetic-to-real nearest neighbor distance.
}
\label{tab:fidelity-bstabdiff}
\renewcommand{\arraystretch}{1.18}
\resizebox{\linewidth}{!}{%
\begin{tabular}{|l|l|c|c|c|c|}
\hline
\rowcolor{fidelityheader}
\textbf{Metric} & \textbf{Statistic} & \textbf{COL} & \textbf{AML} & \textbf{LNG} & \textbf{GLI} \\
\hline

\rowcolor{fidelityKS}
\textbf{KS (per feature)} & mean & 0.0848 & 0.1047 & 0.2001 & 0.1592 \\
\cline{2-6}
\rowcolor{fidelityKS}
                 & max  & 0.2387 & 0.2917 & 0.6650 & 0.4471 \\
\hline

\rowcolor{fidelityW}
\textbf{W1} & mean & 0.1252 & 0.1563 & 0.1036 & 3.76$\times 10^{2}$ \\
\cline{2-6}
\rowcolor{fidelityW}
   & max  & 0.3910 & 0.8026 & 0.6745 & 1.63$\times 10^{4}$ \\
\hline

\rowcolor{fidelityPearson}
\textbf{Pearson $|\Delta\mathrm{corr}|$} & mean & 0.2183 & 0.1874 & 0.2072 & 0.1695 \\
\cline{2-6}
\rowcolor{fidelityPearson}
                                & max  & 1.1046 & 1.2400 & 0.9875 & 1.3613 \\
\hline

\rowcolor{fidelitySpearman}
\textbf{Spearman $|\Delta\mathrm{corr}|$} & mean & 0.2179 & 0.1893 & 0.2110 & 0.1705 \\
\cline{2-6}
\rowcolor{fidelitySpearman}
                                 & max  & 1.0889 & 1.1925 & 1.0696 & 1.2419 \\
\hline

\rowcolor{fidelityHOM}
\textbf{HOM} 
& mean   & 7.99$\times 10^{-2}$ & 8.94$\times 10^{-2}$ & 7.88$\times 10^{-2}$ & 3.19$\times 10^{2}$ \\
\cline{2-6}
\rowcolor{fidelityHOM}
& var    & 1.55$\times 10^{-1}$ & 2.53$\times 10^{-1}$ & 5.08$\times 10^{-2}$ & 2.40$\times 10^{6}$ \\
\cline{2-6}
\rowcolor{fidelityHOM}
& skew   & 0.1666 & 0.4664 & 0.4580 & 0.5397 \\
\cline{2-6}
\rowcolor{fidelityHOM}
& kurt   & 0.4819 & 2.3330 & 2.2799 & 2.8462 \\
\hline

\rowcolor{fidelityMI}
\textbf{$|\Delta \mathrm{MI}(\text{feature}, y)|$} 
& mean & 0.2218 & 0.1656 & 0.6274 & 0.1661 \\
\cline{2-6}
\rowcolor{fidelityMI}
& max  & 0.5115 & 0.5435 & 1.1457 & 0.5458 \\
\hline

\rowcolor{fidelityNaN}
\textbf{NaN fraction} & real      
& 0.0000 & 0.0000 & 0.0000 & 0.0000 \\
\cline{2-6}
\rowcolor{fidelityNaN}
            & synthetic  
& 1.10$\times 10^{-4}$ & 1.03$\times 10^{-4}$ & 1.13$\times 10^{-4}$ & 1.02$\times 10^{-4}$ \\
\hline

\rowcolor{fidelityPrivacy}
\textbf{Priv.} 
& mean & 4.97$\times 10^{1}$ & 9.21$\times 10^{1}$ & 1.53$\times 10^{1}$ & 4.67$\times 10^{5}$ \\
\cline{2-6}
\rowcolor{fidelityPrivacy}
& min  & 4.68$\times 10^{1}$ & 8.42$\times 10^{1}$ & 1.20$\times 10^{1}$ & 4.27$\times 10^{5}$ \\
\cline{2-6}
\rowcolor{fidelityPrivacy}
& p1   & 4.72$\times 10^{1}$ & 8.43$\times 10^{1}$ & 1.21$\times 10^{1}$ & 4.28$\times 10^{5}$ \\
\cline{2-6}
\rowcolor{fidelityPrivacy}
& p5   & 4.77$\times 10^{1}$ & 8.54$\times 10^{1}$ & 1.23$\times 10^{1}$ & 4.36$\times 10^{5}$ \\
\hline

\rowcolor{fidelityC2ST}
\textbf{C2ST ACC} & mean $\pm$ std 
& 0.2880 $\pm$ 0.1023 
& 0.2600 $\pm$ 0.1149 
& 0.7401 $\pm$ 0.0437 
& 0.2286 $\pm$ 0.1059 \\
\hline

\rowcolor{fidelityC2ST}
\textbf{C2ST AUC} & mean $\pm$ std 
& 0.1614 $\pm$ 0.0795 
& 0.1536 $\pm$ 0.1285 
& 0.7029 $\pm$ 0.0610 
& 0.1312 $\pm$ 0.0954 \\
\hline
\end{tabular}%
}
\end{table}
Table~\ref{tab:fidelity-bstabdiff-noperm} reports fidelity diagnostics for BSTabDiff under the Optuna-tuned hyperparameter setting without feature permutation. Overall, the results indicate that BSTabDiff preserves key statistical structure reasonably well even in challenging HDLSS regimes. In particular, Colon and ALLAML show consistently small marginal discrepancies, moderate correlation errors, and stable higher-order moment matching, suggesting that the generator captures both first-order and dependency structure well. Although Lung remains the most difficult case, the model still maintains finite Wasserstein and rank-correlation discrepancies with negligible NaN rates, indicating stable generation despite the stronger label-related complexity. GLI exhibits large scale-sensitive Wasserstein, moment, and privacy values, which is expected given its extreme dimensionality and feature magnitude, but its relatively small KS and correlation discrepancies still suggest that the synthetic data preserves substantial distributional structure at the normalized relational level. Across all four datasets, the near-zero NaN fractions further confirm that BSTabDiff remains numerically robust under the tuned no-permutation configuration.
\begin{table}[t]
\centering
\scriptsize
\caption{
Fidelity diagnostics for BSTabDiff synthetic data under Optuna-tuned hyperparameters without feature permutation. We report marginal distances (KS, W1), correlation discrepancies, higher-order moment (HOM) gaps, label-feature MI gaps, NaN rates, and synthetic-to-real nearest neighbor privacy distance. Here, KS = Kolmogorov-Smirnov, W1 = Wasserstein-1, MI = mutual information, and Priv. = privacy distance.
}
\label{tab:fidelity-bstabdiff-noperm}

\renewcommand{\arraystretch}{1.18}

\resizebox{\linewidth}{!}{%
\begin{tabular}{|l|l|c|c|c|c|}
\hline

\rowcolor{fidelityheader}
\textbf{Metric} &
\textbf{Statistic} &
\textbf{COL} &
\textbf{AML} &
\textbf{LNG} &
\textbf{GLI} \\
\hline

\rowcolor{fidelityKS}
\textbf{KS (per feature)}
& mean & 0.1126 & 0.0945 & 0.1988 & 0.0904 \\
\cline{2-6}
\rowcolor{fidelityKS}
& max & 0.3326 & 0.3111 & 0.6527 & 0.2859 \\
\hline

\rowcolor{fidelityW}
\textbf{W1}
& mean & 0.1753 & 0.1450 & 0.1108
& 1.98$\times 10^{2}$ \\
\cline{2-6}
\rowcolor{fidelityW}
& max & 0.5513 & 0.7975 & 0.6013
& 1.11$\times 10^{4}$ \\
\hline

\rowcolor{fidelityPearson}
\textbf{Pearson $|\Delta\mathrm{corr}|$}
& mean & 0.2178 & 0.1516 & - & 0.1460 \\
\cline{2-6}
\rowcolor{fidelityPearson}
& max & 1.0592 & 1.2636 & - & 1.3425 \\
\hline

\rowcolor{fidelitySpearman}
\textbf{Spearman $|\Delta\mathrm{corr}|$}
& mean & 0.2167 & 0.1563 & 0.1767 & 0.1485 \\
\cline{2-6}
\rowcolor{fidelitySpearman}
& max & 1.0652 & 0.9608 & 0.9693 & 1.1352 \\
\hline

\rowcolor{fidelityHOM}
\textbf{HOM}
& mean
& 1.43$\times 10^{-1}$
& 1.08$\times 10^{-1}$
& 7.66$\times 10^{-2}$
& 1.48$\times 10^{2}$ \\
\cline{2-6}

\rowcolor{fidelityHOM}
& var
& 1.49$\times 10^{-1}$
& 2.00$\times 10^{-1}$
& 6.44$\times 10^{-2}$
& 1.61$\times 10^{6}$ \\
\cline{2-6}

\rowcolor{fidelityHOM}
& skew
& 0.2091 & 0.3620 & 0.3852 & 0.3474 \\
\cline{2-6}

\rowcolor{fidelityHOM}
& kurt
& 0.5005 & 1.9362 & 1.9097 & 2.0479 \\
\hline

\rowcolor{fidelityMI}
\textbf{$|\Delta \mathrm{MI}(\text{feature}, y)|$}
& mean & 0.2228 & 0.2165 & 0.5730 & 0.2117 \\
\cline{2-6}

\rowcolor{fidelityMI}
& max & 0.4818 & 0.5274 & 1.0507 & 0.5161 \\
\hline

\rowcolor{fidelityNaN}
\textbf{NaN fraction}
& real & 0.0000 & 0.0000 & 0.0000 & 0.0000 \\
\cline{2-6}

\rowcolor{fidelityNaN}
& synthetic
& 1.10$\times 10^{-4}$
& 1.11$\times 10^{-4}$
& 9.81$\times 10^{-5}$
& 9.83$\times 10^{-5}$ \\
\hline

\rowcolor{fidelityPrivacy}
\textbf{Priv.}
& mean
& 5.03$\times 10^{1}$
& 9.30$\times 10^{1}$
& 1.73$\times 10^{1}$
& 4.82$\times 10^{5}$ \\
\cline{2-6}

\rowcolor{fidelityPrivacy}
& min
& 4.69$\times 10^{1}$
& 7.79$\times 10^{1}$
& 1.28$\times 10^{1}$
& 4.33$\times 10^{5}$ \\
\cline{2-6}

\rowcolor{fidelityPrivacy}
& p1
& 4.75$\times 10^{1}$
& 7.91$\times 10^{1}$
& 1.33$\times 10^{1}$
& 4.35$\times 10^{5}$ \\
\cline{2-6}

\rowcolor{fidelityPrivacy}
& p5
& 4.84$\times 10^{1}$
& 8.04$\times 10^{1}$
& 1.36$\times 10^{1}$
& 4.41$\times 10^{5}$ \\
\hline

\end{tabular}%
}
\end{table}
\vspace{-1em}
\renewcommand{\thesection}{A12}
\renewcommand{\thesubsection}{A12.\arabic{subsection}}
\setcounter{figure}{0}\renewcommand{\thefigure}{A12.\arabic{figure}}
\setcounter{table}{0}\renewcommand{\thetable}{A12.\arabic{table}}
\section{Additional Ablation Studies}
\label{abla}
\vspace{-1em}
Table~\ref{tab:bstabdiff_compact_ablations} and Fig.~\ref{fig:bstabdiff_ablation_4rows} report compact ablations of BSTabDiff on Colon evaluated with TabPFN-2.5. Overall, BSTabDiff remains robust: most variants achieve high downstream utility (TSTR AUC $\approx$ 0.87-0.90) and consistently track the real data upper bound (TRTR AUC $\approx$ 0.93), yielding strong efficiency ratios (TSTR/TRTR $\approx$ 0.94-0.96). The best TSTR is obtained with more synthetic samples (\texttt{nSyn500}), while changing the latent prior (FlowPrior), removing EMA, or varying the number of blocks ($M=16/64$) has only minor impact on utility. In contrast, the discriminator signal (C2ST AUC) is more sensitive to these choices, suggesting that some settings produce synthetic data that is easier to distinguish even when predictive utility remains comparable.
\begin{table}[t]
\caption{
Key results for BSTabDiff compact ablations on Colon using TabPFN-2.5. Values are the mean $\pm$ standard deviation over the evaluation folds. Efficiency is defined as the AUC ratio of TSTR/TRTR. Lower C2ST AUC indicates that real and synthetic data are harder to distinguish in this experiment.
}
\label{tab:bstabdiff_compact_ablations}
\centering
\small
\setlength{\tabcolsep}{5pt}
\renewcommand{\arraystretch}{1.22}

\begin{tabular}{@{}lcccc@{}}
\toprule

\cellcolor{ablheadsetting}\textbf{Setting} &
\cellcolor{ablheadtstr}\textbf{TSTR AUC $\uparrow$} &
\cellcolor{ablheadtrtr}\textbf{TRTR AUC $\uparrow$} &
\cellcolor{ablheadc2st}\textbf{C2ST AUC $\downarrow$} &
\cellcolor{ablheadeff}\shortstack{
\textbf{Eff. (TSTR/TRTR)}\\
$\uparrow$
} \\

\midrule

\rowcolor{ablbase}
\textbf{BASE}
& 0.886 $\pm$ 0.075
& 0.932 $\pm$ 0.072
& 0.101 $\pm$ 0.058
& 0.950 \\

\midrule

FlowPrior
& 0.896 $\pm$ 0.078
& 0.931 $\pm$ 0.070
& 0.180 $\pm$ 0.100
& \cellcolor{ablwinpurple}\textbf{0.963} \\

\rowcolor{ablstripe}
noEMA
& 0.886 $\pm$ 0.075
& 0.931 $\pm$ 0.072
& 0.192 $\pm$ 0.113
& 0.952 \\

noCC-Marg
& 0.872 $\pm$ 0.098
& 0.930 $\pm$ 0.070
& \cellcolor{ablwinpeach}\textbf{0.070 $\pm$ 0.030}
& 0.938 \\

\rowcolor{ablstripe}
noCC-Miss
& 0.889 $\pm$ 0.078
& 0.931 $\pm$ 0.070
& 0.190 $\pm$ 0.114
& 0.954 \\

M16
& 0.886 $\pm$ 0.075
& 0.932 $\pm$ 0.072
& 0.188 $\pm$ 0.111
& 0.951 \\

\rowcolor{ablstripe}
M64
& 0.896 $\pm$ 0.078
& 0.932 $\pm$ 0.072
& 0.205 $\pm$ 0.110
& 0.961 \\

ep1500
& 0.896 $\pm$ 0.078
& 0.931 $\pm$ 0.070
& 0.236 $\pm$ 0.137
& 0.962 \\

\rowcolor{ablstripe}
ep5000
& 0.893 $\pm$ 0.076
& 0.932 $\pm$ 0.072
& 0.170 $\pm$ 0.089
& 0.958 \\

nSyn100
& 0.892 $\pm$ 0.078
& 0.931 $\pm$ 0.070
& 0.239 $\pm$ 0.117
& 0.958 \\

\rowcolor{ablstripe}
nSyn500
& \cellcolor{ablwinmint}\textbf{0.901 $\pm$ 0.080}
& \cellcolor{ablwinblue}\textbf{0.941 $\pm$ 0.075}
& 0.223 $\pm$ 0.104
& 0.958 \\

\bottomrule
\end{tabular}
\end{table}
\begin{figure*}[t]
  \centering
  \newcommand{\panelhA}{4.0cm}
  \newcommand{\panelhB}{4.0cm}
  \newcommand{\panelhC}{3.2cm}

  \begin{subfigure}[t]{\textwidth}
    \centering
    \includegraphics[height=\panelhA,keepaspectratio]{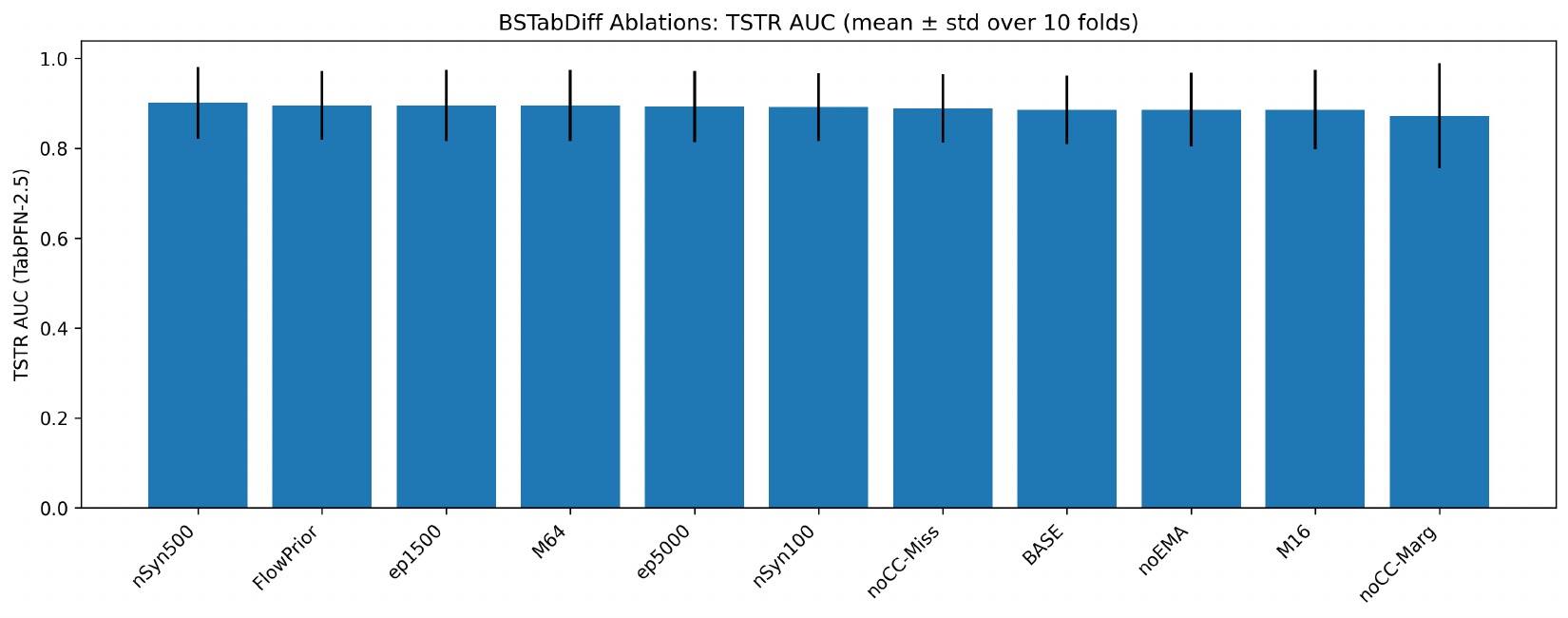}
    \caption{TSTR AUC (compact ablations).}
  \end{subfigure}

  \vspace{0.7em}

  \begin{subfigure}[t]{\textwidth}
    \centering
    \includegraphics[height=\panelhB,keepaspectratio]{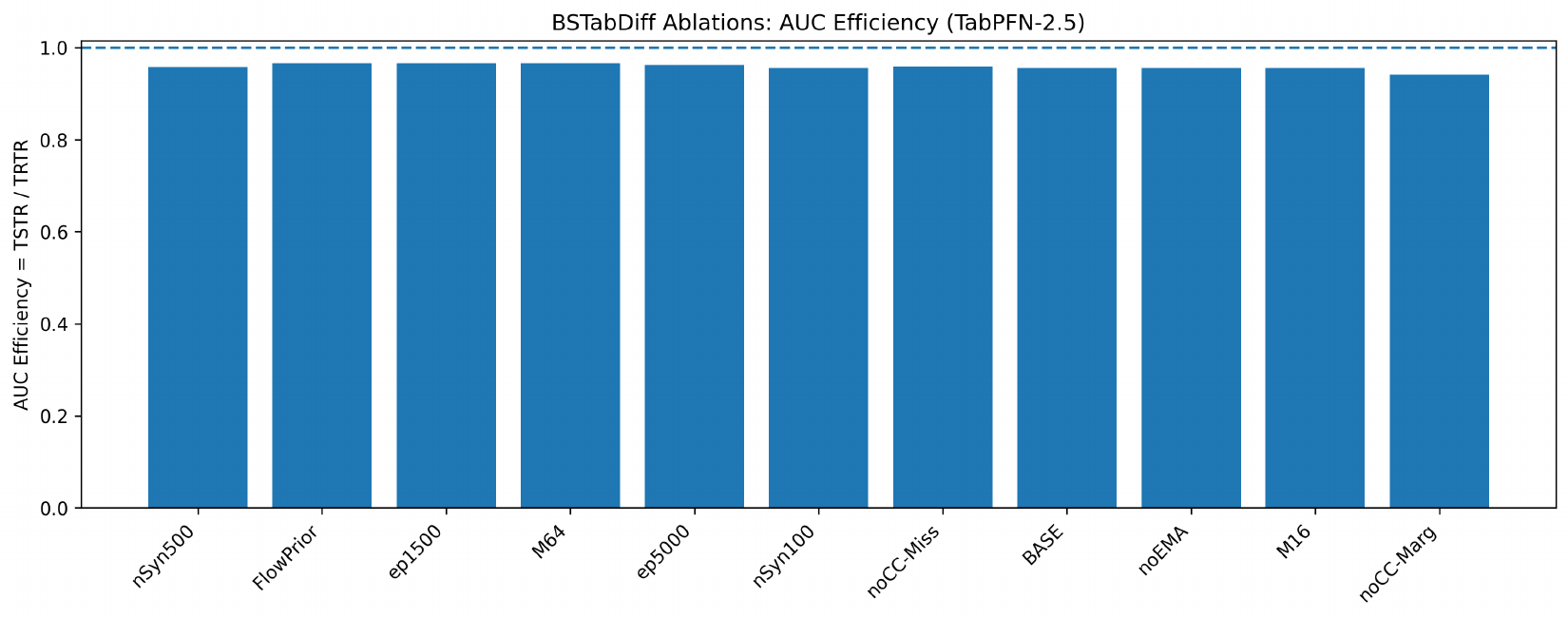}
    \caption{AUC efficiency (TSTR/TRTR).}
  \end{subfigure}

  \vspace{0.7em}

  \begin{subfigure}[t]{0.5\textwidth}
    \centering
    \includegraphics[height=\panelhC,keepaspectratio]{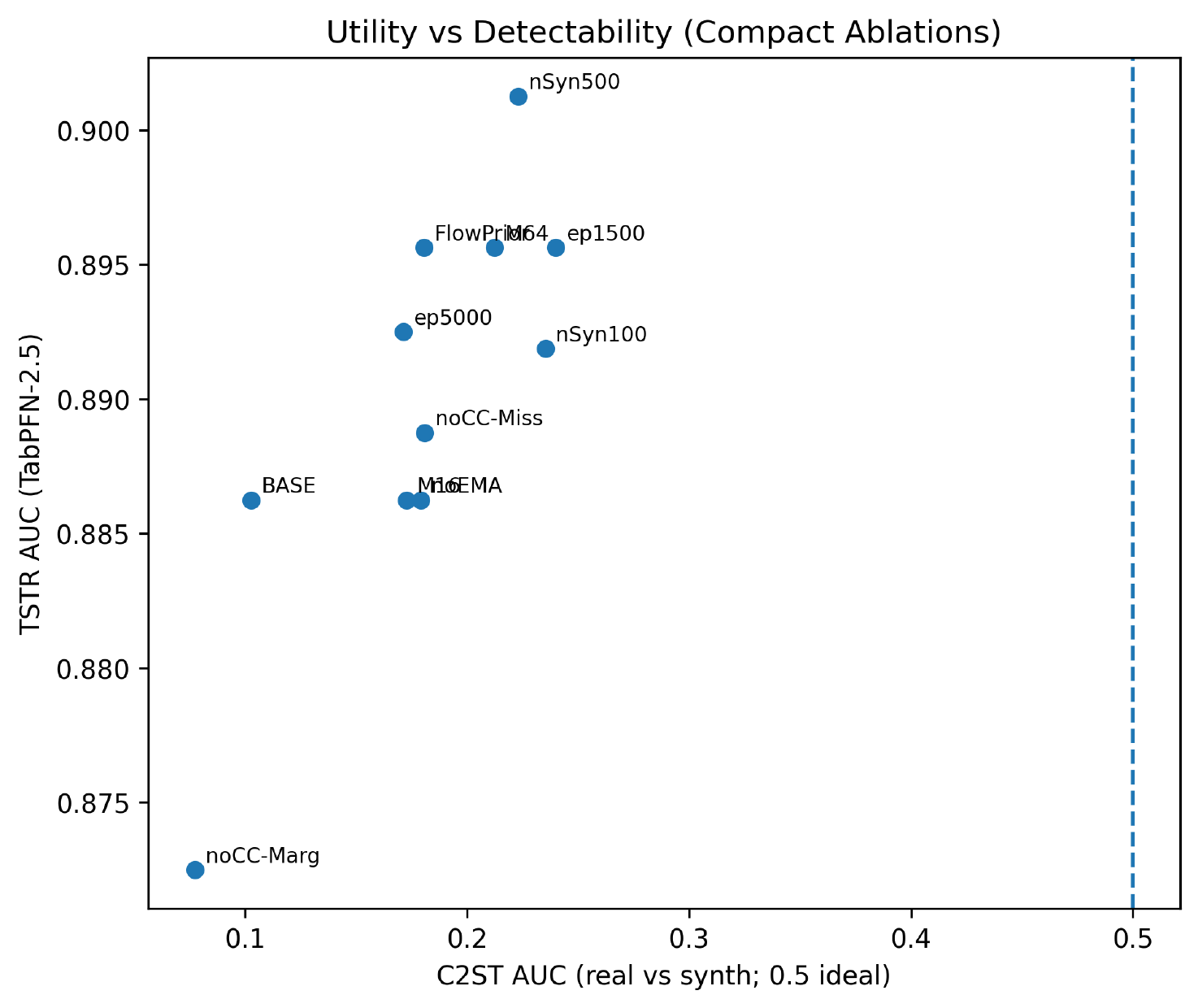}
    \caption{Utility vs detectability (TSTR vs C2ST).}
  \end{subfigure}\hfill
  \begin{subfigure}[t]{0.5\textwidth}
    \centering
    \includegraphics[height=\panelhC,keepaspectratio]{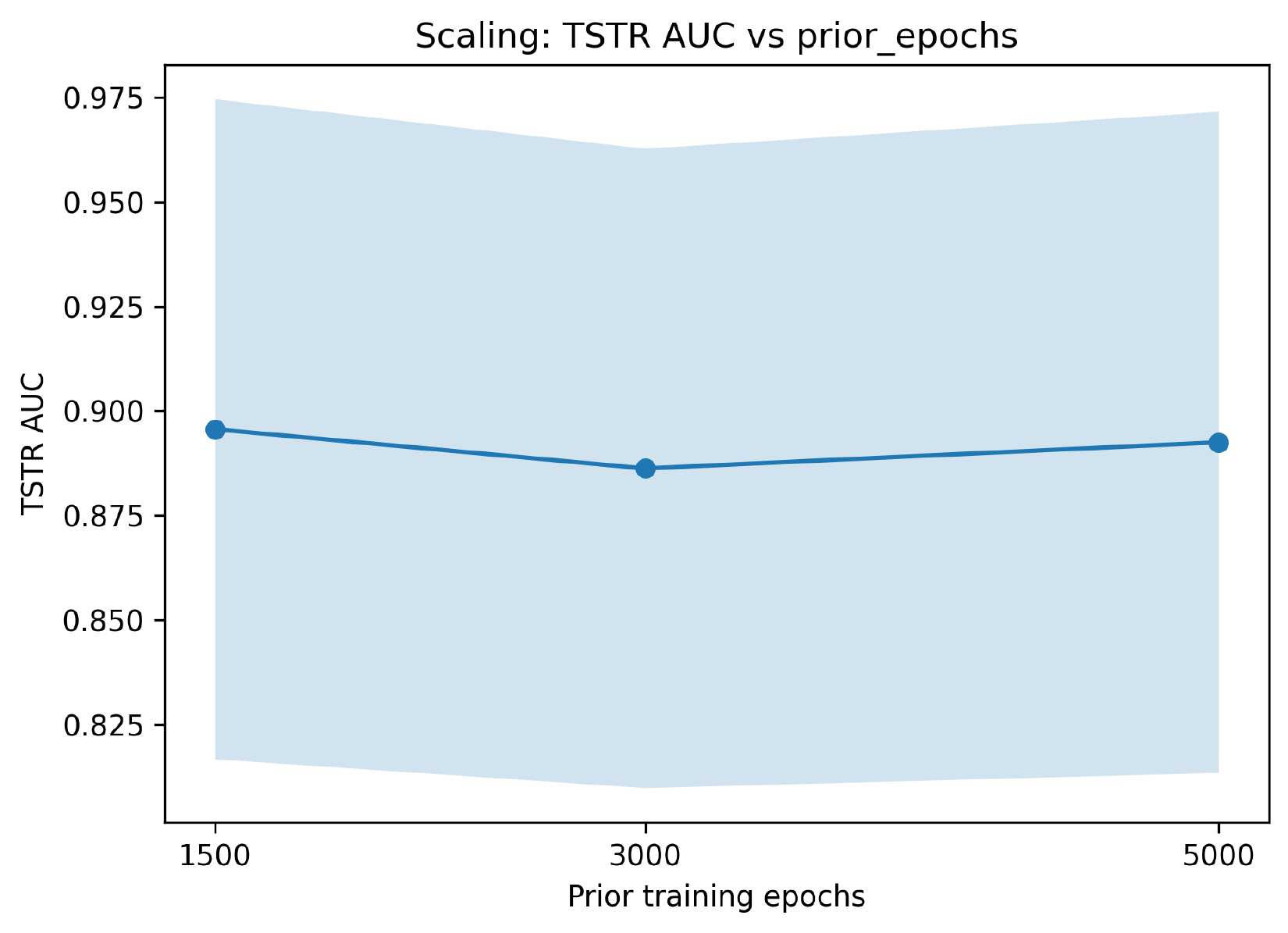}
    \caption{Scaling: TSTR AUC vs prior epochs.}
  \end{subfigure}

  \vspace{0.6em}

  \begin{subfigure}[t]{0.50\textwidth}
    \centering
    \includegraphics[height=\panelhC,keepaspectratio]{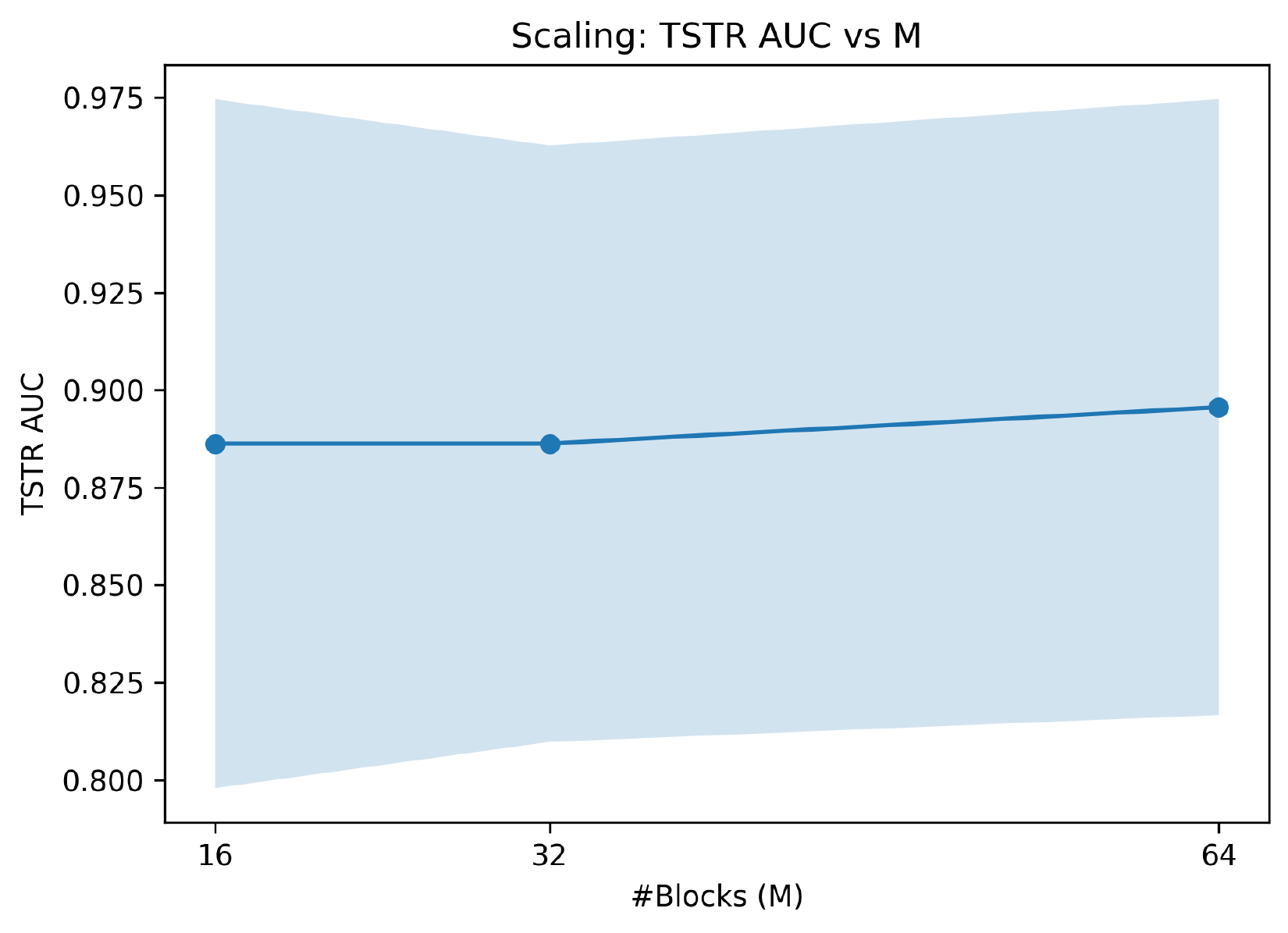}
    \caption{Scaling: TSTR AUC vs \#blocks $M$.}
  \end{subfigure}\hfill
  \begin{subfigure}[t]{0.50\textwidth}
    \centering
    \includegraphics[height=\panelhC,keepaspectratio]{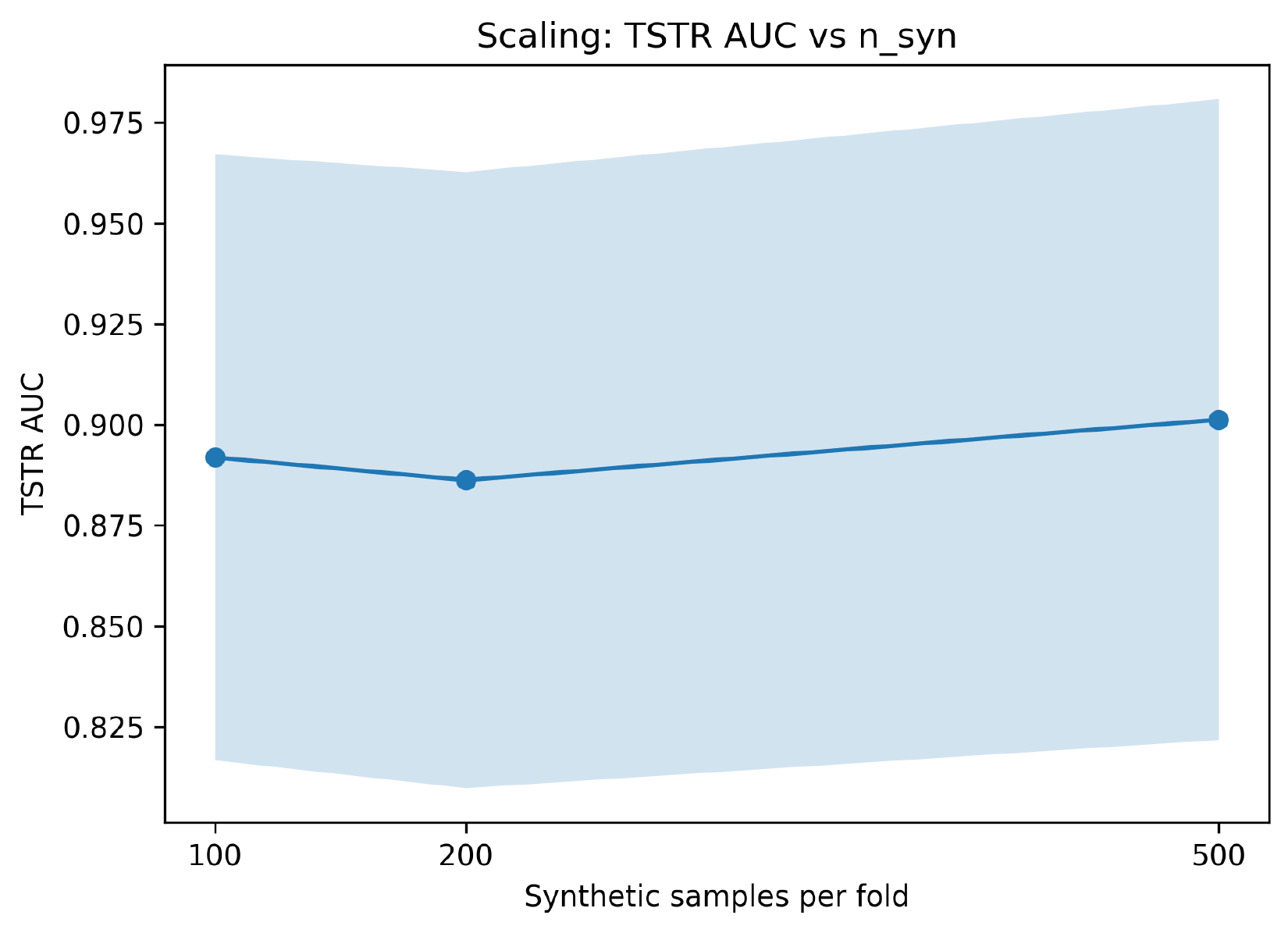}
    \caption{Scaling: TSTR AUC vs $n_{\text{syn}}$.}
  \end{subfigure}

  \caption{BSTabDiff compact ablations (TabPFN-2.5). The first two rows show the main compact-ablation summaries, while the last two rows show utility-privacy tradeoff and scaling trends.}
  \label{fig:bstabdiff_ablation_4rows}
\end{figure*}


\subsection{Selecting GO-BS versus GO-BS-FC}
\label{app:gobs_variant_selection}

\paragraph{Shape and scale-aware motivation.}
GO-BS and GO-BS-FC provide complementary mechanisms for constructing dependency-aware feature blocks. GO-BS first clusters samples to estimate local feature-affinity structures and aggregates their orderings, whereas GO-BS-FC clusters features directly and performs ordering within smaller feature groups. Selecting between these approaches requires considering both the statistical characteristics and computational scale of a dataset. The feature-to-sample ratio $\rho=m/n$, used in DynaTab~\citep{dynatab2} and GOTabPFN~\citep{gotabpfn}, characterizes dataset shape but does not alone capture total matrix scale, since datasets with similar $\rho$ can contain substantially different numbers of data cells. Motivated by BeyondArena~\citep{purucker2026beyond}, we therefore consider both $\rho$ and $\kappa=nm$. We reuse the original regime assignment $\operatorname{Reg}(D)$ and the Micro-Macro refinement introduced in Appendix~\ref{app:micro_macro_regimes}. Specifically, Eq.~\ref{eq:regime_threshold} defines the regime-specific median threshold $T_r$ over a fixed reference collection, and Eq.~\ref{eq:scale} assigns each dataset to Micro or Macro relative to the threshold of its original regime. The resulting classification is illustrated in Fig.~\ref{fig:micro_macro_regime_map}.

\paragraph{Default variant-selection rule.}
Let $\mathcal{D}_{\mathrm{FC}}$ denote the set of datasets assigned to Macro-HDLSS or Macro-HDHSS under the Micro-Macro refinement:

\[
\begin{aligned}
\mathcal{D}_{\mathrm{FC}} =
\bigl\{D:\;&
\operatorname{Scale}(D)=\mathrm{Macro},
&
\operatorname{Reg}(D)\in
\{\mathrm{HDLSS},\mathrm{HDHSS}\}
\bigr\}
\end{aligned}
\]

We recommend the following default ordering-variant selection rule:

\begin{tcolorbox}[
    colback=regimeHDHSS,
    colframe=algborder,
    boxrule=0.5pt,
    arc=1mm,
    boxsep=2pt,
    left=3pt,
    right=3pt,
    top=2pt,
    bottom=2pt,
    before skip=5pt,
    after skip=5pt
]
\begin{equation}
\label{eq:gobs_variant_rule}
\mathcal{V}(D)=
\begin{cases}
\text{GO-BS-FC},
& D\in\mathcal{D}_{\mathrm{FC}}\\
\text{GO-BS},
& \text{otherwise}
\end{cases}
\end{equation}
\end{tcolorbox}

Accordingly, GO-BS-FC is recommended for Macro-HDLSS and Macro-HDHSS datasets, where high feature dimensionality and large data matrix scale may make global feature ordering computationally expensive. Feature clustering decomposes the ordering task into smaller within-cluster problems while preserving the block organization required by BSTabDiff. GO-BS remains the default for Micro-HDLSS, Micro-HDHSS, and other regimes, where sample-clustered feature-affinity estimation is computationally manageable and can capture heterogeneous dependency patterns across sample subpopulations. This rule is a practical, regime-aware heuristic rather than a guarantee that either variant is universally optimal.

\paragraph{Empirical support and computational efficiency.}
The computational ablation in Table~\ref{tab:optional_perm_compute} shows that GO-BS-FC is faster than GO-BS on seven of the eight initial HDLSS datasets. Notable reductions include $151.93$ to $60.61$ seconds on SMK and $62.24$ to $24.42$ seconds on ARC. These savings are accompanied by competitive downstream utility, as reported in Table~\ref{tab:optional_perm_tstr}, supporting the use of feature clustering when ordering cost becomes substantial. Conversely, the sample clustered GO-BS formulation remains useful when feature dependencies vary across local sample populations, since its cluster-specific affinity graphs explicitly represent such heterogeneity.

\paragraph{Comparison with fixed global graph ordering.}
Appendix~\ref{app:same} provides a complementary evaluation of ordering solvers on a fixed global feature-affinity graph. Table~\ref{tab:same_global_graph_tstr} compares downstream TSTR performance across greedy, Christofides, simulated annealing, threshold accepting, genetic algorithm, and ant colony ordering procedures, with GO-BS and GO-BS-FC included as structural reference methods. Table~\ref{tab:same_global_graph_costs} additionally reports their ordering and training costs. In the GO-BS versus GO-BS-FC reference comparison on Colon, GO-BS-FC achieves TSTR accuracy of $87.15\%$ versus $86.82\%$ and TSTR AUC of $90.73\%$ versus $90.30\%$, while reducing ordering time from $23.57$ to $1.08$ seconds (approximately $21.8\times$) and total training time from $28.33$ to $15.60$ seconds. These reference results are reported separately from the identical protocol fixed global graph solver comparison and therefore should not be interpreted as additional solvers evaluated under that same protocol.

\paragraph{Near-threshold selection and stability diagnostics.}
For datasets with $\kappa$ close to the regime-specific threshold $T_r$, the default rule may be supplemented with cluster stability diagnostics under bootstrap resampling. More stable sample clusters favor GO-BS, whereas more stable feature clusters favor GO-BS-FC, particularly when computational costs are comparable. This provides a secondary qualitative diagnostic rather than a separately validated selection criterion. The Micro-Macro thresholds should be estimated from a fixed reference collection before applying the rule to new datasets, thereby avoiding dataset-dependent recalibration during evaluation. Overall, the proposed strategy retains GO-BS as the canonical sample conditioned ordering formulation while recommending GO-BS-FC as a computationally efficient default for Macro-HDLSS and Macro-HDHSS settings.

\renewcommand{\thesection}{A13}
\renewcommand{\thesubsection}{A13.\arabic{subsection}}
\setcounter{figure}{0}\renewcommand{\thefigure}{A13.\arabic{figure}}
\setcounter{table}{0}\renewcommand{\thetable}{A13.\arabic{table}}
\section{HDLSS Literature and Datasets Used}
\label{hdlss}
\subsection{Dataset Categorization Rules.}
There is no universally accepted numerical threshold that uniquely defines a dataset as HDLSS. In the HDLSS literature, the term is commonly used for regimes where the ambient dimension, or number of variables, is much larger than the sample size. Following the empirical regime rules introduced by~\citet{dynatab2, gotabpfn}, we use the feature-to-sample ratio $\rho=m/n$ to develop operational decision boundaries in the $(n,m)$ space (Fig.~\ref{fig:regime_map}). This is often formalized through HDLSS asymptotics~\citep{w2}, where the dimension grows while the sample size is fixed or grows much more slowly~\citep{w1,b57}. Thus, the simple condition ``$m>n$'' provides a useful heuristic, but it is too coarse to describe the full practical range of high-dimensional regimes. To make the categorization more operational, we empirically stratify datasets using the feature-to-sample ratio $\rho=m/n$, which allows us to distinguish qualitatively different settings beyond a binary HDLSS versus non-HDLSS label. Let \(n\) be the number of samples, \(m\) be the number of features, and \(\rho = \frac{m}{n}\) be the feature-to-sample ratio. We can assign each dataset to one of five regimes according to the following empirical \(\rho\)-based criteria according to the rules defined in Table~\ref{tab:dataset_regime_rules}. Fig.~\ref{fig:regime_map} visualizes the empirical regime decision boundaries in the $(n,m)$ space, highlighting how the HDLSS, HDHSS, LDHSS, LDLSS, and MixedRegime categories depend jointly on $n$, $m$, and $\rho=m/n$.
\begin{table}[t]
\centering
\caption{
Empirical dataset regime categorization rules based on
the number of samples \(n\), number of features \(m\),
and feature-to-sample ratio \(\rho=m/n\)~\citep{dynatab2, gotabpfn}.
}
\label{tab:dataset_regime_rules}
\small
\setlength{\tabcolsep}{9pt}
\renewcommand{\arraystretch}{1.25}

\begin{tabular}{@{}ll@{}}
\toprule
\rowcolor{regimeheader}
\textcolor{white}{\textbf{Regime}} &
\textcolor{white}{\textbf{Rule}} \\
\midrule

\rowcolor{regimeHDLSS}
\textbf{HDLSS} &
\(m>1000,\; n<1000,\; \rho>2\) \\

\rowcolor{regimeHDHSS}
\textbf{HDHSS} &
\(m>1000,\; n>10^4,\; 0.005<\rho\le2\) \\

\rowcolor{regimeLDHSS}
\textbf{LDHSS} &
\(m\le100,\; n>10^4,\; \rho\le0.01\) \\

\rowcolor{regimeLDLSS}
\textbf{LDLSS} &
\(m\le100,\; n\le1000,\; \rho\le0.05\) \\

\rowcolor{regimeMixed}
\textbf{MixedRegime} &
\textit{otherwise} \\

\bottomrule
\end{tabular}
\end{table}
\begin{figure*}[t]
    \centering

    \begin{subfigure}[t]{0.49\textwidth}
        \centering
        \includegraphics[width=\linewidth]{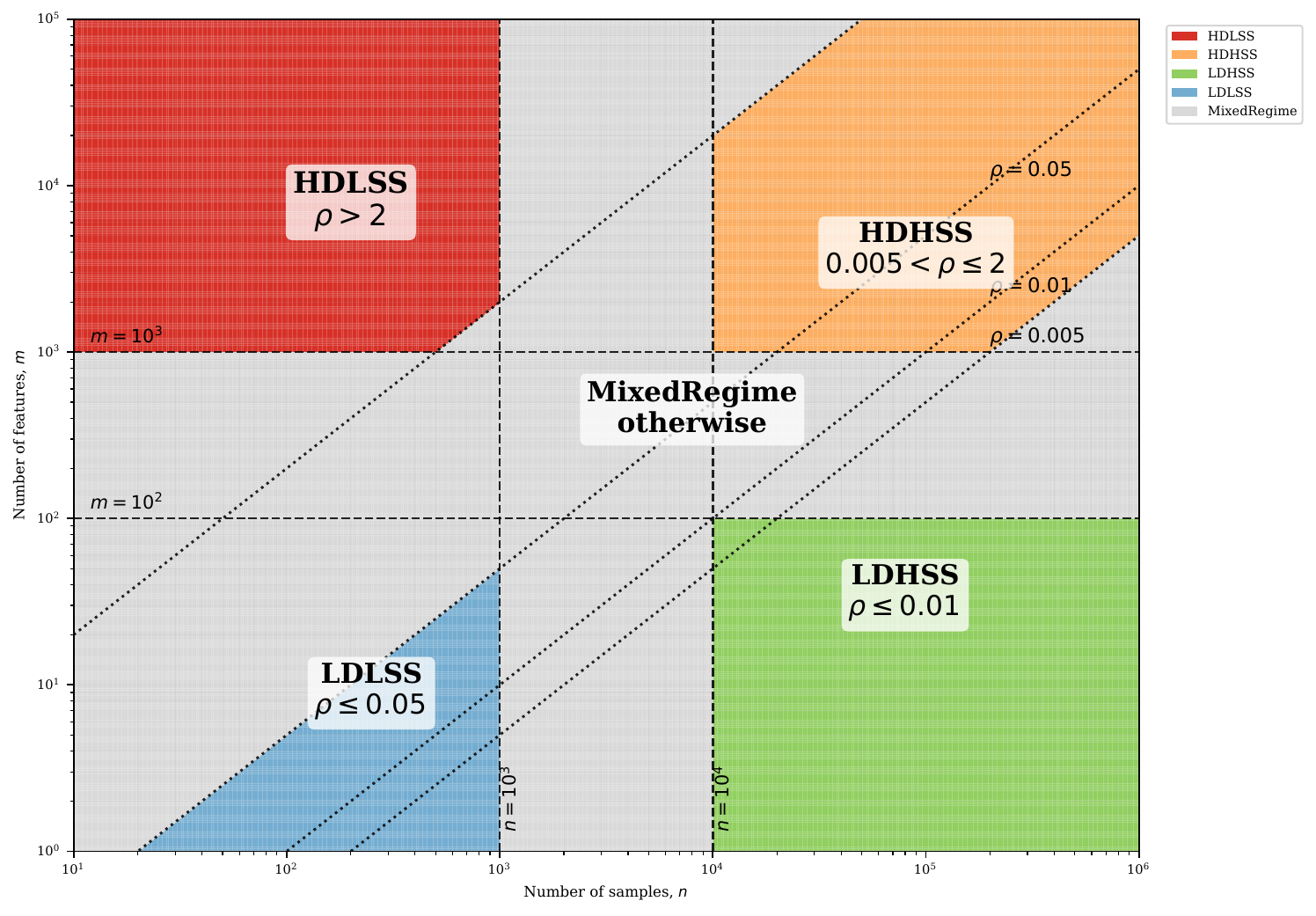}
        \caption{Regime decision map in the $(n,m)$ space.}
        \label{fig:regime_map}
    \end{subfigure}
    \hfill
    \begin{subfigure}[t]{0.49\textwidth}
        \centering
        \includegraphics[width=\linewidth]{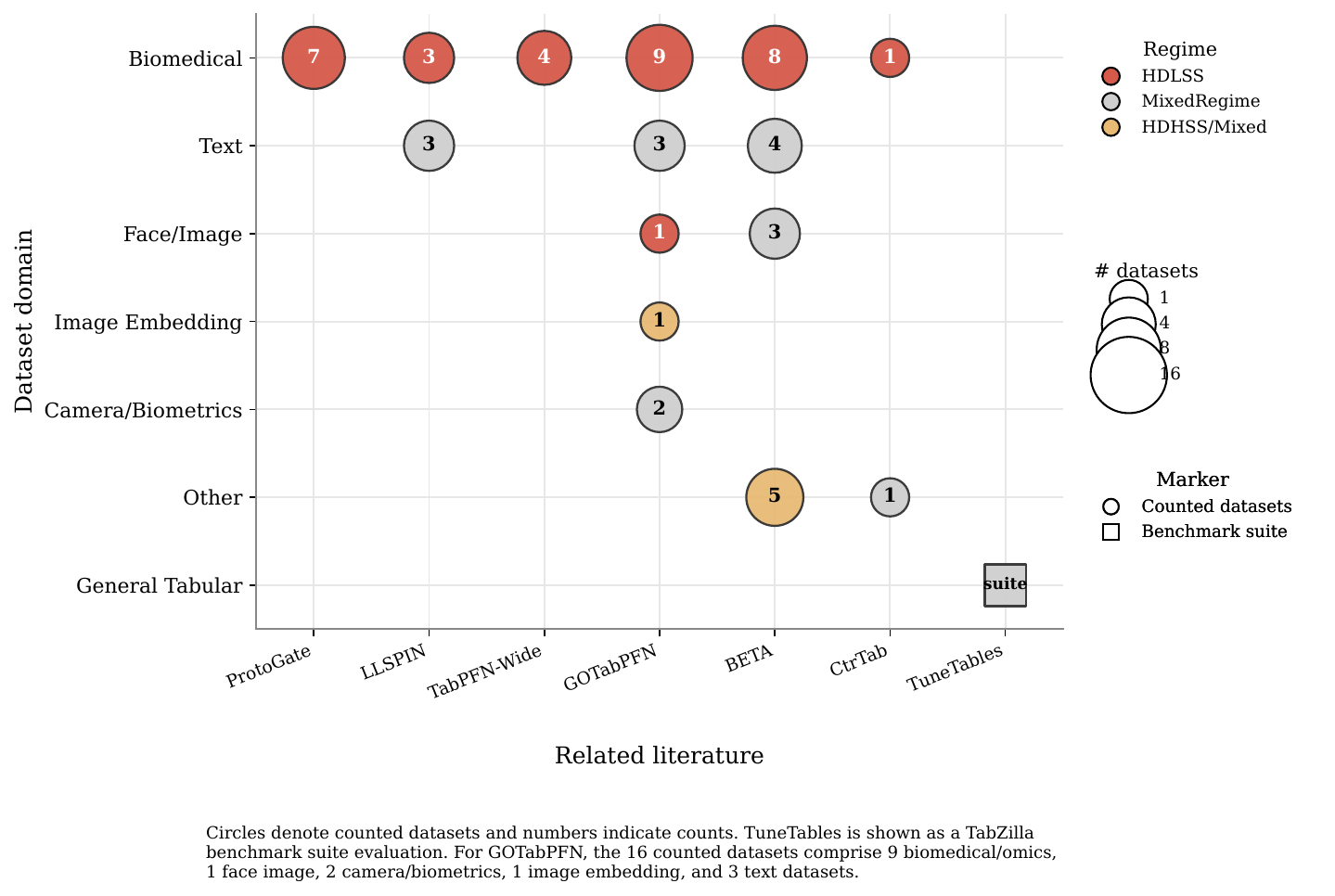}
        \caption{Cross-domain usage in related literature.}
        \label{fig:cross_domain_hdlss}
    \end{subfigure}

    \caption{
Overview of HDLSS and related high-dimensional regimes. 
(a) Dataset regimes in $(n,m)$ space, with $\rho=m/n$~\citep{gotabpfn, dynatab2}; dashed lines mark $n,m$ thresholds and dotted diagonals mark constant $\rho$. 
(b) Cross-domain coverage in prior HDLSS-related literature, highlighting the historical concentration in biomedical datasets and the expanding evaluation scope of recent methods such as GOTabPFN~\citep{gotabpfn}, which covers 16 datasets spanning biomedical, text, face/image, camera/biometrics, and image embedding domains. Bubble sizes indicate dataset counts, while colors distinguish dimensionality regimes; benchmark suite evaluations are represented separately.}
    \label{fig:regime_and_cross_domain}
\end{figure*}

\subsection{Micro-Macro Dataset Regime Refinement}
\label{app:micro_macro_regimes}
\paragraph{Motivation.} The dataset categorization rules used by~\citet{dynatab2, gotabpfn} characterize datasets through the number of samples $n$, the number of features $m$, and the feature-to-sample ratio $\rho=m/n$. These criteria define five empirical regimes: HDLSS, HDHSS, LDHSS, LDLSS, and MixedRegime. Complementarily, BeyondArena~\citep{purucker2026beyond} considers dataset scale in terms of the number of data cells. Motivated by this perspective, we define the shape-scale parameters by Eq.~\ref{eq:shape_scale}.
\begin{equation}
\label{eq:shape_scale}
\rho=\frac{m}{n}, \qquad \kappa=nm
\end{equation}
While $\rho$ characterizes the relative proportions of features and samples, $\kappa$ measures the total number of entries in the data matrix, providing a complementary indicator of dataset size.
\paragraph{Regime-specific scale thresholds.}
The typical magnitude of $\kappa$ can differ substantially across the original regimes. Consequently, a single global threshold would not provide a regime-normalized distinction between small and large datasets. Let us assume we get Eq.~\ref{data} which denotes the set of original regimes, and let $\operatorname{Reg}(D)\in\mathcal{C}$ denote the original regime assignment of dataset $D$ according to Table~\ref{tab:dataset_regime_rules}.
\begin{equation}
\label{data}
\mathcal{C}=
\{\mathrm{HDLSS},\mathrm{HDHSS},\mathrm{LDHSS},
\mathrm{LDLSS},\mathrm{MixedRegime}\}
\end{equation}
Consider a fixed reference benchmark collection $\mathcal{B}_{\mathrm{ref}}=\{D_i\}_{i=1}^{N}$, where dataset $D_i$ contains $n_i$ samples and $m_i$ features, with $\kappa_i=n_i m_i$. For each regime $r\in\mathcal{C}$ represented in
the reference collection, define the regime-specific
median scale threshold by Eq.~\ref{eq:regime_threshold}. The reference collection and resulting thresholds are fixed before assigning Micro-Macro labels.

\begin{tcolorbox}[
    colback=regimeHDHSS,
    colframe=algborder,
    boxrule=0.5pt,
    arc=1mm,
    boxsep=2pt,
    left=3pt,
    right=3pt,
    top=2pt,
    bottom=2pt,
    before skip=5pt,
    after skip=5pt
]
\begin{equation}
\label{eq:regime_threshold}
T_r =
\operatorname{median}
\left\{
\kappa_i :
D_i \in \mathcal{B}_{\mathrm{ref}},
\operatorname{Reg}(D_i)=r
\right\}
\end{equation}
\end{tcolorbox}
\paragraph{Micro-Macro assignment.}
For a dataset $D$ with scale $\kappa=nm$, we assign a Micro or Macro label relative to the reference median of its original regime by Eq.~\ref{eq:scale}. The refined dataset regime is defined as Eq.~\ref{eq:refined_regime}. This yields 10 possible refined categories: Micro-HDLSS, Macro-HDLSS, Micro-HDHSS, Macro-HDHSS, Micro-LDHSS, Macro-LDHSS, Micro-LDLSS, Macro-LDLSS, Micro-MixedRegime, and Macro-MixedRegime.

\begin{tcolorbox}[
    colback=regimeHDHSS,
    colframe=algborder,
    boxrule=0.5pt,
    arc=1mm,
    boxsep=2pt,
    left=3pt,
    right=3pt,
    top=2pt,
    bottom=2pt,
    before skip=5pt,
    after skip=5pt
]
\begin{equation}
\label{eq:scale}
\operatorname{Scale}(D)=
\begin{cases}
\mathrm{Micro},
& \kappa<T_{\operatorname{Reg}(D)}\\
\mathrm{Macro},
& \kappa\geq T_{\operatorname{Reg}(D)}
\end{cases}
\end{equation}
\end{tcolorbox}
\begin{equation}
\label{eq:refined_regime}
\operatorname{Reg}_{\mathrm{MM}}(D)
=
\bigl(
\operatorname{Scale}(D),
\operatorname{Reg}(D)
\bigr)
\end{equation}
\paragraph{Interpretation.}
This refinement preserves the original statistical regime definitions while introducing a relative measure of data-matrix scale within each regime. It distinguishes datasets with similar statistical shapes but substantially different numbers of data cells, enabling more granular, regime-stratified benchmark analysis. The refinement does not directly characterize algorithmic runtime or memory complexity, which also depend on the underlying model architecture and computational operations. Furthermore, the Micro-Macro distinction is relative to the chosen reference collection rather than a universal dataset-size boundary. Fig.~\ref{fig:micro_macro_regime_map} visualizes the refinement in the $(n,m)$ space, where constant-scale boundaries satisfy $\kappa=nm=T_r$ within each original regime.
\begin{figure}[t]
    \centering
    \includegraphics[width=\columnwidth]{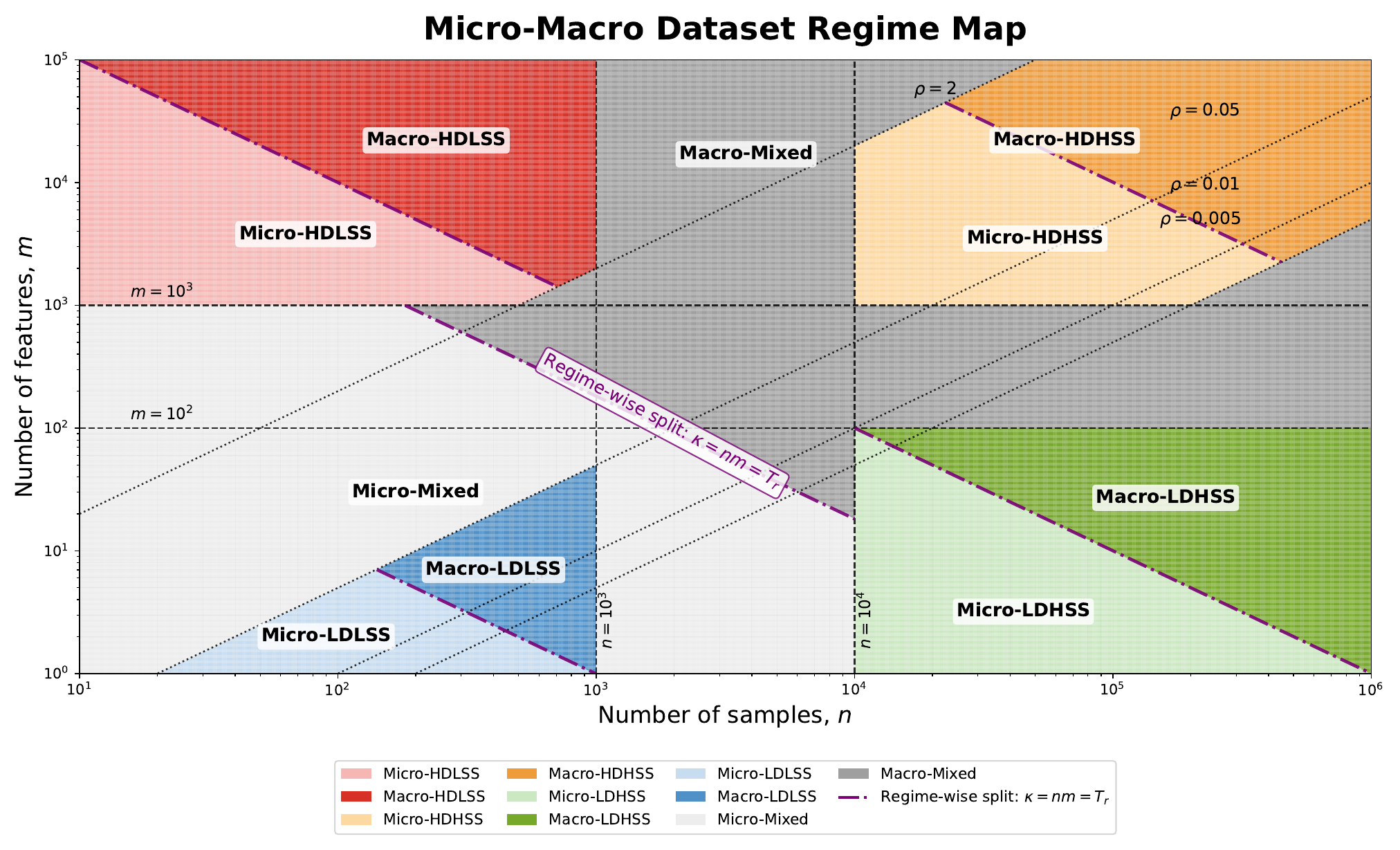}
    \caption{\textbf{Micro-Macro refinement of dataset regimes.}
    The five original regimes are partitioned using the
    shape parameter $\rho=m/n$ and scale parameter
    $\kappa=nm$. Purple dash-dotted curves indicate
    regime-specific median boundaries $\kappa=T_r$,
    separating Micro ($\kappa<T_r$) and
    Macro ($\kappa\geq T_r$) subsets.}
    \label{fig:micro_macro_regime_map}
\end{figure}

\subsection{HDLSS in Cross-Domain Settings.}
The availability of HDLSS datasets varies substantially across domains, but the regime is most commonly observed in biomedical settings, particularly in omics, gene-expression, and microarray datasets. This pattern is also consistent with prior high-dimensional tabular learning studies: TabPFN-Wide~\citep{tabpfnw} explicitly identifies HDLSS as a frequent scenario in biomedical applications, while ProtoGate~\citep{protogate} evaluates on seven real-world biomedical HDLSS datasets. LSPIN/LLSPIN~\citep{spin} considers six real-world high-dimensional datasets spanning biomedical and text domains; however, under our proposed regime rules, its text datasets do not satisfy the HDLSS criteria and instead fall into the MixedRegime category. Similarly, TabPFN-Wide uses four real-world biomedical multi-omics HDLSS datasets, while BETA~\citep{beta}, another TabPFN variant capable of handling HDLSS data, evaluates on a broader collection of high-dimensional classification benchmarks covering biomedical, text, face-image, and other benchmark datasets. Under our categorization, BETA’s benchmarks span HDLSS, HDHSS, and MixedRegime settings, but no explicit financial HDLSS dataset is included among its 20 real-world high-dimensional experiments. TuneTables~\citep{tune}, another high-dimensionality compatible TabPFN variant, evaluates broadly on TabZilla~\citep{tabzilla} and large-scale tabular benchmarks, but likewise does not include high-dimensional financial datasets. For tabular generative modeling, CtrTab~\citep{ctr} is one of the few works to include high-dimensional datasets; among its datasets, Arcene is the sole biomedical HDLSS dataset, whereas Stadyn, with 6248 samples and 1084 features, falls into MixedRegime under our rules in Table~\ref{tab:dataset_regime_rules}. Overall, related high-dimensional literature repeatedly relies on biomedical datasets, with some extensions to text, face-image, and sensor-like domains. This suggests that HDLSS is predominantly a biomedical phenomenon, with partial occurrence in text and image-derived feature domains, whereas financial datasets are not commonly represented in this regime. 
Fig.~\ref{fig:cross_domain_hdlss} summarizes the cross-domain coverage of high-dimensional datasets in related literature, highlighting the historical concentration of HDLSS benchmarks in biomedical applications and the broader evaluation scope of recent methods such as GOTabPFN~\citep{gotabpfn}, which spans 16 datasets across biomedical, text, face/image, camera/biometrics, and image embedding domains. The figure also illustrates the varying dimensionality regimes represented across these domains.

\vspace{-1em}
\renewcommand{\thesection}{A14}
\renewcommand{\thesubsection}{A14.\arabic{subsection}}
\setcounter{figure}{0}\renewcommand{\thefigure}{A14.\arabic{figure}}
\setcounter{table}{0}\renewcommand{\thetable}{A14.\arabic{table}}
\section{Representation Quality via t-SNE}
\label{t-SNE}
\vspace{-1em}
Fig.~\ref{fig:colon_t-SNE_gobs} provides a qualitative sample-space visualization on the Colon dataset. We jointly embed real and synthetic samples using t-SNE~\citep{tSNE} after shared standardization and PCA~\citep{pca} pre-reduction. The non-GO-BS variant generates samples that broadly follow the real data geometry, but the class-conditional structure is less clearly separated. In contrast, the GO-BS variant produces synthetic samples that more closely align with the corresponding real class regions, particularly in the dominant class-specific clusters. This suggests that organizing features through GO-BS before block-subunit generation helps BSTabDiff preserve local dependency patterns and discriminative sample geometry in HDLSS tabular data.
\begin{figure*}[t]
    \centering
    \includegraphics[width=0.96\textwidth]{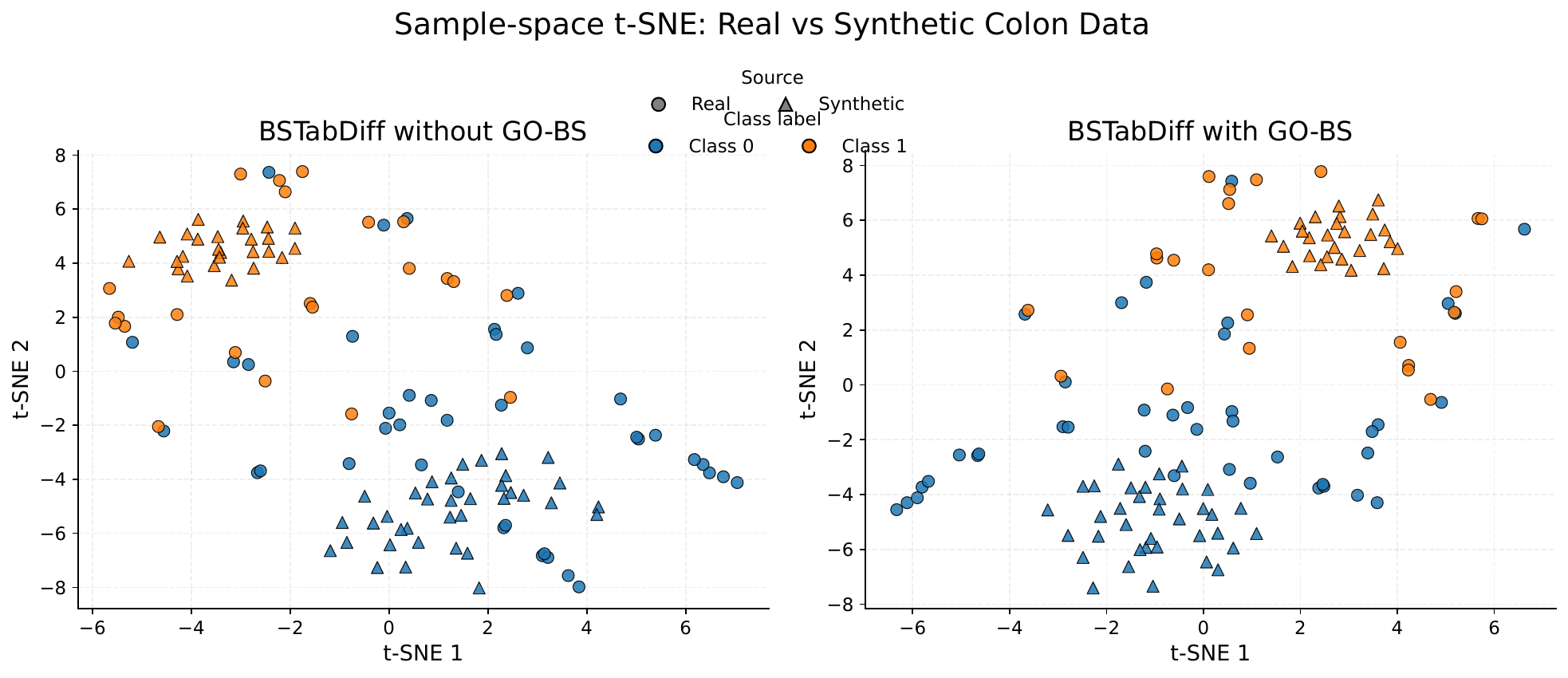}
    \caption{
t-SNE of real and synthetic Colon samples from BSTabDiff with and without GO-BS. GO-BS yields better class-conditional alignment with the real manifold, suggesting improved preservation of HDLSS structure.
}
    \label{fig:colon_t-SNE_gobs}
\end{figure*}
\vspace{-1em}
\renewcommand{\thesection}{A15}
\renewcommand{\thesubsection}{A15.\arabic{subsection}}
\setcounter{figure}{0}\renewcommand{\thefigure}{A15.\arabic{figure}}
\setcounter{table}{0}\renewcommand{\thetable}{A15.\arabic{table}}
\section{Statistical Significance Analysis}
\label{stats}
\vspace{-1em}
We assess statistical significance using TSTR accuracy from Logistic Regression across the eight HDLSS datasets, considering only synthetic generators; Real (TRTR) is retained only as an upper-bound reference. BSTabDiff obtains the highest average TSTR accuracy (\(86.33\pm7.62\)) and ranks first on every dataset. A Friedman~\citep{fried} omnibus test over all synthetic generators indicates significant differences among methods (\(\chi^2=75.43\), \(p=3.91\times10^{-12}\)). Pairwise one-sided Wilcoxon signed-rank tests~\citep{wilcoxon} with Holm correction~\citep{nemen} show that BSTabDiff improves over every baseline on all eight datasets, with corrected \(p_{\mathrm{H}}=0.0391\) in each comparison (Table~\ref{tab:wilcoxon_holm_tstr_lr}).

\begin{table}[t]
\centering
\caption{
Pairwise TSTR comparison of BSTabDiff against synthetic baselines on eight HDLSS datasets.
\(\Delta\)Acc. is the mean gain in percentage points; W/T/L denotes dataset-level wins/ties/losses;
\(p_{\mathrm{H}}\) is Holm-corrected. Light green shades emphasize larger gains, and light blue marks statistically significant Holm-corrected results.
}
\label{tab:wilcoxon_holm_tstr_lr}
\scriptsize
\setlength{\tabcolsep}{3.5pt}
\renewcommand{\arraystretch}{1.08}
\begin{tabular}{lcccc@{\hspace{0.7em}}lcccc}
\toprule
\rowcolor{headergray}
\textbf{Baseline} & \textbf{\(\Delta\)} & \textbf{W/T/L} & \textbf{\(p\)} & \textbf{\(p_{\mathrm{H}}\)}
&
\textbf{Baseline} & \textbf{\(\Delta\)} & \textbf{W/T/L} & \textbf{\(p\)} & \textbf{\(p_{\mathrm{H}}\)} \\
\midrule

CtrTab
& \cellcolor{gainhigh}\textbf{40.56}
& \cellcolor{wincell}\textbf{8/0/0}
& 0.0039
& \cellcolor{sigblue}\textbf{0.0391}
&
BinaryDiff
& \cellcolor{gainmid}\textbf{14.96}
& \cellcolor{wincell}\textbf{8/0/0}
& 0.0039
& \cellcolor{sigblue}\textbf{0.0391} \\

CTGAN
& \cellcolor{gainhigh}\textbf{28.70}
& \cellcolor{wincell}\textbf{8/0/0}
& 0.0039
& \cellcolor{sigblue}\textbf{0.0391}
&
ForestDiff
& \cellcolor{gainmid}\textbf{12.77}
& \cellcolor{wincell}\textbf{8/0/0}
& 0.0039
& \cellcolor{sigblue}\textbf{0.0391} \\

CTAB-GAN
& \cellcolor{gainhigh}\textbf{21.97}
& \cellcolor{wincell}\textbf{8/0/0}
& 0.0039
& \cellcolor{sigblue}\textbf{0.0391}
&
TVAE
& \cellcolor{gainlow}\textbf{5.95}
& \cellcolor{wincell}\textbf{8/0/0}
& 0.0039
& \cellcolor{sigblue}\textbf{0.0391} \\

CTAB-GAN+
& \cellcolor{gainmid}\textbf{21.15}
& \cellcolor{wincell}\textbf{8/0/0}
& 0.0039
& \cellcolor{sigblue}\textbf{0.0391}
&
SMOTE
& \cellcolor{gainlow}\textbf{4.53}
& \cellcolor{wincell}\textbf{8/0/0}
& 0.0039
& \cellcolor{sigblue}\textbf{0.0391} \\

TabDDPM
& \cellcolor{gainmid}\textbf{20.66}
& \cellcolor{wincell}\textbf{8/0/0}
& 0.0039
& \cellcolor{sigblue}\textbf{0.0391}
&
TabDiff
& \cellcolor{gainlow}\textbf{4.42}
& \cellcolor{wincell}\textbf{8/0/0}
& 0.0039
& \cellcolor{sigblue}\textbf{0.0391} \\

\bottomrule
\end{tabular}
\end{table}
\vspace{-1em}
\renewcommand{\thesection}{A16}
\renewcommand{\thesubsection}{A16.\arabic{subsection}}
\setcounter{figure}{0}\renewcommand{\thefigure}{A16.\arabic{figure}}
\setcounter{table}{0}\renewcommand{\thetable}{A16.\arabic{table}}
\section{BSTabDiff Hyperparameters}
\label{hyp}
\vspace{-1em}
Table~\ref{tab:bstabdiff_hparams_all8_rot} summarizes the dataset-specific block latent size $M$ used across the eight HDLSS datasets. All other hyperparameters were kept fixed across datasets: we used a diffusion prior trained for 10{,}000 epochs with batch size 128, learning rate $10^{-3}$, EMA enabled with decay 0.999, no feature permutation, and no predefined blocks. Thus, the main dataset-specific adjustment was the block latent size $M$, which was increased for higher-dimensional datasets (e.g., GLI and SMK) and kept smaller for lower-dimensional ones (e.g., COL and TOX).
\begin{table}[t]
\caption{Dataset-specific block latent size $M$ used for BSTabDiff across the 8 HDLSS datasets.}
\label{tab:bstabdiff_hparams_all8_rot}
\centering
\small
\setlength{\tabcolsep}{8pt}
\renewcommand{\arraystretch}{1.25}

\begin{tabular}{|l|c|c|c|c|c|c|c|c|}
\hline
\rowcolor{latentheader}
\textcolor{white}{\textbf{Dataset}} &
\textcolor{white}{\textbf{COL}} &
\textcolor{white}{\textbf{LNG}} &
\textcolor{white}{\textbf{GLI}} &
\textcolor{white}{\textbf{SMK}} &
\textcolor{white}{\textbf{AML}} &
\textcolor{white}{\textbf{PRS}} &
\textcolor{white}{\textbf{ARC}} &
\textcolor{white}{\textbf{TOX}} \\
\hline

\rowcolor{latentvalues}
\cellcolor{latentlabel}\textbf{$M$}
& 32 & 64 & 128 & 192 & 64 & 64 & 64 & 32 \\
\hline
\end{tabular}
\vspace{2pt}
\end{table}
Table~\ref{tab:optuna_best_params_2000} summarizes the best BSTabDiff hyperparameters selected by Optuna under the no GO-BS, no permutation setting. Across datasets, the tuning consistently preferred relatively small block counts (\(M=8\) in most cases, with \(M=16\) for PRS and TOX), while diffusion priors were favored on most datasets except COL and AML, where flow priors performed best. The selected learning rates and EMA decays varied substantially across datasets, indicating nontrivial dataset-specific optimization structure, whereas the synthetic sample budget used in the inner-loop proxy evaluation typically remained moderate (\(64\)–\(256\)). These results support using dataset-specific tuning following the tuning protocol set by~\citet{w11}.

\begin{table*}[htbp]
\centering
\setlength{\tabcolsep}{4.2pt}
\renewcommand{\arraystretch}{1.12}
\caption{
Best BSTabDiff hyperparameters from 2000 Optuna trials per
HDLSS dataset, selected by Logistic Regression validation
on synthetic data, with GO-BS and feature permutation disabled.
}
\label{tab:optuna_best_params_2000}

\begin{tabular}{|l|c|c|c|c|c|c|c|}
\hline
\rowcolor{optunaheader}
\textcolor{white}{\textbf{Data}} &
\textcolor{white}{\textbf{$M$}} &
\textcolor{white}{\textbf{Prior}} &
\textcolor{white}{\textbf{Epochs}} &
\textcolor{white}{\textbf{Batch}} &
\textcolor{white}{\textbf{LR}} &
\textcolor{white}{\textbf{EMA}} &
\textcolor{white}{\textbf{$n_{\text{syn}}$}} \\
\hline

\rowcolor{optunarow}
\cellcolor{optunadata}\textbf{COL}
& 8 & flow & 300 & 64 & 3.2931e-03 & 0.9926 & 64 \\

\cellcolor{optunadata}\textbf{LNG}
& 8 & diffusion & 1200 & 128 & 5.8095e-05 & 0.9998 & 200 \\

\rowcolor{optunarow}
\cellcolor{optunadata}\textbf{GLI}
& 8 & diffusion & 800 & 32 & 3.1386e-04 & 0.9911 & 64 \\

\cellcolor{optunadata}\textbf{SMK}
& 8 & diffusion & 1200 & 32 & 1.4595e-04 & 0.9995 & 256 \\

\rowcolor{optunarow}
\cellcolor{optunadata}\textbf{AML}
& 8 & flow & 1600 & 32 & 4.9452e-03 & 0.9928 & 128 \\

\cellcolor{optunadata}\textbf{PRS}
& 16 & diffusion & 1200 & 64 & 2.5615e-04 & 0.9983 & 256 \\

\rowcolor{optunarow}
\cellcolor{optunadata}\textbf{ARC}
& 8 & diffusion & 1200 & 16 & 1.6479e-04 & 0.9990 & 200 \\

\cellcolor{optunadata}\textbf{TOX}
& 16 & diffusion & 300 & 64 & 5.8193e-04 & 0.9991 & 64 \\
\hline
\end{tabular}
\end{table*}

Table~\ref{tab:gobs_optuna_params} summarizes the GO-BS-only hyperparameters selected by Optuna when tuning the ordering and block-segmentation module separately from the main BSTabDiff generator. Because GO-BS requires feature-affinity graph construction and graph-based refinement, whose cost scales quadratically with the number of features, we used a smaller 150-trial tuning budget for this component. The selected configurations vary across datasets, indicating that the optimal graph metric, sparsification level, number of sample clusters, and boundary-refinement strength are dataset dependent. In particular, correlation-based affinities are preferred on COL, SMK, AML, and TOX, while cosine, Manhattan, and KL-divergence are selected for other datasets, suggesting that GO-BS benefits from adapting its structural prior to the geometry of each HDLSS dataset.

\begin{table*}[htbp]
\centering
\scriptsize
\setlength{\tabcolsep}{3.2pt}
\renewcommand{\arraystretch}{1.12}
\caption{
Best GO-BS-only hyperparameters from Optuna tuning per HDLSS dataset, with feature permutation disabled. GO-BS is tuned separately due to the approximately \(O(m^2)\) cost of feature-affinity graph operations. Ref. Ord./Pass = local order refinement/passes; Dir. Sel. = direction selection; Bnd. Pass/Win. = boundary refinement passes/window; \(\lambda_{\mathrm{cross}}\) = cross-block leakage penalty; \(\gamma_{\mathrm{bal}}\) = block-balance penalty.
}
\label{tab:gobs_optuna_params}

\begin{tabular}{|l|c|c|c|c|c|c|c|c|c|c|c|}
\hline
\rowcolor{gobsheader}
\textbf{Data} &
\textbf{Clusters} &
\textbf{Metric} &
\textbf{Bins} &
\textbf{Top-\(k\)} &
\textbf{Ref. Ord.} &
\textbf{Dir. Sel.} &
\textbf{Ref. Pass} &
\textbf{Bnd. Pass} &
\textbf{Bnd. Win.} &
\(\boldsymbol{\lambda_{\mathrm{cross}}}\) &
\(\boldsymbol{\gamma_{\mathrm{bal}}}\) \\
\hline

\rowcolor{gobsrow}
\cellcolor{gobsdata}\textbf{COL}
& 3 & correlation & 24 & 16 & False & False
& 3 & 5 & 16 & 2.3989 & 5.19e-06 \\

\cellcolor{gobsdata}\textbf{GLI}
& 3 & manhattan & 32 & 32 & True & False
& 1 & 0 & 24 & 2.4373 & 9.27e-06 \\

\rowcolor{gobsrow}
\cellcolor{gobsdata}\textbf{LNG}
& 3 & cosine & 32 & 16 & True & False
& 1 & 5 & 16 & 1.1691 & 2.65e-08 \\

\cellcolor{gobsdata}\textbf{PRS}
& 3 & kl\_divergence & 24 & 16 & True & True
& 0 & 5 & 12 & 0.1892 & 3.78e-05 \\

\rowcolor{gobsrow}
\cellcolor{gobsdata}\textbf{SMK}
& 7 & correlation & 64 & 64 & True & True
& 2 & 5 & 12 & 0.3032 & 2.68e-07 \\

\cellcolor{gobsdata}\textbf{TOX}
& 4 & correlation & 48 & 64 & False & False
& 1 & 3 & 16 & 2.3373 & 3.25e-08 \\

\rowcolor{gobsrow}
\cellcolor{gobsdata}\textbf{AML}
& 3 & correlation & 16 & 16 & False & True
& 2 & 8 & 8 & 0.1607 & 1.86e-05 \\

\cellcolor{gobsdata}\textbf{ARC}
& 7 & cosine & 48 & 32 & True & False
& 2 & 5 & 12 & 2.7504 & 7.58e-04 \\

\hline
\end{tabular}
\end{table*}
\vspace{-1em}
\renewcommand{\thesection}{A17}
\renewcommand{\thesubsection}{A17.\arabic{subsection}}
\setcounter{figure}{0}\renewcommand{\thefigure}{A17.\arabic{figure}}
\setcounter{table}{0}\renewcommand{\thetable}{A17.\arabic{table}}
\section{Additional Runtime and Resource Usage}
\label{run}
\vspace{-1em}
\paragraph{Runtime and resources.} Table~\ref{tab:bstabdiff_compute} provides an additional runtime and memory profile for BSTabDiff, one fixed-configuration training run finishes in roughly one minute with very small peak GPU usage and modest CPU RSS, supporting the computational efficiency of the block-subunit formulation.
\begin{table}[htbp]
\centering
\caption{BSTabDiff runtime/resource usage and energy footprint measured on Vulcan under the default no-GO-BS configuration. Peak GPU is max CUDA memory; CPU memory is end-of-run RSS. Energy is total consumption (kWh) and CO$_2$e is estimated carbon emission (kg).}
\label{tab:bstabdiff_compute}
\scriptsize
\setlength{\tabcolsep}{4.2pt}
\renewcommand{\arraystretch}{1.15}

\begin{tabular}{|l|c|c|c|c|c@{\hspace{0.8cm}}||l|c|c|c|c|c|}
\hline
\rowcolor{computeheader}
\textcolor{white}{\textbf{Dataset}} &
\textcolor{white}{\textbf{Time}} &
\textcolor{white}{\textbf{GPU}} &
\textcolor{white}{\textbf{CPU}} &
\textcolor{white}{\textbf{Energy}} &
\textcolor{white}{\textbf{CO$_2$e}} &
\textcolor{white}{\textbf{Dataset}} &
\textcolor{white}{\textbf{Time}} &
\textcolor{white}{\textbf{GPU}} &
\textcolor{white}{\textbf{CPU}} &
\textcolor{white}{\textbf{Energy}} &
\textcolor{white}{\textbf{CO$_2$e}} \\

\rowcolor{computeheader}
&
\textcolor{white}{\textbf{(s)}} &
\textcolor{white}{\textbf{(GiB)}} &
\textcolor{white}{\textbf{(GiB)}} &
\textcolor{white}{\textbf{(kWh)}} &
\textcolor{white}{\textbf{(kg)}} &
&
\textcolor{white}{\textbf{(s)}} &
\textcolor{white}{\textbf{(GiB)}} &
\textcolor{white}{\textbf{(GiB)}} &
\textcolor{white}{\textbf{(kWh)}} &
\textcolor{white}{\textbf{(kg)}} \\
\hline

\rowcolor{computerow}
\cellcolor{computedata}\textbf{COL}
& 49.20 & 0.025 & 2.752 & 1.414e-3 & 1.267e-3 &
\cellcolor{computedata}\textbf{SMK}
& 59.61 & 0.043 & 1.284 & 1.522e-3 & 1.364e-3 \\

\cellcolor{computedata}\textbf{GLI}
& 63.32 & 0.044 & 1.269 & 1.643e-3 & 1.473e-3 &
\cellcolor{computedata}\textbf{TOX}
& 49.71 & 0.032 & 1.243 & 1.545e-3 & 1.384e-3 \\

\rowcolor{computerow}
\cellcolor{computedata}\textbf{LNG}
& 53.05 & 0.031 & 1.223 & 1.481e-3 & 1.327e-3 &
\cellcolor{computedata}\textbf{AML}
& 42.67 & 0.028 & 1.220 & 1.439e-3 & 1.290e-3 \\

\cellcolor{computedata}\textbf{PRS}
& 51.63 & 0.027 & 1.219 & 1.618e-3 & 1.450e-3 &
\cellcolor{computedata}\textbf{ARC}
& 52.12 & 0.032 & 1.257 & 1.546e-3 & 1.385e-3 \\

\hline
\end{tabular}
\end{table}
\vspace{-1em}

\paragraph{Assets and licenses.}
All datasets and baseline implementations are appropriately cited, with original sources provided in Tables~\ref{tab:datasets_references} and~\ref{tab:baseline_sources}. We respect their respective licenses and terms of use. The released code includes documentation, dependencies, and licensing information. New assets, including GitHub, PyPI, project webpages, and the Hugging Face demo, are linked in the NeurIPS checklist under ``New assets.''


\renewcommand{\thesection}{A18}
\renewcommand{\thesubsection}{A18.\arabic{subsection}}
\setcounter{figure}{0}
\renewcommand{\thefigure}{A18.\arabic{figure}}
\setcounter{table}{0}
\renewcommand{\thetable}{A18.\arabic{table}}

\section{Generalization across Diverse Data Regimes}
\label{scope}

\textbf{Generalization Beyond the Original HDLSS Scope.}
We extend BSTabDiff's evaluation to nine additional datasets spanning HDLSS, HDHSS, LDHSS, and MixedRegime settings, including TCGA-BRCA (omics), Springleaf (marketing), CIFAR-10 ResNet-50 embeddings, MiniBooNE, Adult, ORLraws10P (face images), and three high-dimensional bag-of-words text datasets (BASEHOCK, RELATHE, and PCMAC). These benchmarks cover heterogeneous data characteristics, including mixed continuous and categorical features, feature-wise missingness, and substantially different sample-to-feature ratios. The six datasets with additional accuracy evaluations range from 100 to 130{,}064 samples and from 14 to 10{,}304 features. As reported in Table~\ref{tab:extended_highdim_auc} of the main paper, at least one BSTabDiff variant achieves the highest synthetic data TSTR AUC on each of the nine datasets, with BSTabDiff-no-$p$ obtaining the best average AUC of $89.21\%$, compared with $85.16\%$ for the strongest non-BSTabDiff baseline (SMOTE). Complementary TSTR accuracy results in Table~\ref{tab:extended_highdim_results} show a consistent pattern across the six datasets with accuracy evaluation: at least one BSTabDiff variant attains the highest synthetic data TSTR accuracy on every dataset. BSTabDiff-no-$p$ achieves the best six dataset average TSTR accuracy of $80.19\%$, followed closely by BSTabDiff-d ($80.05\%$) and BSTabDiff-BS ($80.01\%$), all surpassing the strongest non-BSTabDiff baseline, SMOTE ($79.04\%$), while remaining below the real data TRTR reference ($85.36\%$). These results provide additional evidence that the block-subunit design remains effective beyond the original HDLSS setting.

\textbf{Effect of Feature Ordering on synthetic data Fidelity.}
To investigate the role of feature ordering beyond downstream predictive utility, we compare four BSTabDiff configurations on CIFAR-10 ResNet-50 embeddings. Although their TSTR AUC values are nearly identical ($98.37$-$98.40\%$; Table~\ref{tab:extended_highdim_auc}), Table~\ref{tab:cifar10_fidelity_variants} reveals complementary fidelity characteristics. Tuned BSTabDiff-FC achieves the lowest Kolmogorov-Smirnov (KS), W1, mean, and variance discrepancies, whereas tuned BSTabDiff-BS achieves the lowest Pearson and Spearman correlation gaps. The tuned no-permutation configuration obtains the smallest mutual-information discrepancy. Privacy distances remain comparable across configurations, and all variants exhibit negligible synthetic NaN fractions. These findings indicate that GO-BS and GO-BS-FC can improve different aspects of synthetic data fidelity even when downstream TSTR performance changes only minimally, highlighting the importance of evaluating marginal, dependency, and moment preservation alongside predictive utility.
\begin{table*}[t]
\centering
\scriptsize
\setlength{\tabcolsep}{3.6pt}
\renewcommand{\arraystretch}{1.18}

\caption{
\textbf{Extended TSTR accuracy (\%).}
Results on six additional datasets.
Gold, silver, and bronze cells indicate the first-, second-, and third-highest distinct synthetic data results, respectively; ties share the same color.
Bold values denote the best synthetic data result for each dataset.
TRTR is the real data reference and is excluded from ranking.
Avg(n) denotes the mean over the six available datasets.
}
\label{tab:extended_highdim_results}

\resizebox{\textwidth}{!}{%
\begin{tabular}{lccccccc}
\toprule

\rowcolor{rankheader}
\textcolor{white}{\textbf{Model}}
&
\textcolor{white}{\textbf{BRCA}\(^{*,\dagger,\ddagger}\)}
&
\textcolor{white}{\textbf{CIFAR}\(^{+}\)}
&
\textcolor{white}{\textbf{MB}\(^{\#}\)}
&
\textcolor{white}{\textbf{Spring}\(^{+,\dagger,\ddagger}\)}
&
\textcolor{white}{\textbf{ORL}\(^{*}\)}
&
\textcolor{white}{\textbf{Adult}\(^{\#}\)}
&
\cellcolor{aucAvgHead}\textcolor{white}{\textbf{Avg(n)}}
\\
\midrule


\rowcolor{aucBSTabA}
\cellcolor{aucBSTabName}
\textbf{BSTabDiff-no-\(p\)}
& \(81.16 \pm 3.13\)
& \cellcolor{goldcell}\(\mathbf{82.42 \pm 0.47}\)
& \cellcolor{bronzecell}\(85.62 \pm 2.94\)
& \cellcolor{silvercell}\(65.82 \pm 1.27\)
& \cellcolor{goldcell}\(\mathbf{86.80 \pm 5.64}\)
& \(79.30 \pm 1.59\)
& \cellcolor{goldcell}\(\mathbf{80.19\;(6)}\)
\\

\rowcolor{aucBSTabB}
\cellcolor{aucBSTabName}
\textbf{BSTabDiff-BS}
& \cellcolor{goldcell}\(\mathbf{81.59 \pm 2.93}\)
& \cellcolor{silvercell}\(82.28 \pm 0.44\)
& \(83.76 \pm 11.32\)
& \cellcolor{goldcell}\(\mathbf{66.04 \pm 1.23}\)
& \cellcolor{goldcell}\(\mathbf{86.80 \pm 5.64}\)
& \cellcolor{silvercell}\(79.61 \pm 1.09\)
& \cellcolor{bronzecell}\(80.01\;(6)\)
\\

\rowcolor{aucBSTabA}
\cellcolor{aucBSTabName}
\textbf{BSTabDiff-FC}
& \(81.36 \pm 2.67\)
& \cellcolor{bronzecell}\(82.25 \pm 0.48\)
& \cellcolor{silvercell}\(85.76 \pm 2.51\)
& \(64.73 \pm 1.91\)
& \cellcolor{bronzecell}\(86.00 \pm 6.48\)
& \cellcolor{bronzecell}\(79.48 \pm 1.32\)
& \(79.93\;(6)\)
\\

\rowcolor{aucBSTabB}
\cellcolor{aucBSTabName}
\textbf{BSTabDiff-d}
& \(81.11 \pm 3.29\)
& \(82.22 \pm 0.68\)
& \cellcolor{goldcell}\(\mathbf{86.73 \pm 1.07}\)
& \(63.72 \pm 1.58\)
& \cellcolor{silvercell}\(86.60 \pm 6.28\)
& \cellcolor{goldcell}\(\mathbf{79.93 \pm 1.10}\)
& \cellcolor{silvercell}\(80.05\;(6)\)
\\

\midrule


\rowcolor{aucBaseA}
TabDiff
& \(81.25 \pm 3.97\)
& \(71.89 \pm 2.65\)
& \(82.27 \pm 1.18\)
& \(64.94 \pm 5.55\)
& \(44.20 \pm 10.28\)
& \(76.22 \pm 2.16\)
& \cellcolor{aucAvg}70.13\;(6)
\\

\rowcolor{aucBaseB}
CtrTab
& \(81.40 \pm 3.12\)
& \(76.77 \pm 1.58\)
& \(81.51 \pm 1.60\)
& \(65.12 \pm 2.60\)
& \(82.40 \pm 5.80\)
& \(75.76 \pm 2.12\)
& \cellcolor{aucAvg}77.16\;(6)
\\

\rowcolor{aucBaseA}
TabDDPM
& \cellcolor{bronzecell}\(81.42 \pm 0.29\)
& \(66.10 \pm 2.78\)
& \(76.12 \pm 6.96\)
& \cellcolor{bronzecell}\(65.74 \pm 0.01\)
& \(60.10 \pm 6.01\)
& \(73.69 \pm 5.46\)
& \cellcolor{aucAvg}70.53\;(6)
\\

\rowcolor{aucBaseB}
BinaryDiff
& \(53.99 \pm 18.29\)
& \(52.10 \pm 2.36\)
& \(54.84 \pm 22.08\)
& \(53.54 \pm 12.69\)
& \(51.20 \pm 9.27\)
& \(76.94 \pm 2.39\)
& \cellcolor{aucAvg}57.10\;(6)
\\

\rowcolor{aucBaseA}
SMOTE
& \cellcolor{silvercell}\(81.49 \pm 3.32\)
& \(79.83 \pm 1.46\)
& \(84.21 \pm 2.55\)
& \(65.30 \pm 2.94\)
& \(85.40 \pm 5.58\)
& \(77.98 \pm 1.53\)
& \cellcolor{aucAvg}79.04\;(6)
\\

\midrule


\rowcolor{aucTRTR}
\textbf{TRTR}
& \(90.41 \pm 1.74\)
& \(87.45 \pm 0.54\)
& \(87.45 \pm 0.24\)
& \(68.76 \pm 0.61\)
& \(98.40 \pm 2.73\)
& \(79.67 \pm 0.64\)
& \cellcolor{aucAvg}\(85.36\;(6)\)
\\

\bottomrule
\end{tabular}%
}

\vspace{0.45em}
\raggedright
\footnotesize
\textbf{Dataset notation:}
BRCA = TCGA-BRCA,
CIFAR = CIFAR-10 ResNet-50 embeddings,
MB = MiniBooNE,
Spring = Springleaf,
ORL = ORLraws10P.
\(^{*}\)HDLSS;
\(^{+}\)HDHSS;
\(^{\#}\)LDHSS;
\(^{\dagger}\)mixed continuous and categorical features;
\(^{\ddagger}\)feature-wise missing values.

\vspace{0.25em}
\textbf{Ranking highlights:}
\colorbox{goldcell}{\strut 1st},
\colorbox{silvercell}{\strut 2nd},
\colorbox{bronzecell}{\strut 3rd}.
Ranks are determined from the reported dataset-level mean accuracies and exclude TRTR.
\end{table*}

\begin{table*}[t]
\centering
\scriptsize
\setlength{\tabcolsep}{3.2pt}
\renewcommand{\arraystretch}{1.22}

\caption{
\textbf{CIFAR-10 ResNet-50 fidelity across BSTabDiff variants.}
Gold, silver, and bronze cells indicate the
first-, second-, and third-best distinct values
for each metric, respectively; ties share
the same color.
Bold denotes the best result for each metric.
Lower is better for all discrepancy metrics
and NaN fraction; higher privacy distance
indicates less direct sample proximity.
}
\label{tab:cifar10_fidelity_variants}

\resizebox{\textwidth}{!}{%
\begin{tabular}{lccccccccc}
\toprule


\rowcolor{rankheader}
\textcolor{white}{\textbf{Method}}
&
\textcolor{white}{
\shortstack{\textbf{KS}\\$\downarrow$}}
&
\textcolor{white}{
\shortstack{\textbf{W1}\\$\downarrow$}}
&
\textcolor{white}{
\shortstack{\textbf{Pearson}\\$\downarrow$}}
&
\textcolor{white}{
\shortstack{\textbf{Spearman}\\$\downarrow$}}
&
\textcolor{white}{
\shortstack{\textbf{Mean gap}\\$\downarrow$}}
&
\textcolor{white}{
\shortstack{\textbf{Variance gap}\\$\downarrow$}}
&
\textcolor{white}{
\shortstack{
$\boldsymbol{|\Delta\mathrm{MI}|}$\\
$\downarrow$}}
&
\textcolor{white}{
\shortstack{\textbf{NaN fraction}\\$\downarrow$}}
&
\textcolor{white}{
\shortstack{\textbf{Privacy mean}\\$\uparrow$}}
\\

\midrule


\rowcolor{aucBSTabA}
\cellcolor{aucBSTabName}
\textbf{BSTabDiff-no-\(p\)}
& \fidsilver{0.0810}
& \fidsilver{0.0556}
& \fidsilver{0.0966}
& \fidsilver{0.1009}
& \fidsilver{0.0314}
& \(0.0672\)
& \fidgold{0.2322}
& \fidgold{8.79\times10^{-5}}
& \fidsilver{17.733}
\\


\rowcolor{aucBSTabB}
\cellcolor{aucBSTabName}
\textbf{BSTabDiff-BS}
& \fidbronze{0.0817}
& \fidbronze{0.0563}
& \fidgold{0.0965}
& \fidgold{0.1008}
& \fidbronze{0.0319}
& \fidbronze{0.0666}
& \fidsilver{0.2323}
& \fidgold{8.79\times10^{-5}}
& \fidbronze{17.698}
\\


\rowcolor{aucBSTabA}
\cellcolor{aucBSTabName}
\textbf{BSTabDiff-d}
& \(0.0915\)
& \(0.0630\)
& \(0.1009\)
& \(0.1064\)
& \(0.0436\)
& \fidsilver{0.0657}
& \(0.2375\)
& \(1.03\times10^{-4}\)
& \fidgold{17.812}
\\


\rowcolor{aucBSTabB}
\cellcolor{aucBSTabName}
\textbf{BSTabDiff-FC}
& \fidgold{0.0773}
& \fidgold{0.0538}
& \fidbronze{0.0967}
& \fidbronze{0.1015}
& \fidgold{0.0306}
& \fidgold{0.0655}
& \fidbronze{0.2350}
& \fidgold{8.79\times10^{-5}}
& \(17.613\)
\\

\bottomrule
\end{tabular}%
}

\vspace{0.4em}

\raggedright
\footnotesize

\textbf{Configurations:}
no-\(p\): tuned without permutation;
BS: tuned with GO-BS;
FC: tuned with GO-BS-FC;
d: default configuration.

\vspace{0.2em}

\textbf{Metrics:}
KS, W1, correlation, moment, and
mutual-information values represent
mean feature-wise discrepancies.
Privacy mean is a sample-distance proxy,
not a formal privacy guarantee.

\vspace{0.2em}

\textbf{Ranking highlights:}
\colorbox{goldcell}{\strut 1st},
\colorbox{silvercell}{\strut 2nd},
\colorbox{bronzecell}{\strut 3rd}.
The three tuned configurations share the
lowest reported NaN fraction.

\end{table*}


\renewcommand{\thesection}{A19}
\renewcommand{\thesubsection}{A19.\arabic{subsection}}
\setcounter{figure}{0}
\renewcommand{\thefigure}{A19.\arabic{figure}}
\setcounter{table}{0}
\renewcommand{\thetable}{A19.\arabic{table}}

\section{Block-Latent Dimension Selection and Sensitivity Analysis}
\label{app:m_selection_sensitivity}

\textbf{Motivation and notation.} The number of blocks $M$ is a central hyperparameter of BSTabDiff, controlling both the dimensionality of the block-latent representation $h\in\mathbb{R}^M$ and the granularity of the feature partition. An excessively small $M$ may limit the resolution of the learned block structure, whereas a large $M$ can increase the computational burden of dependency-aware block construction. We therefore distinguish between a practical initialization rule and subsequent dataset-specific hyperparameter selection. Consistent with Appendix~\ref{app:micro_macro_regimes}, let $n$ denote the number of samples and $m$ the number of input features. We use the shape parameter $\rho=m/n$ and the scale parameter $\kappa=nm$, where $\rho$ characterizes feature-to-sample imbalance and $\kappa$ represents the total number of data-matrix cells. The block count satisfies $1\leq M\leq m$.

\textbf{Practical selection of $M$.} For tuned configurations, Optuna~\citep{optuna} selects dataset-specific values of $M$ together with other hyperparameters. For the default configuration and as practical guidance for previously unseen datasets, we introduce a two stage initialization based on $\rho$ and $\kappa$. This strategy is motivated by the complexity-adaptive NSC design of GOTabPFN~\citep{gotabpfn}, which avoids unnecessary compression for low-dimensional inputs and adjusts representation complexity according to dataset characteristics. Considering sample size and feature dimensionality jointly, rather than relying on $m$ alone, is also consistent with BeyondArena's dataset-scale characterization~\citep{purucker2026beyond}. Specifically, low-dimensional datasets with $m\leq64$ retain one block per feature. For datasets with $m>64$ but $\rho<10$, we use a nominal budget $M=64$. For strongly high-dimensional datasets with $m>64$ and $\rho\geq10$, the initial budget increases according to the matrix scale $\kappa$:

\begin{tcolorbox}[
    colback=regimeHDHSS,
    colframe=algborder,
    boxrule=0.5pt,
    arc=1mm,
    boxsep=2pt,
    left=3pt,
    right=3pt,
    top=2pt,
    bottom=2pt,
    before skip=5pt,
    after skip=5pt
]
\begin{equation}
\label{eq:m_selection_init}
M_{\mathrm{init}}=
\begin{cases}
m,
& m\leq64\\[2pt]
64,
& m>64,\ \rho<10\\[2pt]
\min\{M_0(\kappa),m\},
& m>64,\ \rho\geq10
\end{cases}
\end{equation}

where
\begin{equation}
\label{eq:m_selection_budget}
M_0(\kappa)=
\begin{cases}
32,
& \kappa<2.5\times10^5\\[2pt]
64,
& 2.5\times10^5\leq\kappa<10^6\\[2pt]
128,
& 10^6\leq\kappa<3\times10^6\\[2pt]
192,
& \kappa\geq3\times10^6
\end{cases}
\end{equation}
\end{tcolorbox}

The condition $\rho\geq10$ is used only to identify strong feature-to sample imbalance for block-budget initialization; it does not redefine the five dataset regimes in Table~\ref{tab:dataset_regime_rules}. Likewise, the thresholds in Eq.~\ref{eq:m_selection_budget} are fixed heuristic budget thresholds, distinct from the regime-specific Micro-Macro median thresholds $T_r$ in Appendix~\ref{app:micro_macro_regimes}. Consequently, $M_{\mathrm{init}}$ is a practical starting point rather than a theoretically optimal block count.

\textbf{Dataset-wise application.}
Table~\ref{tab:m_selection_application} summarizes the application of the proposed $(\rho,\kappa)$-based block-budget selection rule across the eight original HDLSS benchmarks and nine additional datasets evaluated later. The rule reproduces the reported default budgets for COL, LNG, AML, PRS, GLI, and SMK, while ARC and TOX retain smaller values because their block-coherence or validation-utility diagnostics stabilized at more conservative budgets. For the additional datasets, MiniBooNE and Adult retain $M=m$ because $m\leq64$, whereas TCGA-BRCA, Springleaf, CIFAR-10, BASEHOCK, RELATHE, and PCMAC use the nominal budget $M=64$ because they do not satisfy the strong-HDLSS initialization condition. The extreme-HDLSS ORLraws10P dataset receives $M=128$. For consistency with the original benchmark metadata, $\rho$ and $\kappa$ are computed from the full dataset dimensions rather than individual training-fold sizes. These results illustrate how the proposed rule adapts the default block budget across diverse dimensionality regimes while allowing conservative adjustments when smaller values provide satisfactory validation utility. The proposed selection rule provides a simple initialization strategy across different dataset shapes and scales. In particular, datasets with $m\leq64$ retain one block per feature, datasets with $m>64$ and $\rho<10$ use a nominal budget of $64$, and strongly HDLSS datasets use a scale-dependent budget determined by $\kappa$. Although the default rule reproduces most reported settings, the ARC and TOX exceptions illustrate that practitioners may retain smaller values when block coherence or downstream validation utility stabilizes. Thus, the rule should be interpreted as a practical initialization rather than a universally optimal choice of $M$.

\begin{table*}[t]
\centering
\scriptsize
\setlength{\tabcolsep}{3.5pt}
\renewcommand{\arraystretch}{1.18}

\caption{\textbf{Dataset-wise application of the $(\rho,\kappa)$-based block-budget selection rule.} Here, $n$ and $m$ denote the number of samples and features, respectively, with $\rho=m/n$ and $\kappa=nm$. $M_{\mathrm{init}}$ denotes the heuristic initialization, while $M_{\mathrm{default}}$ denotes the block budget used in the default configuration. Peach cells highlight cases where a smaller budget was retained. All values are based on full dataset dimensions.}
\label{tab:m_selection_application}

\begin{tabular*}{\textwidth}{@{\extracolsep{\fill}}llrrrrcc@{}}
\toprule

\rowcolor{rankheader}
\textcolor{white}{\textbf{Dataset}} &
\textcolor{white}{\textbf{Regime}} &
\textcolor{white}{\textbf{$n$}} &
\textcolor{white}{\textbf{$m$}} &
\textcolor{white}{\textbf{$\rho$}} &
\textcolor{white}{\textbf{$\kappa$}} &
\textcolor{white}{\textbf{$M_{\mathrm{init}}$}} &
\textcolor{white}{\textbf{$M_{\mathrm{default}}$}} \\
\midrule

\multicolumn{8}{l}{\textit{Original HDLSS benchmarks}} \\
\midrule

\rowcolor{aucBSTabA}
COL & HDLSS & 62 & 2,000 & 32.26 & 124,000 & 32 & 32 \\

\rowcolor{aucBSTabB}
LNG & HDLSS & 203 & 3,312 & 16.32 & 672,336 & 64 & 64 \\

\rowcolor{aucBSTabA}
AML & HDLSS & 72 & 7,129 & 99.01 & 513,288 & 64 & 64 \\

\rowcolor{aucBSTabB}
PRS & HDLSS & 102 & 5,966 & 58.49 & 608,532 & 64 & 64 \\

\rowcolor{aucBSTabA}
GLI & HDLSS & 85 & 22,283 & 262.15 & 1,894,055 & 128 & 128 \\

\rowcolor{aucBSTabB}
SMK & HDLSS & 187 & 19,993 & 106.91 & 3,738,691 & 192 & 192 \\

\rowcolor{aucBSTabA}
ARC & HDLSS & 200 & 10,000 & 50.00 & 2,000,000 & 128 & \cellcolor{regimeLDLSS}\textbf{64} \\

\rowcolor{aucBSTabB}
TOX & HDLSS & 171 & 5,748 & 33.61 & 982,908 & 64 & \cellcolor{regimeLDLSS}\textbf{32} \\

\midrule
\multicolumn{8}{l}{\textit{Additional datasets evaluated}} \\
\midrule

\rowcolor{aucBSTabA}
TCGA-BRCA & HDLSS & 705 & 1,940 & 2.75 & 1,367,700 & 64 & 64 \\

\rowcolor{aucBSTabB}
CIFAR-10 & HDHSS & 11,000 & 2,048 & 0.186 & 22,528,000 & 64 & 64 \\

\rowcolor{aucBSTabA}
MiniBooNE & LDHSS & 130,064 & 50 & 0.00038 & 6,503,200 & 50 & 50 \\

\rowcolor{aucBSTabB}
Springleaf & HDHSS & 29,046 & 1,876 & 0.065 & 54,490,296 & 64 & 64 \\

\rowcolor{aucBSTabA}
ORLraws10P & HDLSS & 100 & 10,304 & 103.04 & 1,030,400 & 128 & 128 \\

\rowcolor{aucBSTabB}
Adult & LDHSS & 48,842 & 14 & 0.00029 & 683,788 & 14 & 14 \\

\rowcolor{aucBSTabA}
BASEHOCK & MixedRegime & 1,993 & 4,862 & 2.44 & 9,689,966 & 64 & 64 \\

\rowcolor{aucBSTabB}
RELATHE & MixedRegime & 1,427 & 4,322 & 3.03 & 6,167,494 & 64 & 64 \\

\rowcolor{aucBSTabA}
PCMAC & MixedRegime & 1,943 & 3,289 & 1.69 & 6,390,527 & 64 & 64 \\

\bottomrule
\end{tabular*}

\vspace{0.3em}
\raggedright
\footnotesize
\textbf{Notes:} The original eight datasets follow the block-budget settings reported in Table~\ref{tab:bstabdiff_hparams_all8_rot}. For ARC and TOX, the default budgets were reduced relative to the heuristic initialization. The nine additional datasets include mixed-type omics and marketing data, image embeddings, particle physics, census data, face images, and high-dimensional bag of words text representations. The reported budgets refer to the default selection procedure and do not represent independently Optuna-tuned block counts.

\end{table*}

\textbf{Completed sensitivity analysis.} To evaluate the sensitivity of BSTabDiff to the block-latent dimension, we conduct experiments on the Colon dataset for $M\in\{8,16,32,64,128\}$, both with and without GO-BS. The experiments use Logistic Regression for downstream TSTR evaluation and strict $5\times5$ repeated cross-validation. All other generator settings are fixed at their tuned values within each configuration, and every reported row contains $25/25$ completed folds. Table~\ref{tab:M_sensitivity_colon_complete} presents the completed results, which supersede the preliminary partial-fold sensitivity analysis reported during additional experiments.

\begin{table*}[t]
\centering
\scriptsize
\setlength{\tabcolsep}{6pt}
\renewcommand{\arraystretch}{1.20}

\caption{\textbf{Completed sensitivity analysis of block-latent dimension $M$ on Colon.} All configurations are evaluated using strict $5\times5$ repeated cross-validation with $25/25$ completed folds. All other generator hyperparameters are fixed within each configuration. TSTR accuracy and AUC are reported as mean with subscript standard deviation. Time is the mean fitting time per fold.}
\label{tab:M_sensitivity_colon_complete}

\begin{tabular*}{\textwidth}{@{\extracolsep{\fill}}lcccc@{}}
\toprule

\rowcolor{rankheader}
\textcolor{white}{\textbf{Configuration}} &
\textcolor{white}{\textbf{$M$}} &
\textcolor{white}{\textbf{TSTR Acc. (\%)}} &
\textcolor{white}{\textbf{TSTR AUC (\%)}} &
\textcolor{white}{\textbf{Time/Fold (s)}} \\
\midrule


\rowcolor{aucBSTabA}
\cellcolor{aucBSTabName}\textbf{Tuned GO-BS}
& 8
& \aucval{84.18}{10.69}
& \aucval{88.55}{9.45}
& \cellcolor{regimeLDHSS}11.76 \\

\rowcolor{aucBSTabB}
\cellcolor{aucBSTabName}\textbf{Tuned GO-BS}
& 16
& \aucval{84.18}{10.69}
& \aucval{88.22}{9.83}
& 36.02 \\

\rowcolor{aucBSTabA}
\cellcolor{aucBSTabName}\textbf{Tuned GO-BS}
& 32
& \aucval{84.18}{10.69}
& \aucval{88.40}{9.98}
& 179.09 \\

\rowcolor{aucBSTabB}
\cellcolor{aucBSTabName}\textbf{Tuned GO-BS}
& 64
& \aucval{84.16}{10.46}
& \aucval{88.67}{9.52}
& 1090.90 \\

\rowcolor{aucBSTabA}
\cellcolor{aucBSTabName}\textbf{Tuned GO-BS}
& 128
& \aucval{83.87}{10.56}
& \aucval{88.35}{9.77}
& \cellcolor{regimeLDLSS}\textbf{7574.60} \\

\midrule


\rowcolor{aucBaseA}
\textbf{Tuned no-GO-BS}
& 8
& \aucval{84.18}{10.69}
& \aucval{88.62}{9.76}
& \cellcolor{regimeLDHSS}5.13 \\

\rowcolor{aucBaseB}
\textbf{Tuned no-GO-BS}
& 16
& \aucval{83.26}{10.95}
& \aucval{88.32}{9.52}
& 5.89 \\

\rowcolor{aucBaseA}
\textbf{Tuned no-GO-BS}
& 32
& \aucval{84.18}{10.69}
& \aucval{88.53}{9.95}
& 6.20 \\

\rowcolor{aucBaseB}
\textbf{Tuned no-GO-BS}
& 64
& \aucval{84.18}{10.69}
& \aucval{88.45}{9.72}
& 6.13 \\

\rowcolor{aucBaseA}
\textbf{Tuned no-GO-BS}
& 128
& \aucval{83.87}{10.56}
& \aucval{88.35}{9.39}
& 5.97 \\

\bottomrule
\end{tabular*}

\vspace{0.35em}
\raggedright
\footnotesize
\textbf{Runtime highlights:} Mint identifies the lowest fit time within each configuration; peach indicates the most expensive GO-BS setting. All reported means are based on the full 25 completed folds.

\end{table*}

\textbf{Sensitivity findings.} The completed results indicate that downstream utility is relatively stable over a substantial variation in $M$. With GO-BS, mean TSTR AUC ranges from $88.22\%$ to $88.67\%$, a difference of only $0.45$ percentage points. Without GO-BS, the corresponding range is $88.32\%$-$88.62\%$, a difference of $0.30$ percentage points. Mean accuracy varies by at most $0.31$ and $0.92$ percentage points with and without GO-BS, respectively.These differences are small relative to the reported cross-validation variability and provide no consistent evidence that larger $M$ improves downstream predictive performance. In contrast, fitting time increases substantially with $M$ when GO-BS is enabled. The mean time rises from $11.76$ s/fold at $M=8$ to $36.02$ s/fold at $M=16$, $179.09$ s/fold at $M=32$, $1090.90$ s/fold at $M=64$, and $7574.60$ s/fold at $M=128$. This corresponds to an approximately $644\times$ increase from $M=8$ to $M=128$. By comparison, the no-GO-BS configuration remains near $5$-$6$ s/fold throughout the investigated range. Thus, increasing $M$ provides no consistent TSTR utility gain in this experiment while substantially increasing the computational cost of the GO-BS-enabled configuration. The results motivate selecting a sufficiently small block budget rather than increasing the latent resolution without additional validation evidence.

\textbf{Practical recommendation.} For previously unseen datasets, we recommend using the $(\rho,\kappa)$-based rule in Eqs.~\ref{eq:m_selection_init} and~\ref{eq:m_selection_budget} to initialize the block budget. The practitioner may then refine the value through dataset-specific validation, considering both downstream predictive utility and block-coherence diagnostics. When multiple candidate budgets yield comparable performance, the smallest satisfactory value of $M$ should be preferred, particularly when GO-BS is used. The Colon sensitivity experiment supports this conservative strategy: relatively small block budgets achieve comparable downstream performance at substantially lower computational cost. However, the observed stability is an empirical finding for the evaluated dataset and settings, not a guarantee that the same range of $M$ is optimal for all datasets.
\renewcommand{\thesection}{A20}
\renewcommand{\thesubsection}{A20.\arabic{subsection}}
\setcounter{figure}{0}
\renewcommand{\thefigure}{A20.\arabic{figure}}
\setcounter{table}{0}
\renewcommand{\thetable}{A20.\arabic{table}}

\section{Necessity of the Block-Subunit Structure}
\label{app:block_subunit_necessity}

\textbf{Motivation and parameter efficiency.}
A dense latent-factor model provides a natural alternative to the block-subunit formulation of BSTabDiff. For an observed feature dimension $m$ and latent dimension $M\ll m$, a dense factor decoder of the form $Z=Wh+b(Y)+\xi$ uses a loading matrix $W\in\mathbb{R}^{m\times M}$, requiring $O(mM)$ loading coefficients and exhibiting the usual rotational ambiguity of factor models. In contrast, BSTabDiff assigns each observed feature to a single block latent and estimates one scalar loading per feature, reducing the emission-side loading complexity to $O(m)$. Importantly, the joint prior $p_\theta(h\mid Y)$ continues to model dependence among all $M$ block latents, so the blocks are not assumed mutually independent. The block-subunit formulation therefore provides a parameter-efficient and interpretable inductive bias for HDLSS generation rather than imposing a universally necessary modeling assumption.

\textbf{Relationship to structured latent-factor learning.}
The motivation for dependency-aware blocks is consistent with the use of sparse factor models~\citep{ac1}, graphical models~\citep{ac2}, latent trees~\citep{ac3}, and spectral biclustering~\citep{ac4} to represent structured dependencies in high-dimensional data. BSTabDiff differs from a dense latent-factor decoder by explicitly associating each latent coordinate with a dependency-coherent feature block. GO-BS estimates this organization from training data rather than assuming a meaningful original column order. Although the exact GO-BS optimization problem is NP-hard (Theorem~\ref{thm:gobs_nphard}), the implementation uses a sparse approximate solver rather than exact combinatorial optimization. Related evidence from GOTabPFN~\citep{gotabpfn} further motivates locality-aware feature ordering and compact subunit representations for HDLSS datasets, particularly when meaningful feature-block structure is present.

\textbf{Matched experimental design.}
To examine whether the block-subunit structure provides benefits beyond generic latent dimensionality reduction, we evaluate three configurations on the Colon dataset ($n=62$, $m=2000$): (i) DenseFactor-PCA, which uses a dense latent-factor decoder; (ii) BlockSparse no-GO-BS, which imposes block-sparse loadings without learned feature ordering; and (iii) BSTabDiff w/ GO-BS, which additionally learns the dependency-aware block organization. All configurations use the same reduced latent dimension $M=8$, flow prior, marginal-transformation model, missingness model, synthetic data budget, and downstream evaluation protocol. TSTR performance is evaluated using Logistic Regression under strict $5\times5$ repeated cross-validation. The comparison isolates the effects of decoder sparsity and learned block organization while keeping the remaining generative components matched.

\textbf{Downstream utility.}
Table~\ref{tab:dense_block_cv_utility} presents the matched downstream evaluation. DenseFactor-PCA achieves the highest mean TSTR accuracy ($85.62\%$), whereas the block-sparse configurations achieve higher mean TSTR AUC. Specifically, the dense baseline obtains $88.40\%$ AUC, which increases to $89.28\%$ with block-sparse loadings without GO-BS and to $89.93\%$ when the block organization is learned through GO-BS. The corresponding reported AUC efficiencies are $1.0184$, $1.0304$, and $1.0380$, respectively. Thus, BSTabDiff w/ GO-BS improves mean TSTR AUC by $1.53$ percentage points over DenseFactor-PCA while using substantially fewer decoder loading parameters. Nevertheless, these differences are small relative to the cross-validation variability, and the dense baseline retains the highest TSTR accuracy. The results therefore support a metric dependent advantage of block-subunit factorization rather than universal superiority.


\begin{table*}[t]
\centering
\scriptsize
\setlength{\tabcolsep}{3pt}
\renewcommand{\arraystretch}{1.20}

\caption{\textbf{Matched dense-factor versus block-subunit TSTR evaluation on Colon dataset.} All configurations use $m=2000$, $M=8$, and strict $5\times5$ cross-validation with Logistic Regression. AUC efficiency measures TSTR AUC relative to the TRTR AUC reference. Values are reported as mean with subscript standard deviation. Gold highlights the best result for each varying utility metric.}
\label{tab:dense_block_cv_utility}

\begin{tabular*}{\textwidth}{@{\extracolsep{\fill}}lcccccc@{}}
\toprule

\rowcolor{rankheader}
\textcolor{white}{\textbf{Method}} &
\textcolor{white}{\textbf{$M$}} &
\textcolor{white}{\textbf{TSTR Acc.}} &
\textcolor{white}{\textbf{TSTR AUC}} &
\textcolor{white}{\textbf{TRTR Acc.}} &
\textcolor{white}{\textbf{TRTR AUC}} &
\textcolor{white}{\textbf{AUC Eff.}} \\
\midrule

\rowcolor{aucBaseA}
DenseFactor-PCA
& 8
& \aucgoldstd{0.8562}{0.0817}
& \aucval{0.8840}{0.1029}
& \aucval{0.8431}{0.1058}
& \aucval{0.8703}{0.1114}
& \aucval{1.0184}{0.0504} \\

\rowcolor{aucBSTabA}
\cellcolor{aucBSTabName}\textbf{BlockSparse no-GO-BS}
& 8
& \aucval{0.8331}{0.1157}
& \aucsilverstd{0.8928}{0.1014}
& \aucval{0.8431}{0.1058}
& \aucval{0.8703}{0.1114}
& \aucsilverstd{1.0304}{0.0782} \\

\rowcolor{aucBSTabB}
\cellcolor{aucBSTabName}\textbf{BSTabDiff w/ GO-BS}
& 8
& \aucval{0.8362}{0.1270}
& \aucgoldstd{0.8993}{0.1030}
& \aucval{0.8431}{0.1058}
& \aucval{0.8703}{0.1114}
& \aucgoldstd{1.0380}{0.0830} \\

\bottomrule
\end{tabular*}

\vspace{0.3em}
\raggedright
\footnotesize
\textbf{Notes:} Higher TSTR accuracy, TSTR AUC, and AUC efficiency are better. TRTR values are identical across the matched configurations and are included as real data references. Gold and silver denote the highest and second highest varying results, respectively.

\end{table*}

\textbf{Marginal fidelity and real-synthetic distinguishability.}
The full data synthetic sample evaluation provides additional evidence concerning distributional fidelity. DenseFactor-PCA obtains mean KS and W1 discrepancies of $0.0964$ and $0.1486$, respectively. Imposing block-sparse loadings without GO-BS reduces these values to $0.0797$ and $0.1127$, while BSTabDiff w/ GO-BS achieves comparable discrepancies of $0.0799$ and $0.1132$. Thus, both block-sparse variants improve marginal fidelity relative to the dense baseline, although GO-BS does not further reduce KS or W1 compared with the block-sparse configuration without ordering. Table~\ref{tab:dense_block_cv_c2st_compute} reports complementary Classifier 2 Sample Test (C2ST) results over the $25$ cross-validation folds. DenseFactor-PCA achieves the lowest C2ST separability ($85.62\%$), compared with $86.67\%$ for BlockSparse no-GO-BS and $88.04\%$ for BSTabDiff w/ GO-BS. Since C2ST separability closer to $50\%$ indicates greater real-synthetic indistinguishability, this comparison demonstrates that improvements in downstream AUC and marginal fidelity do not necessarily coincide with improvements in C2ST.


\begin{table*}[t]
\centering
\scriptsize
\setlength{\tabcolsep}{3pt}
\renewcommand{\arraystretch}{1.20}

\caption{\textbf{C2ST and computational evaluation over the 25 Colon dataset folds.} C2ST separability is defined as $\max(\mathrm{AUC},1-\mathrm{AUC})$, with values closer to $0.5$ preferred. Raw C2ST accuracy and AUC are retained as measured; separability accounts for classifier label orientation. Fit time and peak GPU usage are reported as mean with subscript standard deviation. Gold highlights the best separability, fit time, and memory values.}
\label{tab:dense_block_cv_c2st_compute}

\begin{tabular*}{\textwidth}{@{\extracolsep{\fill}}lccccc@{}}
\toprule

\rowcolor{rankheader}
\textcolor{white}{\textbf{Method}} &
\textcolor{white}{\textbf{C2ST Acc.}} &
\textcolor{white}{\textbf{Raw AUC}} &
\textcolor{white}{\textbf{C2ST Sep.}} &
\textcolor{white}{\textbf{Fit Time (s)}} &
\textcolor{white}{\textbf{GPU (GiB)}} \\
\midrule

\rowcolor{aucBaseA}
DenseFactor-PCA
& \aucval{0.2160}{0.0746}
& \aucval{0.1438}{0.0637}
& \aucgoldstd{0.8562}{0.0637}
& \aucval{6.02}{0.80}
& \aucval{0.029795}{0.000000} \\

\rowcolor{aucBSTabA}
\cellcolor{aucBSTabName}\textbf{BlockSparse no-GO-BS}
& \aucval{0.2693}{0.0933}
& \aucval{0.1333}{0.0551}
& \aucval{0.8667}{0.0551}
& \aucgoldstd{5.89}{0.88}
& \aucgoldstd{0.029744}{0.000000} \\

\rowcolor{aucBSTabB}
\cellcolor{aucBSTabName}\textbf{BSTabDiff w/ GO-BS}
& \aucval{0.2707}{0.0735}
& \aucval{0.1196}{0.0652}
& \aucval{0.8804}{0.0652}
& \aucval{10.48}{0.80}
& \aucval{0.076296}{0.000034} \\

\bottomrule
\end{tabular*}

\vspace{0.3em}
\raggedright
\footnotesize
\textbf{Notes:} Lower C2ST separability, fit time, and peak GPU usage are preferred. Raw C2ST AUC values below $0.5$ are converted to separability using $\max(\mathrm{AUC},1-\mathrm{AUC})$. The reported statistics are based on the matched $5\times5$ evaluation.

\end{table*}

\textbf{Parameter efficiency and computational cost.}
Table~\ref{tab:dense_block_full_resources} compares the parameter counts and computational resources measured in the full data experiments. DenseFactor-PCA uses $mM=16{,}000$ loading coefficients, whereas both block-sparse variants use only $m=2{,}000$, yielding an eightfold reduction in decoder loading complexity. All three configurations share a flow prior containing $531{,}680$ parameters, resulting in total reported parameter counts of $547{,}680$ for DenseFactor-PCA and $533{,}680$ for the block-sparse variants. Thus, the eightfold reduction applies specifically to decoder loadings rather than to the entire generative model. BlockSparse no-GO-BS achieves the lowest full data training time ($6.001$ s) and peak GPU usage ($0.0297$ GiB), whereas DenseFactor-PCA achieves the lowest CPU resident memory ($1.162$ GiB). Enabling GO-BS increases the mean cross-validation fit time from $5.89$ to $10.48$ s/fold and the full data training time from $6.001$ to $7.810$ s relative to BlockSparse no-GO-BS, reflecting the additional cost of constructing a dependency-aware feature organization.


\begin{table*}[t]
\centering
\scriptsize
\setlength{\tabcolsep}{2.5pt}
\renewcommand{\arraystretch}{1.20}

\caption{\textbf{Full data resource and parameter comparison on Colon dataset.} The raw dimension is $m=2000$ and the latent dimension is $M=8$. DenseFactor-PCA uses $mM=16{,}000$ loading coefficients, while the block-sparse models require $m=2{,}000$. All configurations share the same flow prior parameter count. Gold highlights the lowest parameter and resource values, with ties sharing the same color.}
\label{tab:dense_block_full_resources}

\begin{tabular*}{\textwidth}{@{\extracolsep{\fill}}lccccccc@{}}
\toprule

\rowcolor{rankheader}
\textcolor{white}{\textbf{Method}} &
\textcolor{white}{\textbf{$M$}} &
\textcolor{white}{\shortstack{\textbf{Loading}\\\textbf{Params.}}} &
\textcolor{white}{\shortstack{\textbf{Prior}\\\textbf{Params.}}} &
\textcolor{white}{\shortstack{\textbf{Total}\\\textbf{Params.}}} &
\textcolor{white}{\shortstack{\textbf{Train}\\\textbf{Time (s)}}} &
\textcolor{white}{\shortstack{\textbf{GPU}\\\textbf{(GiB)}}} &
\textcolor{white}{\shortstack{\textbf{CPU RSS}\\\textbf{(GiB)}}} \\
\midrule

\rowcolor{aucBaseA}
DenseFactor-PCA
& 8
& 16,000
& 531,680
& 547,680
& 6.164
& 0.0298
& \cellcolor{goldcell}\textbf{1.162} \\

\rowcolor{aucBSTabA}
\cellcolor{aucBSTabName}\textbf{BlockSparse no-GO-BS}
& 8
& \cellcolor{goldcell}\textbf{2,000}
& 531,680
& \cellcolor{goldcell}\textbf{533,680}
& \cellcolor{goldcell}\textbf{6.001}
& \cellcolor{goldcell}\textbf{0.0297}
& 1.260 \\

\rowcolor{aucBSTabB}
\cellcolor{aucBSTabName}\textbf{BSTabDiff w/ GO-BS}
& 8
& \cellcolor{goldcell}\textbf{2,000}
& 531,680
& \cellcolor{goldcell}\textbf{533,680}
& 7.810
& 0.0765
& 1.476 \\

\bottomrule
\end{tabular*}

\vspace{0.3em}
\raggedright
\footnotesize
\textbf{Notes:} Parameter counts follow the matched experimental accounting of decoder loadings and the common flow prior. Full data training times are distinct from mean cross-validation fit times in Table~\ref{tab:dense_block_cv_c2st_compute}. GPU denotes peak GPU memory, and CPU RSS denotes resident CPU memory.

\end{table*}

\textbf{Relationship to block-structure assumptions.}
The block-subunit formulation is motivated by distinct within-block dependencies and sparse cross-block interactions (Assumption~\ref{ass:distinct_blocks}), but these conditions are not strict requirements for fitting BSTabDiff. We formalize the meaning of sparse latent dependence, provide practical diagnostics, and discuss assumption violations and possible remedies in Appendix~\ref{app:block_assumption_diagnostics}. The preceding experiment evaluates the empirical benefit of block-subunit factorization but does not imply that the underlying dependence assumptions hold universally. Assumption~\ref{ass:distinct_blocks} introduces distinct within-block dependence and sparse cross-block interactions as sufficient structural conditions motivating the identifiability and sample complexity arguments. Such assumptions should be assessed through within-block affinity, normalized cross-block leakage, latent conditional dependence sparsity, and bootstrap stability of the recovered feature partitions. When these diagnostics indicate weak block structure or dense cross-block dependence, the benefits of block factorization may diminish. In these settings, practitioners may relax the cross-block penalty, reconsider the block budget $M$, or disable GO-BS when it provides no validation benefit. The joint prior $p_\theta(h\mid Y)$ nevertheless continues to model dependencies among the block latents, so the model does not require strict cross-block independence.

\textbf{Interpretation and limitations.}
The controlled Colon ablation supports block-subunit factorization as a parameter-efficient, interpretable, and potentially variance-reducing inductive bias for HDLSS generation. Relative to DenseFactor-PCA, BSTabDiff w/ GO-BS achieves a higher mean TSTR AUC ($89.93\%$ versus $88.40\%$), higher AUC efficiency ($1.0380$ versus $1.0184$), and lower KS and W1 discrepancies while using $2{,}000$ rather than $16{,}000$ decoder loadings. However, the dense baseline achieves the highest TSTR accuracy and the lowest C2ST separability, while the block-sparse configuration without GO-BS attains slightly better KS and W1 values and lower computational cost than BSTabDiff w/ GO-BS. Consequently, these results do not establish that block structure is universally necessary or superior for all data distributions and metrics. Rather, they demonstrate that explicitly constrained and learned feature blocks can provide useful gains in selected utility and fidelity measures while substantially reducing emission-side loading complexity.


\renewcommand{\thesection}{A21}
\renewcommand{\thesubsection}{A21.\arabic{subsection}}
\setcounter{subsection}{0}
\setcounter{figure}{0}
\renewcommand{\thefigure}{A21.\arabic{figure}}
\setcounter{table}{0}
\renewcommand{\thetable}{A21.\arabic{table}}

\section{Component Sensitivity Analysis and Practical Configuration Guide}
\label{app:component_sensitivity}

\subsection{``One Factor at a Time" Component Sensitivity}

\textbf{Experimental protocol.}
To assess the sensitivity of BSTabDiff to individual modeling and optimization choices, we conduct an additional ``one factor at a time" analysis on the Colon dataset ($n=62$, $m=2000$). We evaluate tuned configurations with and without GO-BS using Logistic Regression under strict $5\times5$ repeated cross-validation. Each experiment varies one component while holding the remaining tuned generator parameters fixed. The investigated factors include the latent prior family, prior training epochs, prior learning rate, exponential moving average (EMA), and synthetic sample budget $n_{\mathrm{syn}}$. We report TSTR accuracy, TSTR AUC, raw C2ST AUC, and mean fitting time in Table~\ref{tab:component_sensitivity_full}. The sensitivity of the block count $M$ is evaluated separately in Appendix~\ref{app:m_selection_sensitivity}.

\textbf{Overall sensitivity.}
Completed experiments demonstrate that downstream TSTR AUC is comparatively stable across investigated settings. With GO-BS, mean TSTR AUC ranges from $88.10\%$ to $89.03\%$, compared w/ tuned baseline of $88.55\%$. Without GO-BS, corresponding range is $88.22\%$-$89.10\%$, compared w/ baseline of $88.62\%$. The largest absolute change from respective tuned baseline is $0.0048$ in both configurations, equivalent to $0.48$ percentage points. These differences are small relative to reported cross-validation variability. Consequently, none of investigated prior or optimization settings produces a consistent large improvement in downstream AUC within tested ranges.


\begin{table*}[t]
\centering
\scriptsize
\setlength{\tabcolsep}{3pt}
\renewcommand{\arraystretch}{1.18}

\caption{\textbf{Full component sensitivity analysis on Colon dataset.} Results use strict $5\times5$ repeated cross-validation with Logistic Regression. Each setting changes one component while holding the remaining tuned parameters fixed within each GO-BS configuration. Values are reported as mean with subscript standard deviation. C2ST AUC is the raw classifier AUC as reported in the experiment, without conversion to orientation-independent separability. Fit time is in seconds per fold.}
\label{tab:component_sensitivity_full}

\begin{tabular*}{\textwidth}{@{\extracolsep{\fill}}llcccc@{}}
\toprule

\rowcolor{rankheader}
\textcolor{white}{\textbf{Variant}} &
\textcolor{white}{\textbf{Setting}} &
\textcolor{white}{\textbf{TSTR Acc.}} &
\textcolor{white}{\textbf{TSTR AUC}} &
\textcolor{white}{\textbf{C2ST AUC}} &
\textcolor{white}{\textbf{Fit Time (s)}} \\
\midrule


\rowcolor{aucBSTabA}
\textbf{GO-BS} & \cellcolor{aucBSTabName}\textbf{Base}
& \aucval{0.8418}{0.1069}
& \aucval{0.8855}{0.0945}
& \aucval{0.1948}{0.1192}
& 11.758 \\

\rowcolor{aucBSTabB}
GO-BS & Matched diffusion prior
& \aucval{0.8326}{0.1049}
& \aucval{0.8903}{0.0914}
& \aucval{0.2740}{0.1305}
& 9.395 \\

\rowcolor{aucBSTabA}
GO-BS & Prior epochs $=150$
& \aucval{0.8321}{0.1179}
& \aucval{0.8825}{0.0968}
& \aucval{0.2340}{0.1239}
& 8.400 \\

\rowcolor{aucBSTabB}
GO-BS & Prior epochs $=800$
& \aucval{0.8423}{0.1158}
& \aucval{0.8853}{0.0992}
& \aucval{0.2124}{0.1117}
& 21.754 \\

\rowcolor{aucBSTabA}
GO-BS & Prior epochs $=1500$
& \aucval{0.8356}{0.1065}
& \aucval{0.8832}{0.0991}
& \aucval{0.1804}{0.1167}
& 35.037 \\

\rowcolor{aucBSTabB}
GO-BS & Prior LR $=0.5\times$
& \aucval{0.8418}{0.1069}
& \aucval{0.8810}{0.0994}
& \aucval{0.1980}{0.1112}
& 11.815 \\

\rowcolor{aucBSTabA}
GO-BS & Prior LR $=2\times$
& \aucval{0.8385}{0.1085}
& \aucval{0.8850}{0.0983}
& \aucval{0.1856}{0.0943}
& 11.682 \\

\rowcolor{aucBSTabB}
GO-BS & No EMA
& \aucval{0.8418}{0.1069}
& \aucval{0.8855}{0.0945}
& \aucval{0.1948}{0.1192}
& 11.448 \\

\rowcolor{aucBSTabA}
GO-BS & $n_{\mathrm{syn}}=100$
& \aucval{0.8390}{0.1050}
& \aucval{0.8830}{0.0978}
& \aucval{0.1632}{0.0642}
& 11.963 \\

\rowcolor{aucBSTabB}
GO-BS & $n_{\mathrm{syn}}=500$
& \aucval{0.8354}{0.1123}
& \aucval{0.8845}{0.0956}
& \aucval{0.1904}{0.1171}
& 11.775 \\

\midrule


\rowcolor{aucBaseA}
\textbf{No GO-BS} & \cellcolor{aucAvg}\textbf{Base}
& \aucval{0.8418}{0.1069}
& \aucval{0.8862}{0.0976}
& \aucval{0.1732}{0.1075}
& 5.129 \\

\rowcolor{aucBaseB}
No GO-BS & Matched diffusion prior
& \aucval{0.8326}{0.1095}
& \aucval{0.8910}{0.0947}
& \aucval{0.2784}{0.1330}
& 3.558 \\

\rowcolor{aucBaseA}
No GO-BS & Prior epochs $=150$
& \aucval{0.8321}{0.1179}
& \aucval{0.8822}{0.0980}
& \aucval{0.2344}{0.1227}
& 2.821 \\

\rowcolor{aucBaseB}
No GO-BS & Prior epochs $=800$
& \aucval{0.8392}{0.1208}
& \aucval{0.8842}{0.0978}
& \aucval{0.2104}{0.1022}
& 13.416 \\

\rowcolor{aucBaseA}
No GO-BS & Prior epochs $=1500$
& \aucval{0.8418}{0.1069}
& \aucval{0.8822}{0.0986}
& \aucval{0.1564}{0.0969}
& 25.185 \\

\rowcolor{aucBaseB}
No GO-BS & Prior LR $=0.5\times$
& \aucval{0.8418}{0.1069}
& \aucval{0.8842}{0.0968}
& \aucval{0.1724}{0.1072}
& 5.436 \\

\rowcolor{aucBaseA}
No GO-BS & Prior LR $=2\times$
& \aucval{0.8418}{0.1069}
& \aucval{0.8853}{0.0968}
& \aucval{0.1744}{0.1087}
& 5.335 \\

\rowcolor{aucBaseB}
No GO-BS & No EMA
& \aucval{0.8418}{0.1069}
& \aucval{0.8862}{0.0976}
& \aucval{0.1732}{0.1075}
& 6.208 \\

\rowcolor{aucBaseA}
No GO-BS & $n_{\mathrm{syn}}=100$
& \aucval{0.8323}{0.1078}
& \aucval{0.8858}{0.0930}
& \aucval{0.1660}{0.0786}
& 6.489 \\

\rowcolor{aucBaseB}
No GO-BS & $n_{\mathrm{syn}}=500$
& \aucval{0.8323}{0.1174}
& \aucval{0.8862}{0.0971}
& \aucval{0.1836}{0.1276}
& 5.836 \\

\bottomrule
\end{tabular*}

\vspace{0.3em}
\raggedright
\footnotesize
\textbf{Notes:} Base denotes the tuned reference configuration. LR denotes the prior learning rate, EMA denotes exponential moving average, and $n_{\mathrm{syn}}$ denotes the synthetic sample budget. The GO-BS and no-GO-BS configurations are evaluated separately under the same repeated cross-validation protocol.

\end{table*}

\textbf{Component-wise changes.}
Table~\ref{tab:component_sensitivity_summary} summarizes the largest absolute change in mean TSTR AUC relative to the tuned baseline for each component. The prior family change produces the largest observed deviation ($0.0048$) in both configurations. Varying prior training epochs changes AUC by at most $0.0030$ with GO-BS and $0.0040$ without GO-BS, whereas the learning rate changes produce deviations of at most $0.0045$ and $0.0020$, respectively. Changes in the synthetic sample budget have smaller effects, and disabling EMA does not change the reported mean TSTR AUC at four decimal precision. These results suggest that the remaining prior and optimization controls are comparatively stable within the investigated ranges.


\begin{table}[t]
\centering
\scriptsize
\setlength{\tabcolsep}{4pt}
\renewcommand{\arraystretch}{1.16}

\caption{\textbf{Maximum absolute change in mean TSTR AUC relative to the tuned baseline.} Values summarize the ``one factor at a time" experiments in Table~\ref{tab:component_sensitivity_full}. Lower values indicate less sensitivity within the tested settings.}
\label{tab:component_sensitivity_summary}

\begin{tabular*}{\columnwidth}{@{\extracolsep{\fill}}lcc@{}}
\toprule

\rowcolor{rankheader}
\textcolor{white}{\textbf{Component}} &
\textcolor{white}{\textbf{GO-BS}} &
\textcolor{white}{\textbf{No GO-BS}} \\
\midrule

\rowcolor{aucBSTabA}
Prior family & 0.0048 & 0.0048 \\

\rowcolor{aucBSTabB}
Prior epochs & 0.0030 & 0.0040 \\

\rowcolor{aucBSTabA}
Prior learning rate & 0.0045 & 0.0020 \\

\rowcolor{aucBSTabB}
Synthetic sample budget & 0.0025 & 0.0005 \\

\rowcolor{aucBSTabA}
EMA & 0.0000 & 0.0000 \\

\bottomrule
\end{tabular*}

\vspace{0.3em}
\raggedright
\footnotesize
\textbf{Notes:} The summary follows the reported additional experiments' values. The no GO-BS synthetic budget change is reported as $0.0005$; the displayed four decimal means in Table~\ref{tab:component_sensitivity_full} imply $0.0004$ because of numerical precision.

\end{table}

\textbf{Computational sensitivity.}
Although downstream AUC varies only modestly, training cost is more sensitive to the prior epoch budget. With GO-BS, fitting time increases from $8.400$ s at 150 prior epochs to $35.037$ s at 1500 epochs, while mean TSTR AUC remains between $0.8825$ and $0.8853$ across the three tested epoch settings. Without GO-BS, the corresponding fitting times increase from $2.821$ to $25.185$ s, with mean AUC between $0.8822$ and $0.8842$. The matched diffusion prior reduces fitting time relative to the tuned baseline in both configurations, while changing the synthetic sample budget or disabling EMA has a smaller effect on mean TSTR AUC. Therefore, increasing prior training effort does not necessarily provide proportional improvements in downstream utility.

\textbf{Importance of class conditional marginals.}
The original compact ablation in Appendix~\ref{abla}, Table~\ref{tab:bstabdiff_compact_ablations}, provides complementary evidence about the importance of class conditional marginal transformations. In that experiment, which uses TabPFN-2.5~\citep{tabpfn2.5} rather than the Logistic Regression evaluation employed here, disabling class conditional marginals reduces mean TSTR AUC from $0.886$ to $0.872$. This is the clearest utility degradation among the original component removal settings and suggests that preserving class-specific feature distributions is particularly important. Class conditional marginals should therefore remain enabled in the default model rather than being treated as an optional optimization parameter. The two ablation protocols are complementary and should not be interpreted as a single jointly controlled experiment.


\subsection{Practical Configuration Guide}
\label{app:practical_configuration_guide}

\textbf{Recommended tuning priorities.}
The original ablation and the completed component sensitivity analysis motivate a practical tuning order. First, retain class conditional marginal transformations, which preserve class-specific feature distributions and were particularly important in the compact ablation. Second, select the block count $M$ and construct an appropriate dependency-aware feature partition, since these choices control the representation structure and can substantially affect computational cost. Third, tune prior-specific settings, including the prior family, learning rate, training epochs, EMA, and synthetic sample budget. The relatively small AUC deviations in Table~\ref{tab:component_sensitivity_summary} suggest that extensive prior-specific tuning may offer limited additional utility once an appropriate structural configuration has been established.

\textbf{Selecting the block budget and ordering variant.}
For previously unseen datasets, practitioners can initialize the block count using the ``shape-scale" rule based on $\rho=m/n$ and $\kappa=nm$ in Appendix~\ref{app:m_selection_sensitivity}. The selected value may subsequently be refined using block-coherence diagnostics and held-out utility, with preference given to the smallest satisfactory budget when GO-BS computation is expensive. The choice between sample-clustered GO-BS and feature-clustered GO-BS-FC follows the Micro-Macro regime-aware selection rule in Appendix~\ref{app:gobs_variant_selection}. These rules provide practical starting configurations rather than guarantees of optimality.

\textbf{Block-structure diagnostics.}
A selected block organization can be examined using three complementary diagnostics: (i) low normalized cross-block leakage relative to within-block feature affinity, indicating that strongly dependent features are concentrated within recovered blocks; (ii) stable block assignments under bootstrap resampling, indicating that the discovered partition is not overly sensitive to individual training samples; and (iii) limited residual cross-block dependence after conditioning on the inferred block latents, indicating that the latent representation adequately explains observed dependency structure. Additional discussion of sparse latent conditional dependence, assumption violations, and recovery diagnostics is provided in Appendix~\ref{app:block_assumption_diagnostics}. These diagnostics are recommended validation procedures and are not claimed to constitute formal tests of block identifiability.

\textbf{Missingness and failure modes.}
For complete datasets such as Colon, the missingness mechanism requires no meaningful dataset-specific tuning. For datasets containing missing values, practitioners can compare real and synthetic per-feature missingness rates as well as class-conditional missingness patterns when labels are available. If the recovered blocks show weak coherence or poor bootstrap stability, and GO-BS does not improve held-out utility or fidelity, practitioners may reduce the cross-block leakage penalty, reconsider the block budget $M$, or disable GO-BS. Such decisions should be guided by validation evidence and computational cost rather than by the assumption that learned feature ordering always improves synthetic data quality.

\textbf{Summary.}
The combined experiments support prioritizing class-conditional marginal preservation and structural configuration before extensive prior-specific optimization. Once the block budget and ordering variant are appropriately selected, moderate changes to prior-family and training controls produce relatively small variations in TSTR AUC on Colon dataset, although computational cost and raw C2ST behavior can still vary. This provides a practical configuration strategy for deploying BSTabDiff while recognizing that optimal settings and sensitivity patterns may differ across datasets and dimensionality regimes.



\renewcommand{\thesection}{A22}
\renewcommand{\thesubsection}{A22.\arabic{subsection}}
\setcounter{subsection}{0}
\setcounter{figure}{0}
\renewcommand{\thefigure}{A22.\arabic{figure}}
\setcounter{table}{0}
\renewcommand{\thetable}{A22.\arabic{table}}

\section{Additional Methodological Notes and Clarifications}
\label{app:additional_methodological_notes}

\textbf{Selection of Tabular Generative Baselines.}
Our baseline selection considers recent tabular generative approaches across four families: GAN/VAE-based, diffusion/score-based, LLM-based, and HDLSS/selection-oriented methods. We examined their publicly available implementations, supported sample and feature ranges, benchmark regimes, and practical applicability to raw HDLSS data, selecting ten representative baselines from the applicable families, including the high-dimensional generator CtrTab~\citep{ctr}. Appendix~\ref{morerelated} provides a detailed survey of these methods, while Fig.~\ref{fig:bstabdiff_scope} and Table~\ref{tab:gen_benchmark_scope} compare their evaluation breadth and dimensionality coverage. Among the surveyed methods, BSTabDiff is evaluated against the largest collection of baselines and covers extreme HDLSS settings involving very small sample sizes and up to $42{,}728$ raw features (Appendix~\ref{extreme}). Although LLM-based generators are relevant to general tabular synthesis, directly representing thousands or tens of thousands of feature columns as tokens introduces substantial sequence length and attention computation costs for conventional transformer-based generators. We therefore prioritize methods applicable to direct high-dimensional tabular synthesis rather than relying on LLM-based generation at these dimensionalities. The resulting benchmark selection emphasizes methodological diversity, reproducibility, and suitability for the target HDLSS regime.

\textbf{Clarity of the End-to-End Generative Pipeline.}
BSTabDiff integrates dependency-aware feature ordering, block-latent inference, emission estimation, and a compact generative prior into a structured synthesis pipeline. As illustrated in Fig.~\ref{fig:bstabdiff_architecture}, GO-BS or GO-BS-FC estimates the feature permutation and contiguous block boundaries from training data before generator fitting. Observed continuous features are rank-Gaussianized using pooled training values, and block latents $h\in\mathbb{R}^M$ are inferred deterministically through standardized within-block averages rather than a learned encoder. Continuous emissions are fitted by least squares on observed Gaussianized entries, categorical emissions by cross-entropy, and missingness probabilities from empirical training masks, optionally conditioned on class labels. The prior $p_\theta(h\mid Y)$ then learns the joint distribution of inferred block latents through denoising score matching for diffusion or conditional maximum likelihood for flows. During generation, sampled block latents are decoded using copula-Gaussian emissions, learned inverse marginal CDFs, and categorical logits, with missingness sampled from estimated observation probabilities. Finally, the fixed learned permutation restores generated features to the original column order. Algorithms~\ref{alg:block_subunit_gen} and~\ref{alg:block_subunit_train} in Appendix~\ref{pseudocode} detail generation and training, while Appendix~\ref{sec:gobs} describes feature ordering and block segmentation.

\textbf{Significance of Block-Structured HDLSS Generation.}
HDLSS tabular synthesis is challenging because the number of observed features can greatly exceed the available sample count, making direct estimation of high-dimensional dependence statistically difficult. The block-subunit formulation addresses this challenge by concentrating global dependence learning in a compact latent space $\mathbb{R}^{M}$ while retaining feature-specific probabilistic decoding in the original $\mathbb{R}^{m}$ space. This approach is motivated by the presence of correlated feature groups and dependency modules in scientific datasets, particularly omics and gene-expression data, and is consistent with classical sparse factor~\citep{ac1}, graphical model~\citep{ac2}, latent tree~\citep{ac3}, and biclustering~\citep{ac4} approaches discussed in Appendix~\ref{morerelated}. BSTabDiff extends this structural perspective through data-driven feature-block discovery, diffusion or flow priors, copula-based emissions, heterogeneous marginals, and explicit missingness modeling. The primary experimental evaluation in Section~\ref{exp}, the extension to diverse data regimes in Appendix~\ref{scope}, and the matched dense-factor comparison in Appendix~\ref{app:block_subunit_necessity} provide complementary evidence of the practical relevance and limitations of this approach.

\textbf{Originality Relative to Conventional Latent Diffusion.}
BSTabDiff differs from a conventional latent diffusion model with an unconstrained neural encoder and decoder because its latent representation has an explicit feature-block interpretation. Rather than learning arbitrary compressed coordinates, the framework associates each latent variable $h_t$ with a discovered block $\mathcal{S}_t$ of observed features, as specified in Definition~\ref{def:block_subunit_hdlss_gen}. Each original feature is then generated through a block-conditioned probabilistic emission with a feature-specific loading, label-dependent location term, potentially heteroscedastic noise, marginal transformation, and missingness mechanism. The diffusion component is not the defining contribution but one possible prior for learning $p_{\theta}(h\mid Y)$; the same block-subunit formulation also supports flow, mixture, or graphical priors. Consequently, the central structural contribution is the combination of dependency-aware block discovery and constrained feature-wise probabilistic decoding, rather than applying diffusion to a generic low-dimensional embedding. The corresponding identifiability and informal sample complexity motivations are discussed in Propositions~\ref{prop:identifiability} and~\ref{prop:sample_complexity}, subject to their stated assumptions.

\textbf{Practical Construction of Feature Blocks.}
The feature-block organization is estimated from the training data through GO-BS, as detailed in Appendix~\ref{sec:gobs}. Each feature is represented as a vertex in a weighted affinity graph $G=(V,E,W)$, where edge weights quantify feature dependence using measures such as absolute correlation or other suitable affinity scores. For scalability, the graph retains the top-$k$ strongest neighbors per feature and is symmetrized to obtain a sparse dependency representation. GO-BS then estimates a permutation $\hat{\pi}$ that brings strongly connected features closer together, as illustrated in Fig.~\ref{fig:feature_reordering_graph}, and identifies $M-1$ boundaries defining contiguous feature blocks. Its objective rewards within-block affinity, penalizes cross-block leakage, and regularizes block sizes; the segmentation mechanism is illustrated in Fig.~\ref{fig:gobs_block_cut}. The resulting ordering and boundaries define the canonical partition used during latent inference and generation, with synthetic observations subsequently restored to the original feature order. Although exact optimization of the GO-BS objective is NP-Hard (Theorem~\ref{thm:gobs_nphard}), the practical implementation employs a sparse approximate solver described in Algorithm~\ref{alg:gobs}. The feature-clustered alternative, GO-BS-FC, is detailed in Algorithm~\ref{alg:gobsfc}, with practical variant selection guidance in Appendix~\ref{app:gobs_variant_selection}.

\textbf{Inference of Block Latent Variables.}
Once the canonical feature blocks $\{\mathcal{S}_t\}_{t=1}^{M}$ are established, each training sample is represented by an $M$-dimensional latent vector $h=(h_1,\ldots,h_M)$. Observed continuous features are rank-Gaussianized, and each block factor $h_t$ is estimated by averaging its observed Gaussianized coordinates, followed by standardization across samples. This lightweight factor score estimator admits a closed-form maximum-likelihood interpretation under a simplified Gaussian block model with unit feature loadings and common residual variance, where each factor minimizes the sum of squared deviations from its observed block coordinates. The inferred factors are subsequently used to fit the more general block-conditioned emissions, using least squares for continuous features and cross-entropy for categorical features on observed entries, consistent with Eq.~\ref{eq:block_factorization}. The joint prior $p_\theta(h\mid Y)$ is then trained using denoising score matching for diffusion or conditional maximum likelihood for flows. No separate inference encoder or iterative per-sample optimization is required. Algorithm~\ref{alg:block_subunit_train} in Appendix~\ref{pseudocode} summarizes the procedure.

\textbf{Learning Feature-Specific Marginal Transformations.}
BSTabDiff separates feature-wise marginal estimation from block-dependent latent inference. For each continuous feature $j$, empirical marginal CDFs $\hat F_{j,y}$ are estimated nonparametrically from observed training values, optionally conditioned on class $Y=y$, excluding missing entries. For latent inference, observed values are rank-Gaussianized using pooled training samples across classes: $z_{ij}=\Phi^{-1}((r_{ij}-0.5)/N_j)$, where $r_{ij}$ is the observed rank and $N_j$ is the number of observed values for feature $j$. The resulting Gaussianized coordinates are used to infer block latents and fit continuous emissions on observed entries, separating marginal behavior from block-level dependence. For generation, class-conditional inverse marginal functions $\hat F^{-1}_{j,y}$ are retained through sorted training observations and interpolation. Each sampled copula coordinate $Z_j$ is mapped to $U_j=\Phi(Z_j)$ and decoded as $\tilde X_j=\hat F^{-1}_{j,y}(U_j)$, following Eqs.~\ref{eq:copula_gauss} and~\ref{eq:inverse_marginal}. This preserves non-Gaussian characteristics, including skewness and heavy tails, without requiring a high-capacity marginal network. The training and inverse-decoding procedures are detailed in Algorithms~\ref{alg:block_subunit_train} and~\ref{alg:block_subunit_gen}, while the contribution of class-conditional marginals is evaluated in Table~\ref{tab:bstabdiff_compact_ablations} and Appendix~\ref{app:component_sensitivity}.

\textbf{Distinct Methodological Contributions.}
The defining contribution of BSTabDiff is the block-subunit generative factorization rather than the particular choice of a diffusion prior. GO-BS introduces a data-driven dependency-aware feature organization, and the resulting partition associates each block latent with a restricted subset of observed features rather than an unconstrained dense decoder. The emission model uses feature-specific scalar loadings, heteroscedastic noise, label-conditional marginal transformations, and explicit missingness mechanisms to reconstruct the full observed feature dimension from a compact latent representation. Global interactions are modeled through the joint prior $p_{\theta}(h\mid Y)$, so the block formulation does not require independent block latents. This architecture combines interpretable feature-block organization with a parameter-efficient emission structure whose loading complexity is $O(m)$ rather than the $O(mM)$ required by a dense latent-factor decoder. Appendix~\ref{app:block_subunit_necessity} provides a matched empirical comparison of these representations, while Appendix~\ref{app:block_assumption_diagnostics} discusses the structural assumptions, diagnostics, and possible failure modes. The resulting methodological contribution is therefore a structured generative family for HDLSS data, with diffusion serving as one interchangeable implementation of the block-latent prior.

\textbf{Handling Mixed-Type Data and Missingness.}
Although original HDLSS benchmarks primarily contain numerical measurements, extended evaluation includes mixed-type datasets with categorical features and heterogeneous missingness patterns. Specifically, Springleaf is an HDHSS marketing dataset containing $29{,}046$ samples and $1{,}876$ features, including $22$ categorical variables, high-cardinality attributes, timestamps, and substantial missingness. TCGA-BRCA contains $705$ samples and $1{,}940$ features, including three categorical receptor-status variables and explicitly represented missing values. These datasets extend the evaluation to continuous and categorical features in both HDLSS and HDHSS settings, providing evidence of applicability beyond predominantly numerical omics benchmarks. Corresponding downstream TSTR AUC results are reported in Table~\ref{tab:extended_highdim_auc}, while complementary accuracy results and broader cross-regime analyses appear in Appendix~\ref{scope}, including Table~\ref{tab:extended_highdim_results}. Generative model's explicit missingness mechanism is defined in Eq.~\ref{eq:missing_mask}, while feature-specific decoding formulation is summarized in Definition~\ref{def:block_subunit_hdlss_gen}. Together, these experiments and modeling components illustrate BSTabDiff's applicability to heterogeneous tabular settings, although the main methodological motivation remains high-dimensional generation under limited sample availability.

\textbf{Computational Complexity and Practical Configuration.}
BSTabDiff supports both dataset-specific hyperparameter optimization and a fixed default configuration, so extensive tuning is not required for the model to operate. Appendix~\ref{hyp} documents the default and tuned hyperparameters, including the dataset-specific block budgets in Table~\ref{tab:bstabdiff_hparams_all8_rot} and the results of the 2000 trial Optuna~\citep{optuna} search in Table~\ref{tab:optuna_best_params_2000}. The main dataset-dependent structural choices are the number of blocks $M$ and the optional feature ordering variant. We provide a practical initialization of $M$ based on the feature-to-sample ratio $\rho=m/n$ and data matrix scale $\kappa=nm$ in Appendix~\ref{app:m_selection_sensitivity}, together with a completed sensitivity analysis demonstrating the trade-off between block resolution and GO-BS computational cost. A complementary regime-aware decision rule for GO-BS versus GO-BS-FC is provided in Appendix~\ref{app:gobs_variant_selection}. Class-conditional marginals are treated as a core modeling component rather than an optional tuning parameter, and the missingness mechanism requires no meaningful adjustment for fully observed datasets while being estimated from observed missingness patterns otherwise. Appendix~\ref{app:practical_configuration_guide} consolidates the recommended tuning priorities, structural diagnostics, default block-budget selection, and ordering-variant choice. The Optuna~\citep{optuna} experiments provide optimized baseline comparisons, whereas the default configuration offers a practical starting point for deployment without extensive dataset-specific search.

\renewcommand{\thesection}{A23}
\renewcommand{\thesubsection}{A23.\arabic{subsection}}
\setcounter{figure}{0}\renewcommand{\thefigure}{A23.\arabic{figure}}
\setcounter{table}{0}\renewcommand{\thetable}{A23.\arabic{table}}
\section{Dataset and Baseline Sources}
\label{app:dataset}

This appendix documents the datasets and baseline implementations used throughout our experiments to facilitate reproducibility and transparent evaluation. We consider 18 publicly available datasets, comprising eight initial HDLSS benchmarks, the extreme dimensional Cell Cycle RNA-seq dataset, and nine additional datasets evaluated during additional experiments, spanning HDLSS, HDHSS, LDHSS, and mixed dimensionality regimes. The collection covers omics, gene expression, face images, ResNet-50 image embeddings, high-dimensional text representations, particle physics, marketing, and census data, with sample sizes ranging from $n=62$ to $130{,}064$ and feature dimensions ranging from $m=14$ to $42{,}728$. Table~\ref{tab:datasets_references} provides the original dataset sources and associated references, while Table~\ref{tab:all_datasets_summary} summarizes the sample and feature dimensions of all 18 datasets, including the processed Springleaf subset and CIFAR-10 image embeddings. For experimental reproducibility, Table~\ref{tab:baseline_sources} additionally lists the public implementations and documentation for all ten baseline methods, covering oversampling, GAN/VAE-based generation, and diffusion-based tabular synthesis.

\begin{table*}[t]
\centering
\caption{\textbf{Dataset sources and references.} Original data sources and associated references for all 18 datasets used in our experimental evaluation, including the eight original HDLSS benchmarks, Cell Cycle, and nine additional datasets evaluated during additional experiments.}
\label{tab:datasets_references}

\small
\setlength{\tabcolsep}{6pt}
\renewcommand{\arraystretch}{1.25}

\begin{tabularx}{\textwidth}{
    @{}
    >{\raggedright\arraybackslash}p{0.16\textwidth}
    >{\raggedright\arraybackslash}X
    >{\raggedright\arraybackslash}p{0.16\textwidth}
    >{\raggedright\arraybackslash}X
    @{}
}
\toprule

\rowcolor{rankheader}
\textcolor{white}{\textbf{Dataset}} &
\textcolor{white}{\textbf{Source / Reference}} &
\textcolor{white}{\textbf{Dataset}} &
\textcolor{white}{\textbf{Source / Reference}} \\

\midrule

\rowcolor{aucBSTabA}
\cellcolor{aucBSTabName}\textbf{Cell Cycle}
& \citet{cell1, cell2}
& \cellcolor{regimeLDHSS}\textbf{ORLraws10P}
& \citet{hdlsss, hdlsss2} \\

\rowcolor{aucBSTabB}
\cellcolor{aucBSTabName}\textbf{TCGA-BRCA}
& \citet{brca1, brca2}
& \cellcolor{regimeLDHSS}\textbf{CIFAR-10 (ResNet-50)}
& \citet{cifa} \\

\rowcolor{aucBSTabA}
\cellcolor{aucBSTabName}\textbf{Springleaf (Subset)}
& \citet{springleaf}
& \cellcolor{regimeLDHSS}\textbf{MiniBooNE}
& \citet{miniboone} \\

\rowcolor{aucBSTabB}
\cellcolor{aucBSTabName}\textbf{Adult}
& \citet{adult}
& \cellcolor{regimeLDHSS}\textbf{COL}
& \citet{hdlsss, hdlsss2} \\

\rowcolor{aucBSTabA}
\cellcolor{aucBSTabName}\textbf{RELATHE}
& \citet{hdlsss, hdlsss2}
& \cellcolor{regimeLDHSS}\textbf{GLI}
& \citet{hdlsss, hdlsss2} \\

\rowcolor{aucBSTabB}
\cellcolor{aucBSTabName}\textbf{PCMAC}
& \citet{hdlsss, hdlsss2}
& \cellcolor{regimeLDHSS}\textbf{SMK}
& \citet{hdlsss, hdlsss2} \\

\rowcolor{aucBSTabA}
\cellcolor{aucBSTabName}\textbf{Lung}
& \citet{hdlsss, hdlsss2}
& \cellcolor{regimeLDHSS}\textbf{ALLAML}
& \citet{hdlsss, hdlsss2} \\

\rowcolor{aucBSTabB}
\cellcolor{aucBSTabName}\textbf{PRS}
& \citet{hdlsss, hdlsss2}
& \cellcolor{regimeLDHSS}\textbf{ARC}
& \citet{hdlsss, hdlsss2} \\

\rowcolor{aucBSTabA}
\cellcolor{aucBSTabName}\textbf{BASEHOCK}
& \citet{hdlsss, hdlsss2}
& \cellcolor{regimeLDHSS}\textbf{TOX}
& \citet{hdlsss, hdlsss2} \\

\bottomrule
\end{tabularx}
\end{table*}

\begin{table*}[t]
\centering
\caption{\textbf{Baseline implementation sources.} Official implementations and documentation for the ten baseline methods used in our experimental comparisons. All repository and documentation links are clickable.}
\label{tab:baseline_sources}

\small
\setlength{\tabcolsep}{6pt}
\renewcommand{\arraystretch}{1.25}

\begin{tabularx}{\textwidth}{
    @{}
    >{\raggedright\arraybackslash}p{0.16\textwidth}
    >{\raggedright\arraybackslash}X
    >{\raggedright\arraybackslash}p{0.16\textwidth}
    >{\raggedright\arraybackslash}X
    @{}
}
\toprule

\rowcolor{rankheader}
\textcolor{white}{\textbf{Baseline}} &
\textcolor{white}{\textbf{Implementation / Documentation}} &
\textcolor{white}{\textbf{Baseline}} &
\textcolor{white}{\textbf{Implementation / Documentation}} \\

\midrule

\rowcolor{aucBSTabA}
\cellcolor{aucBSTabName}\textbf{SMOTE}
& \href{https://imbalanced-learn.org/stable/references/generated/imblearn.over_sampling.SMOTE.html}{\texttt{imbalanced-learn / SMOTE}}
& \cellcolor{regimeLDHSS}\textbf{CTGAN}
& \href{https://github.com/sdv-dev/CTGAN}{\texttt{sdv-dev/CTGAN}} \\

\rowcolor{aucBSTabB}
\cellcolor{aucBSTabName}\textbf{CTAB-GAN}
& \href{https://github.com/Team-TUD/CTAB-GAN}{\texttt{Team-TUD/CTAB-GAN}}
& \cellcolor{regimeLDHSS}\textbf{TVAE}
& \href{https://github.com/sdv-dev/CTGAN}{\texttt{sdv-dev/CTGAN}} \\

\rowcolor{aucBSTabA}
\cellcolor{aucBSTabName}\textbf{CTAB-GAN+}
& \href{https://github.com/Team-TUD/CTAB-GAN-Plus}{\texttt{CTAB-GAN-Plus}}
& \cellcolor{regimeLDHSS}\textbf{TabDDPM}
& \href{https://github.com/yandex-research/tab-ddpm}{\texttt{yandex-research/tab-ddpm}} \\

\rowcolor{aucBSTabB}
\cellcolor{aucBSTabName}\textbf{ForestDiffusion}
& \href{https://github.com/SamsungSAILMontreal/ForestDiffusion}{\texttt{ForestDiffusion}}
& \cellcolor{regimeLDHSS}\textbf{TabDiff}
& \href{https://github.com/minkaixu/tabdiff}{\texttt{minkaixu/tabdiff}} \\

\rowcolor{aucBSTabA}
\cellcolor{aucBSTabName}\textbf{Binary Diffusion}
& \href{https://github.com/vkinakh/binary-diffusion-tabular}{\texttt{binary-diffusion-tabular}}
& \cellcolor{regimeLDHSS}\textbf{CtrTab}
& \href{https://github.com/zuqingli0404/CtrTab}{\texttt{zuqingli0404/CtrTab}} \\

\bottomrule
\end{tabularx}

\vspace{0.3em}
\raggedright
\footnotesize
\textbf{Notes:} CTGAN and TVAE are two distinct baseline models provided through the same repository. SMOTE is an oversampling baseline implemented through the imbalanced-learn library. The remaining methods are linked to their respective public GitHub implementations.

\end{table*}


\begin{table*}[t]
\centering
\caption{\textbf{Summary of all 18 datasets used in our experiments.} Here, $n$ denotes the number of samples and $m$ denotes the number of features. The evaluation includes eight initial HDLSS benchmarks, the extreme dimensional Cell Cycle dataset, and nine additional datasets spanning diverse dimensionality regimes.}
\label{tab:all_datasets_summary}

\small
\setlength{\tabcolsep}{5pt}
\renewcommand{\arraystretch}{1.20}

\begin{tabular*}{\textwidth}{@{\extracolsep{\fill}}llrr@{\hspace{12pt}}llrr@{}}
\toprule

\rowcolor{datasetheader}
\multicolumn{4}{c}{\textcolor{white}{\textbf{Initial HDLSS + Cell Cycle}}} &
\multicolumn{4}{c}{\textcolor{white}{\textbf{Additional Datasets}}} \\

\midrule

\rowcolor{datasetheader}
\textcolor{white}{\textbf{Abbr}} &
\textcolor{white}{\textbf{Name}} &
\textcolor{white}{\textbf{$n$}} &
\textcolor{white}{\textbf{$m$}} &
\textcolor{white}{\textbf{Abbr}} &
\textcolor{white}{\textbf{Name}} &
\textcolor{white}{\textbf{$n$}} &
\textcolor{white}{\textbf{$m$}} \\

\midrule

\rowcolor{datasetrowa}
COL & Colon & 62 & 2,000
& BRCA & TCGA-BRCA & 705 & 1,940 \\

\rowcolor{datasetrowb}
GLI & GLI-85 & 85 & 22,283
& CIFAR & CIFAR-10 & 11,000 & 2,048 \\

\rowcolor{datasetrowa}
LNG & Lung & 203 & 3,312
& MB & MiniBooNE & 130,064 & 50 \\

\rowcolor{datasetrowb}
PRS & Prostate & 102 & 5,966
& SPR & Springleaf (Subset) & 29,046 & 1,876 \\

\rowcolor{datasetrowa}
SMK & SMK & 187 & 19,993
& ORL & ORLraws10P & 100 & 10,304 \\

\rowcolor{datasetrowb}
TOX & TOX171 & 171 & 5,748
& ADULT & Adult & 48,842 & 14 \\

\rowcolor{datasetrowa}
AML & ALLAML & 72 & 7,129
& BASE & BASEHOCK & 1,993 & 4,862 \\

\rowcolor{datasetrowb}
ARC & Arcene & 200 & 10,000
& REL & RELATHE & 1,427 & 4,322 \\

\rowcolor{datasetrowa}
CC & Cell Cycle & 1,067 & 42,728
& PCMAC & PCMAC & 1,943 & 3,289 \\

\bottomrule
\end{tabular*}

\vspace{0.3em}
\raggedright
\footnotesize
\textbf{Notes:} The first eight datasets are the original HDLSS benchmarks. Cell Cycle is an RNA-seq dataset with the highest feature dimensionality ($m=42{,}728$) in our evaluation. CIFAR-10 uses 2,048-dimensional ResNet-50 image embeddings rather than raw pixels. Springleaf refers to the processed subset. The nine additional datasets include omics, image embeddings, particle physics, marketing, census, face image, and text classification data.

\end{table*}

\clearpage
\input{checklist.tex}

\end{document}

%% file: checklist.tex
\section*{NeurIPS Paper Checklist}

\begin{enumerate}

\item {\bf Claims}
    \item[] Question: Do the main claims made in the abstract and introduction accurately reflect the paper's contributions and scope?
    \item[] Answer: \answerYes{} 
    \item[] Justification: 
The abstract and introduction accurately summarize BSTabDiff's main contributions, including the HDLSS ($n\ll m$) block-subunit generator, compact block-latent priors, Graph-guided Ordering with Block Segmentation (GO-BS), theoretical analysis, empirical evaluation, and energy-carbon analysis. The primary claims focus on HDLSS tabular generation, while additional experiments examine broader dataset regimes and practical limitations. We do not claim universal superiority across all data regimes.


\item {\bf Limitations}
    \item[] Question: Does the paper discuss the limitations of the work performed by the authors?
    \item[] Answer: \answerYes{} 
    \item[] Justification: The paper includes a dedicated limitations discussion before the conclusion. It acknowledges that BSTabDiff is primarily designed for High-Dimensional Low-Sample Size (HDLSS; $n\ll m$) tabular data, while additional experiments explore other dataset regimes (e.g. HDHSS, LDHSS, MixedRegime). It also discusses the dependence of Graph-guided Ordering with Block Segmentation (GO-BS) on meaningful local feature dependence structure, the potential computational overhead of feature ordering, and the possibility of reduced benefits when such structure is weak. The paper does not claim universal superiority across all dataset regimes.


\item {\bf Theory assumptions and proofs}
    \item[] Question: For each theoretical result, does the paper provide the full set of assumptions and a complete (and correct) proof?
    \item[] Answer: \answerYes{}{} 
    \item[] Justification: The paper states the relevant assumptions and provides proofs or proof sketches for its theoretical results, including identifiability, within-block SNR scaling, GO-BS NP-hardness, exact recovery under block dominance, and sparse computational complexity. The formal GO-BS results are detailed in Appendix~\ref{sec:gobs}. The identifiability result is accompanied by a proof sketch, while the sample complexity advantage is explicitly labeled as informal and intended to provide theoretical intuition rather than a complete formal guarantee.

    \item {\bf Experimental result reproducibility}
    \item[] Question: Does the paper fully disclose all the information needed to reproduce the main experimental results of the paper to the extent that it affects the main claims and/or conclusions of the paper (regardless of whether the code and data are provided or not)?
    \item[] Answer: \answerYes{} 
    \item[] Justification: The paper provides dataset descriptions and original source links in Table~\ref{tab:datasets_references}, along with baseline implementation sources in Table~\ref{tab:baseline_sources}, evaluation protocols, cross-validation setup, hyperparameter tuning procedures, compute environment, and implementation details needed to reproduce the main experimental results. The supplemental material includes the BSTabDiff implementation and most experimental \texttt{.ipynb} notebooks with executed outputs and reported results. The notebooks have been cleaned and validated through reproducibility smoke tests, including installation and execution using the released PyPI package. Additional hyperparameters, pseudocode, energy-measurement protocols, and runtime/resource details are provided in the appendix. The implementation and reproducibility resources are publicly available on GitHub (\url{https://github.com/zadid6pretam/BSTabDiff}) and PyPI (\url{https://pypi.org/project/bstabdiff}), with installation supported via \texttt{pip install bstabdiff}.


\item {\bf Open access to data and code}
    \item[] Question: Does the paper provide open access to the data and code, with sufficient instructions to faithfully reproduce the main experimental results, as described in supplemental material?
    \item[] Answer: \answerYes{} 
    \item[] Justification: The datasets used in this study are publicly available through their original sources, with access links provided in Table~\ref{tab:datasets_references}. Baseline implementations and their source links are documented in Table~\ref{tab:baseline_sources}. We provide a supplemental code package containing the BSTabDiff implementation and most experimental \texttt{.ipynb} notebooks with executed outputs and reported results. The notebooks have been cleaned and validated through reproducibility smoke tests, including installation and execution using the released PyPI package. Detailed experimental settings, hyperparameters, pseudocode, and evaluation protocols are provided in the appendix. The source code and reproducibility resources are publicly available on GitHub (\url{https://github.com/zadid6pretam/BSTabDiff}) and PyPI (\url{https://pypi.org/project/bstabdiff/}), with installation supported via \texttt{pip install bstabdiff}. The datasets are accessed through their original sources rather than redistributed.


\item {\bf Experimental setting/details}
    \item[] Question: Does the paper specify all the training and test details (e.g., data splits, hyperparameters, how they were chosen, type of optimizer) necessary to understand the results?
    \item[] Answer: \answerYes{} 
    \item[] Justification: The paper specifies the training and evaluation procedures, including publicly available datasets, baseline implementations, the $5\times5$ repeated cross-validation protocol, training and test splits, TSTR/TRTR evaluation, and Optuna-based hyperparameter optimization with 2000 trials. Detailed hyperparameter configurations, search spaces, optimization procedures, pseudocode, compute environment, energy-measurement protocols, and runtime/resource information are provided in the appendix. The supplemental code package and experimental notebooks further document the implementation and experimental settings, supporting reproducibility of the reported results.


\item {\bf Experiment statistical significance}
    \item[] Question: Does the paper report error bars suitably and correctly defined or other appropriate information about the statistical significance of the experiments?
    \item[] Answer: \answerYes{} 
    \item[] Justification: The main TSTR/TRTR results are reported as mean $\pm$ standard deviation over strict $5\times5$ repeated cross-validation (25 evaluations), capturing variability across train/test splits. Additional ablation studies follow the same protocol where applicable, while fidelity diagnostics report complementary measures, including KS, W1, correlation discrepancy, mutual information (MI) discrepancy, privacy distance, and C2ST. Statistical significance is further assessed in Appendix~\ref{stats} using a Friedman omnibus test, followed by pairwise one-sided Wilcoxon signed-rank tests with Holm correction, as reported in Table~\ref{tab:wilcoxon_holm_tstr_lr}.

\item {\bf Experiments compute resources}
    \item[] Question: For each experiment, does the paper provide sufficient information on the computer resources (type of compute workers, memory, time of execution) needed to reproduce the experiments?
    \item[] Answer: \answerYes{} 
    \item[] Justification: The paper reports the experimental compute environment, including the Vulcan cluster equipped with NVIDIA RTX A6000 GPUs, Intel Xeon Gold CPUs, and 503\,GB of system RAM. Table~\ref{tab:bstabdiff_compute_optuna_noperm} provides per-dataset execution time, GPU memory usage, CPU Resident Set Size (RSS), energy consumption, and CO$_2$e emissions for the main BSTabDiff runs. Additional computational resource analyses, runtime measurements, and energy-carbon diagnostics are provided in Appendices~\ref{app:green_ai}, \ref{run}, \ref{bsgob}, \ref{app:gobsmeta}, and~\ref{app:gobsextreme}, including CodeCarbon~\citep{g8}-based measurements and GO-BS/GO-BS-FC computational analyses. These details document the hardware requirements and computational costs associated with reproducing the experiments.

    
\item {\bf Code of ethics}
    \item[] Question: Does the research conducted in the paper conform, in every respect, with the NeurIPS Code of Ethics \url{https://neurips.cc/public/EthicsGuidelines}?
    \item[] Answer: \answerYes{} 
    \item[] Justification: We have reviewed the NeurIPS Code of Ethics (\url{https://neurips.cc/public/EthicsGuidelines}) and confirm that our research conforms to its principles. The study uses publicly available datasets with appropriate citations and source links, does not involve new human-subject data collection, deception, or intervention, and reports methodological limitations, computational requirements, and energy-carbon impacts. We also provide publicly accessible code and reproducibility resources through GitHub and PyPI, supporting transparency and responsible research practices.


\item {\bf Broader impacts}
    \item[] Question: Does the paper discuss both potential positive societal impacts and negative societal impacts of the work performed?
    \item[] Answer: \answerYes{} 
    \item[] Justification: The paper includes a dedicated Limitations and Impact Statement discussing both positive and potential negative societal impacts. Positive impacts include synthetic data generation for scarce HDLSS and sensitive scientific domains, improved data augmentation and benchmarking, and transparent reporting of energy consumption and CO$_2$e emissions. Potential risks include privacy leakage, over-reliance on synthetic data, propagation of existing biases, and inappropriate use in high-stakes applications. We emphasize that synthetic data should complement rather than replace real data validation, domain expert review, privacy assessments, and appropriate safety checks before downstream deployment.
    
\item {\bf Safeguards}
    \item[] Question: Does the paper describe safeguards that have been put in place for responsible release of data or models that have a high risk for misuse (e.g., pre-trained language models, image generators, or scraped datasets)?
    \item[] Answer: \answerNA{} 
    \item[] Justification: The paper does not release high-risk pretrained language models, image generators, or scraped datasets. The experiments use publicly available, properly cited tabular datasets, with original sources documented in Table~\ref{tab:datasets_references}. The released materials consist of the BSTabDiff implementation, experimental notebooks, and reproducibility instructions rather than sensitive datasets or pretrained models intended for unrestricted deployment. We do not redistribute the original datasets and acknowledge potential privacy, misuse, and downstream validation risks in the Limitations and Impact Statement. Users are encouraged to assess privacy and data utility before applying BSTabDiff to sensitive or high-stakes domains.


\item {\bf Licenses for existing assets}
    \item[] Question: Are the creators or original owners of assets (e.g., code, data, models), used in the paper, properly credited and are the license and terms of use explicitly mentioned and properly respected?
    \item[] Answer: \answerYes{} 
    \item[] Justification: The paper appropriately cites the original datasets, baseline implementations, libraries, and tools, with source links provided in Tables~\ref{tab:datasets_references} and~\ref{tab:baseline_sources}. The supplemental README documents asset sources and licensing information, and the released BSTabDiff code includes its own license. We respect the original assets' licensing and terms of use.


\item {\bf New assets}
    \item[] Question: Are new assets introduced in the paper well documented and is the documentation provided alongside the assets?
    \item[] Answer: \answerYes{} 
    \item[] Justification: 
BSTabDiff is documented through its project webpage (\url{https://www.zadidhabib.com/bstabdiff.html} or alternatively \url{https://zadid6pretam.github.io/bstabdiff}), GitHub repository (\url{https://github.com/zadid6pretam/BSTabDiff}), PyPI package (\url{https://pypi.org/project/bstabdiff/}), Hugging Face Space (\url{https://huggingface.co/spaces/zadid6pretam/BSTabDiff-Studio}), and supplemental code package, which provide the implementation, interactive demonstration, experimental notebooks, and reproducibility instructions. The paper and appendix further document the methodology, pseudocode, hyperparameters, computational requirements, and limitations.
    

\item {\bf Crowdsourcing and research with human subjects}
    \item[] Question: For crowdsourcing experiments and research with human subjects, does the paper include the full text of instructions given to participants and screenshots, if applicable, as well as details about compensation (if any)? 
    \item[] Answer: \answerNA{} 
    \item[] Justification: The paper does not involve crowdsourcing experiments, new human-subject data collection, participant studies, or compensated human labor. All experiments use publicly available, cited tabular datasets.

\item {\bf Institutional review board (IRB) approvals or equivalent for research with human subjects}
    \item[] Question: Does the paper describe potential risks incurred by study participants, whether such risks were disclosed to the subjects, and whether Institutional Review Board (IRB) approvals (or an equivalent approval/review based on the requirements of your country or institution) were obtained?
    \item[] Answer: \answerNA{} 
    \item[] Justification: The paper does not involve crowdsourcing, new human-subject data collection, participant studies, or interventions with study participants. All experiments use publicly available, cited tabular datasets, so IRB approval or equivalent human-subjects review is not applicable.

\item {\bf Declaration of LLM usage}
    \item[] Question: Does the paper describe the usage of LLMs if it is an important, original, or non-standard component of the core methods in this research? Note that if the LLM is used only for writing, editing, or formatting purposes and does \emph{not} impact the core methodology, scientific rigor, or originality of the research, declaration is not required.
    \item[] Answer: \answerNA{} 
    \item[] Justification: The core method development, theory, experiments, and evaluation of BSTabDiff do not use Large Language Models (LLMs) as an important, original, or non-standard component. LLM assistance was limited to writing, editing, LaTeX formatting, and polishing repetitive experimental notebook code, without altering the core methodology, scientific claims, or experimental results.


\end{enumerate}